\documentclass[10pt,journal,commsoc]{IEEEtran}

\IEEEoverridecommandlockouts

\ifCLASSOPTIONcompsoc
  \usepackage[nocompress]{cite}
\else
  \usepackage{cite}
\fi

\usepackage{xcolor}
\usepackage{amssymb}
\usepackage{amsmath}
\allowdisplaybreaks[4]
\usepackage{multicol}
\usepackage{amsfonts}
\usepackage{amsthm}
\usepackage{array}
\usepackage{bbm}
\usepackage[caption=false,font=normalsize,labelfont=sf,textfont=sf]{subfig}
\usepackage{textcomp}
\usepackage{url}
\usepackage{verbatim}
\usepackage{graphicx}
\usepackage{multirow}
\usepackage{textcase}
\usepackage[lined,linesnumbered,ruled,longend]{algorithm2e}
\usepackage{cite}
\usepackage{makecell}
\usepackage{float}
\usepackage{stfloats}
\usepackage{url}
\usepackage{verbatim}
\usepackage{xr}
\newtheorem{thm}{\textbf{Theorem}}
\newtheorem{assumption}{\textbf{Assumption}}

\newtheorem{lem}{\textbf{Lemma}}

\newtheorem{defn}{\textbf{Definition}}
\newtheorem{remark}{\textbf{Remark}}
\newtheorem{fact}{\textbf{Fact}}
\newtheorem{corollary}{\textbf{Corollary}}
\newcommand*\tcircle[1]{%
  \raisebox{0pt}{%
    \textcircled{\fontsize{7pt}{0}\fontfamily{phv}\selectfont #1}%
  }%
}

\makeatletter

\newcommand{\Rmnum}[1]{\expandafter\@slowromancap\romannumeral #1@}
\makeatother



\begin{document}

\title{DPDL: Towards Differential Privacy Preservation in Decentralized Stochastic Learning on Non-IID Data}

\author{Yunsheng~Yuan,
        Xue~Xiao,
        Lina~Wang,
        Feng~Li
        
\IEEEcompsocitemizethanks{\IEEEcompsocthanksitem Y. Yuan, L. Wang and F. Li are with School of Computer Science and Technology, Shandong University, Qingdao 266237, China.
E-mail: yuanys1028@gmail.com, linawang425@mail.sdu.edu.cn, fli@sdu.edu.cn  \\
X. Xiao is with Inspur Cloud Information Technology
Co, Ltd. E-mail: xiaoxue@inspur.com
}
}



\maketitle

\begin{abstract}
  In the paradigm of decentralized learning, a group of agents collaborate to train a global model using distributed datasets without a central server. Although the power of collaboration has been verified by many state-of-the-art studies, it entails extensive gradient information exchanging among the agents and thus induces high risk of privacy leakage for the individual agents. Moreover, in real-world applications, the training data are usually non-identically and independently distributed across the agents, inducing more challenges to enable privacy-preserved decentralized learning. To address these issues, we propose a privacy-preserved decentralized learning algorithm with non-IID data, DPDL, which leverages the notion of \textit{Differential Privacy} (DP) in cross-gradient aggregation through a similarity-based calibration technique. Specifically, in each round, each agent perturbs the cross-gradients (i.e., the derivatives of its neighbors' local model in its private local data) by Gaussian noise mechanism before sharing them with its neighbors; it then adopt cosine similarity to calibrate the received perturbed cross-gradients such that the aggregation of the calibrated cross-gradients can be utilized to effectively update local model in a momentum-like manner. Our rigorous theoretical analysis not only reveals the minimum noise level required to achieve a specific level of privacy preservation, but also illustrates that our algorithm still achieves a linear speedup in training with non-IID data. We finally conduct extensive experiments on real-world dataset to validate the effectiveness of our algorithm in defending privacy attacks and in training accurate models.

  %
\end{abstract}


\section{Introduction} \label{sec:intro}
  In response to the significant growth of data produced on an array of smart devices (e.g., smartphones, tablets, and drones), distributed machine learning, aiming to train a global model based on the data dispersed across different devices, has been garnering increasing attention~\cite{VerbraekenWKKVR-CSUR20,LiWWHWLLH-TKDE23,BeltranPSBBPPC-COMST23,ChuL-SIGMETRICS24}. For example, in \textit{Federated Learning} (FL), a group of devices (which we will refer to as \textit{agents} henceforth) are coordinated by a central server~\cite{McmahanMRHA-AISTATS17}. Although the agents never share their local data with the central server, they must consistently communicate their local updates (e.g., the gradient information calculated according to their local dataset) with the server, resulting in the limitations such as having a single point of failure, trust dependencies, and bottlenecks at the server. To address these limitations, decentralized learning~\cite{LianZZHZL-NIPS17,NedicOR-IEEE18,LianZZL-ICML18,BeltranPSBBPPC-COMST23} has been proposed, such that each agent communicates its updates only with its neighbors without the coordination of a central server.

  One challenging issue for decentralized learning is the heterogeneous data distribution across the agents. For example, due to the diversities of user behaviors and sensor devices, the data across different agents are distributed non-independently and identically (non-IID), while the non-IID data may result in a degradation in convergence and accuracy~\cite{LinKSJ-ICML21,LiDCH-ICDE22}. Many efforts have been made to address the challenge of data heterogeneity in decentralized learning~\cite{LinKSJ-ICML21,LiLV-IOTJ21,ShiSWSYWT-ICML23,AketiHR-NeurIPS24}. Recently, \cite{ZhangFLYLZ-MobiHoc22} conducts gradient corrections recursively in multi-step local updates to resolve the problem of data heterogeneity. In \cite{EsfandiariTJBHHS-ICML21}, each agent aggregates the cross-gradients (i.e., derivatives of its model with respect to its neighbors’ datasets) and its own self-gradient to update its local model.

  Another concern for decentralized learning is privacy preservation. In the paradigm of decentralized learning, agents collaborate with each other for model training by exchanging information. For example, in \cite{LianZZHZL-NIPS17,YuJY-ICML19,EsfandiariTJBHHS-ICML21,AketiHR-NeurIPS24}, each agent communicates its local gradient information with its peers. However, many studies have demonstrated that an adversary can eavesdrop on the communications and infer the private local data of the agents through the intercepted gradient information~\cite{ZhuLH-NIPS19,ZhaoMB-ARXIV20,WangHL-INFOCOM24}. To ensure the data privacy of participants in decentralized learning, one of the most popular choices is \textit{Differential Privacy} (DP)~\cite{Dwork-ICALP06,DworkR-FTTCS14,PonomarevaVXMKZ-KDD23}. Although DP has been extensively used in centralized distributed learning (e.g., FL)~\cite{ShokriS-CCS15,LiWWHWLLH-TKDE23,LinWLSHJ-TDSC23,ChenDBZJ-TKDE24}, there are only a few proposals investigating privacy-preserved decentralized learning algorithms. For instance, \cite{XuZW-TPAMI21} incorporates DP in the decentralized parallel SGD framework~\cite{LianZZHZL-NIPS17}. \cite{CyffersEBM-NIPS22} presents a differentially private decentralized learning algorithm that alternates between local gradient descent and gossip averaging. Unfortunately, these state-of-the-art proposals assume that the training data are identically and independently distributed across different agents, whereas such an assumption may not always hold in practice, as mentioned previously. Applying DP to decentralized learning with non-IID data is highly non-trivial. DP makes the decentralized learning algorithm robust against privacy attacks by injecting random noise into the private information (e.g., local gradient information) shared by the agents. However, the noise mechanism can significantly impair the utility of the shared information~\cite{SeemanS-JRC24}. Moreover, due to the heterogeneity of the data, the consistency of the local information shared by individual agents cannot be guaranteed. Therefore, how to harness the perturbed information for efficient model training in the face of non-IID data remains an open problem.

  In this paper, we propose DPDL, a differential privacy-preserved decentralized stochastic learning algorithm with non-IID data. We innovate in applying DP in an elaborately calibrated cross-gradient aggregation to ensure both privacy preservation and training efficiency in the face of non-IID data. In particular, in each round, each agent adopts a Gaussian noise mechanism to perturb the cross-gradients before sharing the perturbed gradients with its neighbors. The agent then calibrates the cross-gradients received from its neighbors based on the similarity between the cross-gradients and its own self-gradient. Finally, each agent updates its local model in a momentum-like manner based on the aggregation of the calibrated cross-gradients and self-gradient . Our solid theoretical analysis not only reveals the lower bound on the noise level to satisfy the demand of privacy preservation, but also demonstrates our algorithm achieves a linear speedup even with ``perturbed'' collaboration of the agents. We finally conduct extensive experiments to verify the efficacy of our algorithm in defending privacy attacks (e.g., gradient inversion attacks) and training accurate models with non-IID data. Our main contributions are as follows:
  \begin{itemize}
    \item We propose a novel privacy-preserved algorithm with non-IID data, DPDL, by leveraging the concept of DP in the cross-gradient aggregation through a similarity-based calibration.
    \item We perform solid theoretical analysis to quantify the level of privacy preservation we can achieve through our algorithm as the convergence rate of our algorithm.
    \item We conduct extensive experiments on real datasets to verify the efficacy of our DPDL algorithm in terms of both privacy preservation and model training. 
  \end{itemize}
 
  The remaining of this paper is organized as follows. Related literature is surveyed in Sec.~\ref{sec:work}. We then introduce our system model and some preliminaries in Sec.~\ref{sec:pre}. We give the details of our DPDL algorithm in Sec.~\ref{sec:dpdl} and analyze the algorithm theoretically in Sec.~\ref{sec:theory}. We finally conduct extensive experiments to verify the efficacy of our algorithm in Sec.~\ref{sec:exp} and conclude this paper in Sec.~\ref{sec:conclusion}.

\section{Related Literature}\label{sec:work}
  Recent years have witnessed the advancement of the paradigm of decentralized learning. In \cite{LianZZHZL-NIPS17}, \textit{Stochastic Gradient Descent} (SGD) is combined with gossip-averaging~\cite{BoydGPS-TIT06} to develop \textit{Decentralized Parallel Stochastic Gradient Descent} (DPSGD). \cite{ScamanBBLM-ICML17} considers strongly convex and smooth decentralized optimization in a network and proposes the \textit{Multi-Step Dual accelerated} (MSDA) algorithm with provable linear convergence rate. \cite{ScamanBBLM-NIPS18} proposes the first optimal algorithm, called \textit{Multi-Step Primal-Dual} (MSPD) algorithm, for the more challenging case of non-smooth convex optimization. Inspired by the fact that momentum methods are widely adopted to train machine learning models as they  converge faster and generalize better, a momentum extension of decentralized parallel SGD, so-called \textit{Decentralized Momentum Stochastic Gradient Descent} (DMSGD), is proposed in \cite{YuJY-ICML19}. 
  
  The above studies usually assume that training data are independently and identically distributed across different agents. Unfortunately, this assumption may not always hold in real-world applications; therefore, significant efforts have been made to address the challenge of data heterogeneity in decentralized learning. \textit{Quasi-Global} (QG) momentum is proposed in~\cite{LinKSJ-ICML21}. It is based on locally approximating the global optimization direction such that the smoothness and robustness of the decentralized learning process can be ensured with heterogeneous data. In \cite{LiLV-IOTJ21}, a decentralized learning algorithm through mutual knowledge transfer is proposed. In this algorithm, only a small subset of agents train their local models and transmit the local models to another set of agents. In \cite{ShiSWSYWT-ICML23}, each agent trains its local model through Sharpness Aware Minimization (SAM), which searches for models with uniformly low loss values. Moreover, various tracking approaches, e.g., gradient tracking~\cite{PuN-MP21} and momentum tracking~\cite{TakezawaBNSY-TMLR23}, have been proposed to resolve the data heterogeneity in decentralized learning. Different from the traditional tracking-based methods that track average gradients, \cite{AketiHR-NeurIPS24} proposes to let agents store a copy their neighbors' model parameters and then track the model updates instead of gradients. NET-FLEET\cite{ZhangFLYLZ-MobiHoc22} incorporates a recursive gradient correction technique to approximate the global stochastic gradient efficiently, combined with multi-round local update methods to accelerate the speed of convergence. Recently, \cite{EsfandiariTJBHHS-ICML21} proposes \textit{Cross gradient aggregation} (CGA) algorithm for decentralized learning on non-IID data. In CGA, each agent aggregates cross-gradients (i.e., derivatives of its model with respect to its neighbors’ datasets) and its local self-gradient based on \textit{quadratic programming} (QP). In \cite{AketiKR-TMLR23}, each agent takes into account two types of cross-gradients, i.e., the derivatives computed using its local data on its neighbors’ model parameters and the ones of its local model on its neighbors’ dataset.
  
  As mentioned in Sec.~\ref{sec:intro}, another challenging issue for distributed learning is privacy preservation. Differential privacy~\cite{Dwork-ICALP06,DworkR-FTTCS14} is a popular approach for privacy-preserving machine learning. Although it has been extensively used in centralized distributed learning (e.g., FL)~\cite{ShokriS-CCS15,LiWWHWLLH-TKDE23,LinWLSHJ-TDSC23,ChenDBZJ-TKDE24}, there are only a few proposals investigating privacy-preserved decentralized learning algorithms. \cite{HuangG-arXiv20} proposes a differentially private decentralized \textit{Alternating Direction Method of Multipliers} (ADMM) algorithm. In \cite{XuZW-TPAMI21}, a differentially private version of DPSGD is proposed. Specifically, each agent adds Gaussian noise to its local stochastic gradient before sharing with its peers. \cite{CyffersEBM-NIPS22} proposes Muffliato, a privacy amplification mechanism composed of local Gaussian noise injection at the node level followed by gossiping for averaging the private values. Unfortunately, none of these studies takes into account the heterogeneous data distribution. Recently, PDSL \cite{WangYWL-ICDCS25} proposed a decentralized differential privacy algorithm based on shapley values, but the calculation of shapley values is complex and requires a specifically generated validation set.

\section{System Model and Preliminary} \label{sec:pre}

  \subsection{Decentralized Stochastic Learning} \label{ssec:decentlearning}
    Consider a decentralized learning system consisting of a set of $N$ agents $\mathcal{N}=\{1,2,\cdots,N\}$. Let $\mathcal{D}_i$ represent the data distribution of agent $i \in \mathcal{N}$, and $D_i$ denote the number of data samples in $\mathcal{D}_i$. The main objective of decentralized learning is to let the agents collaboratively train a global model using their local datasets. This objective can be mathematically formulated as follows.

    \begin{equation} \label{eq:decentobj}
      \min_{{x} \in \mathbb{R}^d} \mathcal{F}({x}) = \frac{1}{N} \sum_{i=1}^{N} \underbrace{\mathbb{E}_{\zeta_i \sim \mathcal{D}_i} \left[F_i({x}; \zeta_i)\right]}_{\triangleq f_i({x})} 
    \end{equation}
    where ${x} \in \mathbb{R}^d$ represents the model parameters, $f_i({x}) \triangleq \mathbb{E}_{\zeta_i \sim \mathcal{D}_i} [F_i({x}; \zeta_i)]$ denotes the loss function for agent $i$, and $\zeta_i$ represents the randomly sampled data from $\mathcal{D}_i$. Without loss of generality, we assume $F_i (\cdot; \cdot) = F$ for $\forall i\in\mathcal{N}$.

    Suppose the agents communicate with each other over an undirected weighted graph $\mathcal{G} = (\mathcal{N},\mathbf{W})$, where $\mathbf{W} \in [0,1]^{N \times N}$ is the adjacency matrix of graph $\mathcal{G}$. Without loss of generality, we assume $\mathbf{W}$ is a symmetric doubly stochastic matrix, such that $\sum_{i = 1}^N w_{ij} = \sum_{j = 1}^N w_{ij} = 1$ and $w_{ij} = w_{ji}$ for $\forall i,j \in \mathcal{N}$, where $w_{ij}$ denotes the element in the $i$-th row and $j$-th column of the matrix $\mathbf{W}$~\cite{LianZZHZL-NIPS17, EsfandiariTJBHHS-ICML21, YuJY-ICML19, ZhangCHWY-ICML22}. We use $w_{ij}$ to encode how much agent $i$ can affect node $j$, while $w_{ij}=0$ implies agents $i$ and $j$ are disconnected. Let $\mathcal{N}_i \subseteq \mathcal{N}$ denote the set of neighbors of agent $i \in \mathcal{N}$ and $\mathcal{N}_i = \{ j \in \mathcal{N} \mid w_{i,j} > 0 \}$. Note that agent $i$ can communicate with itself by default such that $w_{ii} > 0$; hence, $\mathcal{N}_i$ includes agent $i$ itself, i.e., $i\in\mathcal{N}_i$. Assume $N_i = |\mathcal{N}_i|$ denotes the cardinality of $\mathcal{N}_i$.

  \subsection{Cross-gradient and Self-gradient} \label{ssec:crossgrad}
    As mentioned in Sec.~\ref{sec:intro}, in our DPDL algorithm, each agent updates it local model based on \textit{cross-gradients} received from its neighbors and \textit{self-gradient} calculated locally. We define the two key concepts as follows.
    \begin{defn} \label{def:sg}
      Consider agent $i \in \mathcal{N}$ with local dataset $\mathcal{D}_i$ and a random data sample \(\zeta_i \in \mathcal{D}_i\). Let $x_i$ denotes the local model parameters of agent $i$. The self-gradient $g^{ii}$ of agent $i$ with respect to $\zeta_i$ is defined as:
      \begin{equation} \label{eq:selfgrad}
        g^{ii} \triangleq \nabla_{x} F_i \left( x_i; \mathcal{\zeta}_i \right)
      \end{equation}
    \end{defn}
    \begin{defn}\label{def:cg}
      For any pair of agents $i$ and $j \in \mathcal{N}$ and a random data sample $\zeta_j\in\mathcal{D}_j$, the cross-gradient $g_t^{ij}$ of agent $i$ with respect to $\zeta_j$ is defined as:
      \begin{equation}\label{eq:cg}
        g^{ij} \triangleq \nabla_{x} F_j(x_i; \zeta_j)
      \end{equation}
      where $x_i$ denotes the local model parameters of agent $i$.
    \end{defn} 

    \subsection{Differential Privacy} \label{ssec:dp}
      Throughout the decentralized learning process, the agents collaborate with each other to train a global model through information exchange (e.g., exchanging the private gradient information in our case as shown above). This entails a significant demand on privacy preservation, as an adversary may eavesdrop on their communications to obtain the private information. In this paper, we leverage the notion of \textit{Differential Privacy} (DP)~\cite{Dwork-ICALP06,DworkR-FTTCS14} for the purpose of privacy preservation. Specifically, a randomized mechanism is differentially private, if the probability of generating a specific output does not depend significantly on whether or not a particular data point is involved in the input. The formal definition is given as in \textbf{Definition}~\ref{def:dp}.
      \begin{defn} \label{def:dp}
        A randomized mechanism $M: \mathbb{D} \rightarrow \mathbb{R}$ with domain $\mathbb{D}$ and range $\mathbb{R}$ is said to satisfy $(\epsilon, \delta)$-DP, if for any two adjacent inputs $\mathcal{D}, \mathcal{D}^{\prime} \in \mathbb{D}$ differing in a single item and any subset of outputs $\mathcal{S} \subseteq \mathbb{R}$, we have
        \begin{equation} \label{eq:dp}
          \Pr\left(\mathcal{M}(\mathcal{D}) \in \mathcal{S}\right) \leq e^{\epsilon} \Pr\left(\mathcal{M}(\mathcal{D}^{\prime}) \in \mathcal{S}\right) + \delta,
        \end{equation}
      \end{defn} 
      \noindent $\epsilon$ is the so-called \textit{privacy budget} specifying the privacy loss of a DP mechanism. A higher privacy budget indicates a greater potential for privacy loss.
      
      One approach to approximating a deterministic function, e.g., $f:\mathbb{D} \rightarrow \mathbb{R}^d$, with a DP mechanism is via additive noise calibrated to the $f$'s sensitivity. As shown in \textbf{Definition}~\ref{def:l2-sensitivity}, the $L_2$-sensitivity~\cite{Dwork-ICALP06,DworkR-FTTCS14,KasiviswanathanLNRS-FOCS08} of $f$,  i.e., $\Delta_2 f$, is defined as the maximum of the Euclidean norm $\left\Vert f(\mathcal{D}) - f(\mathcal{D}^{\prime}) \right\Vert$ where $\mathcal{D}$ and $\mathcal{D}^{\prime}$ are adjacent input; hence, it indicates the maximum influence that a single individual (i.e., the item in which $\mathcal{D}$ and $\mathcal{D}^{\prime}$ differ) can have on the output of $f$.
      \begin{defn} \label{def:l2-sensitivity}
        For $f : \mathbb{D} \rightarrow \mathbb{R}^d $, the $L_2$-sensitivity of $f$ is defined as
        \begin{equation} \label{eq:l2-sensitivity}
          \Delta_2 f = \max_{\mathcal{D}, \mathcal{D}^{\prime} \in \mathsf{D}} \left\Vert f(\mathcal{D}) - f(\mathcal{D}^{\prime}) \right\Vert
        \end{equation}
        where $\mathcal{D}$ and $\mathcal{D}^{\prime}$ differ by at most one element and $\Vert \cdot \Vert$ denotes Euclidean norm.
      \end{defn} 
      For example, Let $\mathbf{I}$ denotes the identity matrix, the Gaussian mechanism is defined by $M(\mathcal{D}) = f(\mathcal{D}) + \mathcal{N}(0, \sigma^2 \mathbf{I})$
      %
      %
      where $\mathcal{N}(0, \sigma^2 \mathbf{I})$ is a multivariate Gaussian distribution with zero mean and covariance $\sigma^2 \mathbf{I}$. $M$ is $(\epsilon, \delta)$-DP if $\sigma \geq \epsilon^{-1} \Delta_2 f \sqrt{2\ln(1.25/\delta)}$~\cite{Dwork-ICALP06,DworkR-FTTCS14,KasiviswanathanLNRS-FOCS08}.
      Thus, a smaller privacy budget $\epsilon$ requires Gaussian noise with a larger variance $\sigma^2$, and higher variance is also necessary when the sensitivity $\Delta_2 f$ is greater. 
      In addition, we introduce the definition of moment \cite{AbadiCGMMTZ-CCS16} in \textbf{Definition}~\ref{def:lambdamoments} which describes the $\lambda$-th moment of the privacy loss random variable $\mathcal{L}$ for randomized mechanism $M$. Intuitively, it indicates how much privacy is lost when mechanism $M$ is used.
      %
      
      \begin{defn} \label{def:lambdamoments}
        The ${\lambda}{th}$ moments of randomized mechanism $M: \mathbb{D} \rightarrow \mathbb{R}$ with input $\mathcal{D}$ is
        \begin{align} \label{eq:lambdamoments}
          \alpha_M(\lambda) &= \max_{\mathcal{D}, \mathcal{D}^{\prime}}\log \mathbb{E}_{\psi \sim M(\mathcal{D})} e^{\lambda \cdot \mathcal{L}(M, \mathcal{D}, \mathcal{D}^{\prime})}  \\
          &= \max_{\mathcal{D}, \mathcal{D}^{\prime}} \log \mathbb{E}_{\psi \sim M(\mathcal{D})} \left[ \frac{\Pr[M(\mathcal{D})= \psi]}{\Pr[M(\mathcal{D}^{\prime})= \psi]} \right]^{\lambda}
        \end{align}
        where $\mathcal{L}(M, \mathcal{D},\mathcal{D}^{\prime}) = \ln \left(\frac{\Pr[M(\mathcal{D})= \psi]}{\Pr[M(\mathcal{D}^{\prime})= \psi]} \right)$ represents the privacy loss between $\mathcal{D}$ and $\mathcal{D}^{\prime}$.
      \end{defn} 
      

\section{Our Privacy-Preserved Decentralized Stochastic Learning Algorithm}~\label{sec:dpdl}
  
  Our DPDL algorithm mainly consists of three phase: i) cross-gradient sharing with differential privacy, ii) gradient calibration based on similarity, and iii) local update with momentum. In the first phase, each agent calculates and shares cross-gradient information with its neighbors. For the purpose of preserving differential privacy, the gradient information is perturbed by Gaussian noise. In the second phase, considering the Gaussian noise and the heterogeneous data distribution, we calibrate the gradient information based on cosine similarity. The local model of each agent is finally updated in a momentum-like manner based on the calibrated gradient information in the third phase. The pseudo-code of our algorithm is shown in \textbf{Algorithm}~\ref{alg:dpdl}.

  \subsection{Cross-Gradient Sharing with Differential Privacy} \label{ssec:cgradexchange}
    In this section, we carry out both the computation and exchange of cross-gradients and the computation of self-gradients. Due to data heterogeneity, agents client models exhibit substantial divergence in parameter updates across training iterations. The cross-gradient mechanism serves to quantify the efficacy of neighboring agents' datasets when applied to the local model architecture, while simultaneously establishing directional guidance for subsequent optimization trajectories. 

    In each round $t=1,2,3,\ldots,T$, each agent $i$ first transmits its current model $x^i_{t-1}$ to its neighbors and samples batch data $\mathcal{B}^i_t$ from its local dataset (see Lines~\ref{ln:txcurmodel}$\sim$\ref{ln:randsample}). When receiving $x^j_{t-1}$ from its neighbor $j \in \mathcal{N}_i \setminus \{i\}$, agent $i$ computes cross-gradient of each $\zeta^i_{t,b} \in \mathcal{B}^i_t$ on model $x^j_{t-1}$ to pass information of $i$ to neighbor $j$:
    
    \begin{equation} \label{eq:compcrossgrad}
      g^{ji}_{t,b} = g^{ji}_{t} \left( \zeta^i_{t,b}  \right) =\nabla F_i \left( x^j_{t-1}; \zeta^i_{t,b} \right), \forall \zeta^i_{t,b} \in \mathcal{B}^i_t
    \end{equation}
    Given that cross-gradients are designed to be transmitted to specific neighbors for further computations, they become vulnerable to malicious interception during communication, particularly susceptible to gradient inversion attacks~\cite{GeipingBDM-NIPS20, FredriksonJR-CCS15}, which pose underlying privacy risks. To mitigate this, we apply DP-mechanism on cross-gradients (see Lines~\ref{ln:clipeachcrossgrad}$\sim$\ref{ln:pertcrossgrad}). Specifically, we first clip each individual cross-gradient generated in (\ref{eq:compcrossgrad}) to limit their $l_2$-sensitivity
    \begin{equation} \label{eq:clipcrossgrad}
      \hat{g}^{ji}_{t,b} = \hat{g}^{ji}_{t} \left( \zeta^i_{t,b} \right) = g^{ji}_{t,b} \cdot \min \left\{ 1, C\left\| g^{ji}_{t,b} \right\|^{-1} \right\}, \forall \zeta^i_{t,b} \in \mathcal{B}^i_t
    \end{equation}
    where $C$ denotes the clipping threshold. Then we perturb the averaged clipped gradient over each batch data $\mathcal{B}^i_t$ by adding independent Gaussian noise $z^{ji}_t \sim \mathcal{N}(0, \sigma^2 C^2 \mathbf{I})$ to preserve DP in the training process. 
    \begin{equation} \label{eq:noisecrossgrad}
        \ddot{g}_{t}^{ji} = \frac{1}{B} \left( \sum_b \hat{g}^{ji}_{t,b} + z^{ji}_t \right)
    \end{equation}
    where $\mathbf{I}$ denotes the identity matrix. After that, each agent $i$ sends perturbed cross-gradient $\ddot{g}_{t}^{ji}$ to its neighbor $j$ while receiving perturbed cross-gradient $\ddot{g}_{t}^{ij}$ computed by neighbor $j$. This process allows each agent to obtain gradients containing information from neighboring agents, thereby facilitating the design of update mechanisms for data heterogeneity scenario.
    %
    
    Then, following privacy-preserved methodology of cross-gradients, each agent $i$ computes its intrinsic self-gradient of each data sample $\zeta^i_{t,b} \in \mathcal{B}^i_t$ on its local model $x^i_{t-1}$:
    %
    %
    \begin{equation} \label{eq:compselfgrad}
        g^{ii}_{t,b} = g^{ii}_{t} \left( \zeta^i_{t,b}  \right) =\nabla F_i \left( x^i_{t-1}; \zeta^i_{t,b} \right), \forall \zeta^i_{t,b} \in \mathcal{B}^i_t
    \end{equation}
    Since the self-gradients of agent $i$ directly influence the update direction of model $x^i_{t-1}$ which will subsequently be exchanged among agents (as will be shown later in Sec.~\ref{ssec:gradcal}), and given the presence of self-loops, it is necessary to apply DP protection to self-gradients to prevent potential privacy leakage. Similar with (\ref{eq:clipcrossgrad})$\sim$(\ref{eq:noisecrossgrad}), we clip self-gradients of agent $i$ by each sample and perform batch perturbation by adding independent Gaussian noise $z^{ii}_{t} \sim \mathcal{N}(0, \sigma^2 \mathbf{I})$
    \begin{equation} \label{eq:clipselfgrad}
        \hat{g}^{ii}_{t,b} = \hat{g}^{ii}_{t} \left( \zeta^i_{t,b} \right) = g^{ii}_{t,b} \cdot \min \left\{ 1, C\left\| g^{ii}_{t,b} \right\|^{-1} \right\}, \forall \zeta^i_{t,b} \in \mathcal{B}^i_t
    \end{equation}
    %
    \begin{equation} \label{eq:noiseselfgrad}
      \ddot{g}_{t}^{ii} = \frac{1}{B} \left( \sum_b \hat{g}^{ii}_{t,b} + z^{ii}_t \right)
    \end{equation}
    as shown in Lines~\ref{ln:clipeachselfgrad}$\sim$\ref{ln:pertselfgrad} . It is noteworthy that Line~\ref{ln:clipselfgrad} necessitates an independent calculation of batch-averaged clipped self-gradient, which will be employed to guide the orientation of model update in forthcoming optimization step
    \begin{equation} \label{eq:clipbatchedselfgrad}
      \hat{g}^{ii}_{t} = \frac{1}{B} \sum_b \hat{g}^{ii}_{t,b}
    \end{equation}

    \begin{algorithm*}[!t]
      \caption{DPDL (at each agent $i \in \mathcal{N}$)}
      \label{alg:dpdl}
        \KwIn{Local dataset $\mathcal{D}_{i}$, doubly stochastic weight matrix $\mathbf{W}$, initial model parameter $x^i_0$ and momentum $v^i_0$, batch size $B$, learning rate $\eta$, momentum coefficient $\beta$, the maximum number of communication rounds $T$, clipping threshold $C$, variance of Gaussian noise $\sigma$.}
        \KwOut{Model $x_{T}^i$ at agent $i$ in the end of round $T$.}
        \ForEach{$t=1,2,3,\cdots,{T}$}{
          $\rhd$~\textit{Gradient sharing with differential privacy}: \\
          Send current model $x^i_{t-1}$ to each neighbor $j \in \mathcal{N}_i \setminus \{i\}$; \label{ln:txcurmodel}\\
          Randomly sample $B$ data as batch from local dataset $\mathcal{D}_i$, denote as $\mathcal{B}^{i}_t$; \label{ln:randsample}\\
          \ForEach{$j \in \mathcal{N}_i \setminus \{i\}$}{ \label{ln:sharecg-start}
            Receive $x^j_{t-1}$ from agent $j$; \label{ln:rxcurmodel}\\
            Compute cross-gradient $g^{ji}_{t,b} = g^{ji}_{t} \left( \zeta^{i}_{t,b}  \right) =\nabla F_i \left( x^j_{t-1}; \zeta^{i}_{t,b} \right)$ for any  $\zeta^{i}_{t,b} \in \mathcal{B}^{i}_t$; \label{ln:compeachcrossgrad}\\
            Clip the cross-gradient $\hat{g}^{ji}_{t,b} = \hat{g}^{ji}_{t} \left( \zeta^{i}_{t,b} \right) = g^{ji}_{t,b} \cdot \min \left\{ 1, C\left\| g^{ji}_{t,b} \right\|^{-1} \right\}$ for any $\zeta^{i}_{t,b} \in \mathcal{B}^{i}_t$; \label{ln:clipeachcrossgrad}\\
            Perturb the cross-gradient $\ddot{g}_{t}^{ji} = \frac{1}{B} \left( \sum_b \hat{g}^{ji}_{t,b} + z^{ji}_t \right)$ where $z^{ji}_t \sim \mathcal{N}(0, \sigma^2C^2\mathbf{I})$; \label{ln:pertcrossgrad}\\
            Send $\ddot{g}^{ji}_{t}$ to each neighbor $j \in \mathcal{N}_i \setminus \{i\}$; \label{ln:txpertgrad}\\
          } \label{ln:sharecg-end}
          Compute self-gradient $g^{ii}_{t,b} = g^{ii}_{t} \left( \zeta^{i}_{t,b}  \right) =\nabla F_i \left( x^i_{t-1}; \zeta^{i}_{t,b} \right)$ for any $\zeta^{i}_{t,b} \in \mathcal{B}^{i}_t$; \label{ln:compeachselfgrad}\\
          Clip the self-gradient $\hat{g}^{ii}_{t,b} = \hat{g}^{ii}_{t} \left( \zeta^{ii}_{t,b} \right) = g^{ii}_{t,b} \cdot \min \left\{ 1, C\left\| g^{ii}_{t,b} \right\|^{-1} \right\}$ for any $\zeta^{i}_{t,b} \in \mathcal{B}^{i}_t$; \label{ln:clipeachselfgrad}\\
          Aggregate the above clipped self-gradients $\hat{g}^{ii}_{t} = \frac{1}{B} \sum_b \hat{g}^{ii}_{t,b} $; \label{ln:clipselfgrad}\\
          Perturb the self-gradient $\ddot{g}_{t}^{ii} = \frac{1}{B} \left( \sum_b \hat{g}^{ii}_{t,b} + z^{ii}_t \right)$ where $z^{ii}_t \sim \mathcal{N}(0, \sigma^2C^2\mathbf{I})$; \label{ln:pertselfgrad}\\
          $\rhd$~\textit{Gradient calibration based on similarity}: \\
          Receive $\ddot{g}_{t}^{ij}$ from each neighbor $j \in \mathcal{N}_i \setminus \{i\}$; \label{ln:rxpertgrad}\\
          Compute cosine similarity $S \left( \ddot{g}^{ij}_{t}, \hat{g}^{ii}_{t} \right) = \frac{\left\langle\ddot{g}^{ij}_{t}, \hat{g}^{ii}_{t} \right\rangle}{\big\| \ddot{g}^{ij}_{t} \big\| \cdot \big\| \hat{g}^{ii}_{t} \big\|}$ for each $j \in \mathcal{N}_i$; \label{ln:compsimilarity}\\
          Compute calibration parameter $C^{ij}_t = \frac{1}{1+e^{S \left( \ddot{g}^{ij}_{t}, \hat{g}^{ii}_{t} \right)}}$ for each $j \in \mathcal{N}_i$; \label{ln:compcalparameter}\\
          Calibrate gradient ${\tilde{g}}^{ij}_{t} = \frac{1}{\sqrt{w_{ij}}N} \ddot{g}^{ij}_{t} + \alpha w_{ij} C^{ij}_t \hat{g}^{ii}_{t}$ for each $j \in \mathcal{N}_i$; \label{ln:compcalibration}\\
          Aggregate ${\tilde{g}}^i_t = \sum_{j \in \mathcal{N}_i} \tilde{g}^{ij}_{t}$; \label{ln:agggrad}\\
          $\rhd$~\textit{Local update with momentum:} \\
          Update momentum parameter $\tilde{v}^i_{t} = \beta v^i_{t-1} + {\tilde{g}}_t^i$; \label{ln:updatemompara}\\
          Update model parameter $\tilde{x}^i_{t} = x_{t-1}^i - \eta\tilde{v}^i_{t}$; \label{ln:updatemodelpara}\\
          Send $\tilde{v}^i_{t}$ and $\tilde{x}^i_{t}$ to each neighbor $j \in \mathcal{N}_i \setminus \{i\}$; \label{ln:txmompara}\\
          Receive $\tilde{v}^j_{t}$ and $\tilde{x}^j_{t}$ from each neighbor $j \in \mathcal{N}_i \setminus \{i\}$; \label{ln:rxmompara}\\
          Aggregate momentum parameters $v^i_{t} = \sum_{j \in \mathcal{N}_i} w_{ij} \tilde{v}^j_{t} $; \label{ln:aggremompara}\\
          Aggregate local models $x^i_{t} = \sum_{j \in \mathcal{N}_i} w_{ij} \tilde{x}^j_{t}$; \label{ln:aggremodel}
      }
    \end{algorithm*}

  \subsection{Gradient Calibration based on Similarity} \label{ssec:gradcal}
    To tackle the challenge of non-IID data, the cross-gradient information is aggregated through QP~\cite{EsfandiariTJBHHS-ICML21} or weighted averaging~\cite{AketiKR-TMLR23}. Unfortunately, as mentioned above, the cross-gradient information is perturbed by Gaussian noise, while the traditional gradient aggregation methods cannot effectively address the random perturbation, as will be shown later in Sec.~\ref{sec:exp}. Therefore, we propose to calibrate the cross-gradient information based on similarity before aggregating the gradients.
     
    %
    According to (\ref{eq:noisecrossgrad}), for each agent, both the cross-gradients received from its neighbors and the self-gradient are deviated from each other, since they are calculated based on heterogeneous datasets and are perturbed by Gaussian noise for privacy preservation. 
    To address this, our algorithm employs \textit{cosine similarity} to calibrate the gradients and mitigate such deviations, as it is a widely used metric for measuring the similarity between two non-zero vectors.
    Specifically, for any $j \in \mathcal{N}_i$, we compute the cosine similarity between $\ddot{g}_{t}^{ij}$ and $\hat{g}_{t}^{ii}$ as follows
    \begin{equation} \label{eq:similarity}
      S \left( \ddot{g}^{ij}_{t}, \hat{g}^{ii}_{t} \right) = \frac{\left\langle\ddot{g}^{ij}_{t}, \hat{g}^{ii}_{t} \right\rangle}{\big\| \ddot{g}^{ij}_{t} \big\| \cdot \big\| \hat{g}^{ii}_{t} \big\|}
    \end{equation}
    as shown in Lines~\ref{ln:rxpertgrad}$\sim$\ref{ln:compsimilarity}. We then adopt a sigmoid function to calculate a calibration parameter
    \begin{equation} \label{eq:calparameter}
      \mathcal{C}^{ij}_t = \frac{1}{1+e^{S\left( \ddot{g}_{t}^{ij}, \hat{g}_{t}^{ii} \right)}}
    \end{equation}
    which can make the similarity in (\ref{eq:similarity}) inversely proportional to the calibration parameter and avoid the appearance of negative coefficients. Based on which, we calibrate $\ddot{g}^{ij}_t$ by
    \begin{equation} \label{eq:calibrate}
      \tilde{g}^{ij}_{t} = \frac{1}{\sqrt{w_{ij}}N} \ddot{g}^{ij}_{t} + \alpha w_{ij} \mathcal{C}^{ij}_t \hat{g}^{ii}_{t}
    \end{equation}
    where $\alpha$ is a hyper-parameter to reduce the oscillation of the algorithm in sparse topology. According to (\ref{eq:similarity}) and (\ref{eq:calparameter}), for each neighbor $j \in \mathcal{N}_i \setminus \{i\}$ of agent $i$, we use $\hat{g}^{ii}_t$ to "rectify" the deviation of the perturbed gradient $\ddot{g}^{ij}_t$. Specifically, if $\ddot{g}^{ij}_t$ is closely aligned with $\hat{g}^{ii}_t$ (which implies they have a high cosine similarity), we apply a small correction, since the deviation is minor. Conversely, when the cosine similarity is low, indicating a significant deviation, a larger adjustment is made during the calibration process. Finally, as demonstrated in Line~\ref{ln:agggrad}, we aggregate all the calibrated cross-gradients and the self-gradient by
    
    %
    \begin{equation} \label{eq:aggragategrad}
      \tilde{g}^i_t = \sum_{j \in \mathcal{N}_i} \tilde{g}^{ij}_{t}
    \end{equation}

  \subsection{Local Update with Momentum} \label{ssec:upmomentum}
    To accelerate the convergence of our algorithm, we adopt momentum-like mechanism~\cite{Qian-NN19} to update the local models. In particular, as shown in Lines~\ref{ln:updatemompara}$\sim$\ref{ln:updatemodelpara}, we update the momentum parameter and the model parameter by
    \begin{equation}\label{eq:updatemompara}
      \tilde{v}^i_{t} = \beta v^i_{t-1} + \tilde{g}^i_t
    \end{equation}
    and
    \begin{equation}\label{eq:updatemodelpara}
      \tilde{x}^i_t = x^i_{t-1} - \eta\tilde{v}^i_t
    \end{equation}
    respectively. We then let agent $i$ broadcast $\tilde{v}^i_{t}$ and $\tilde{x}^i_{t}$ to its neighbors (see Line~\ref{ln:txmompara}). After receiving $\tilde{v}^j_{t}$ and $\tilde{x}^j_{t}$ from each neighbor $j \in \mathcal{N}_i \setminus \{i\}$, agent $i$ finally performs the following update
    \begin{equation}\label{eq:aggremompara}
      v^i_{t} = \sum_{j \in \mathcal{N}_i} w_{ij} \tilde{v}^j_{t}
    \end{equation}
    and
    \begin{equation}\label{eq:aggremodel}
      x^i_{t} = \sum_{j \in \mathcal{N}_i} w_{ij} \tilde{x}^j_{t}
    \end{equation}
    as shown in Lines~\ref{ln:aggremompara}$\sim$\ref{ln:aggremodel}, respectively.

\section{Theoretical Analysis} \label{sec:theory}
  In this section, we provide a theoretical analysis of the DPDL algorithm from two perspectives. Specifically, in Sec.~\ref{ssec:privacy}, we show how Gaussian noise is leveraged to achieve \((\epsilon, \delta)\)-DP. In Sec.~\ref{ssec:convergence}, we establish the convergence properties of our algorithm under DP and non-IID data. 

  \subsection{Privacy Preservation} \label{ssec:privacy}
  %
    %
    As shown in Sec.~\ref{sec:dpdl}, each agent clips the gradient of every individual data sample in the batch and then perturbs the averaged gradient over the batch with Gaussian noise determined by $\sigma$ and $C$. Taking a larger value for $\sigma$ implies that we introduce greater noise to perturb the gradient information, and thus results in stronger privacy preservation. In \textbf{Theorem}~\ref{thm:dp}, we reveal the lower bound of $\sigma$ to characterize the Gaussian noise based on which $(\epsilon, \delta)$-DP of our algorithm is ensured.
    %

    \begin{thm} \label{thm:dp}
      Let $w_{min} = \min_{i \in \mathcal{N}, j\in\mathcal{N}_i} w_{ij}$, $w_{max} = \max_{i \in \mathcal{N}, j\in\mathcal{N}_i} w_{ij}$ and $q = \frac{B}{N \min_{i \in \mathcal{N}} D_i}$. For any $0<\delta<1$ and $\epsilon \leq c_1 q^2 T$, \textbf{Algorithm}~\ref{alg:dpdl} is $(\epsilon, \delta)$-DP throughout $T$ rounds, when
      \begin{equation} \label{eq:global_lwsigma}
        \sigma \geq \frac{\left( \frac{2}{\sqrt{w_{min}}} + \frac{15\alpha N}{8}+ \frac{\alpha w_{max} (h+1)^2N}{2h} \right) c_2q \sqrt{T \ln(1/\delta)}}{\sqrt{\sum_{j \in \mathcal{N}_i} \frac{1}{w_{max}}} \cdot \epsilon}
      \end{equation}
      where $h = \max_{i \in \mathcal{N}, t<T, \zeta^{i}_{t,b} \in \mathcal{B}^{i}_t} \frac{\left\| \frac{1}{B}\sum_b \hat{g}^{ii}_{t,b} \right\|}{\left\| \frac{1}{B}\sum_b \hat{g}^{ii}_{t,b} + z^{ii}_t \right\|}+1$ and $c_1, c_2 >0$ are constants.
    \end{thm}
    The lower bound of $\sigma$ is primarily determined by privacy budget $\epsilon$, implying that a smaller $\epsilon$ necessitates a larger $\sigma$. Additionally, it is also positively correlated with to the following parameters: the training rounds $T$, the agents number $N$, the global sampling rate $q$ and the hyper-parameters $\alpha$. The corresponding proof of \textbf{Theorem}~\ref{thm:dp} is provided in  Appendix~\ref{sec:proofdp}.

  \subsection{Convergence} \label{ssec:convergence}
    Before presenting our main results about the convergence of our algorithm, we introduce \textbf{Assumptions}~\ref{assum:smooth}$\sim$\ref{assum:dsm} which have been commonly used in the theoretical analysis of stochastic learning algorithms~\cite{EsfandiariTJBHHS-ICML21, YuJY-ICML19}. Note that when these assumptions do not hold, our algorithm still works but the convergence may not be ensured. \textbf{Assumption}~\ref{assum:smooth} implies the Lipschitz smoothness of $f_i(\cdot)$ for $\forall i\in\mathcal{N}$. In \textbf{Assumption}~\ref{assum:variance}, gradient variances induced by data heterogeneity are bounded. In \textbf{Assumption}~\ref{assum:dsm}, we suppose the adjacent matrix of the communication graph is doubly stochastic.
    \begin{assumption} \label{assum:smooth}
      \textbf{Lipschitz smoothness}. Each function $f_i(\cdot)$ is $L$-smooth for any $i\in\mathcal{N}$. Specifically, there exist a constant $L \geq 0$ such that
      \begin{equation} \label{eq:smooth}
        f_i(x) \leq f_i(y)+\nabla f^\top_i(y)(x-y)+\frac{L}{2}\|x-y\|^2, ~\forall i\in\mathcal{N}
      \end{equation}
    \end{assumption}
    \begin{assumption} \label{assum:variance}
      \textbf{Bounded gradient variance}. There exist $\xi>0$ and $\kappa>0$ such that
      \begin{equation} \label{var_1}
        \mathbb{E}_{\zeta\sim\mathcal{D}_i} \big[ \left\|   \nabla F_i(x;\zeta) - \nabla f_i(x) \right\|^2 \big] \leq \xi^2, ~\forall i\in\mathcal{N}
      \end{equation}
      and
      \begin{equation} \label{var_2}
        \frac{1}{N} \sum_{i=1}^N \big\| \nabla f_i({x})-\nabla \mathcal{F}({x}) \big\|^2 \leq \kappa^2, ~\forall i\in\mathcal{N}
      \end{equation}
      %
    %
    \end{assumption}
    \begin{assumption} \label{assum:dsm}
      \textbf{Doubly Stochastic Matrix}. The adjacent matrix of the communication graphs $\mathcal{G}$, i.e., $\mathbf{W}$, is doubly stochastic matrix such that
      \begin{equation} \label{eq:dsm}
        \lambda_1(\mathbf{W}) = 1 ~\text{and}~ \max \left\{  \left|\lambda_2(\mathbf{W})\right|, \left|\lambda_N(\mathbf{W}) \right| \right\} \leq \sqrt{\rho}
      \end{equation}
      where $\lambda_i(\mathbf{W})$ denotes the $i$-th largest eigenvalue of $\mathbf{W}$ and $\rho \in [0,1]$ is a constant.
    \end{assumption}

    It is revealed in \textbf{Theorem}~\ref{thm:convergence} that, when \textbf{Assumptions}~\ref{assum:smooth}-\ref{assum:dsm} hold, the average gradient magnitude achieved by our algorithm is mainly bounded by the difference between the initial value of the objective function and the optimal objective value, if we carefully tune the learning rate. In addition, the level of the Gaussian noise represented by $\sigma$ and the clipping threshold $C$, also influence the convergence of our algorithm. Nevertheless, they are minor polynomial factors and are positively correlated with the convergence bound (larger values of $\sigma$ or $C$ result in more relax convergence bound). The detailed proof of \textbf{Theorem}~\ref{thm:convergence} can be found in Appendix~\ref{sec:convergenceproof}. 
    %
    %
    %
    \begin{thm} \label{thm:convergence}
      Let \textbf{Assumptions}~\ref{assum:smooth}-\ref{assum:dsm} hold. Let $\bar{x}_t = \frac{1}{N} \sum^N_{i=1} x^i_t$ and $w_{min} = \min_{i \in \mathcal{N}, j\in\mathcal{N}_i} w_{ij}$. When learning rate $\eta$ satisfies the following condition
      \begin{align}\label{eq:learningrate}
      \begin{cases}
        \eta & > \frac{(1-\beta)^2}{256} \\
        \eta & \leq \frac{(1-\beta)(1-\rho)}{2\sqrt{30}L} \\
        \eta & \leq \frac{(1-\sqrt{\rho}) \sqrt{256^2+15(1-\beta)^2L^2N^2}-256(1-\sqrt{\rho})^2}{30L^2N}
      \end{cases}
      \end{align}
      for any $T \geq 1$, we have
      \begin{align}\label{eq:convergence}
        & \frac{1}{T} \sum^{T}_{t=1} \mathbb{E}\left[ \left\|\nabla \mathcal{F}\left(\bar{x}_{t-1}\right)\right\|^{2}\right] \nonumber \\
        \leq & \frac{1}{C_{1}T}\left( \mathcal{F}\left(\bar{{x}}_{0}\right) -\mathcal{F}^{\star}\right) 
        + \frac{60\eta^2C_{6}} {(1-\beta)^2(1-\sqrt{\rho})^2} \kappa^2  \nonumber \\
        & + \left(\frac{4C_5}{N} + \frac{60\eta^2C_{6}} {(1-\beta)^2(1-\sqrt{\rho})^2} \right) \xi^2 \nonumber \\
        & + \left(C_{2} \frac{\eta^2\beta^2}{(1-\beta)^4} + C_{3} + 2C_5\right)\frac{3C^2\sigma^2d}{w_{min}^2B^2N^2} \nonumber \\
        & + \frac{24\eta^2C_6 C^2\sigma^2d} {w_{min}^2(1-\beta)^2(1-\sqrt{\rho})^2B^2N} C^2\sigma^2d \nonumber \\
        & + \left(C_{2} \frac{\eta^2\beta^2}{(1-\beta)^4} + C_{3} + 2C_5 \right) \frac{48 C^2 + 27w_{min}^2\alpha^2 C^2N}{16w_{min}^2N} \nonumber \\
        & + \frac{\left(48B^2 + 27w_{min}^2\alpha N\right) \eta^2 C_6 C^2} {w_{min}^2(1-\beta)^2(1-\sqrt{\rho})^2B^2N}
      \end{align}
      where
      \begin{equation} \label{eq:conconstants}
      \begin{cases}
        &C_{1} = \frac{\eta}{2(1-\beta)}-\frac{1-\beta}{32} \vspace{.5ex}, \;\; C_{2}=\frac{(1-\beta) L}{2 C_{1} \beta} \vspace{.5ex}, \;\; C_{3}=\frac{L \eta^2}{2 C_1 (1-\beta)^3} \vspace{.5ex}\\
        &C_{4}=\frac{\eta}{2 C_{1} (1-\beta)}, \;\; C_{5}=\frac{8\eta^2}{2C_1(1-\beta)^3}, \;\; C_{6}=\frac{L^{2} \eta}{2 C_{1} (1-\beta)}
      \end{cases}
      \end{equation}
    \end{thm}

    From the above theorem, we can observe that the convergence of our algorithm is also influenced by the following parameters: 1) $\xi$ denotes the standard variance of self-gradients; 2) $\kappa$ is the standard variance of gradients among agents; 3) $L$ represents Lipschitz smoothness; 4) $\rho$ quantifies the information mixing rate of the communication graph. These parameters are directly proportional to the convergence bound. For instance, smaller values of $L$ and $\rho$ indicate tighter convergence bounds.
    Nevertheless, they are typically treated as fixed constant with reasonable values according to~\cite{YuJY-ICML19, EsfandiariTJBHHS-ICML21}, thus do not play a dominant role in determining the convergence. We then present \textbf{Corollary}~\ref{cor:convergence} to provide a more concise explanation of the convergence of our algorithm.

    \begin{corollary} \label{cor:convergence}
      Assume all assumptions in \textbf{Theorem}~\ref{thm:convergence} and set $\sigma$ as the lower bound in \textbf{Theorem}~\ref{thm:dp} for any $(\epsilon, \delta)$-DP satisfying the constraints, fixing $\beta$, $\rho$, $L$ and $\alpha$ as acceptable constants, we have
      \begin{align} \label{eq:corconvergence}
        &\frac{1}{T}\sum_{t=1}^{T} \mathbb{E} \left[\left\| \nabla \mathcal{F} \left( \bar{x}_{t-1} \right) \right\|^{2}\right] \nonumber \\
        = & \mathcal{O} \left( \frac{1}{\eta T} + \eta^2 + \frac{\eta C^2\sigma^2d}{w_{min}^2B^2N^2} + \frac{\eta^2 C^2\sigma^2d}{w_{min}^2B^2N} + \frac{\eta C^2}{w_{min}^2N} \right) \nonumber \\
        = & \mathcal{O} \left( \frac{1}{\eta T} + \eta^2 + \frac{\eta C^2Td}{\epsilon^2N^2} + \frac{\eta^2 C^2Td}{\epsilon^2N} + \frac{\eta C^2}{w_{min}^2N} \right)
      \end{align}
      Especially, when $\eta = \mathcal{O}\left(\frac{N\epsilon}{T}\right)$ and $T$ is sufficiently large with:
      \begin{equation}\label{Tbound}
        T \geq \max \left\{ 
        \begin{aligned} 
        &\frac{2\sqrt{30}NL\epsilon}{(1-\beta)(1-\rho)}, \\
        & \frac{\frac{30}{1-\sqrt{\rho}}N^2L^2\epsilon}{\sqrt{256+15(1-\beta)L^2N^2}-16(1-\sqrt{\rho})}
        \end{aligned}
        \right\}
      \end{equation}
      we have
      \begin{align} \label{eq:corconvergencespec}
        \frac{1}{T}\sum_{t=1}^{T} \mathbb{E} \left[\left\| \nabla \mathcal{F} \left( \bar{x}_{t-1} \right) \right\|^{2}\right]
        = \mathcal{O} \left( \frac{1}{N\epsilon} + \frac{C^2d}{T} \right)
      \end{align}
      %
    \end{corollary}
    %
    
    %
    %
    %
    This corollary implies that when given fixed $C$, $N$ and $d$, the convergence rate of our algorithm is mainly dominated by $\epsilon$ and $T$. If we completely ignore Gaussian noises, our algorithm can achieve $O(\frac{1}{T})$ convergence rate. This convergence rate is comparable to the well-known results in decentralized SGD algorithms~\cite{EsfandiariTJBHHS-ICML21, XuZW-TPAMI21}. We present its proof in Appendix~\ref{sec:cor_proof}.

\section{Experiments} \label{sec:exp}
  In this section, we conduct extensive experiments to demonstrate the performance of our DPDL algorithm. In Sec.~\ref{ssec:setting}, we present the settings of our experiments and briefly introduce the state-of-the-art reference algorithms. We report our experiment results about convergence and the ones about test accuracy in Sec.~\ref{ssec:expconvergence} and Sec.~\ref{ssec:expaccuracy}, respectively. In Sec.~\ref{ssec:expattacks}, we present our experiment results about the efficacy of our DPDL algorithm in safeguarding against privacy attacks (e.g., gradient inversion attacks~\cite{GeipingBDM-NIPS20, FredriksonJR-CCS15}).

  \subsection{Experiment Settings} \label{ssec:setting}
    We conduct our experiments on MNIST\cite{LecunBBH-IEEE98} and CIFAR-10\cite{Krizhevsky-09}  datasets. MNIST is a database of handwritten digits, containing $60,000$ training images and $10,000$ test images. CIFAR-10 is an image dataset containing $10$ categories, including airplane, automobile, bird, cat, etc. The CIFAR-10 dataset contains $50,000$ training images and $10,000$ test images. We adopt Dirichlet distribution parameterized by $\alpha_d$ to produce heterogeneous data distributions. A smaller $\alpha_d$ value implies the data distribution is more imbalanced. We let $\alpha_d = 0.25, 1.0$ for both datasets.

    We adopt three different communication graph topologies in our experiments: fully connected, bipartite, and ring graphs, with decreasing levels of connectivity. To evaluate the scalability of our algorithm, we also vary the number of agents (e.g., $N = 10, 20, 30$). Due to space limitations, we hereby present only the results for the bipartite and fully connected graphs. Additional results for the ring topology can be found in Appendix~\ref{sec:supp_ex_con}.

    We train a LeNet model~\cite{SimonyanZ-ARXIV14} on the MNIST dataset. This model comprises two convolutional layers with kernel sizes $1 \times 5 \times 5$ and $6 \times 5 \times 5$, respectively. Each convolutional layer is followed by a $2 \times 2$ max pooling layer and then a fully connected layer. For the CIFAR-10 dataset, we train a CNN model consisting of convolutional layers with kernel sizes $3 \times 3 \times 3$ and $16 \times 3 \times 3$, respectively, each followed by a $2 \times 2$ max pooling layer and a fully connected layer. Moreover, we let the learning rate $\eta=0.005$, clipping threshold $C=2$ and privacy budget $\epsilon=0.25, 0.5, 1.0$ for MNIST, and $\eta=0.03$, $C=3$ and $\epsilon= 2.0, 4.0, 8.0$ for CIFAR-10. Additionally, we let the training batch size $B=216$ and the momentum coefficient $\beta=0.7$ for both datasets. We also set hyper-parameter $\alpha=1.5$ for both the bipartite and fully connected graphs and $\alpha=2$ for the ring graph. Note that all these parameters satisfy the conditions shown in \textbf{Theorem}~\ref{thm:dp} and \textbf{Theorem}~\ref{thm:convergence}.
    

    In our experiments, we compare our DPDL algorithm with the following state-of-the-art ones.
    %
    %
    \begin{itemize}
      \item \textbf{DP-DPSGD}: In \cite{XuZW-TPAMI21}, Gaussian mechanism is applied in decentralized parallel SGD~\cite{LianZZHZL-NIPS17}. Specifically, each agent perturbs its local gradient shared with its neighbors for preserving its local privacy. Unfortunately, since each agent computes its local gradient based solely on its own model and dataset, data heterogeneity is not addressed.

      \item \textbf{MUFFLIATO}: MUFFLIATO adds Gaussian noises to the updated model parameters without clipping, followed by multi-rounds gossip averaging based on the spectral gap~\cite{CyffersEBM-NIPS22}. Nevertheless, MUFFLIATO does not address the data heterogeneity problem either. 
      \item \textbf{DP-NETFLEET}: In \cite{ZhangFLYLZ-MobiHoc22}, NET-FLEET algorithm is proposed. It adopts a recursive gradient correction technique adopted to handle non-IID data, such that each agent runs multiple local updates between two consecutive communication rounds with their neighbors. Since it does not account for privacy preservation, we adapt the algorithm by applying a Gaussian noise mechanism, similar to that in~\cite{XuZW-TPAMI21}, to ensure a fair comparison.
      \item \textbf{DP-CGA}: \cite{ZhuHZ-ICML21} proposes \textit{Cross-Gradient Aggregation} (CGA) algorithm, where each agent aggregates cross-gradients and its local self-gradient based on quadratic programming. In this variant of CGA, we apply Gaussian noise mechanism such that each agent perturbs the cross-gradient information before sharing it with its neighborhood. Specifically, following the perturbation method commonly used in existing studies \cite{WeiLDMYFJQP-TIFS20, ChenYZYC-TII22, YuZCTTLC-TVV21}, each agent first computes the average cross-gradient (or self-gradient) over its local data batch, and then conducts clipping operation and adds Gaussian noise to ensure privacy.   
    \end{itemize}
  \begin{figure*}[htb!]
  \begin{center}
    \parbox{.3\textwidth}{\center\includegraphics[width=.3\textwidth]{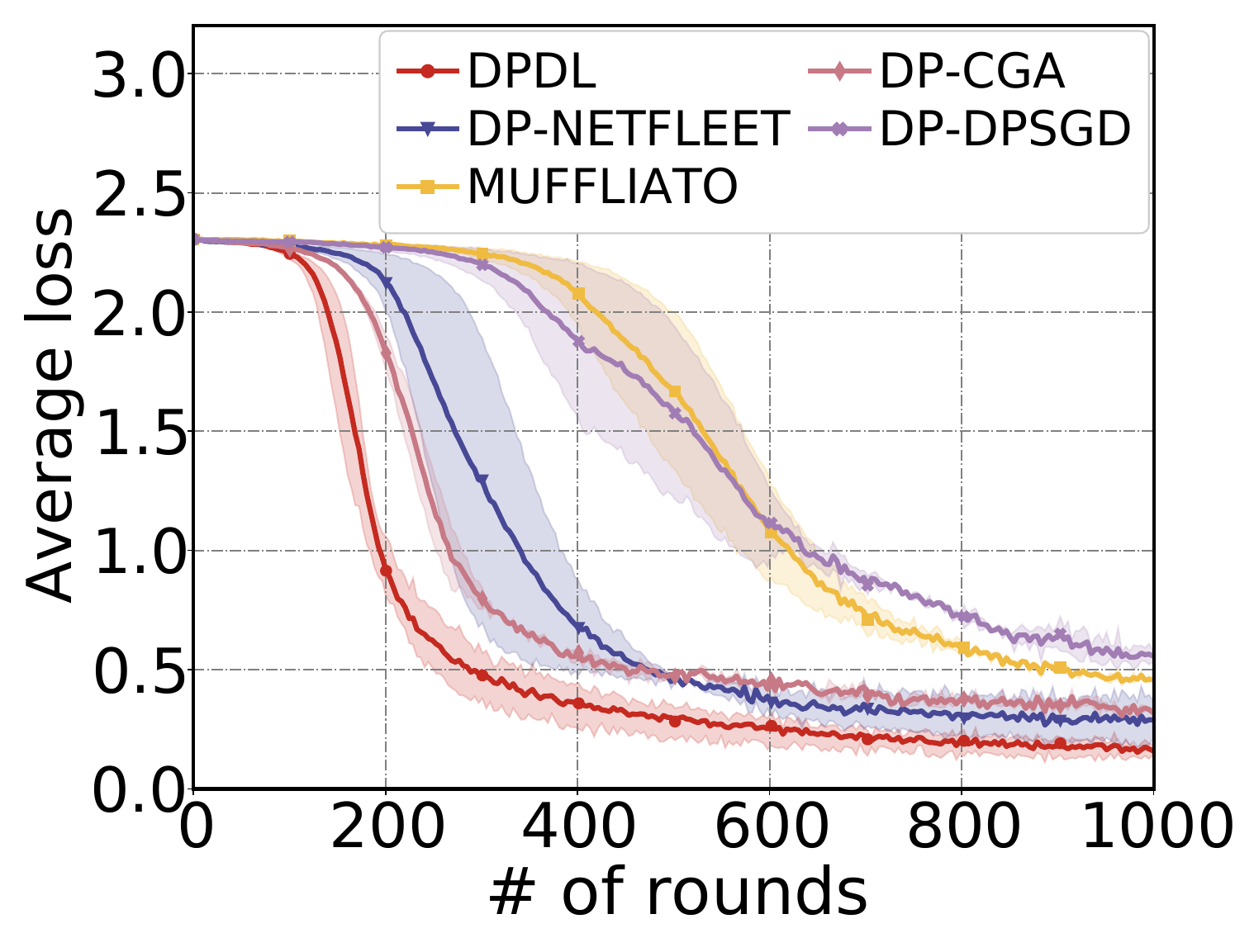}}
    \parbox{.3\textwidth}{\center\includegraphics[width=.3\textwidth]{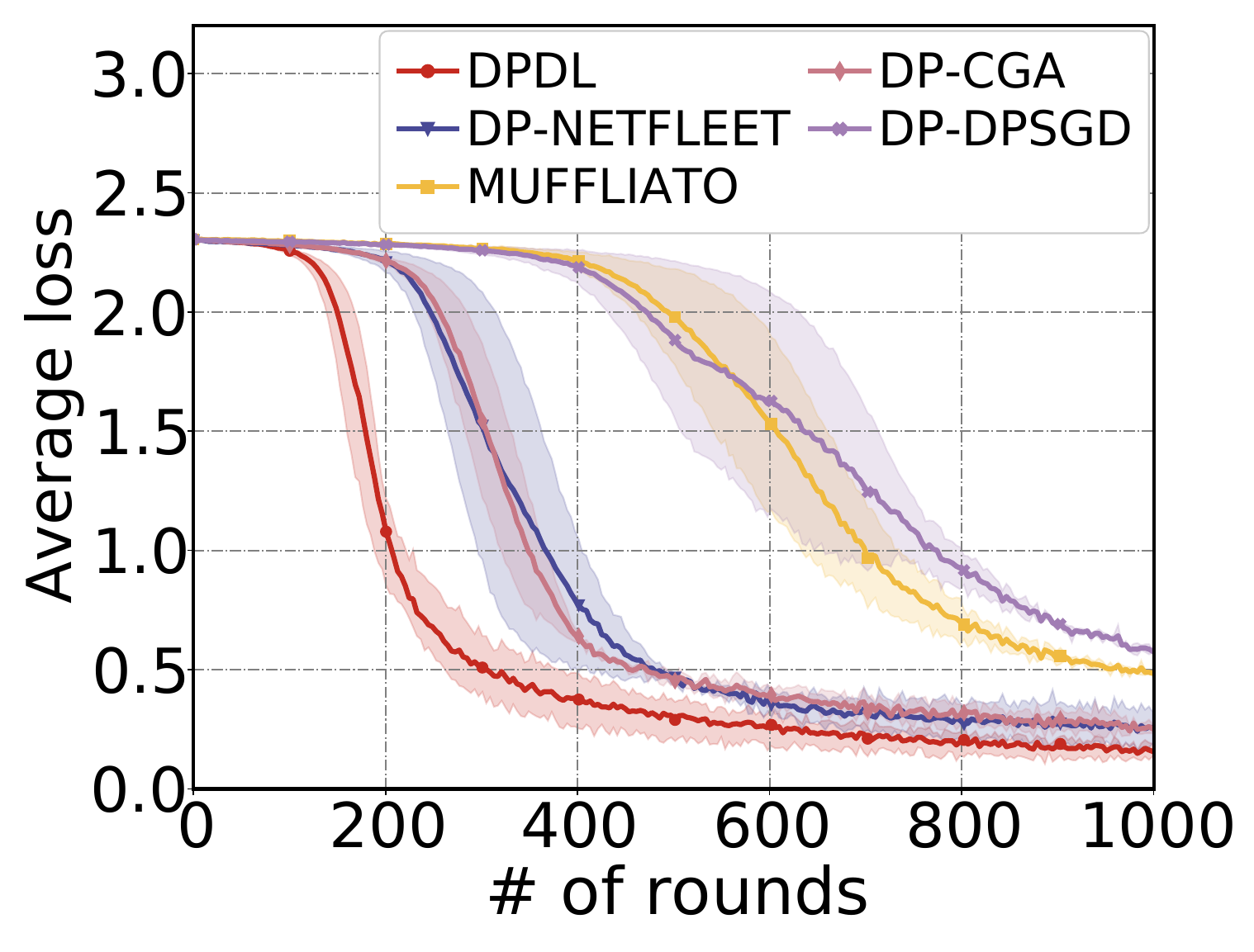}}
    \parbox{.3\textwidth}{\center\includegraphics[width=.3\textwidth]{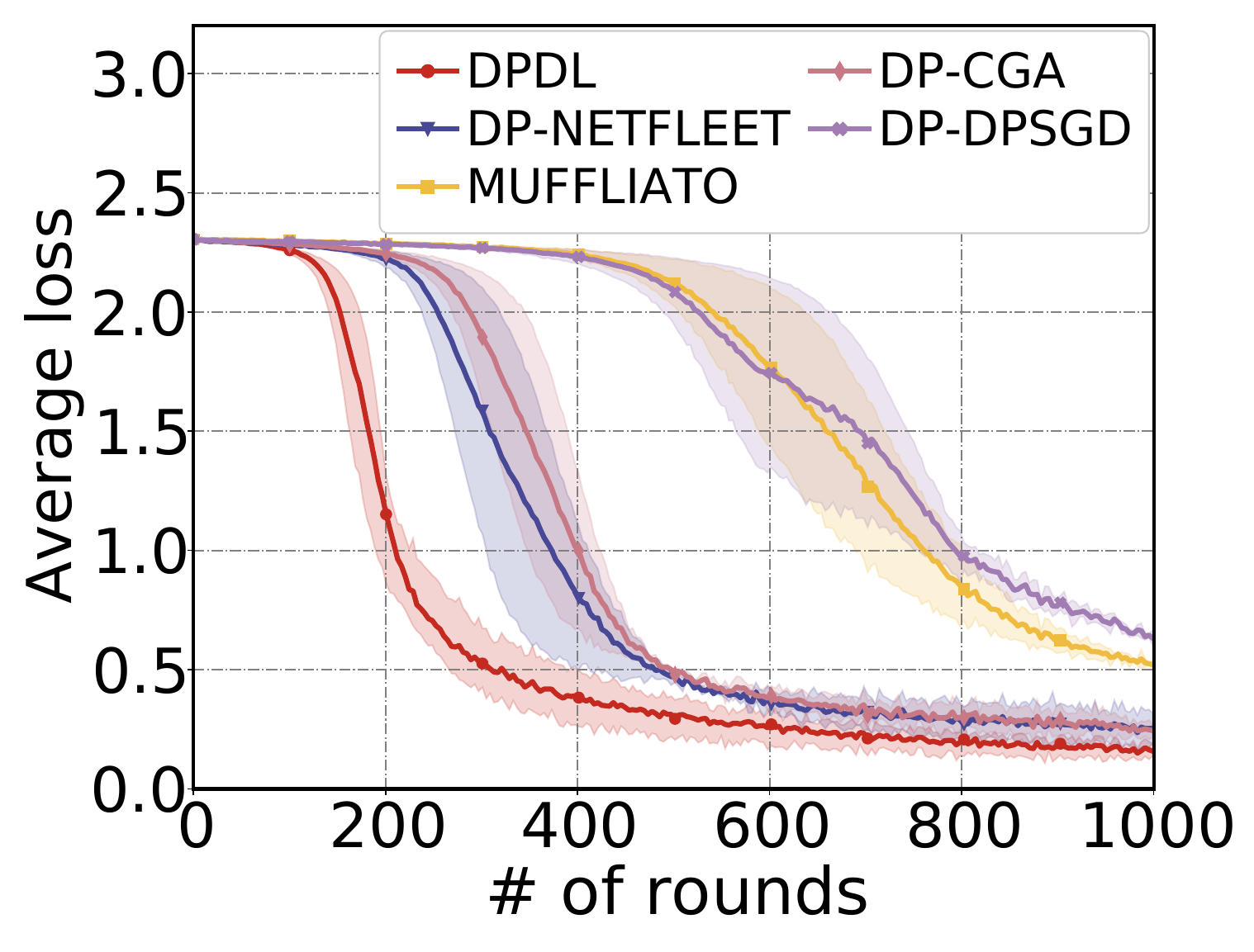}}
    \parbox{.3\textwidth}{\center\scriptsize(a1) $\epsilon=0.25, N=10$}
    \parbox{.3\textwidth}{\center\scriptsize(b1) $\epsilon=0.5, N=10$}
    \parbox{.3\textwidth}{\center\scriptsize(c1) $\epsilon=1.0, N=10$}
    \parbox{.3\textwidth}{\center\includegraphics[width=.3\textwidth]{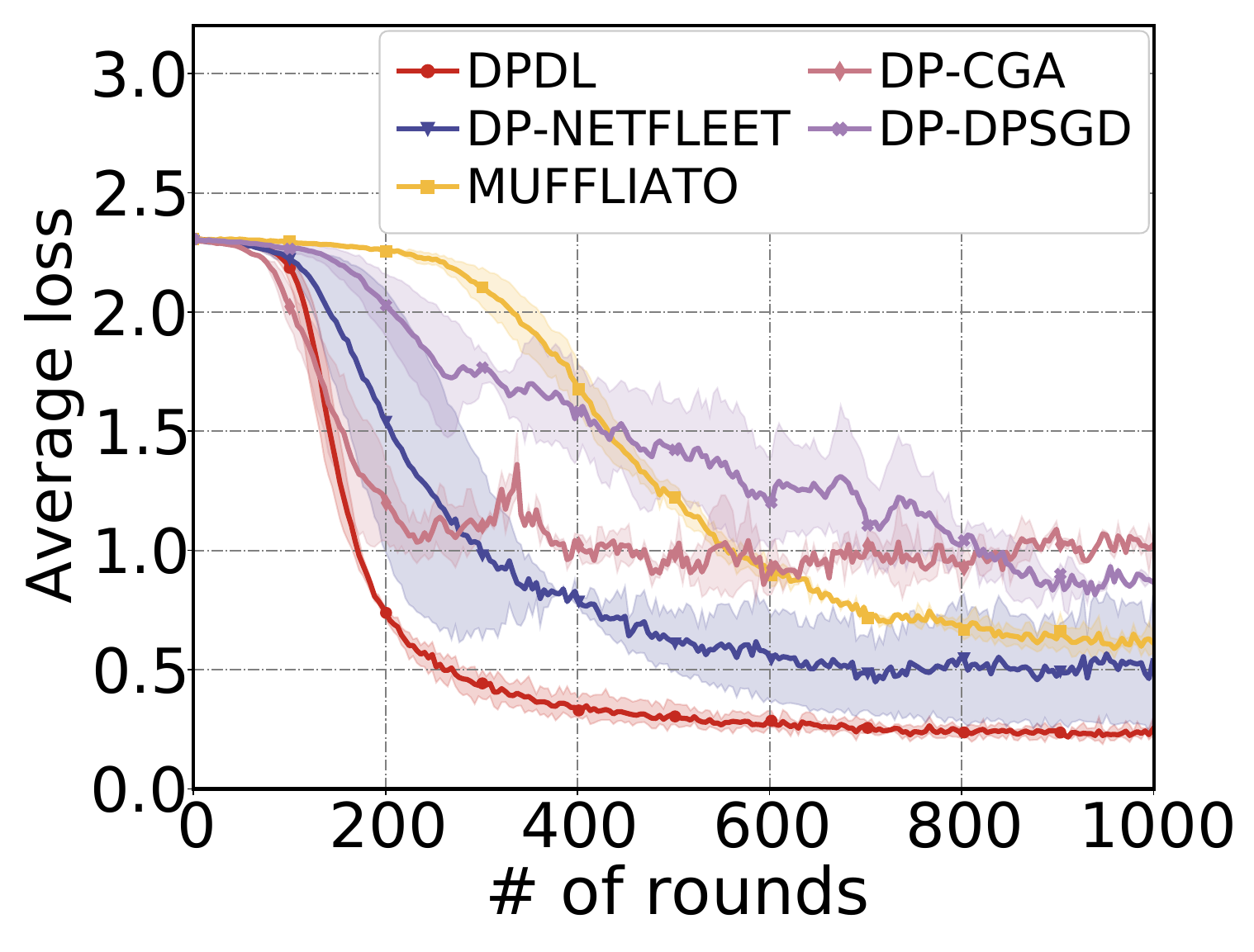}}
    \parbox{.3\textwidth}{\center\includegraphics[width=.3\textwidth]{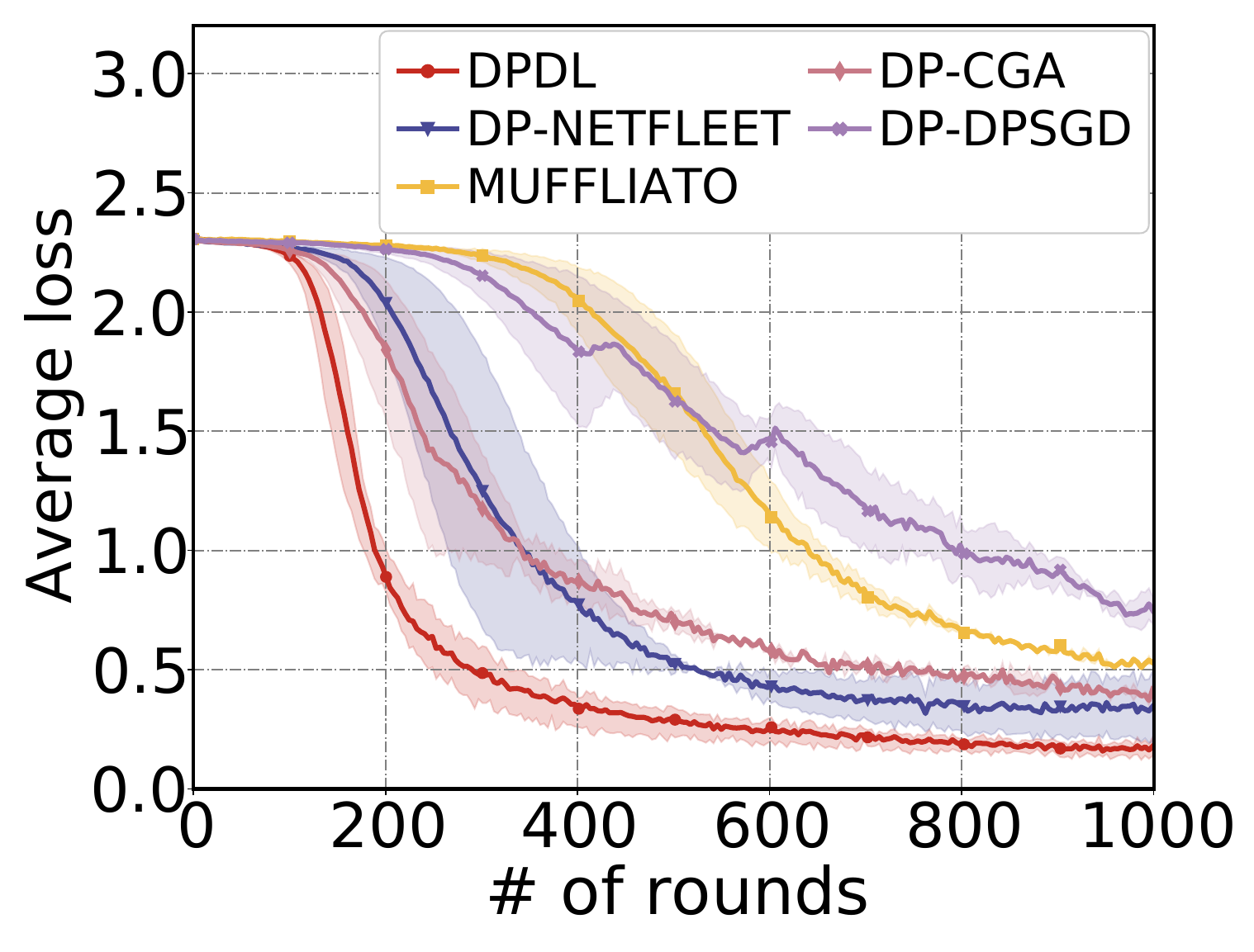}}
    \parbox{.3\textwidth}{\center\includegraphics[width=.3\textwidth]{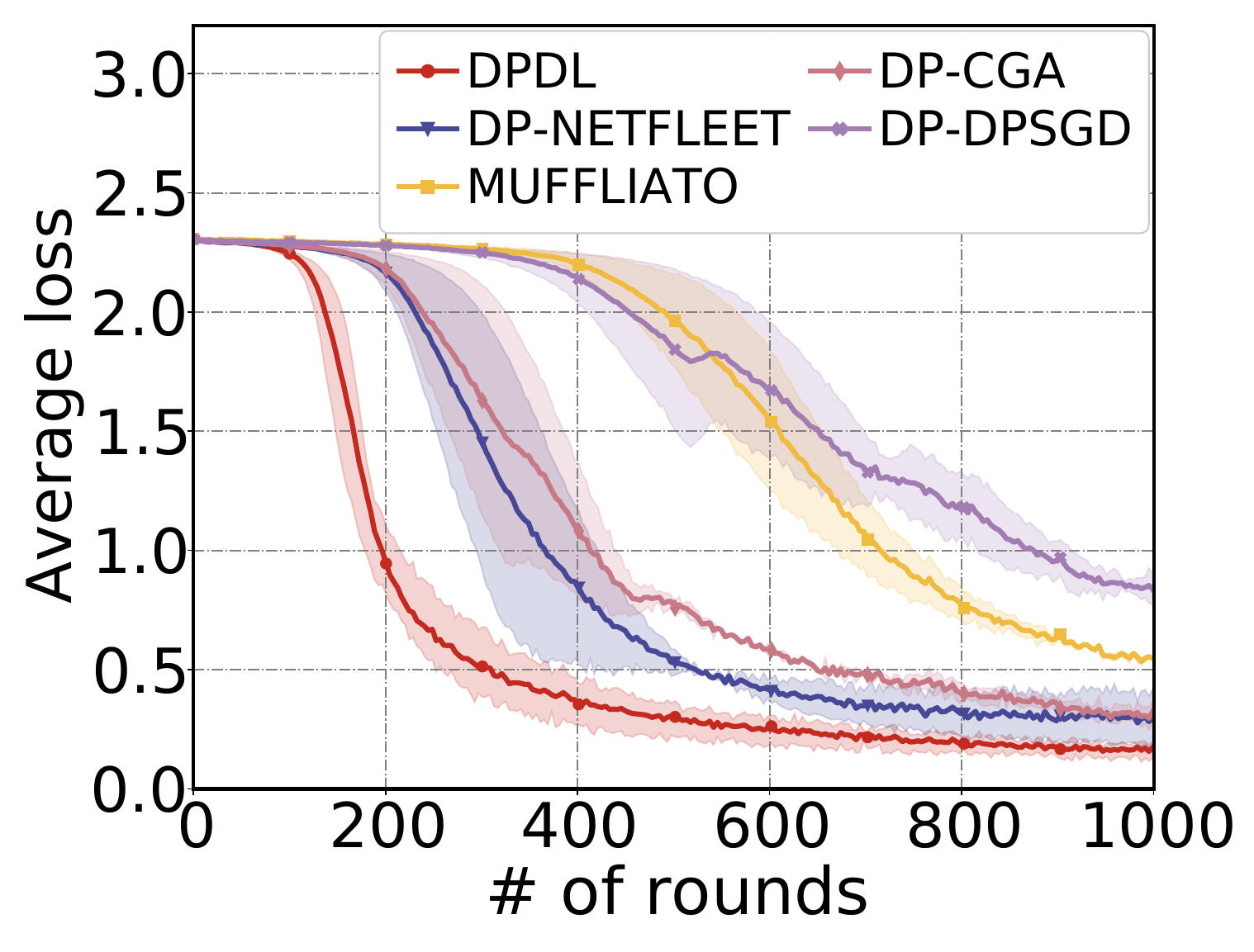}}
    \parbox{.3\textwidth}{\center\scriptsize(a2) $\epsilon=0.25, N=20$}
    \parbox{.3\textwidth}{\center\scriptsize(b2) $\epsilon=0.5, N=20$}
    \parbox{.3\textwidth}{\center\scriptsize(c2) $\epsilon=1.0, N=20$}
    \parbox{.3\textwidth}{\center\includegraphics[width=.3\textwidth]{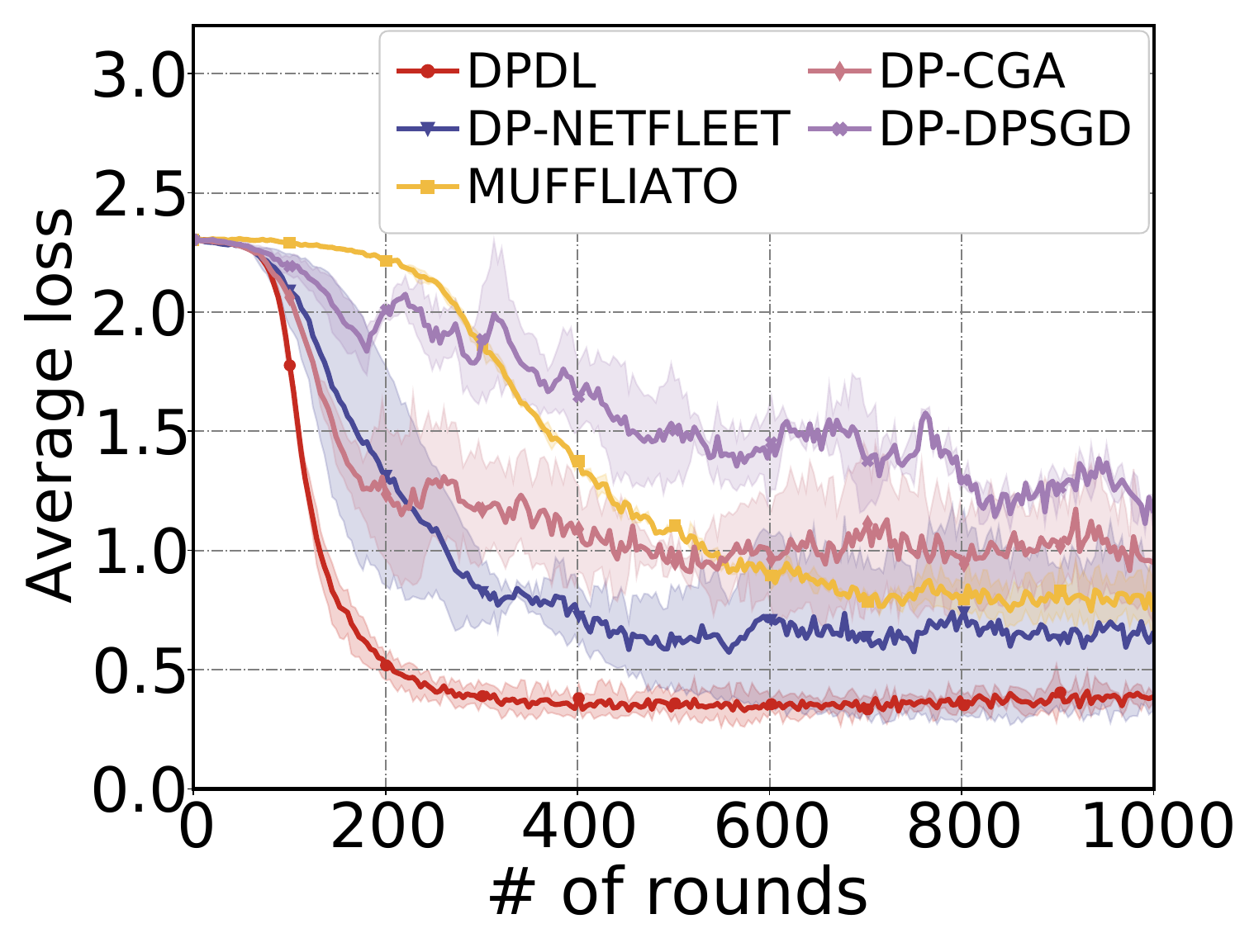}}
    \parbox{.3\textwidth}{\center\includegraphics[width=.3\textwidth]{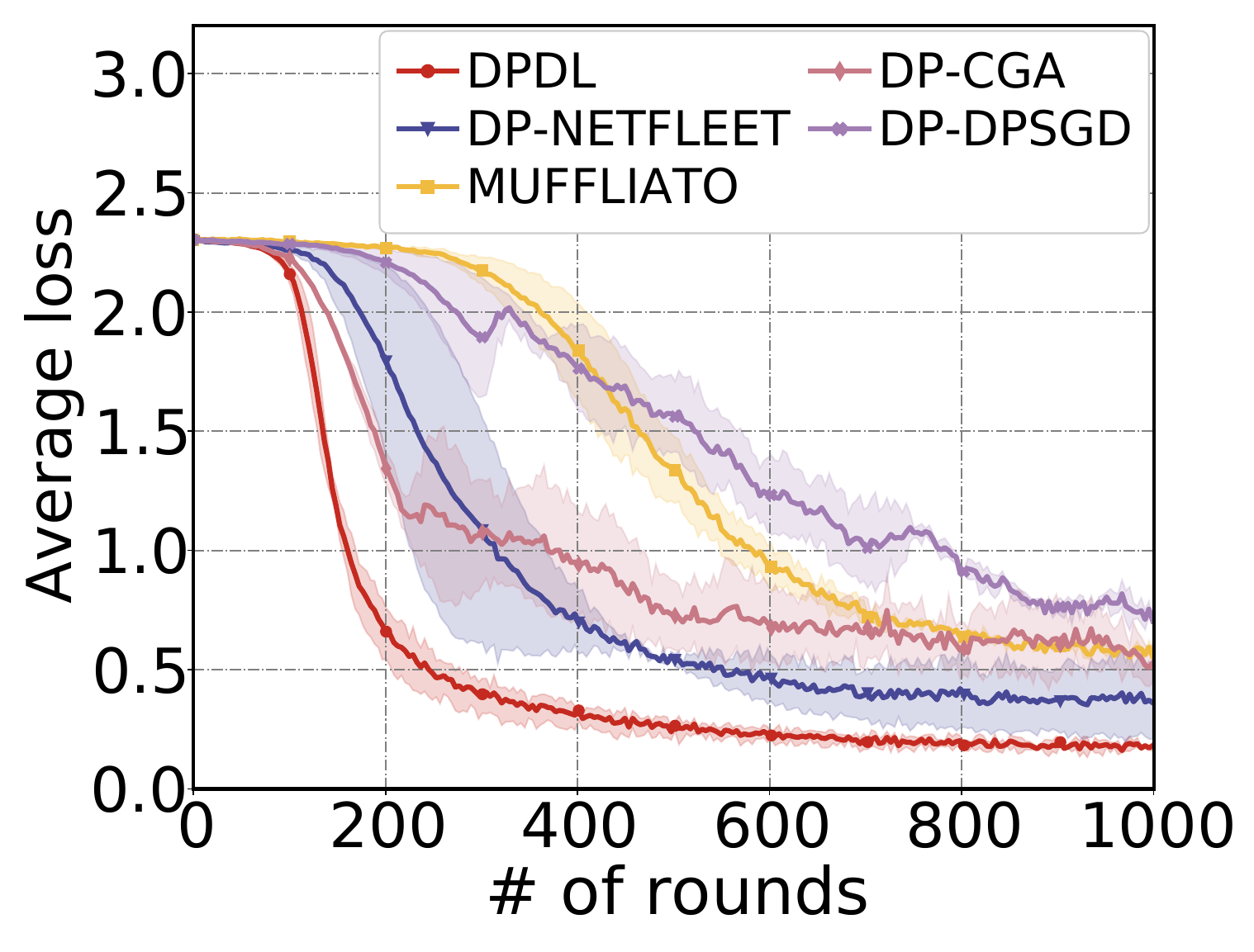}}
    \parbox{.3\textwidth}{\center\includegraphics[width=.3\textwidth]{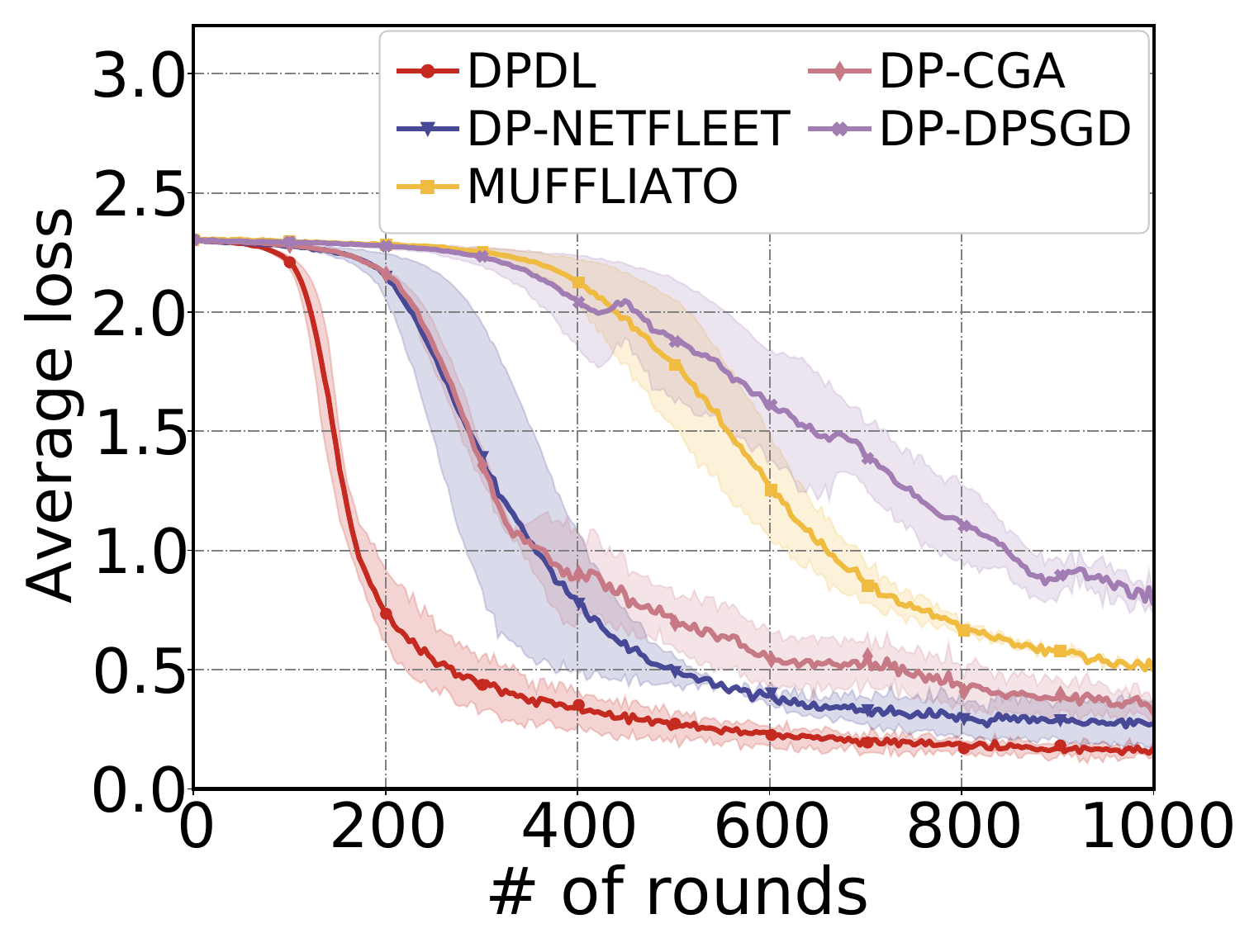}}
    \parbox{.3\textwidth}{\center\scriptsize(a3) $\epsilon=0.25, N=30$}
    \parbox{.3\textwidth}{\center\scriptsize(b3) $\epsilon=0.5, N=30$}
    \parbox{.3\textwidth}{\center\scriptsize(c3) $\epsilon=1.0, N=30$}
    \caption{Experiment results about convergence on MNIST dataset over bipartite graphs.}
    \label{fig:bipar-mnist}
    \end{center}
  \end{figure*}
  \begin{figure*}[htb!]
  \begin{center}
    \parbox{.3\textwidth}{\center\includegraphics[width=.3\textwidth]{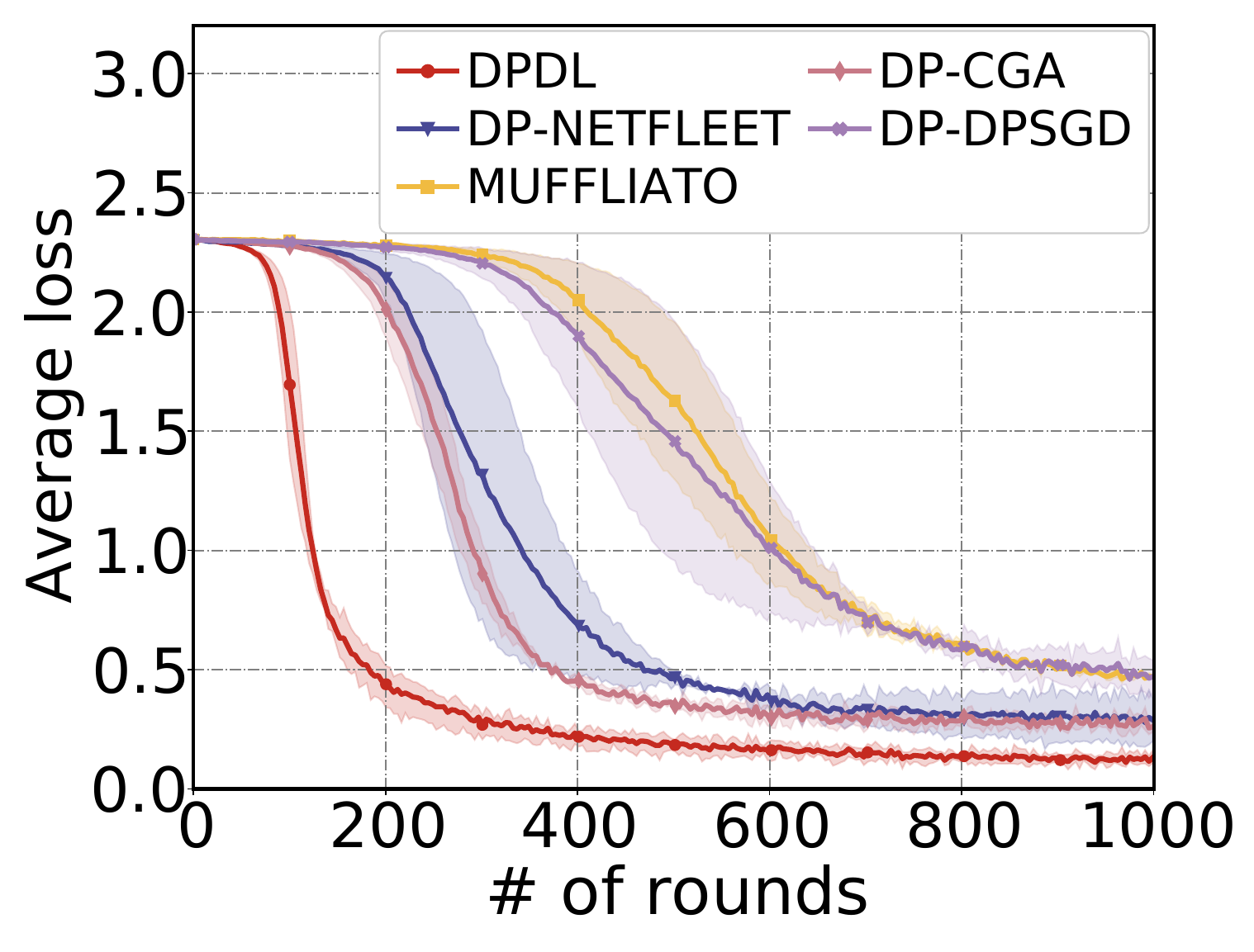}}
    \parbox{.3\textwidth}{\center\includegraphics[width=.3\textwidth]{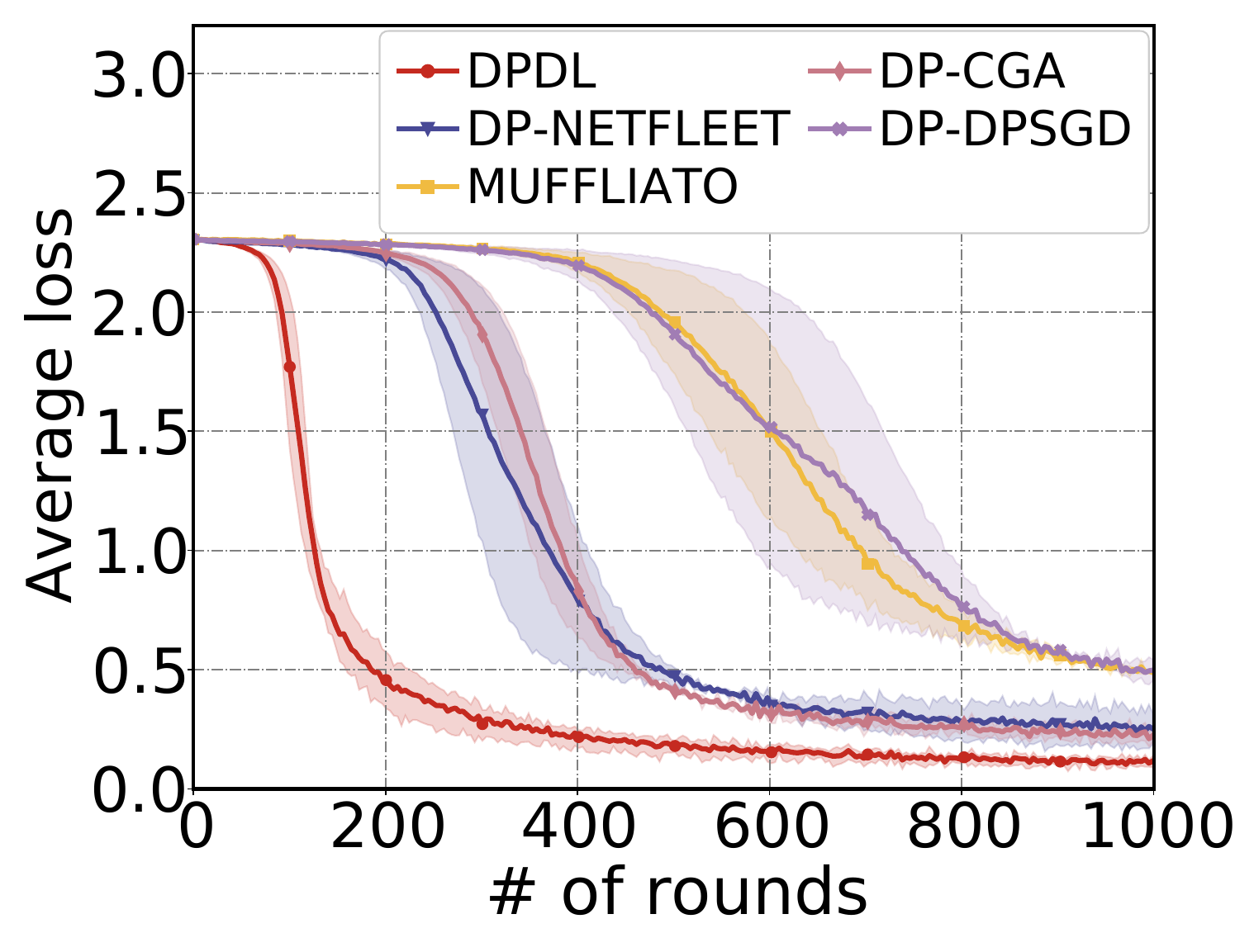}}
    \parbox{.3\textwidth}{\center\includegraphics[width=.3\textwidth]{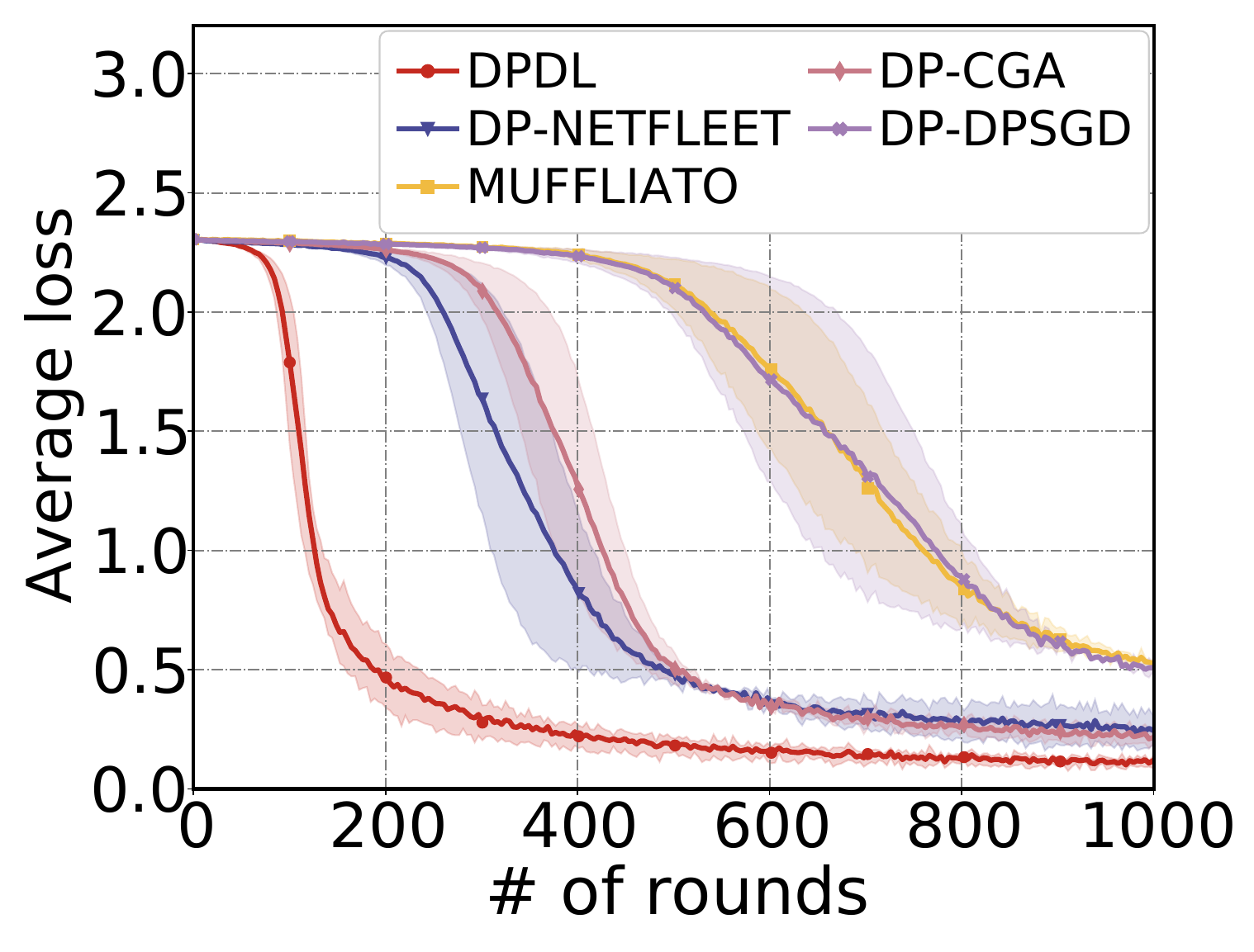}}
    \parbox{.3\textwidth}{\center\scriptsize(a1) $\epsilon=0.25, N=10$}
    \parbox{.3\textwidth}{\center\scriptsize(b1) $\epsilon=0.5, N=10$}
    \parbox{.3\textwidth}{\center\scriptsize(c1) $\epsilon=1.0, N=10$}
    \parbox{.3\textwidth}{\center\includegraphics[width=.3\textwidth]{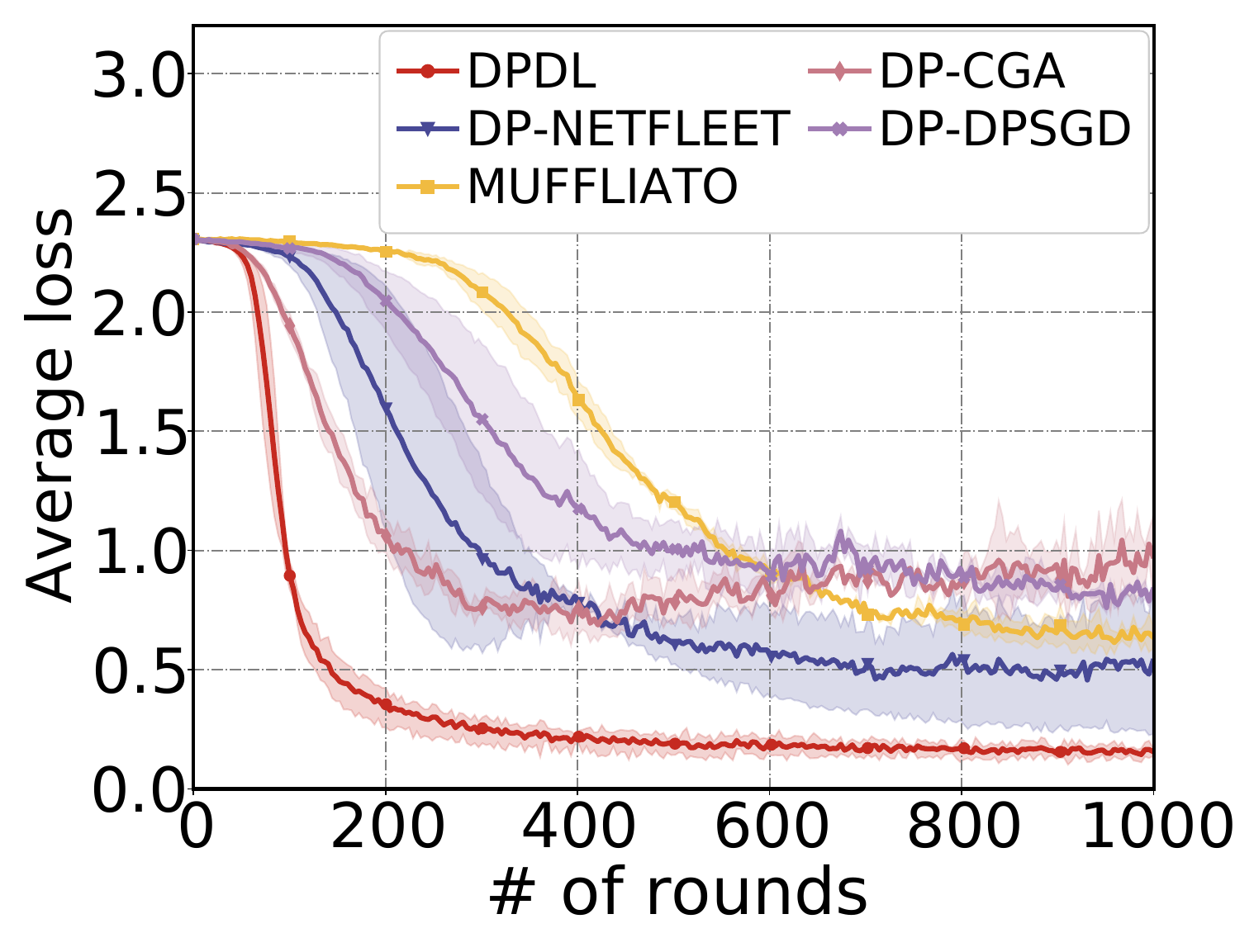}}
    \parbox{.3\textwidth}{\center\includegraphics[width=.3\textwidth]{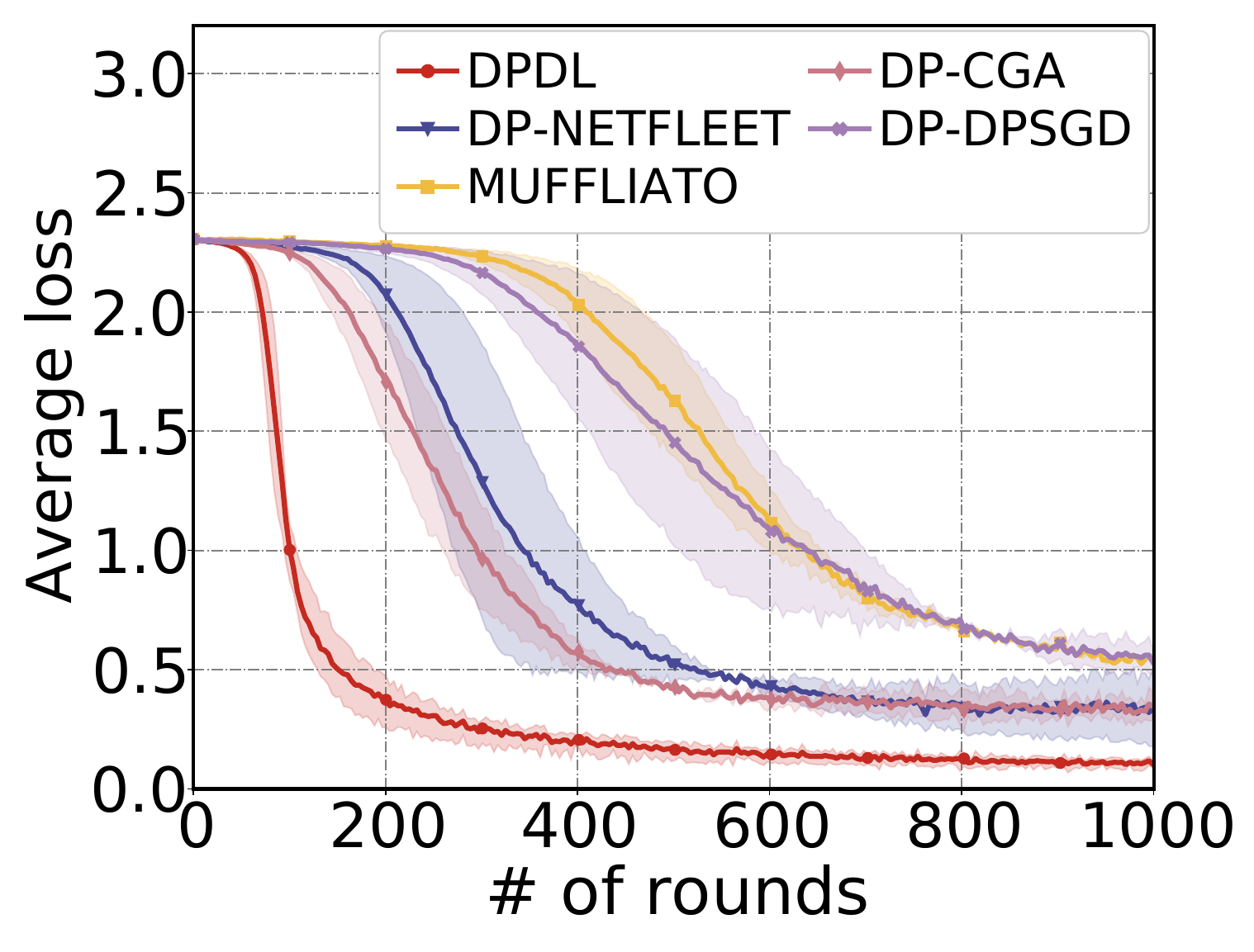}}
    \parbox{.3\textwidth}{\center\includegraphics[width=.3\textwidth]{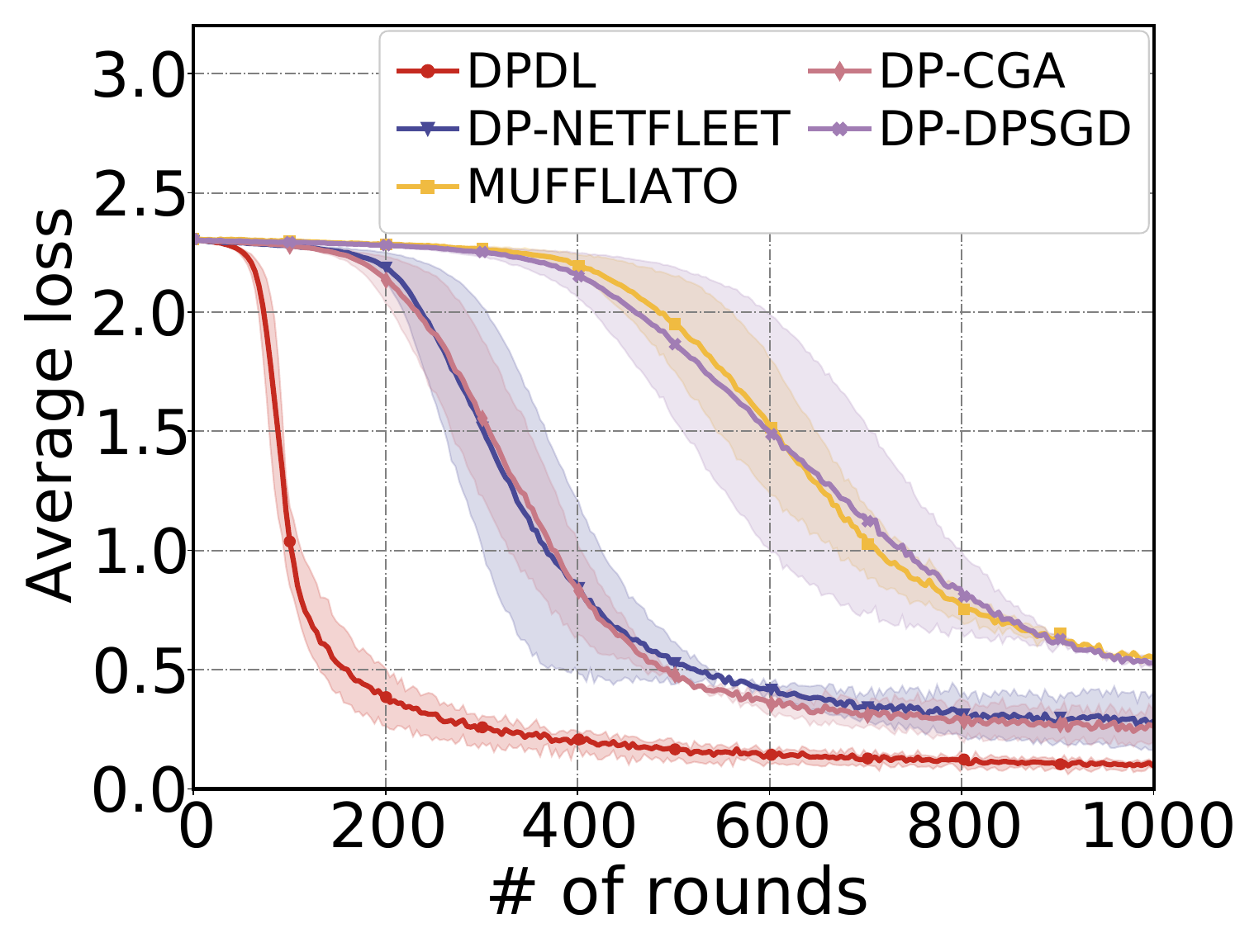}}
    \parbox{.3\textwidth}{\center\scriptsize(a2) $\epsilon=0.25, N=20$}
    \parbox{.3\textwidth}{\center\scriptsize(b2) $\epsilon=0.5, N=20$}
    \parbox{.3\textwidth}{\center\scriptsize(c2) $\epsilon=1.0, N=20$}
    \parbox{.3\textwidth}{\center\includegraphics[width=.3\textwidth]{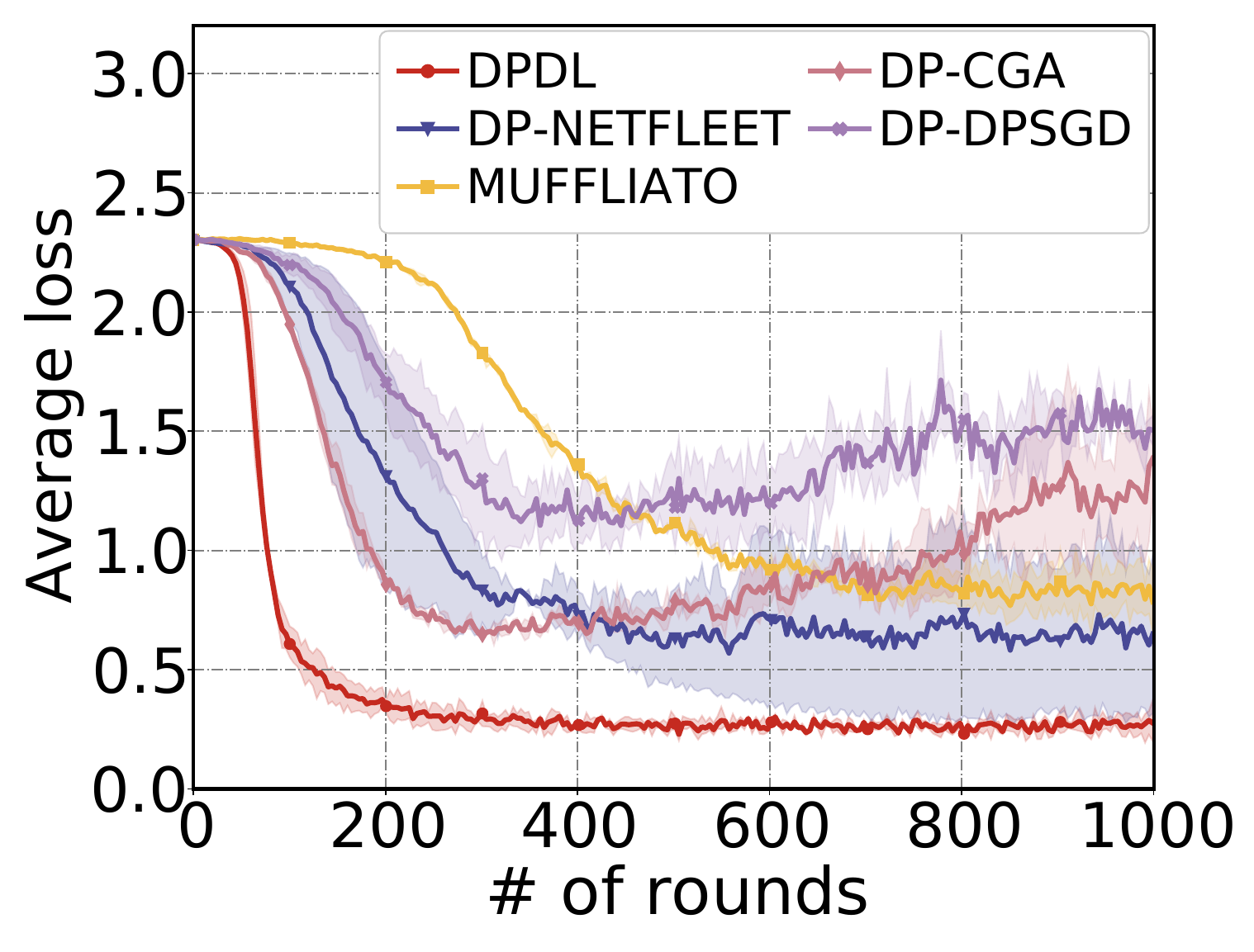}}
    \parbox{.3\textwidth}{\center\includegraphics[width=.3\textwidth]{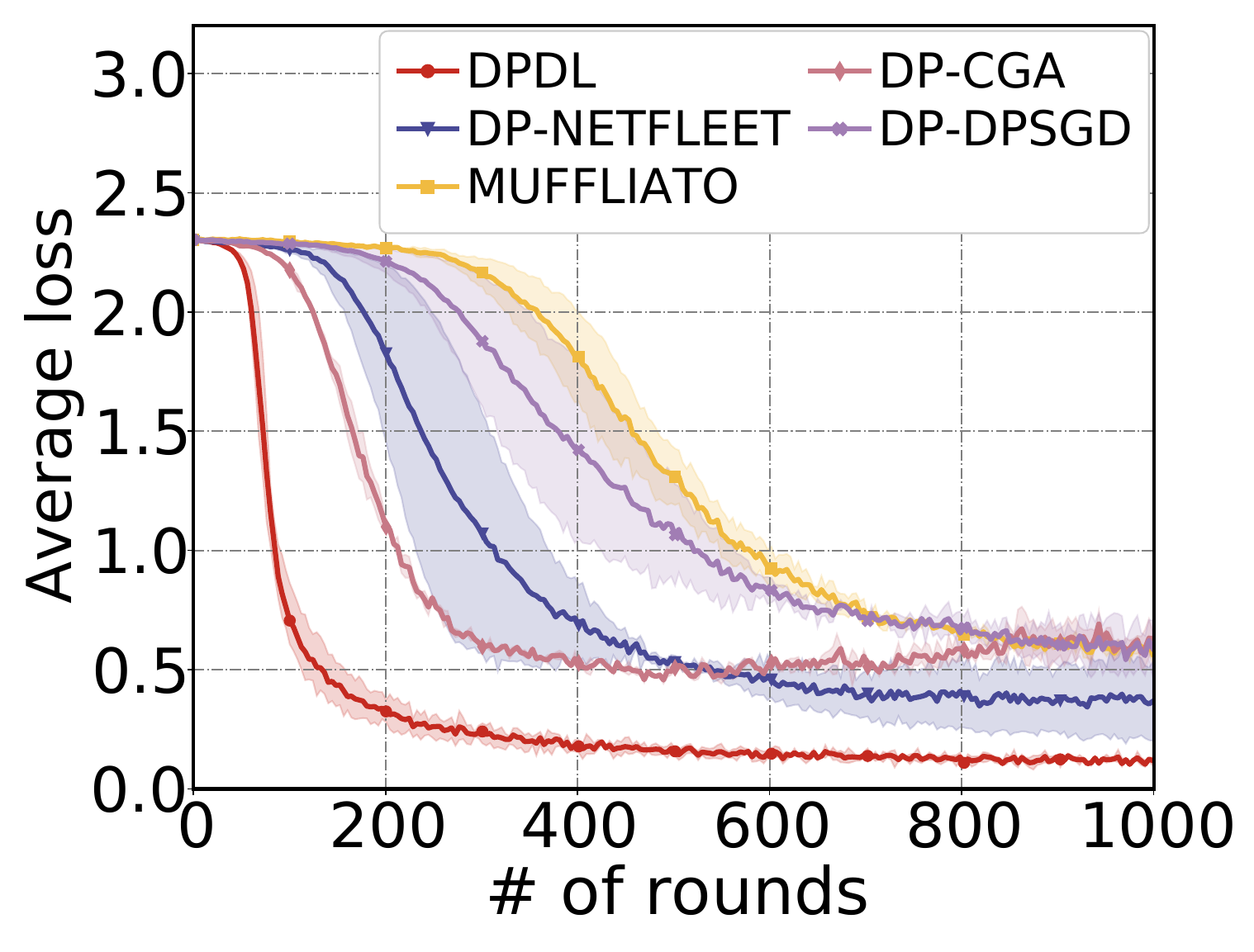}}
    \parbox{.3\textwidth}{\center\includegraphics[width=.3\textwidth]{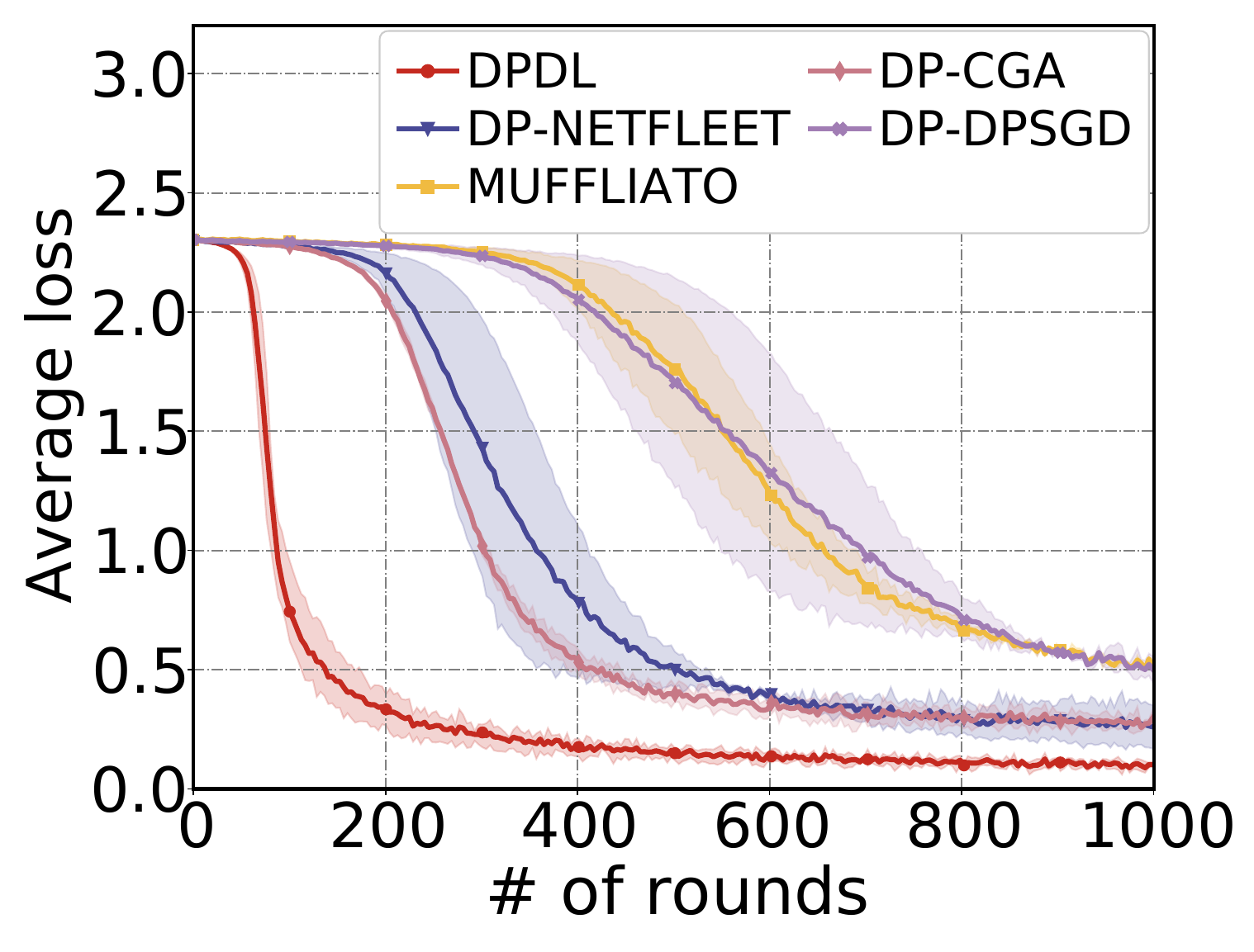}}
    \parbox{.3\textwidth}{\center\scriptsize(a3) $\epsilon=0.25, N=30$}
    \parbox{.3\textwidth}{\center\scriptsize(b3) $\epsilon=0.5, N=30$}
    \parbox{.3\textwidth}{\center\scriptsize(c3) $\epsilon=1.0, N=30$}
  \caption{Experiment results about convergence on MNIST dataset over fully connected graphs.}
  \label{fig:fc-mnist}
  \end{center}
  \end{figure*}
    \begin{figure*}[htb!]
    \begin{center}
      \parbox{.3\textwidth}{\center\includegraphics[width=.3\textwidth]{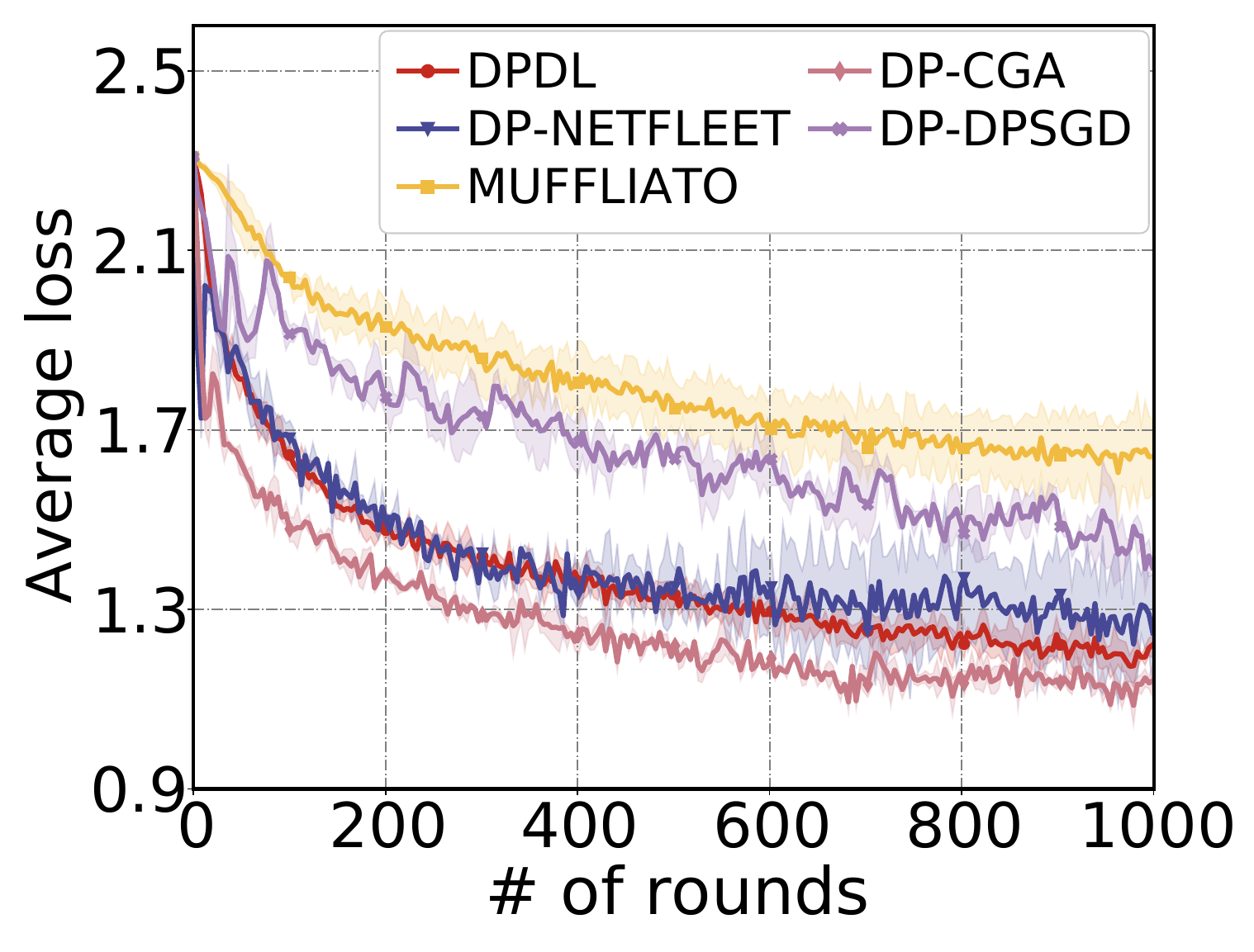}}
      \parbox{.3\textwidth}{\center\includegraphics[width=.3\textwidth]{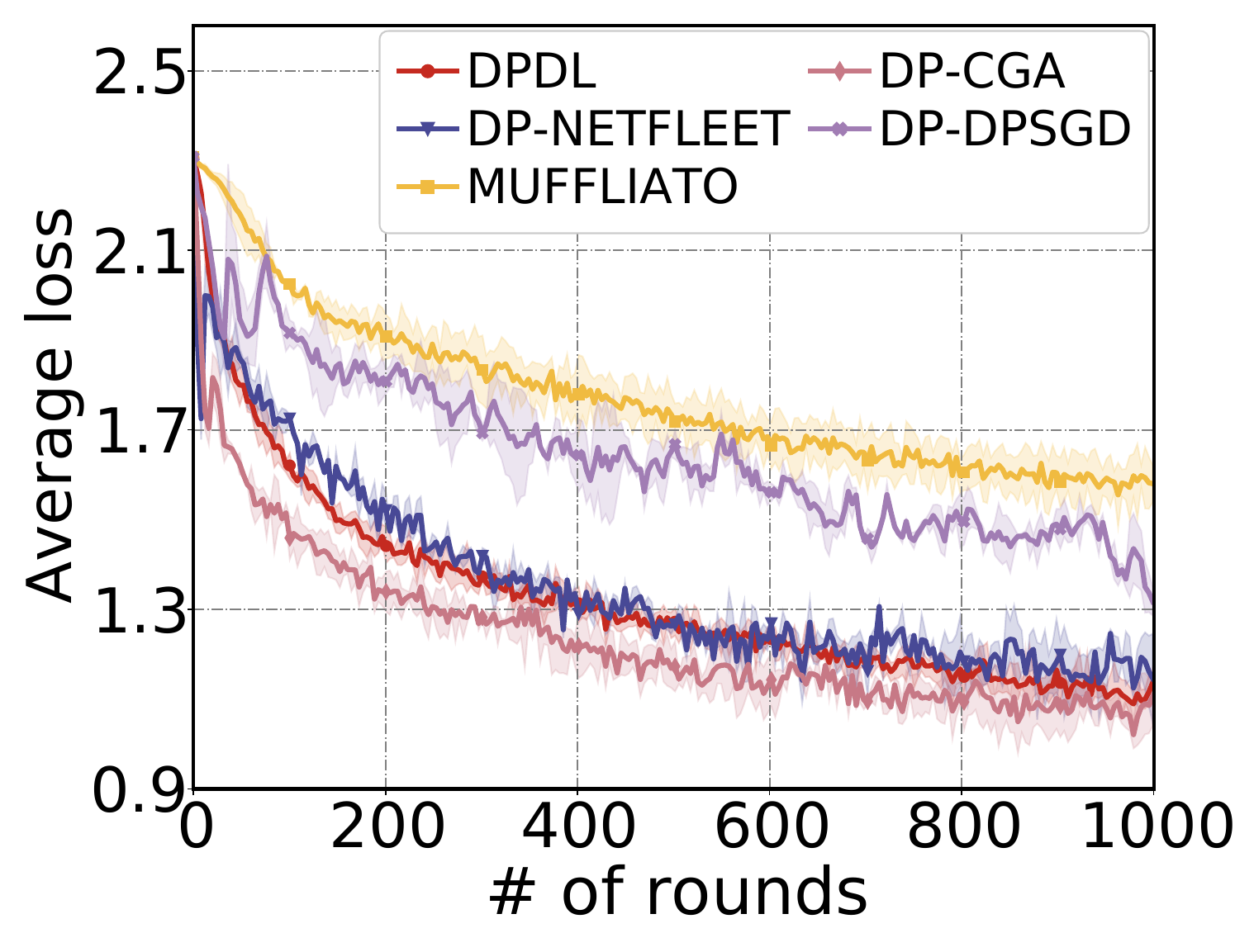}}
      \parbox{.3\textwidth}{\center\includegraphics[width=.3\textwidth]{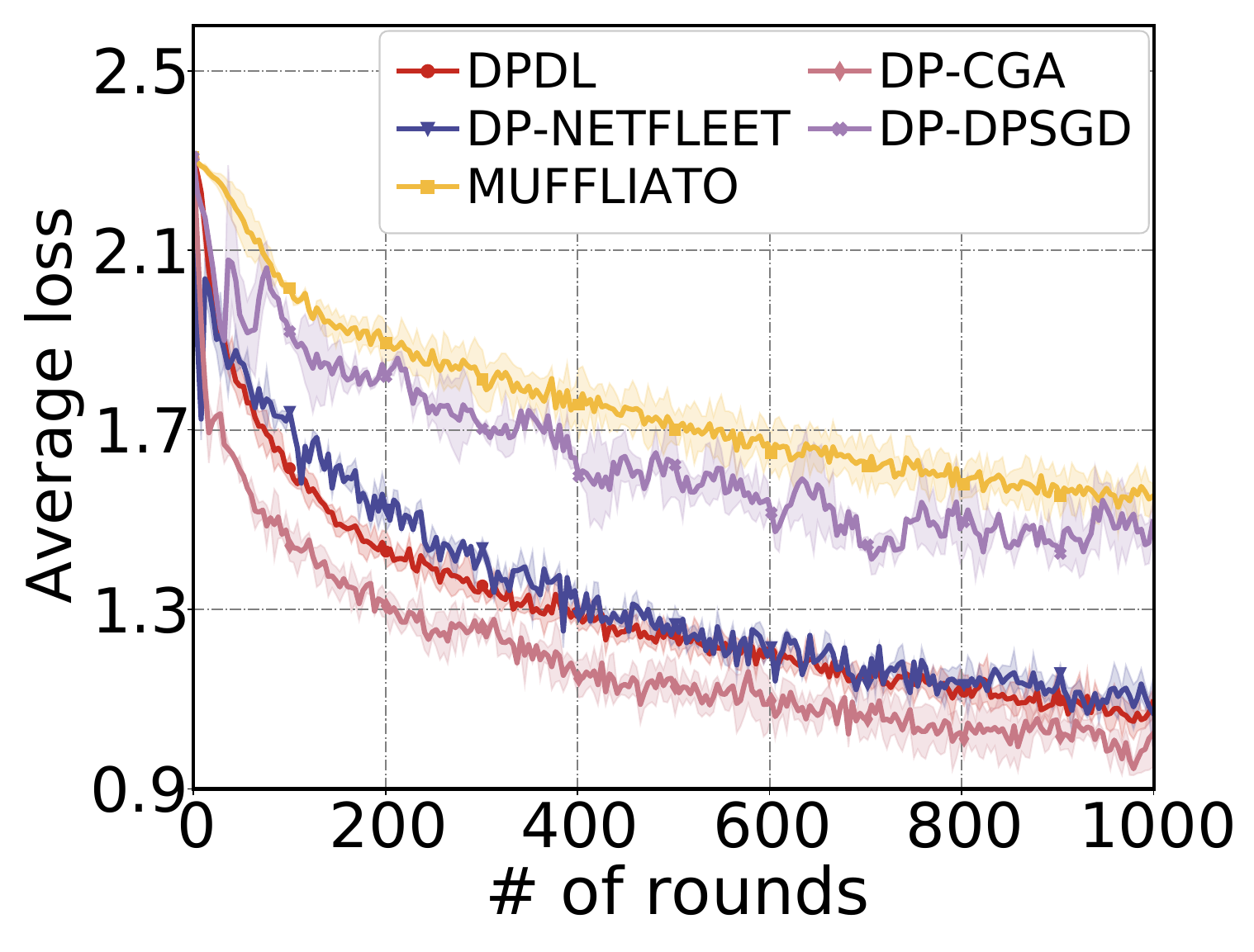}}
      \parbox{.3\textwidth}{\center\scriptsize(a1) $\epsilon=2.0, N=10$}
      \parbox{.3\textwidth}{\center\scriptsize(b1) $\epsilon=4.0, N=10$}
      \parbox{.3\textwidth}{\center\scriptsize(c1) $\epsilon=8.0, N=10$}
      \parbox{.3\textwidth}{\center\includegraphics[width=.3\textwidth]{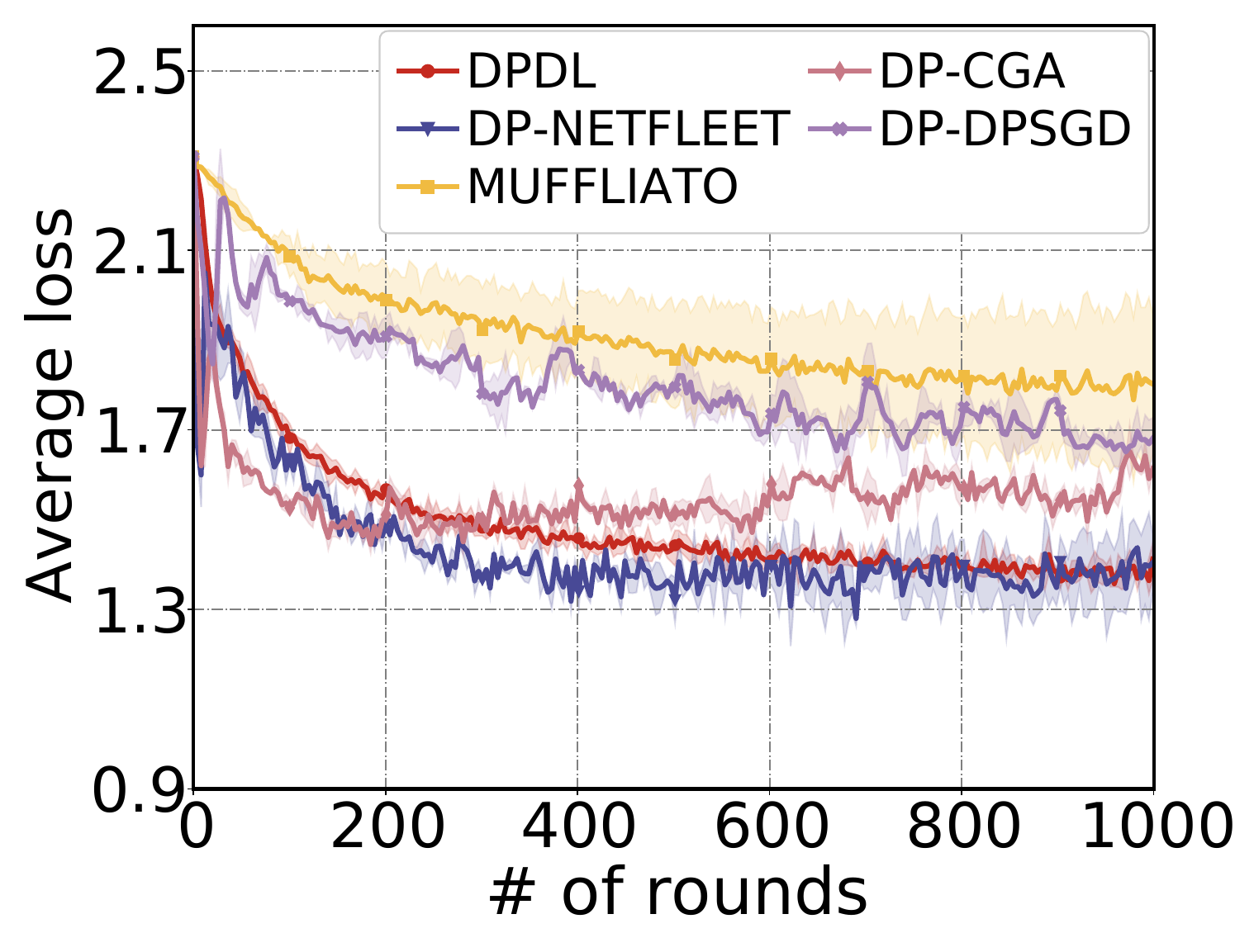}}
      \parbox{.3\textwidth}{\center\includegraphics[width=.3\textwidth]{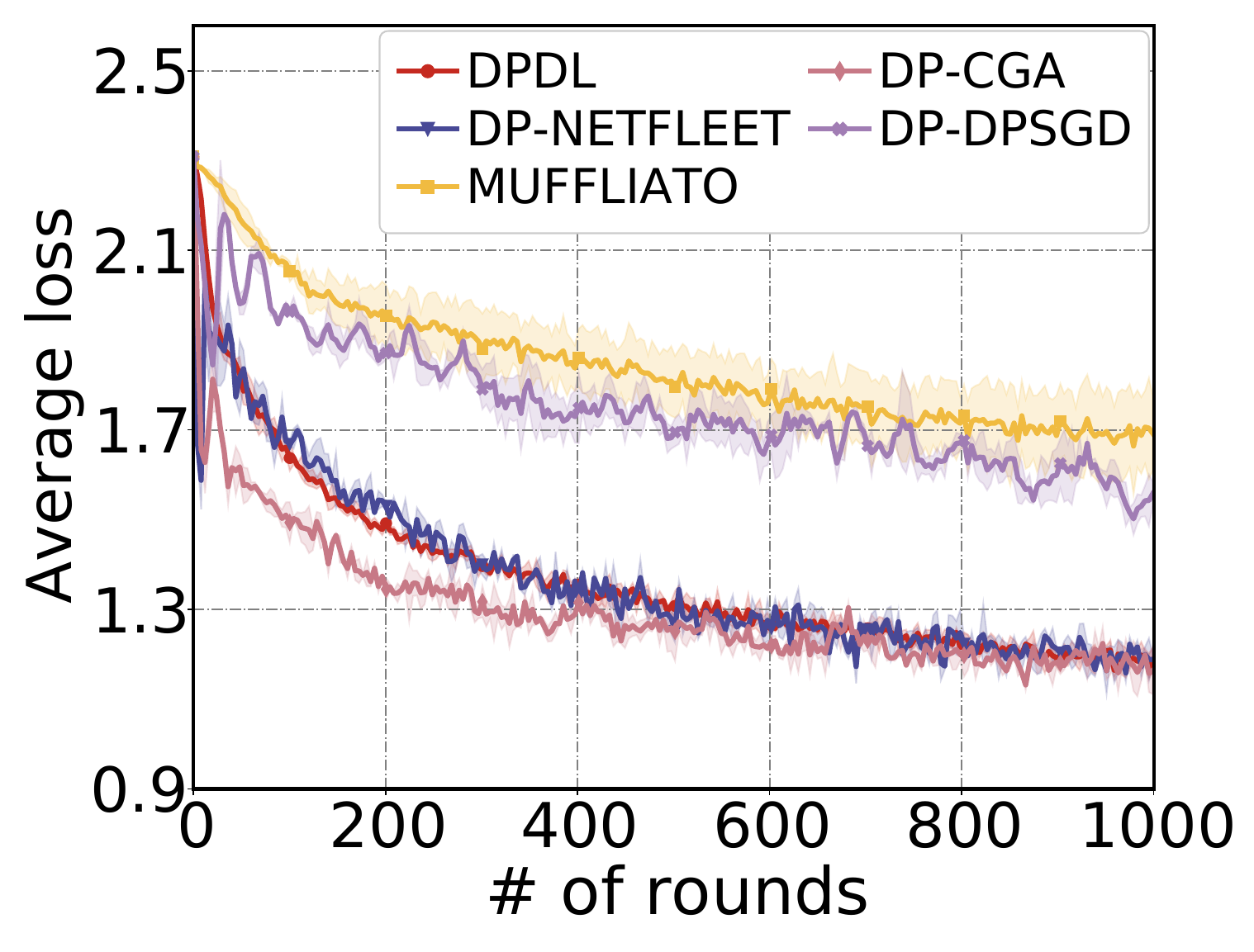}}
      \parbox{.3\textwidth}{\center\includegraphics[width=.3\textwidth]{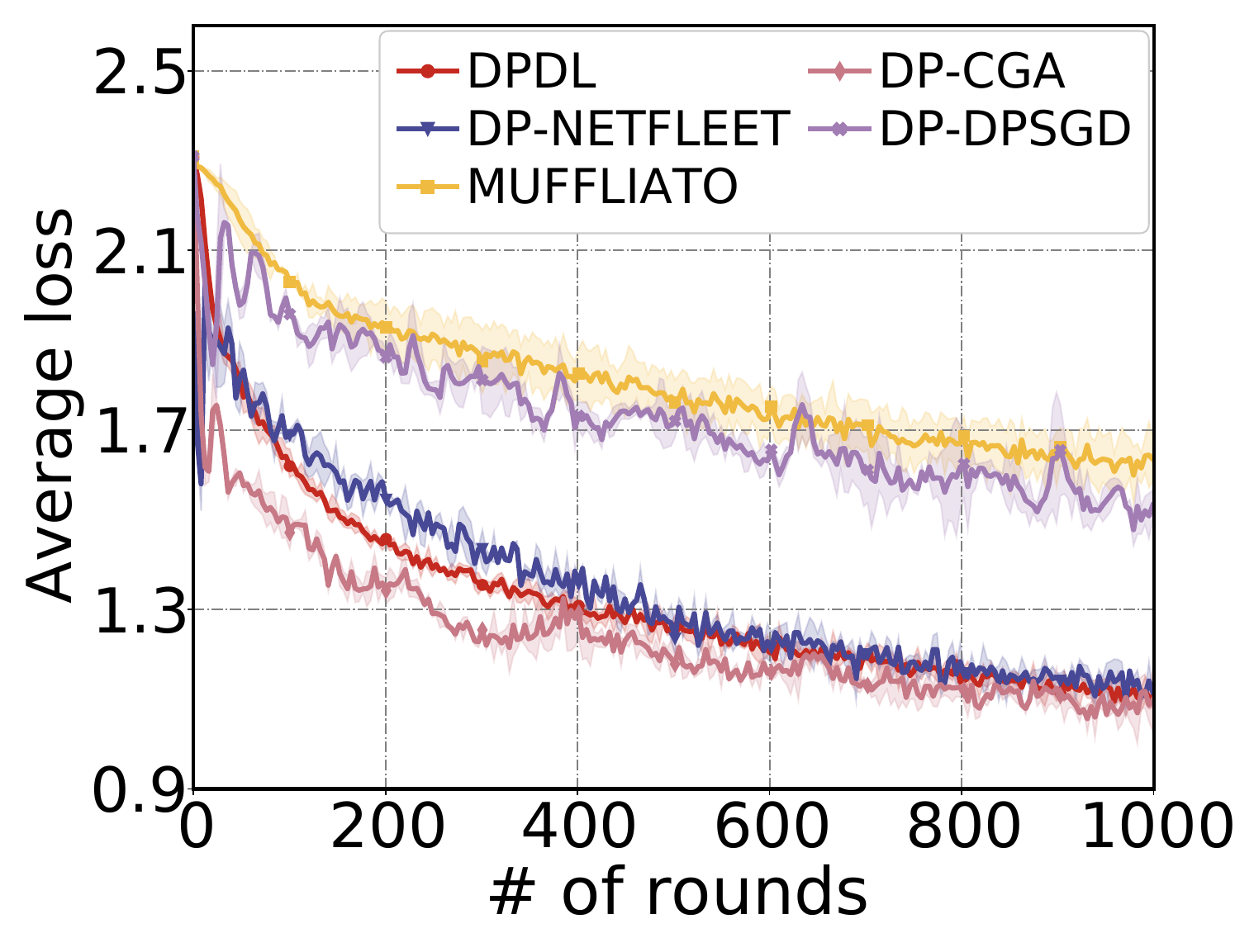}}
      \parbox{.3\textwidth}{\center\scriptsize(a2) $\epsilon=2.0, N=20$}
      \parbox{.3\textwidth}{\center\scriptsize(b2) $\epsilon=4.0, N=20$}
      \parbox{.3\textwidth}{\center\scriptsize(c2) $\epsilon=8.0, N=20$}
      \parbox{.3\textwidth}{\center\includegraphics[width=.3\textwidth]{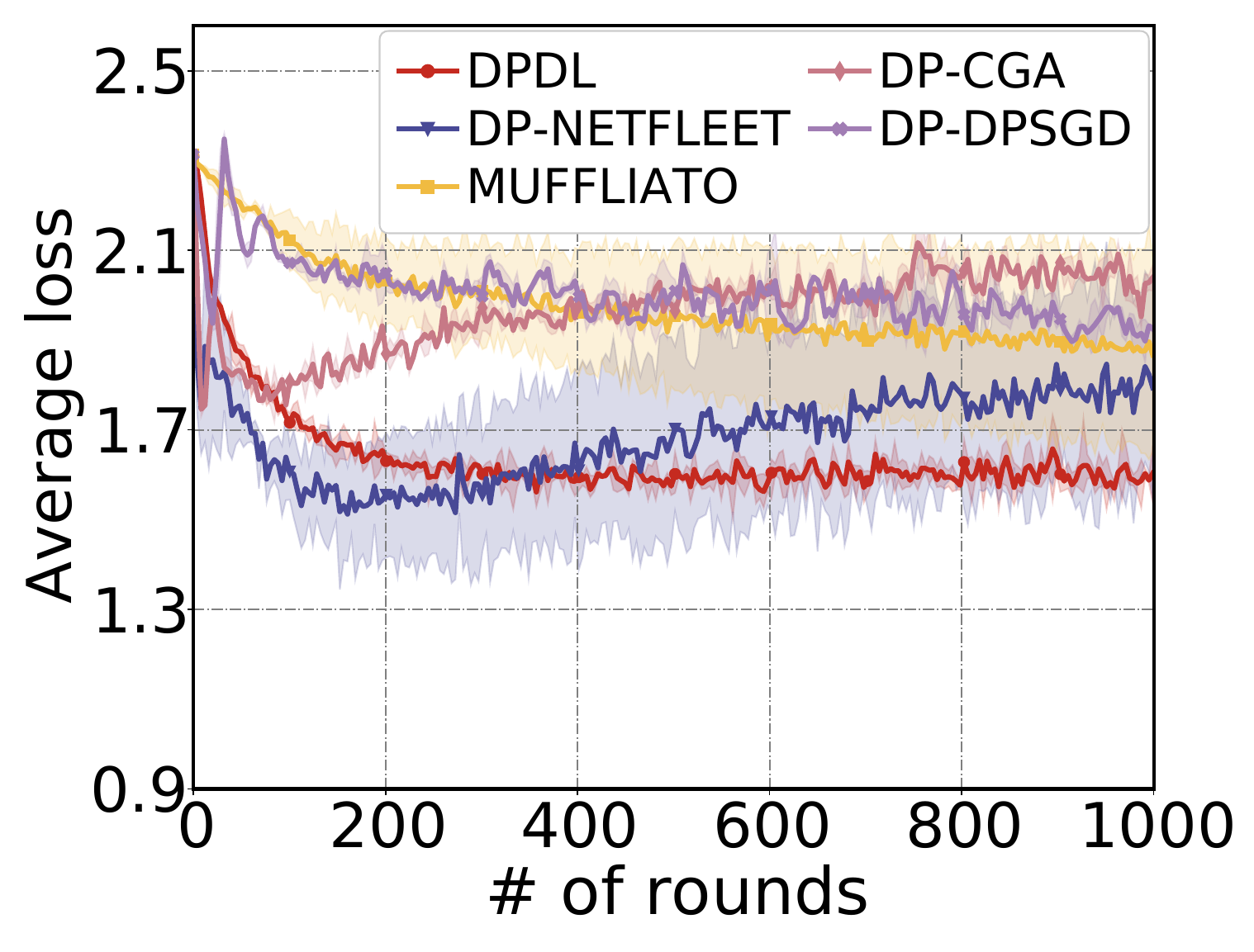}}
      \parbox{.3\textwidth}{\center\includegraphics[width=.3\textwidth]{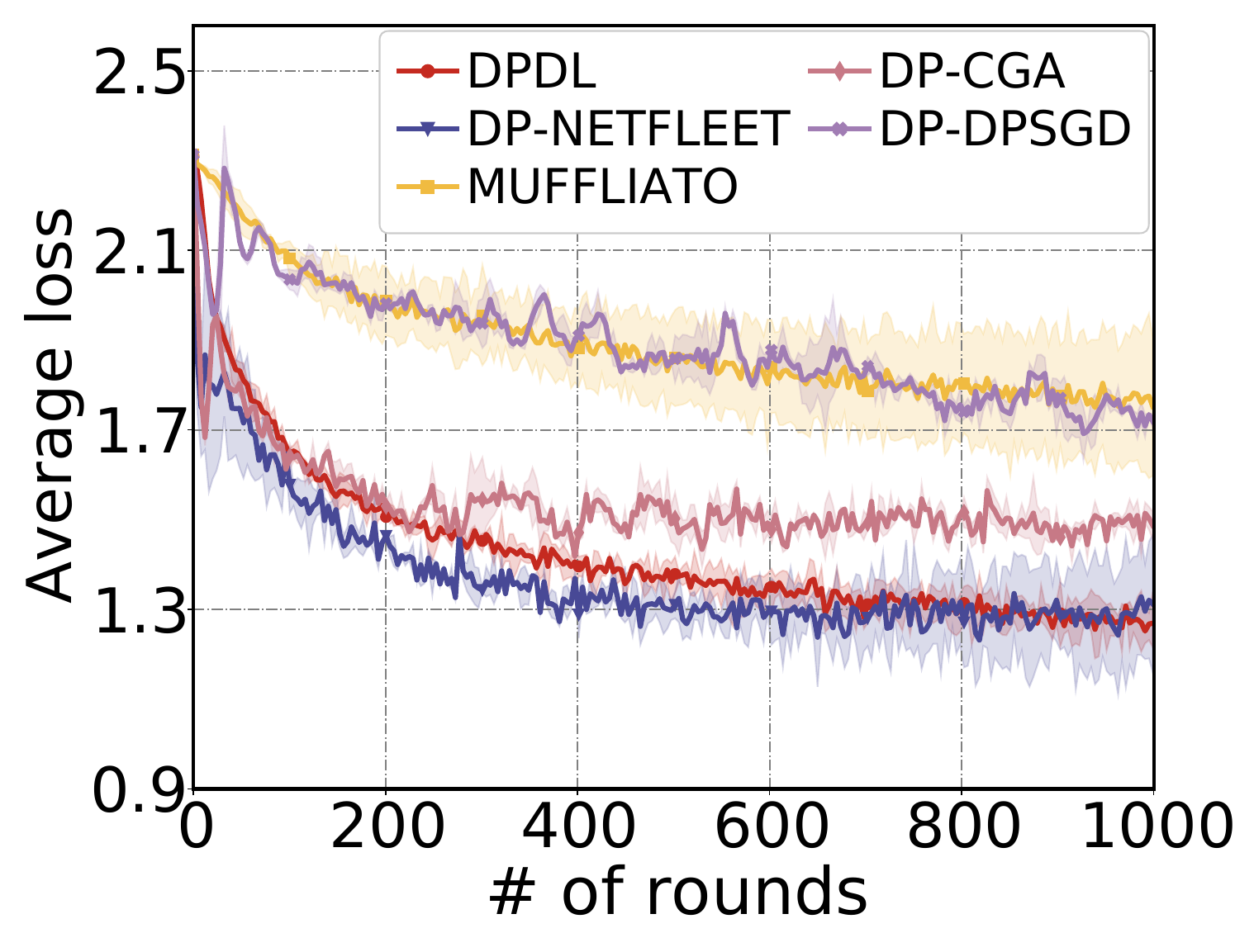}}
      \parbox{.3\textwidth}{\center\includegraphics[width=.3\textwidth]{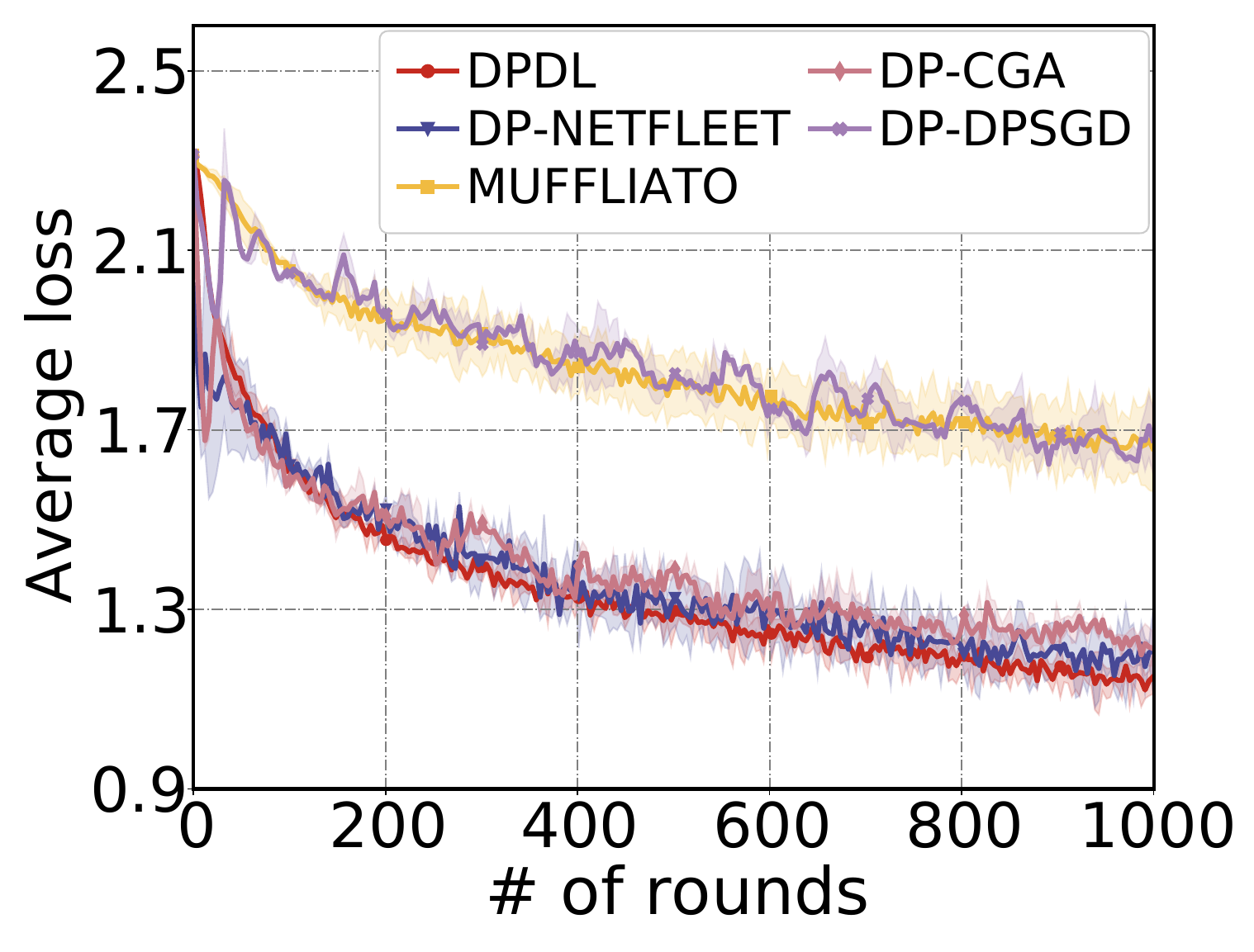}}
      \parbox{.3\textwidth}{\center\scriptsize(a3) $\epsilon=2.0, N=30$}
      \parbox{.3\textwidth}{\center\scriptsize(b3) $\epsilon=4.0, N=30$}
      \parbox{.3\textwidth}{\center\scriptsize(c3) $\epsilon=8.0, N=30$}
    \caption{Experiment results about convergence on CIFAR-10 dataset over bipartite graphs.}
    \label{fig:bipar-CIFAR10}
    \end{center}
    \end{figure*}
    \begin{figure*}[htb!]
    \begin{center}
      \parbox{.3\textwidth}{\center\includegraphics[width=.3\textwidth]{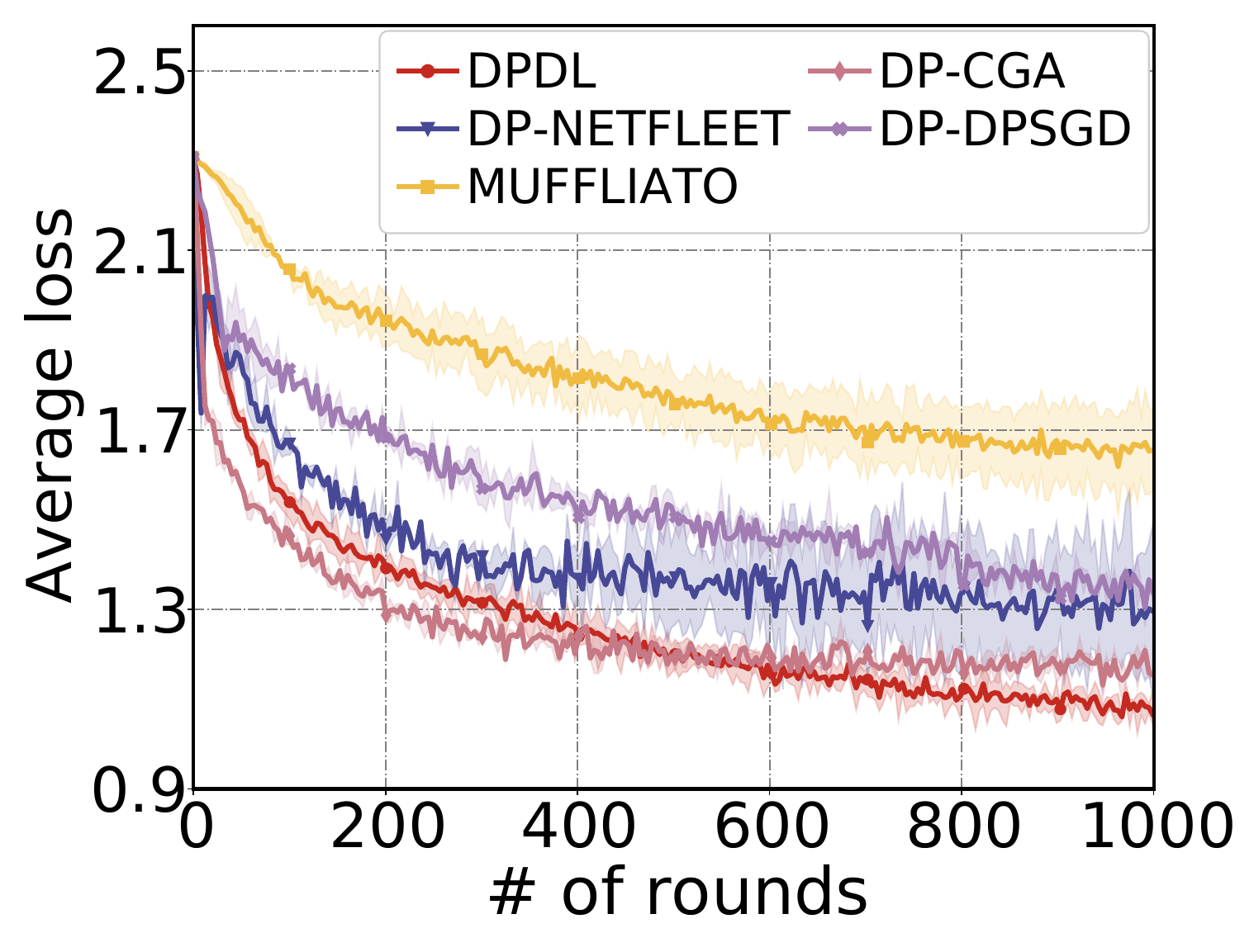}}
      \parbox{.3\textwidth}{\center\includegraphics[width=.3\textwidth]{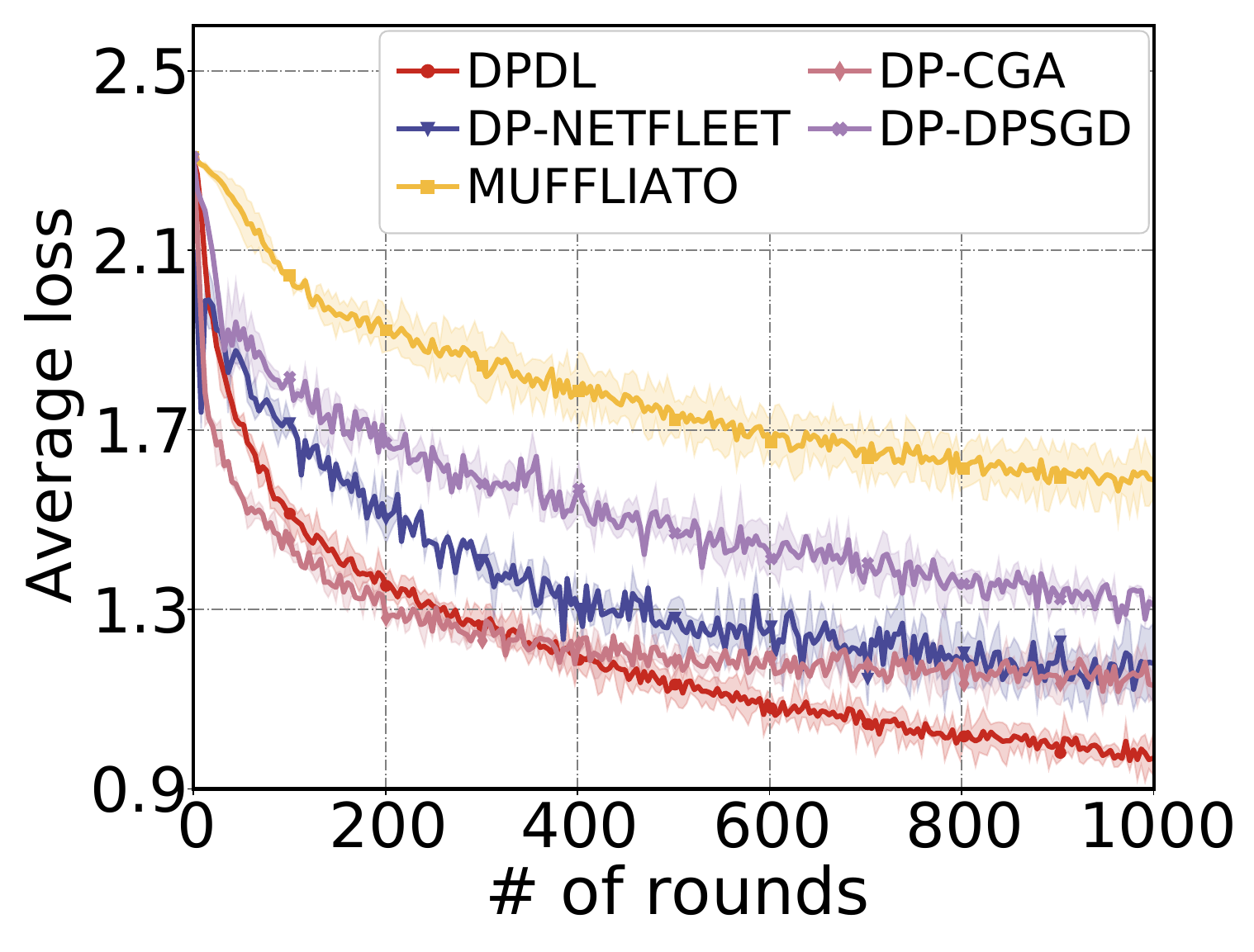}}
      \parbox{.3\textwidth}{\center\includegraphics[width=.3\textwidth]{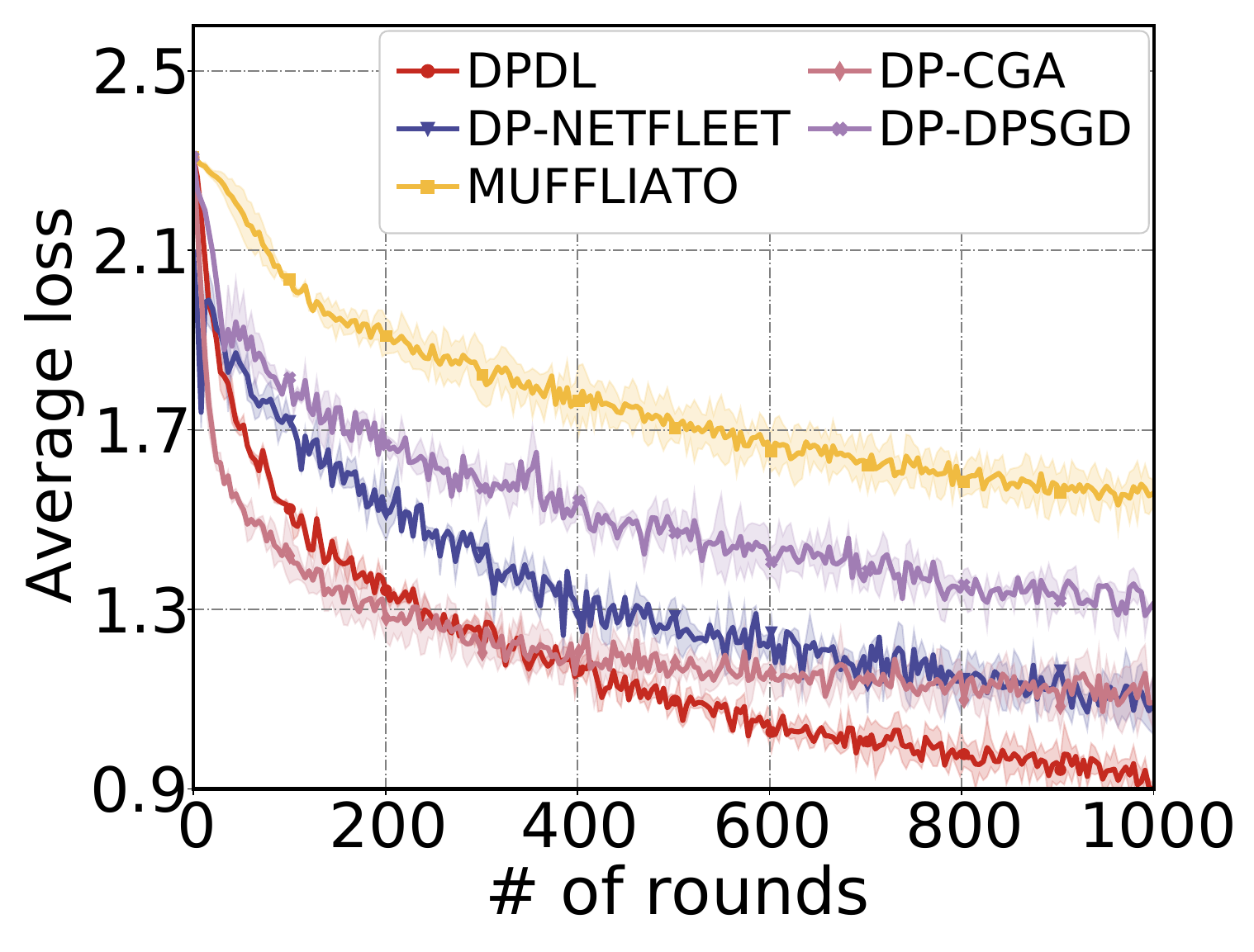}}
      \parbox{.3\textwidth}{\center\scriptsize(a1) $\epsilon=2.0, N=10$}
      \parbox{.3\textwidth}{\center\scriptsize(b1) $\epsilon=4.0, N=10$}
      \parbox{.3\textwidth}{\center\scriptsize(c1) $\epsilon=8.0, N=10$}
      \parbox{.3\textwidth}{\center\includegraphics[width=.3\textwidth]{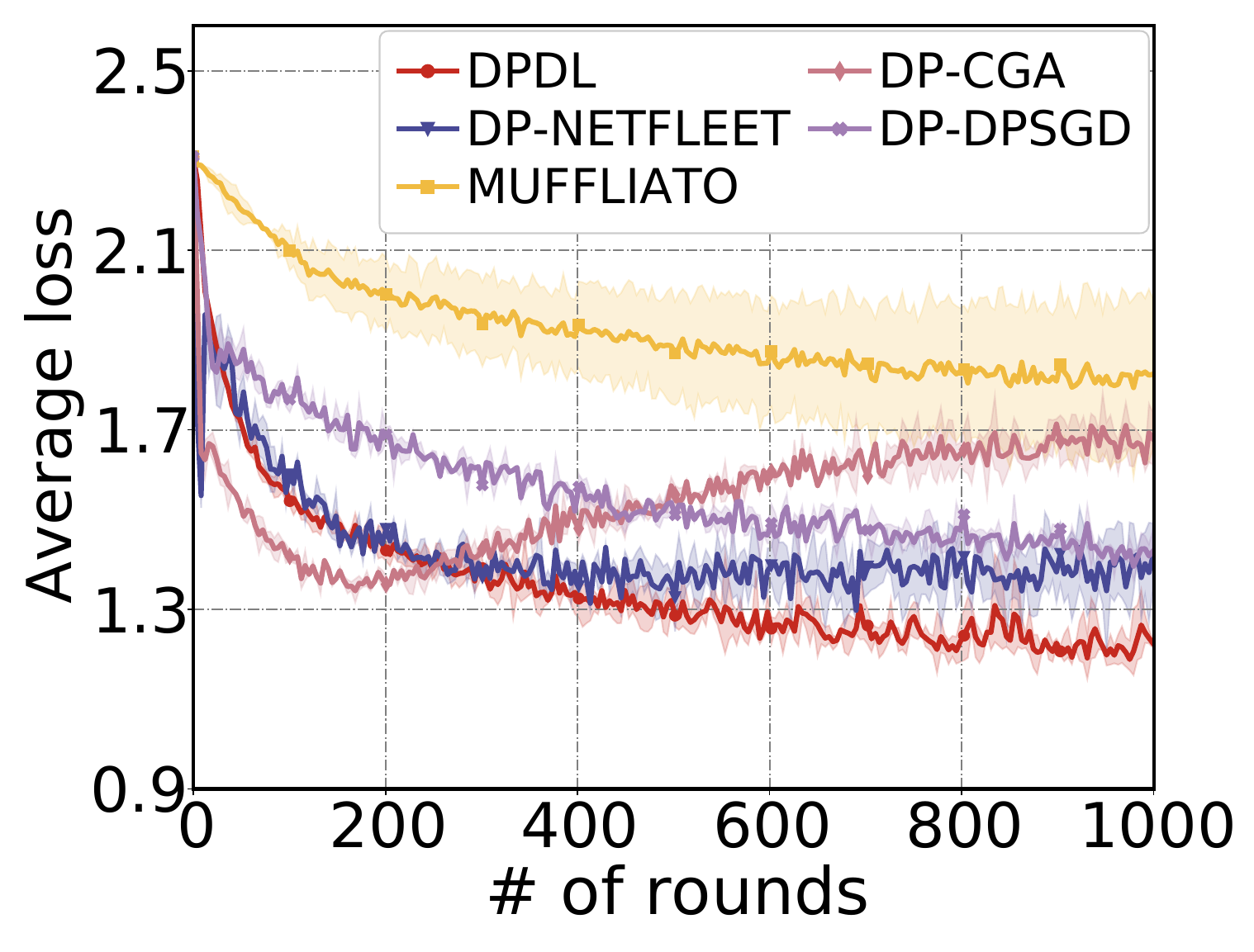}}
      \parbox{.3\textwidth}{\center\includegraphics[width=.3\textwidth]{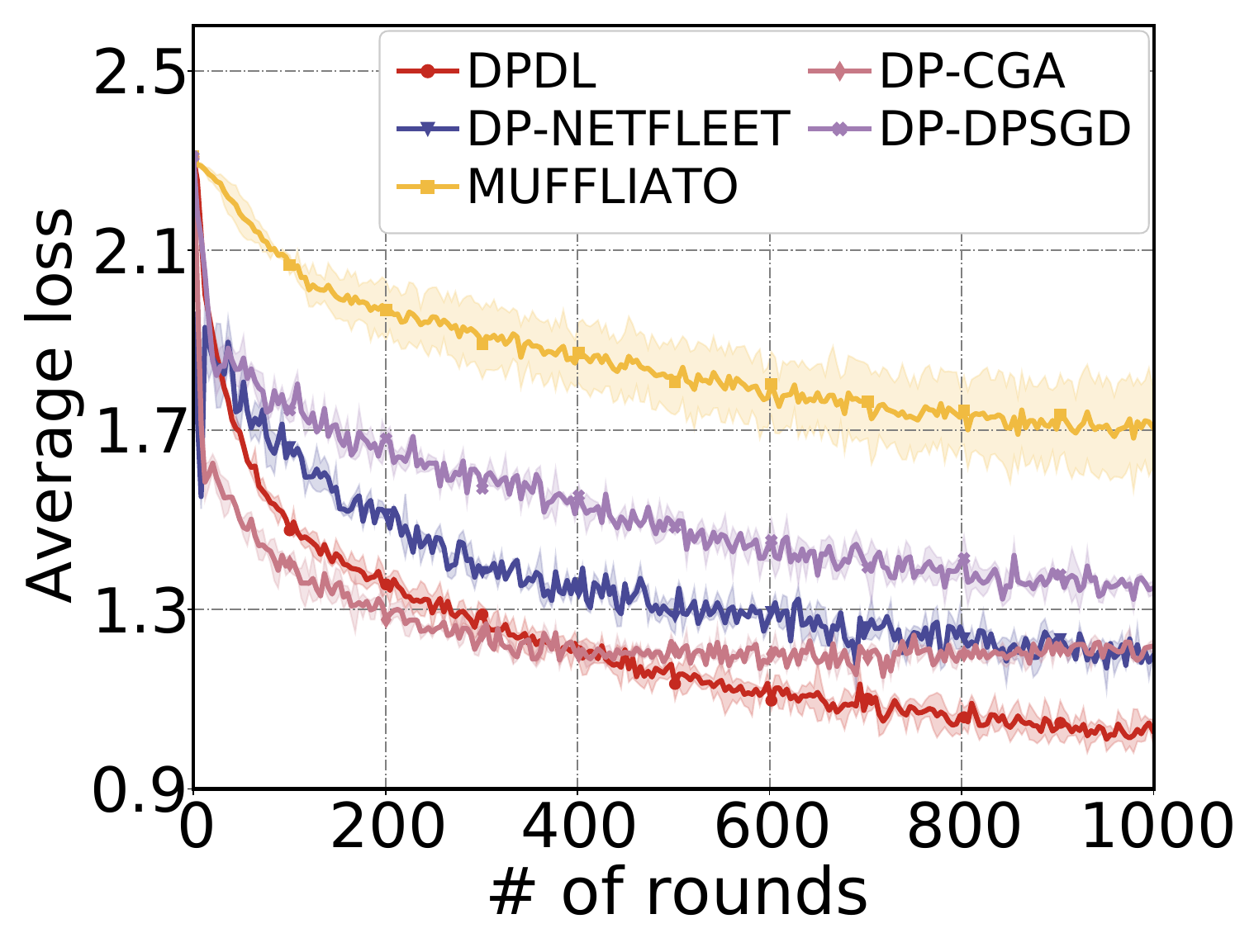}}
      \parbox{.3\textwidth}{\center\includegraphics[width=.3\textwidth]{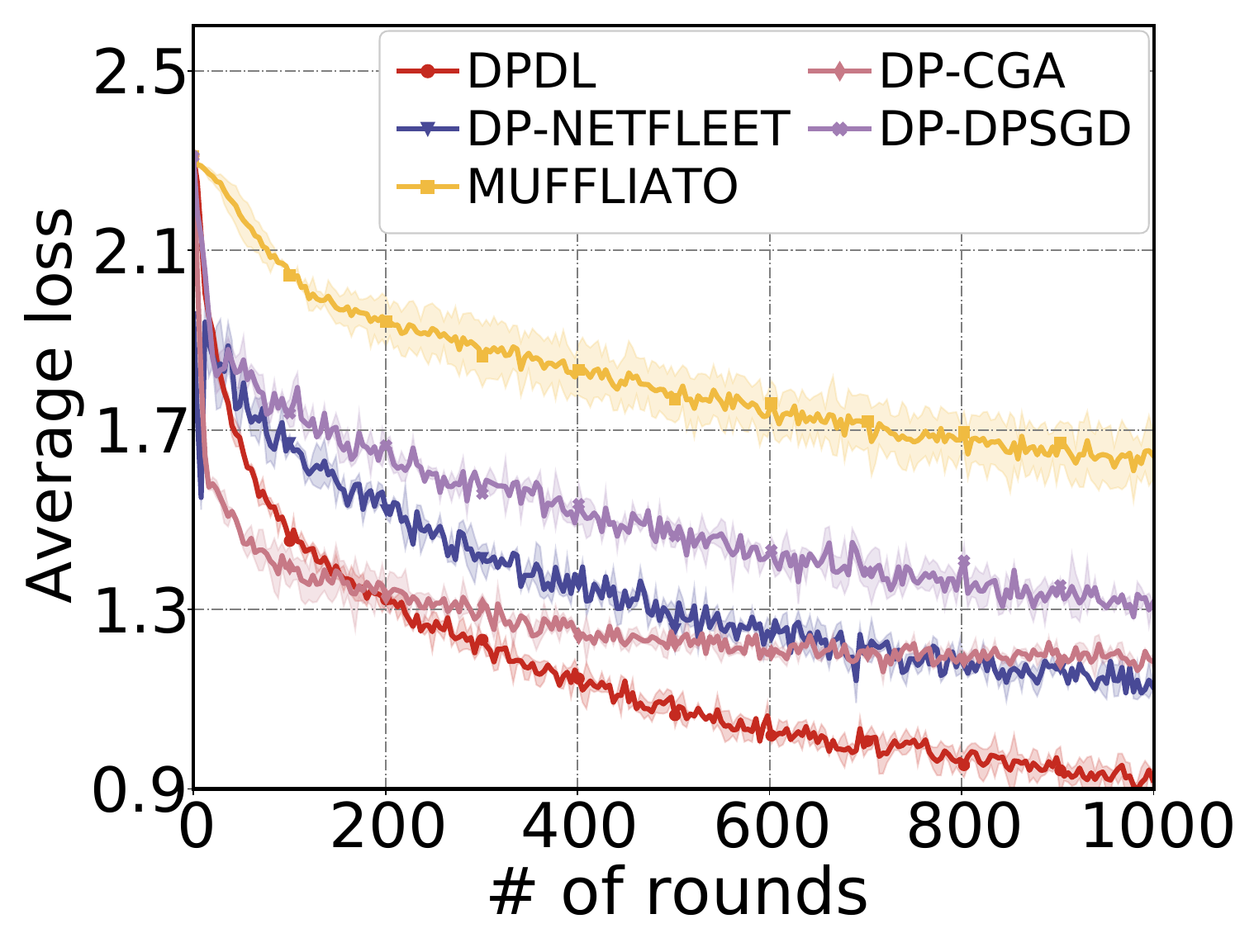}}
      \parbox{.3\textwidth}{\center\scriptsize(a2) $\epsilon=2.0, N=20$}
      \parbox{.3\textwidth}{\center\scriptsize(b2) $\epsilon=4.0, N=20$}
      \parbox{.3\textwidth}{\center\scriptsize(c2) $\epsilon=8.0, N=20$}
      \parbox{.3\textwidth}{\center\includegraphics[width=.3\textwidth]{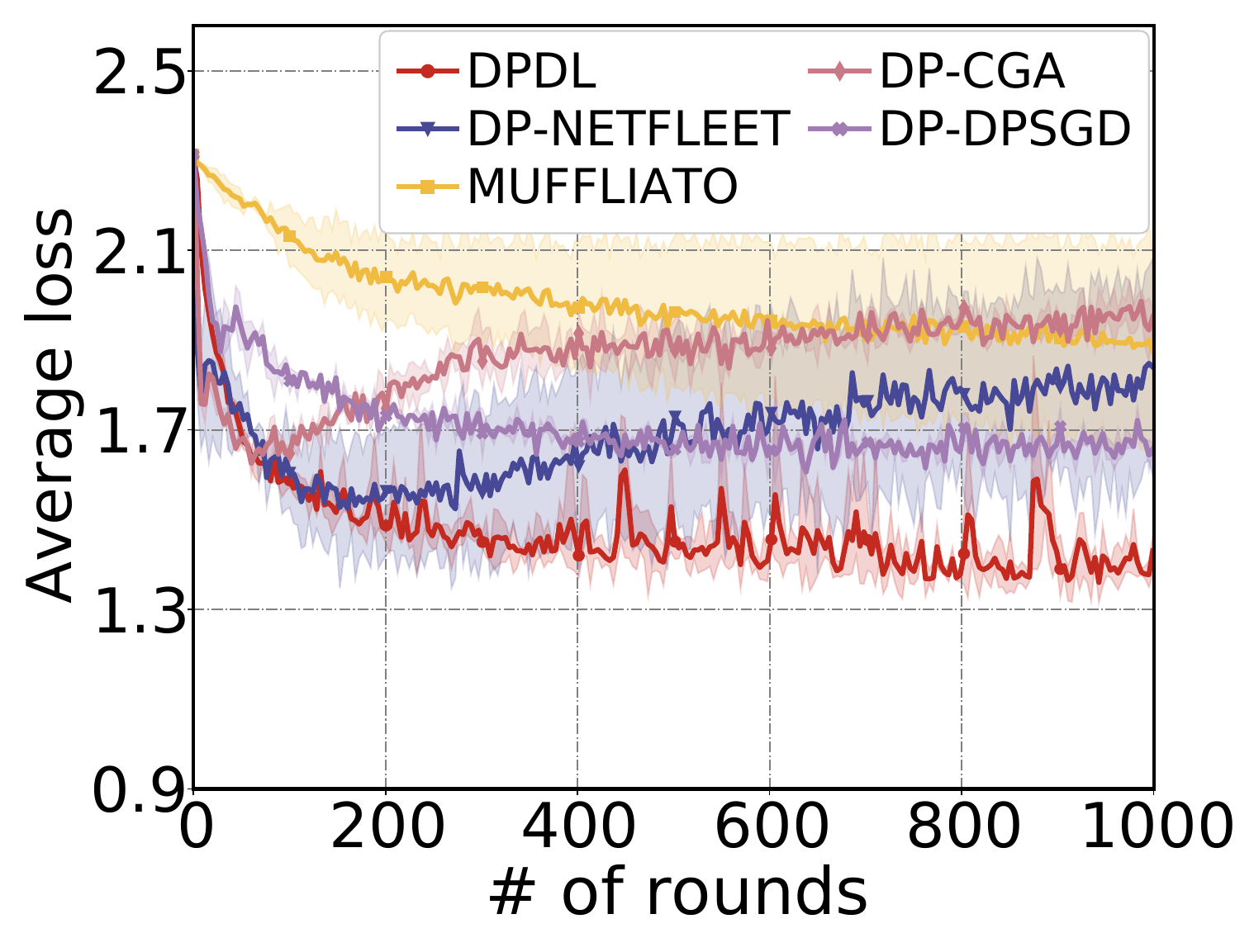}}
      \parbox{.3\textwidth}{\center\includegraphics[width=.3\textwidth]{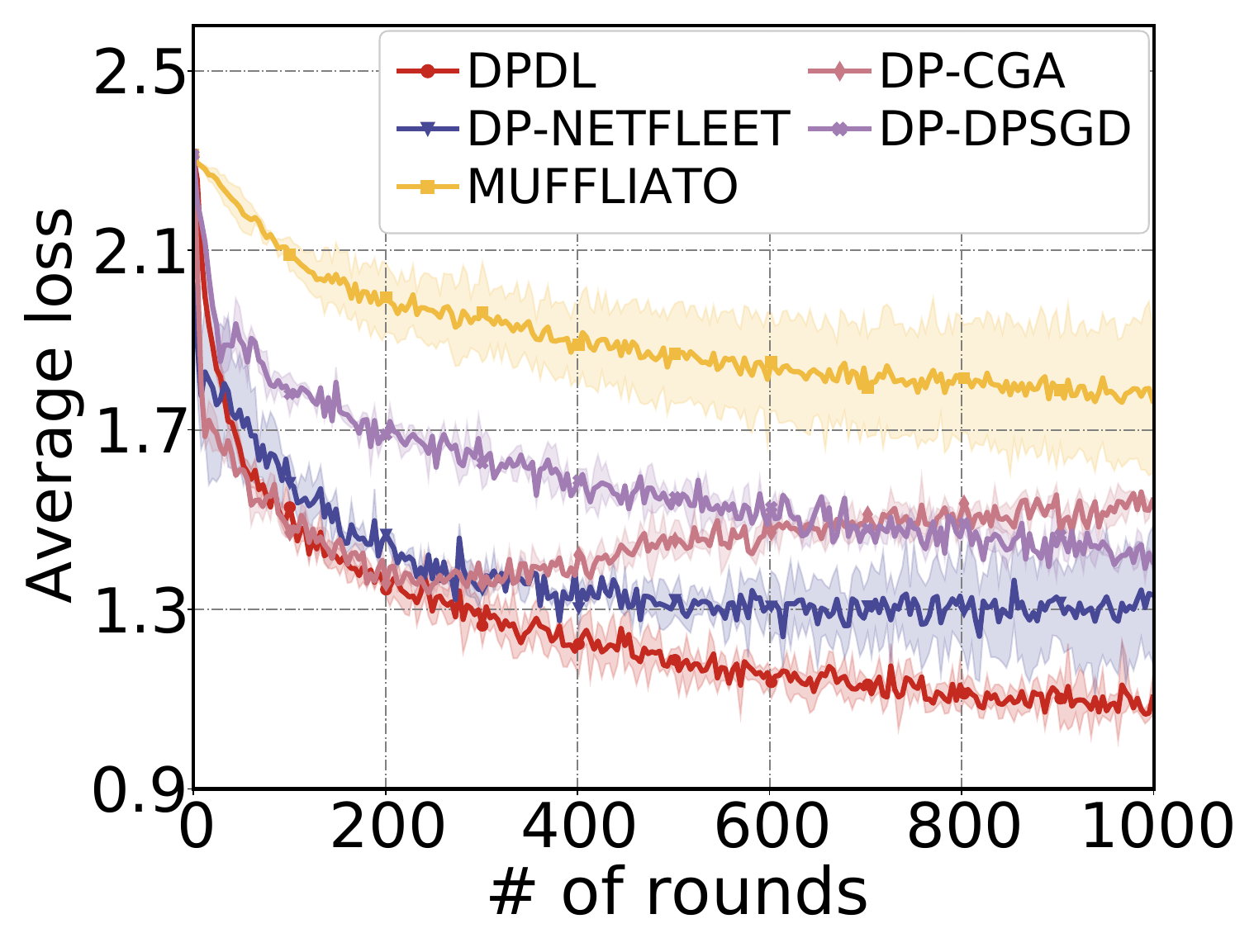}}
      \parbox{.3\textwidth}{\center\includegraphics[width=.3\textwidth]{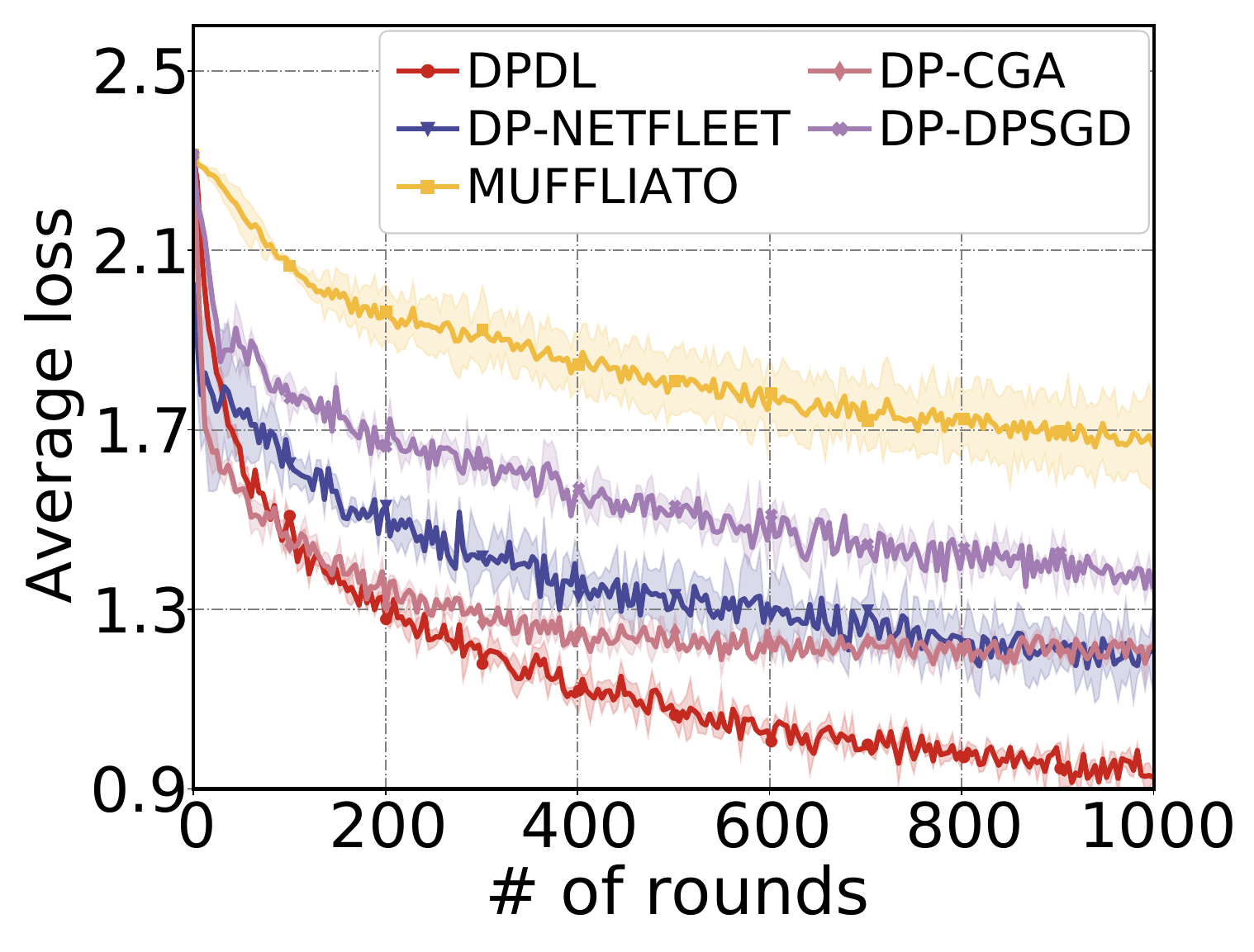}}
      \parbox{.3\textwidth}{\center\scriptsize(a3) $\epsilon=2.0, N=30$}
      \parbox{.3\textwidth}{\center\scriptsize(b3) $\epsilon=4.0, N=30$}
      \parbox{.3\textwidth}{\center\scriptsize(c3) $\epsilon=8.0, N=30$}
     \caption{Experiment results about convergence on CIFAR-10 dataset over fully connected graphs.}
  \label{fig:fc-CIFAR10}
  \end{center}
  \end{figure*}

  \subsection{Experiment Results about Convergence} \label{ssec:expconvergence}
  In this section, we plot loss curves for each algorithm under different settings with $\alpha_d=0.25$ (loss curves for $\alpha_d=1.0$ exhibit similar trends to those of $\alpha_d=0.25$, therefore, we present only the results for $\alpha_d=0.25$ due to space limitations). Each algorithm runs twenty times, and the final results are obtained by averaging the outcomes of these runs. To better illustrate the experimental findings, we present both the average loss (the mean of the loss values across all agents) and the individual variations, using shaded regions to reflect variability.

    \subsubsection{Experiment Results about Convergence on MNIST Dataset} \label{sssec:expconvergence-mnist}
      We first present the convergence trajectories of the different algorithms in terms of training loss versus training rounds over bipartite graphs on the MNIST dataset, as shown in Fig.~\ref{fig:bipar-mnist}. It is demonstrated that, our algorithm achieves convergence much more rapidly and more smoothly than the other reference ones. For example, when $\epsilon=0.5$ and $N=20$, the training loss of our DPDL algorithm achieves its minimum value around $0.27$ in about $500$ rounds, while the ones of the other four reference algorithms, i.e., DP-CGA, DP-NETFLEET, DP-DPSGD, and MUFFLIATO, are $0.69$, $0.53$, $1.62$, and $1.64$, respectively, within the same time horizon. Although DP-NETFLEET and DP-CGA have their final training loss close to our algorithm, their convergence exhibit considerable fluctuations especially with $N \geq 20$ and $\epsilon \leq 0.5$.

      As mentioned in \textbf{Algorithm}~\ref{alg:dpdl}, each agent has to add "more" noise to the gradient information and (thus allocates less privacy budget) for "stronger" privacy preservation, resulting in a reduction in the utility of the gradient information. For example, as shown in Fig.~\ref{fig:bipar-mnist}, when $\epsilon$ is decreased from $1.0$ to $0.25$, our DPDL algorithm still achieves convergence rapidly and smoothly, whereas the convergences of the other four algorithms cannot be ensured. On the other hand, increasing the privacy budget (and thereby reducing the noise) may lead to improved convergence performance. However, such gains are not always worthwhile, particularly given that a higher privacy budget implies a more significant risk of privacy leakage. For instance, when $\epsilon$ is increased from $0.25$ to $0.5$, our DPDL algorithm achieves a noticeably lower training loss after convergence. Yet, further increasing $\epsilon$ to $1.0$ yields little to no improvement in the minimum training loss.
      
      We also investigate the impact of the number of agents on the convergence of the different algorithms, as shown in Fig.~\ref{fig:bipar-mnist}. The results indicate that increasing the number of agents has only a minimal effect on the convergence rate of our algorithm in practice. For example, when $\epsilon = 0.5$, our algorithm consistently reaches its minimum training loss (approximately $0.25$ to $0.29$) within $500$ rounds for $N = 10, 20, 30$. In contrast, the performance of the other reference algorithms, particularly DP-CGA and MUFFLIATO, degrades significantly as the number of agents increases.

      We report the experimental results obtained over the fully connected graph in Fig.~\ref{fig:fc-mnist}. It can be observed that the convergence performance of our algorithm improves with denser communication topologies, and its advantage over the other reference algorithms becomes more pronounced. For example, when $\epsilon = 0.5$ and $N = 20$, our algorithm reaches its minimum training loss of approximately $0.29$ within $400$ rounds over the bipartite graph, while it achieves a similar minimum in only $200$ rounds over the fully connected graph. In contrast, the training losses of the other reference algorithms (DP-CGA, DP-NETFLEET, DP-DPSGD, and MUFFLIATO) are $0.93$, $0.69$, $1.76$, and $1.81$, respectively, which are approximately $2.38$ to $6.24$ times higher than that of our algorithm. 
      %
      %
      %

    \subsubsection{Experiment Results about Convergence on CIFAR-10 Dataset} \label{sssec:expconvergence-cifar}
      We present the experimental results on the CIFAR-10 dataset in Fig.~\ref{fig:bipar-CIFAR10} and Fig.~\ref{fig:fc-CIFAR10}. It can be observed that across different communication topologies, DPDL, DP-CGA, and DP-NETFLEET exhibit comparable convergence performance. Their training losses converge to approximately $1.0$–$1.3$ within $700$–$900$ rounds when a higher privacy budget is used (e.g., $\epsilon = 4.0, 8.0$). In contrast, the other two reference algorithms, DP-DPSGD and MUFFLIATO, fail to achieve convergence within the same number of rounds, primarily because they do not account for data heterogeneity. It is worth noting that although DP-CGA and DP-NETFLEET may converge slightly faster and reach marginally lower training losses than our DPDL algorithm in some individual cases, their test accuracies are significantly lower, as will be shown later. Moreover, when the privacy budget is small (e.g., $\epsilon = 2.0, 4.0$), both DP-CGA and DP-NETFLEET exhibit poor stability. DP-NETFLEET shows a wide shaded area in its loss curve, indicating high variance, while DP-CGA shows an upward trend. In contrast, our algorithm achieves more stable and reliable convergence.

  \subsection{Experiment Results about Test Accuracy} \label{ssec:expaccuracy}
  In this section, we report the accuracy of each algorithm under different settings. Specifically, we present tables that summarize the accuracy values across varying numbers of agents, privacy budgets $\epsilon$, and data heterogeneity levels $\alpha_d$. The symbol $\pm$ is used to indicate variability across multiple runs, and the best results are highlighted in bold.

    \subsubsection{Experiment Results about Test Accuracy on MNIST dataset} \label{sssec:expaccuracy-mnist}
      \begin{table*}[h]
      \captionsetup{aboveskip=3pt}
      \caption{Experiment results about test accuracy on MNIST dataset.}
      \centering
        \setlength{\tabcolsep}{8pt}
        \setlength{\extrarowheight}{1.5pt} 
        \begin{tabular}{|c|c|c|c|c|c|c|c|c|c|c|c|}
        \hline
        \multirow{2}{*}{$\alpha_d$} &
        \multirow{2}{*}{$\epsilon$} &
        \multirow{2}{*}{\textbf{Algorithms}} &
        \multicolumn{3}{c|}{\textbf{Bipartite}} &
        \multicolumn{3}{c|}{\textbf{Fully connected}} \\ 
        \cline{4-9} & & & $N$=10 & $N$=20 & $N$=30 & $N$=10 & $N$=20 & $N$=30 \\ \hline
        \multirow{15}{*}{$0.25$} 
        & \multirow{5}{*}{$0.25$} 
        & DP-DPSGD &83.5±1.3 &79.0±0.3 &69.7±0.3 &85.4±2.7 &71.9±1.0 &59.9±2.2 \\
        & & DP-CGA &88.0±0.7 &67.9±1.8 &73.2±4.4 &88.8±1.1 &66.7±0.1 &63.9±1.1 \\ 
        & & MUFFLITAO &85.5±0.8 &84.3±1.1 &83.7±1.3 &86.8±1.1 &79.3±2.0 &74.3±2.7  \\
        & & DP-NETFLEET &90.9±2.3 &81.8±5.5 &76.3±7.6 &90.8±2.4 &82.1±6.2 &76.4±7.7 \\
        & & \textbf{DPDL} &\textbf{94.4±0.6} &\textbf{92.3±0.5} &\textbf{90.9±0.7} &\textbf{96.1±0.6} &\textbf{95.3±1.0} &\textbf{91.9±0.3} \\ \cline{2-9}
        & \multirow{5}{*}{$0.5$}
        & DP-DPSGD &84.1±1.9 &81.8±1.2 &84.2±0.9 &85.2±1.8 &82.9±3.0 &81.0±3.3 \\
        & & DP-CGA &90.9±1.8 &86.1±0.1 &81.1±3.8 &91.5±1.1 &84.7±2.3 &77.1±0.4 \\ 
        & & MUFFLITAO &86.6±1.2 &78.6±2.0 &73.5±2.8 &86.4±0.6 &83.9±1.1 &81.2±1.3 \\
        & & DP-NETFLEET &91.2±2.0 &88.2±2.5 &85.9±4.2 &91.5±1.9 &88.4±2.7 &86.4±4.0 \\
        & & \textbf{DPDL} &\textbf{95.3±0.6} &\textbf{94.4±0.5} &\textbf{93.8±0.3} &\textbf{96.5±0.3} &\textbf{97.0±0.5} &\textbf{96.4±0.1} \\ \cline{2-9}
        & \multirow{5}{*}{$1.0$}
        & DP-DPSGD &83.4±1.6 &81.1±1.3 &82.4±1.3 &84.7±1.7 &83.1±1.6 &83.7±1.9 \\
        & & DP-CGA &91.3±2.1 &88.7±1.5 &89.1±1.9 &91.8±1.5 &89.3±3.1 &88.8±2.0 \\ 
        & & MUFFLITAO &85.6±0.6 &84.5±0.7 &84.9±1.0 &85.4±0.6 &84.3±0.8 &84.9±1.1 \\
        & & DP-NETFLEET &91.4±1.8 &89.5±2.3 &89.4±2.0 &91.5±1.8 &89.7±2.3 &89.5±2.0 \\
        & & \textbf{DPDL} &\textbf{96.1±0.6} &\textbf{95.5±0.4} &\textbf{95.0±0.5} &\textbf{96.7±0.2} &\textbf{97.3±0.5} &\textbf{97.4±0.3} \\ \hline
        \multirow{15}{*}{$1.0$} 
        & \multirow{5}{*}{$0.25$} 
        & DP-DPSGD &88.8±0.8 &87.8±0.6 &77.2±0.1 &88.5±0.8 &87.9±0.6 &78.0±0.1 \\
        & & DP-CGA &92.0±0.2 &83.6±2.5 &53.9±5.1 &91.2±0.1 &82.7±1.8 &55.8±7.3 \\ 
        & & MUFFLITAO &86.8±0.6 &86.6±0.6 &86.2±0.9 &86.7±0.6 &86.5±0.6 &86.5±0.9 \\
        & & DP-NETFLEET &91.0±2.3 &85.6±2.6 &75.7±3.2 &91.3±2.2 &85.9±2.7 &76.4±3.0 \\
        & & \textbf{DPDL} &\textbf{95.5±0.5} &\textbf{93.6±0.5} &\textbf{92.5±0.1} &\textbf{96.8±0.4} &\textbf{96.4±0.8} &\textbf{94.7±0.4}  \\ \cline{2-9}
        & \multirow{5}{*}{$0.5$}
        & DP-DPSGD &87.6±1.4 &88.8±1.4 &87.8±1.9 &86.7±0.9 &87.8±0.7 &88.4±0.7 \\
        & & DP-CGA &93.6±0.8 &93.3±0.5 &88.7±1.6 &93.4±0.7 &92.3±1.5 &87.4±1.0 \\ 
        & & MUFFLITAO &87.4±0.3 &87.8±0.8 &87.9±0.6 &87.1±0.3 &87.6±0.8 &87.8±0.6  \\
        & & DP-NETFLEET &91.4±1.5 &91.0±2.0 &88.0±2.4 &91.6±1.5 &91.4±2.1 &88.3±2.6 \\
        & & \textbf{DPDL} &\textbf{95.9±0.7} &\textbf{95.9±0.6} &\textbf{95.5±0.6} &\textbf{97.2±0.5} &\textbf{97.4±0.4} &\textbf{96.9±0.1} \\ \cline{2-9}
        & \multirow{5}{*}{$1.0$}
        & DP-DPSGD &87.2±1.3 &86.3±1.3 &86.0±1.5 &85.2±1.3 &85.4±0.4 &85.6±0.9 \\
        & & DP-CGA &93.5±1.1 &93.7±0.8 &92.9±1.5 &93.6±0.6 &93.8±1.1 &92.7±1.1 \\ 
        & & MUFFLITAO &88.9±0.1 &88.1±0.6 &88.0±0.6 &88.6±0.1 &88.8±0.6 &87.9±0.7\\
        & & DP-NETFLEET &91.6±1.3 &91.4±1.6 &91.0±1.8 &91.9±1.3 &91.8±1.6 &91.3±1.8 \\
        & & \textbf{DPDL} &\textbf{96.2±0.8} &\textbf{96.2±0.5} &\textbf{96.1±0.7} &\textbf{97.3±0.4} &\textbf{97.5±0.4} &\textbf{97.4±0.3} \\ \hline
      \end{tabular}
      \label{tab:acc-mnist}
      \end{table*}

      In this section, we evaluate the performance of different algorithms in terms of test accuracy. The results on the MNIST dataset are reported in Table~\ref{tab:acc-mnist}. It is evident that our DPDL algorithm consistently achieves the highest prediction accuracy across all settings. For example, when $N = 20$, $\epsilon = 0.25$, and $\alpha_d = 0.25$ under a bipartite graph, the test accuracy of DPDL exceeds those of DP-NETFLEET, DP-CGA, MUFFLIATO, and DP-DPSGD by $14.0\%$, $37.4\%$, $18.4\%$, and $18.1\%$, respectively. Moreover, increasing the privacy budget $\epsilon$ leads to higher test accuracy, though at the cost of reduced privacy protection. For instance, with $N = 20$ and a bipartite graph, the test accuracy of our algorithm improves from $92.3\%$ to $95.5\%$ as $\epsilon$ increases from $0.25$ to $1.0$.
      Another observation is that adopting denser communication graphs can lead to higher test accuracy for our algorithm. For example, when $N = 20$, $\epsilon = 0.5$, and $\alpha_d = 0.25$, the test accuracy on the bipartite graph is $94.4\%$, whereas it exceeds $97\%$ on the fully connected graph. 
      Furthermore, when data heterogeneity is relatively low ($\alpha_d = 1.0$), our algorithm still achieves the best performance across all scenarios and outperforms baseline methods. Compared to $\alpha_d = 1.0$, our algorithm exhibits averagely $1.72\%$ and $0.98\%$  accuracy decline under $\alpha_d = 0.25$ over bipartite and fully connected graphs, highlighting its robustness to data heterogeneity.

    \subsubsection{Experiment Results about Test Accuracy on CIFAR-10 dataset} \label{sssec:expaccuracy-cifar}
      \begin{table*}[h]
      \captionsetup{aboveskip=3pt}
      \caption{Experiment results about test accuracy on CIFAR-10 dataset.}
      \centering
        \setlength{\tabcolsep}{8pt}
        \setlength{\extrarowheight}{1.5pt} 
        \begin{tabular}{|c|c|c|c|c|c|c|c|c|c|c|c|}
        \hline
        \multirow{2}{*}{$\alpha_d$} &
        \multirow{2}{*}{$\epsilon$} &
        \multirow{2}{*}{\textbf{Algorithms}} &
        \multicolumn{3}{c|}{\textbf{Bipartite}} &
        \multicolumn{3}{c|}{\textbf{Fully connected}} \\ 
        \cline{4-9} & & & $N$=10 & $N$=20 & $N$=30 & $N$=10 & $N$=20 & $N$=30 \\ \hline
        \multirow{15}{*}{$0.25$} 
        & \multirow{5}{*}{$2.0$} 
        & DP-DPSGD &48.2±0.8 &45.0±1.3 &37.2±0.7 &48.7±1.0 &45.4±0.1 &42.2±0.8  \\
        & & DP-CGA &55.8±1.4 &39.6±0.7 &24.9±0.7 &49.5±0.8 &36.2±2.0 &29.4±0.1 \\ 
        & & MUFFLITAO &46.0±2.6 &40.2±5.8 &37.5±8.1 &45.8±2.8 &39.9±6.1 &37.4±8.2 \\
        & & DP-NETFLEET &53.1±3.2 &49.0±2.0 &34.1±8.0 &51.8±3.6 &48.4±2.1 &34.6±7.9  \\
        & & \textbf{DPDL} &\textbf{61.7±0.4} &\textbf{56.7±1.3} &\textbf{52.6±1.0} &\textbf{67.2±0.6} &\textbf{63.4±0.8} &\textbf{55.6±1.5} \\ \cline{2-9}
        & \multirow{5}{*}{$4.0$}
        & DP-DPSGD &48.9±0.9 &46.6±0.5 &44.1±0.2 &49.5±0.9 &48.1±0.9 &47.4±0.9 \\
        & & DP-CGA &56.4±2.3 &53.9±0.7 &47.1±0.7 &50.3±1.5 &50.3±0.2 &45.7±0.0 \\ 
        & & MUFFLITAO &47.5±1.5 &43.7±3.0 &41.7±4.8 &47.4±1.7 &43.6±3.2 &41.5±5.2  \\
        & & DP-NETFLEET &55.7±1.3 &53.9±0.3 &48.3±4.3 &55.3±1.4 &54.0±0.3 &48.6±4.2  \\
        & & \textbf{DPDL} &\textbf{64.0±0.7} &\textbf{62.1±0.6} &\textbf{60.2±0.7} &\textbf{69.5±0.6} &\textbf{68.1±1.0} &\textbf{67.3±0.8} \\ \cline{2-9}
        & \multirow{5}{*}{$8.0$}
        & DP-DPSGD &48.7±0.6 &47.0±0.5 &45.1±0.7 &49.8±0.8 &48.5±1.0 &47.9±1.0 \\
        & & DP-CGA &58.4±2.7 &57.4±1.3 &55.6±1.0 &52.5±3.3 &51.0±1.0 &50.1±0.2 \\ 
        & & MUFFLITAO &47.6±1.3 &45.8±1.0 &44.8±2.5  &47.9±1.3 &45.8±1.5 &44.6±2.7 \\
        & & DP-NETFLEET &57.3±1.1 &55.6±0.2 &50.9±3.4 &56.7±0.9 &55.5±0.4 &50.7±3.4 \\
        & & \textbf{DPDL} &\textbf{64.7±0.8} &\textbf{64.2±0.1} &\textbf{63.9±1.2} &\textbf{70.2±0.7} &\textbf{70.2±0.6} &\textbf{70.5±0.9} \\ \hline
        \multirow{15}{*}{$1.0$} 
        & \multirow{5}{*}{$2.0$} 
        & DP-DPSGD &49.3±0.5 &46.8±0.5 &40.7±1.7 &49.3±0.5 &46.8±0.5 &40.7±1.7 \\
        & & DP-CGA &57.6±1.0 &45.6±0.6 &37.4±0.2 &58.4±1.3 &45.8±0.6 &33.5±0.4 \\ 
        & & MUFFLITAO &51.6±0.7 &50.5±0.7 &49.1±0.6 &51.6±0.7 &50.5±0.7 &49.1±0.6 \\
        & & DP-NETFLEET&55.5±0.5 &50.3±0.2 &39.5±0.2 &55.6±0.4 &50.2±0.3 &39.7±0.4 \\
        & & \textbf{DPDL} &\textbf{64.0±0.8} &\textbf{61.1±0.1} &\textbf{54.9±0.8} &\textbf{67.9±0.6} &\textbf{65.3±0.2} &\textbf{60.9±0.7} \\ \cline{2-9}
        & \multirow{5}{*}{$4.0$}
        & DP-DPSGD &49.6±0.4 &47.8±0.6 &45.1±1.3 &49.6±0.4 &47.8±0.6 &45.1±1.3 \\
        & & DP-CGA &61.2±0.9 &55.9±1.1 &49.4±0.7 &61.8±1.7 &59.1±0.1 &50.2±0.4 \\ 
        & & MUFFLITAO &51.7±0.5 &50.7±0.6 &49.4±0.7 &51.7±0.5 &50.7±0.6 &49.4±0.7 \\
        & & DP-NETFLEET &56.9±0.3 &54.8±0.1 &51.3±0.1 &56.8±0.1 &54.8±0.3 &51.3±0.3 \\
        & & \textbf{DPDL} &\textbf{65.8±0.5} &\textbf{65.0±0.1} &\textbf{62.8±0.6} &\textbf{70.3±0.3} &\textbf{69.5±0.4} &\textbf{68.1±1.3} \\ \cline{2-9}
        & \multirow{5}{*}{$8.0$}
        & DP-DPSGD &49.8±0.5 &48.2±0.8 &46.1±0.5 &49.8±0.5 &48.2±0.8 &46.1±0.5 \\
        & & DP-CGA &63.0±0.9 &60.5±1.0 &57.6±0.9 &64.2±1.4 &63.5±1.0 &60.0±0.8 \\ 
        & & MUFFLITAO &51.9±0.5 &50.8±0.5 &49.6±0.6 &51.9±0.5 &50.8±0.5 &49.6±0.6 \\
        & & DP-NETFLEET &57.7±0.1 &55.9±0.4 &52.8±0.0 &57.8±0.2 &55.5±0.4 &53.0±0.2 \\
        & & \textbf{DPDL} &\textbf{66.7±0.9} &\textbf{66.4±0.4} &\textbf{65.1±0.6}  &\textbf{71.0±0.4} &\textbf{71.8±0.8} &\textbf{71.2±0.9} \\ \hline
      \end{tabular}
      \label{tab:acc-cifar}
      \end{table*}

    According to the results on the CIFAR-10 dataset (see Table~\ref{tab:acc-cifar}), the advantage of our DPDL algorithm over the other reference methods remains clear. In particular, when $N = 20$, $\epsilon = 4.0$, and $\alpha_d = 0.25$, the test accuracy of our algorithm on the bipartite graph reaches $62.1\%$, while the other reference algorithms achieve only around $43.7\%$ to $53.9\%$, which is at least $13.2\%$ lower than that of DPDL. When the privacy budget $\epsilon$ is increased to $8.0$, the test accuracy of our algorithm further improves to $64.2\%$, as less noise is added to the gradients, thereby preserving more utility.
    Besides, the impact of communication topology on performance observed on the CIFAR-10 dataset is consistent with the findings from the MNIST dataset. For example, fixing $N = 10$, $\epsilon = 4.0$, and $\alpha_d = 0.25$, the test accuracy of our algorithm increases from $64.0\%$ on the bipartite graph to $69.5\%$ on the fully connected graph. 
    Furthermore, when $\epsilon=8.0$ and $\alpha_d=1.0$, our algorithm achieves over $65\%$ accuracy across all agent numbers on the bipartite graph, and even surpass $70\%$ on the fully connected graph, demonstrating that our algorithm remains effective under low data heterogeneity.

  \subsection{Experiment Results about Defending Gradient Attacks} \label{ssec:expattacks}


    In our DPDL framework, each agent collaborates with its peers by exchanging its private gradient information. Specifically, as shown in Lines~\ref{ln:sharecg-start}$\sim$\ref{ln:sharecg-end} in \textbf{Algorithm}~\ref{alg:dpdl}, in each round $t$, each agent $i$ calculate cross-gradient $g^{ji}_t$ for each of its neighbors $j\in\mathcal{N}_i \setminus \{i\}$ based on its local data, and agent $i$ then clips the cross-gradients by sample and perturbs the averaged cross-gradient over the batch before communicating it with the neighbor for the purpose of privacy preservation. Furthermore, as illustrated in Lines~\ref{ln:compcalibration}$\sim$\ref{ln:txmompara} in \textbf{Algorithm}~\ref{alg:dpdl}, each agent $i$ needs to send $\tilde{v}^i_t$ and $\tilde{x}^i_t$ to its neighbors, which also contains private information from aggregated gradient information $\tilde{g}^i_t$.

    In our experiments, we demonstrate the robustness of the proposed DPDL algorithm against privacy attacks, such as gradient inversion attacks~\cite{GeipingBDM-NIPS20, FredriksonJR-CCS15}. Based on the previous discussion, we assume the existence of an adversary aiming to reconstruct the private local data of any agent $i$ from either its perturbed cross-gradient $\ddot{g}^{ji}_t$ or its updated momentum $\tilde{v}^i_t$ (and consequently, its model parameter $\tilde{x}^i_t$). To illustrate the effectiveness of our algorithm, we consider privacy attacks under the following scenarios. Due to space limitations, we present the comparison between our algorithm and a variant without DP preservation applied to $\ddot{g}^{ji}_t$ and $\tilde{v}^i_t$ (attack results on $\tilde{x}^i_t$ is similar to that on $\tilde{v}^i_t$). 

    %
    \begin{figure*}[htb!]
    \begin{center}
      \parbox{.3\textwidth}{\center\includegraphics[width=.3\textwidth]{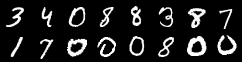}}
      \parbox{.3\textwidth}{\center\includegraphics[width=.3\textwidth]{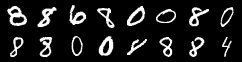}}
      \parbox{.3\textwidth}{\center\includegraphics[width=.3\textwidth]{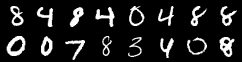}}      
      \parbox{.3\textwidth}{\center\scriptsize(a1) Ground Truth in the 200th round}
      \parbox{.3\textwidth}{\center\scriptsize(a2) Ground Truth in the 400th round}
      \parbox{.3\textwidth}{\center\scriptsize(a3) Ground Truth in the 800th round}
      \parbox{.3\textwidth}{\center\includegraphics[width=.3\textwidth]{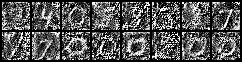}}
      \parbox{.3\textwidth}{\center\includegraphics[width=.3\textwidth]{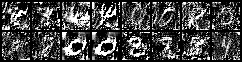}}
      \parbox{.3\textwidth}{\center\includegraphics[width=.3\textwidth]{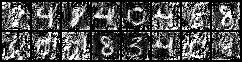}}      
      \parbox{.3\textwidth}{\center\scriptsize(b1) ND-CG in the 200th round}
      \parbox{.3\textwidth}{\center\scriptsize(b2) ND-CG in the 400th round}
      \parbox{.3\textwidth}{\center\scriptsize(b3) ND-CG in the 800th round}
      \parbox{.3\textwidth}{\center\includegraphics[width=.3\textwidth]{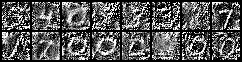}}
      \parbox{.3\textwidth}{\center\includegraphics[width=.3\textwidth]{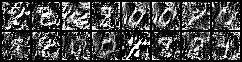}}
      \parbox{.3\textwidth}{\center\includegraphics[width=.3\textwidth]{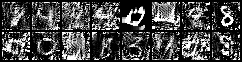}}      
      \parbox{.3\textwidth}{\center\scriptsize(c1) ND-MM in the 200th round}
      \parbox{.3\textwidth}{\center\scriptsize(c2) ND-MM in the 400th round}
      \parbox{.3\textwidth}{\center\scriptsize(c3) ND-MM in the 800th round}
      \parbox{.3\textwidth}{\center\includegraphics[width=.3\textwidth]{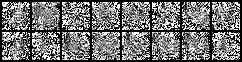}}
      \parbox{.3\textwidth}{\center\includegraphics[width=.3\textwidth]{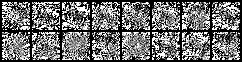}}
      \parbox{.3\textwidth}{\center\includegraphics[width=.3\textwidth]{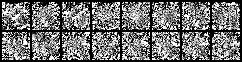}}      
      \parbox{.3\textwidth}{\center\scriptsize(d1) DPDL-CG in the 200th round}
      \parbox{.3\textwidth}{\center\scriptsize(d2) DPDL-CG in the 400th round}
      \parbox{.3\textwidth}{\center\scriptsize(d3) DPDL-CG in the 800th round}
      \parbox{.3\textwidth}{\center\includegraphics[width=.3\textwidth]{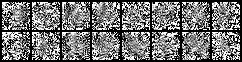}}
      \parbox{.3\textwidth}{\center\includegraphics[width=.3\textwidth]{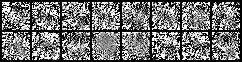}}
      \parbox{.3\textwidth}{\center\includegraphics[width=.3\textwidth]{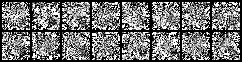}}      
      \parbox{.3\textwidth}{\center\scriptsize(e1) DPDL-MM in the 200th round}
      \parbox{.3\textwidth}{\center\scriptsize(e2) DPDL-MM in the 400th round}
      \parbox{.3\textwidth}{\center\scriptsize(e3) DPDL-MM in the 800th round}
    \caption{Images reconstructed by gradient inversion attacks on MNIST dataset.}
    \label{fig:attack-mnist}
    \end{center}
    \end{figure*}

    \begin{itemize}
      \item \textbf{No-Defense-CG (ND-CG)}: We assume that our algorithm does not apply the Gaussian mechanism to the cross-gradient information, meaning that each agent $i$ computes and transmits its original cross-gradient $g^{ji}_t$ to each of its neighbors $j$. In this case, an adversary can intercept the unperturbed cross-gradient $g^{ji}_t$ of any agent $i$ and apply gradient inversion techniques to reconstruct that agent’s local data in any round $t$.
      \item \textbf{No-Defense-MM (ND-MM)}: Without the Gaussian noise mechanism, each agent $i$ calibrates the received cross-gradient $g^{ij}_t$ from each neighbor $j$, and then aggregates the calibrated cross-gradients along with its self-gradient to perform the upcoming momentum update. In this case, the local data of agent $i$ can be reconstructed by analyzing the difference between $\tilde{v}^i_{t-1}$ and $\tilde{v}^i_t$ in consecutive rounds $t-1$ and $t$.
      \item \textbf{DPDL-CG}: The adversary conducts gradient inversion to reconstruct the local data of any agent $i$ according to its perturbed cross-gradient $\ddot{g}^{ji}_t$ in any round $t$.
      \item \textbf{DPDL-MM}: The inversion attack is performed based on the difference between $\tilde{v}^i_{t-1}$ and $\tilde{v}^i_t$ in consecutive rounds $t-1$ and $t$.
    \end{itemize}

    \begin{figure*}[htb!]
    \begin{center}
      \parbox{.3\textwidth}{\center\includegraphics[width=.3\textwidth]{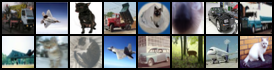}}
      \parbox{.3\textwidth}{\center\includegraphics[width=.3\textwidth]{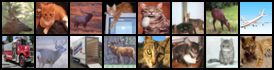}}
      \parbox{.3\textwidth}{\center\includegraphics[width=.3\textwidth]{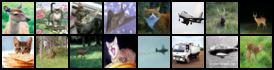}}      
      \parbox{.3\textwidth}{\center\scriptsize(a1) Ground Truth in the 200th round}
      \parbox{.3\textwidth}{\center\scriptsize(a2) Ground Truth in the 400th round}
      \parbox{.3\textwidth}{\center\scriptsize(a3) Ground Truth in the 800th round}
      \parbox{.3\textwidth}{\center\includegraphics[width=.3\textwidth]{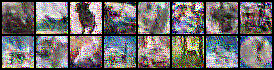}}
      \parbox{.3\textwidth}{\center\includegraphics[width=.3\textwidth]{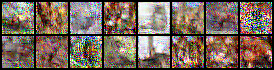}}
      \parbox{.3\textwidth}{\center\includegraphics[width=.3\textwidth]{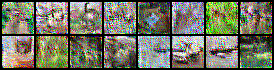}}      
      \parbox{.3\textwidth}{\center\scriptsize(b1) ND-CG in the 200th round}
      \parbox{.3\textwidth}{\center\scriptsize(b2) ND-CG in the 400th round}
      \parbox{.3\textwidth}{\center\scriptsize(b3) ND-CG in the 800th round}
      \parbox{.3\textwidth}{\center\includegraphics[width=.3\textwidth]{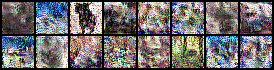}}
      \parbox{.3\textwidth}{\center\includegraphics[width=.3\textwidth]{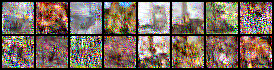}}
      \parbox{.3\textwidth}{\center\includegraphics[width=.3\textwidth]{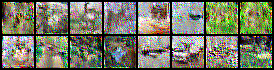}}      
      \parbox{.3\textwidth}{\center\scriptsize(c1) ND-MM in the 200th round}
      \parbox{.3\textwidth}{\center\scriptsize(c2) ND-MM in the 400th round}
      \parbox{.3\textwidth}{\center\scriptsize(c3) ND-MM in the 800th round}
      \parbox{.3\textwidth}{\center\includegraphics[width=.3\textwidth]{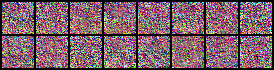}}
      \parbox{.3\textwidth}{\center\includegraphics[width=.3\textwidth]{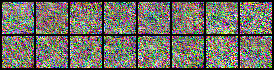}}
      \parbox{.3\textwidth}{\center\includegraphics[width=.3\textwidth]{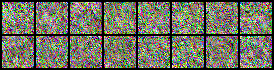}}      
      \parbox{.3\textwidth}{\center\scriptsize(d1) DPDL-CG in the 200th round}
      \parbox{.3\textwidth}{\center\scriptsize(d2) DPDL-CG in the 400th round}
      \parbox{.3\textwidth}{\center\scriptsize(d3) DPDL-CG  in the 800th round}
      \parbox{.3\textwidth}{\center\includegraphics[width=.3\textwidth]{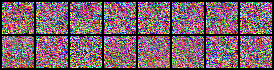}}
      \parbox{.3\textwidth}{\center\includegraphics[width=.3\textwidth]{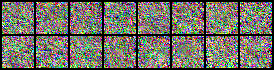}}
      \parbox{.3\textwidth}{\center\includegraphics[width=.3\textwidth]{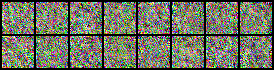}}      
      \parbox{.3\textwidth}{\center\scriptsize(e1) DPDL-MM in the 200th round}
      \parbox{.3\textwidth}{\center\scriptsize(e2) DPDL-MM in the 400th round}
      \parbox{.3\textwidth}{\center\scriptsize(e3) DPDL-MM  in the 800th round}
    \caption{Images reconstructed by gradient inversion attacks on CIFAR-10 dataset.}
    \label{fig:attack-cifar10}
    \end{center}
    \end{figure*}

    We first present the reconstructed images under the aforementioned scenarios in Fig.~\ref{fig:attack-mnist} and Fig.~\ref{fig:attack-cifar10}. Due to space limitations, we show the results obtained with $N = 10$, $\epsilon = 0.5$ (for the MNIST dataset) and $\epsilon = 4.0$ (for the CIFAR-10 dataset) over a bipartite graph in different rounds. The experiment results reveal that, without our Gaussian-based DP mechanism, the reconstructed images obtained through gradient inversion attacks exhibit a noticeable resemblance to the original images, with perceptually recognizable objects. Moreover, as training progresses, the reconstructed images from later rounds tend to reveal more detailed and clearer information. In contrast, when the Gaussian mechanism is applied in our algorithm, the adversary fails to effectively reconstruct the agents' private data. Even as the number of rounds increases, the privacy protection remains robust.

    \begin{table*}[htb!]
    \captionsetup{aboveskip=3pt}
    \caption{The results of gradient inversion attacks on MNIST dataset.}
    \setlength{\tabcolsep}{4.5pt}
    \centering
    \setlength{\tabcolsep}{8pt}
    \setlength{\extrarowheight}{1.5pt} 
    \begin{tabular}{|c|c|c|c|c|c|c|c|c|c|c|}
    \hline
    \multirow{2}{*}{\textbf{Metrics}}  &  \multirow{2}{*}{\textbf{Algorithms}} & 
    \multicolumn{3}{c|}{\textbf{Ring}} & \multicolumn{3}{c|}{\textbf{Bipartite}} & \multicolumn{3}{c|}{\textbf{Fully connected}} \\ 
    \cline{3-11} & & $N$=10 & $N$=20 & $N$=30 & $N$=10 & $N$=20 & $N$=30 & $N$=10 & $N$=20 & $N$=30 \\ \hline
    \multirow{4}{*}{MSE} 
    & {ND-CG} 
    &0.85 &0.75 &0.71 &0.84 &0.86 &0.83 &0.85 &0.83 &0.89 \\
    & {ND-MM} 
    &0.81 &0.86 &0.84 &0.88 &0.82 &0.87 &0.89 &0.90 &0.91 \\ 
    & {DPDL-CG} 
    &1.57 &1.58 &1.57 &1.37 &1.44 &1.43 &1.46 &1.45 &1.36 \\ 
    & {DPDL-MM} 
    &1.60 &1.59 &1.56 &1.39 &1.43 &1.40 &1.33 &1.46 &1.36 \\ \cline{1-11}
    \multirow{4}{*}{PSNR}
    & {ND-CG} 
    &6.84 &7.48 &7.73 &6.88 &6.76 &6.86 &6.83 &6.93 &6.56 \\
    & {ND-MM} 
    &7.12 &6.91 &7.12 &6.69 &7.03 &6.71 &6.57 &6.56 &6.54  \\ 
    & {DPDL-CG} 
    &4.07 &4.03 &4.06 &4.67 &4.46 &4.47 &4.39 &4.41 &4.80 \\
    & {DPDL-MM} 
    &3.99 &4.01 &4.09 &4.58 &4.46 &4.56 &4.55 &4.43 &4.68  \\ \cline{1-11}
    \multirow{4}{*}{\makecell[c]{SSIM \\ $(\times {10}^{\textnormal{-}2})$}}
    & {ND-CG} 
    &19.57 &24.46 &21.50 &17.14 &14.80 &13.57 &15.43 &15.22 &12.14  \\
    & {ND-MM} 
    &19.88 &16.91 &18.30 &15.73 &14.60 &14.11 &15.86 &13.54 &13.26 \\ 
    & {DPDL-CG} 
    &4.92 &3.33 &3.82 &5.75 &4.19 &3.73 &4.33 &4.11 &4.80  \\
    & {DPDL-MM} 
    &4.50 &3.86 &3.93 &5.18 &4.35 &3.50 &4.27 &4.09 &3.52 \\ \hline
    \end{tabular}
    \label{tab:attack-mnist}
    \end{table*}
    \begin{table*}[htb!]
    \captionsetup{aboveskip=3pt}
    \caption{The results of gradient inversion attacks on CIFAR-10 dataset.}
    \setlength{\tabcolsep}{4.5pt}
    \centering
    \setlength{\tabcolsep}{8pt}
    \setlength{\extrarowheight}{1.5pt} 
    \begin{tabular}{|c|c|c|c|c|c|c|c|c|c|c|}
    \hline
    \multirow{2}{*}{\textbf{Metrics}}  &  \multirow{2}{*}{\textbf{Algorithms}} & 
    \multicolumn{3}{c|}{\textbf{Ring}} & \multicolumn{3}{c|}{\textbf{Bipartite}} & \multicolumn{3}{c|}{\textbf{Fully connected}} \\ 
    \cline{3-11} & & $N$=10 & $N$=20 & $N$=30 & $N$=10 & $N$=20 & $N$=30 & $N$=10 & $N$=20 & $N$=30 \\ \hline
    \multirow{4}{*}{MSE} 
    & {ND-CG} 
    &0.90 &1.16 &0.96 &0.91 &1.23 &0.97 &1.20 &1.15 &1.16 \\ 
    & {ND-MM} 
    &1.70 &1.35 &1.85 &1.90 &1.61 &1.84 &1.85 &1.95 &1.95 \\ 
    & {DPDL-CG} 
    &2.47 &2.24 &2.35 &2.40 &2.34 &2.24 &2.40 &2.18 &2.22 \\ 
    & {DPDL-MM} 
    &2.47 &2.28 &2.02 &2.60 &2.26 &2.21 &2.49 &2.20 &2.64 \\ \cline{1-11}
    \multirow{4}{*}{PSNR}
    & {ND-CG} 
    &12.13 &10.20 &11.06 &10.00 &10.19 &10.05 &11.10 &9.98 &9.72   \\
    & {ND-MM} 
    &9.60 &10.21 &11.25 &9.37 &9.50 &10.03 &10.27 &9.42 &10.10\\ 
    & {DPDL-CG} 
    &8.12 &8.58 &9.06 &8.12 &8.58 &8.77 &8.72 &8.78 &8.28 \\
    & {DPDL-MM} 
    &8.12 &8.58 &9.03 &8.02 &8.52 &8.68 &8.17 &8.69 &8.03 \\ \cline{1-11}
    \multirow{4}{*}{\makecell[c]{SSIM \\ $(\times {10}^{\textnormal{-}2})$} }
    & {ND-CG} 
    &8.92 &11.89 &10.41 &12.10 &10.06 &11.56 &9.15 &6.64 &11.65   \\ 
    & {ND-MM} 
    &8.47 &14.15 &7.87 &9.07 &7.18 &7.81 &10.00 &9.33 &10.27 \\ 
    & {DPDL-CG} 
    &2.49 &2.05 &2.14 &2.08 &1.98 &2.47 &2.21 &2.04 &3.44 \\
    & {DPDL-MM} 
    &2.68 &2.33 &2.12 &1.85 &1.88 &2.14 &1.87 &2.17 &2.41  \\ \hline
    \end{tabular}
    \label{tab:attack-cifar10}
    \end{table*}

    We also evaluate image reconstruction performance in the $800$th round using three commonly adopted metrics: \textit{Mean Squared Error} (MSE), \textit{Peak Signal-to-Noise Ratio} (PSNR), and \textit{Structural Similarity Index Measure} (SSIM)~\cite{WangSB-SSC03}, which are widely used in prior work to assess privacy preservation~\cite{WangHL-INFOCOM24, ChuL-SIGMETRICS24}. MSE quantifies the average squared difference between two images at the pixel level, PSNR assesses the level of noise or distortion in a reconstructed image, and SSIM evaluates image quality by considering luminance, contrast, and structural similarity. Given a reconstructed image and its original counterpart, a higher MSE or lower PSNR or SSIM indicates greater divergence and thus stronger privacy preservation. The results are reported in Table~\ref{tab:attack-mnist} and \ref{tab:attack-cifar10}. It is evident that our DPDL algorithm effectively protects data privacy. Specifically, for the MNIST dataset, the MSE values of the reconstructed images using DPDL are $0.44$-$0.86$ higher than those without any privacy protection. In addition, PSNR and SSIM scores are $1.74$-$3.67$ and $7.34$-$21.13$ smaller than the ones with no privacy preservation mechanism applied, respectively. Similarly, on the CIFAR-10 dataset, the MSE values of the images reconstructed when our DPDL algorithm is applied are $1.28$ higher than those from the unprotected setting, while the PSNR and SSIM scores are $2.07$ and $8.39$ lower than the ones with no privacy preservation on average.

\section{Conclusion}\label{sec:conclusion}
  In this paper, to address the risk of privacy leakage in decentralized learning with data heterogeneity, we have proposed DPDL, a novel privacy-preserved decentralized stochastic learning algorithm with non-IID data. In particular, the notion of DP is leveraged in an elaborated calibrated cross-gradient aggregation, such that the perturbed cross-gradient information can be fully exploited for model training with non-IID data. Our theoretical analysis have not only given the minimum level of noise to satisfy the demand of privacy preservation, but also demonstrated that our algorithm has a linear speed up even with non-IID data. Through our extensive experiments on real datasets, we have verified the robustness of our algorithm against privacy attacks as well as the efficacy in training accurate model.


\bibliographystyle{IEEEtran}
\bibliography{IEEEabrv,reference}

\begin{thebibliography}{10}
\providecommand{\url}[1]{#1}
\csname url@samestyle\endcsname
\providecommand{\newblock}{\relax}
\providecommand{\bibinfo}[2]{#2}
\providecommand{\BIBentrySTDinterwordspacing}{\spaceskip=0pt\relax}
\providecommand{\BIBentryALTinterwordstretchfactor}{4}
\providecommand{\BIBentryALTinterwordspacing}{\spaceskip=\fontdimen2\font plus
\BIBentryALTinterwordstretchfactor\fontdimen3\font minus
  \fontdimen4\font\relax}
\providecommand{\BIBforeignlanguage}[2]{{%
\expandafter\ifx\csname l@#1\endcsname\relax
\typeout{** WARNING: IEEEtran.bst: No hyphenation pattern has been}%
\typeout{** loaded for the language `#1'. Using the pattern for}%
\typeout{** the default language instead.}%
\else
\language=\csname l@#1\endcsname
\fi
#2}}
\providecommand{\BIBdecl}{\relax}
\BIBdecl

\bibitem{VerbraekenWKKVR-CSUR20}
J.~Verbraeken, M.~Wolting, J.~Katzy, J.~Kloppenburg, T.~Verbelen, and
  J.~Rellermeyer, ``{A Survey on Distributed Machine Learning},'' \emph{ACM
  Computing Surveys}, vol.~53, no.~2, pp. 1--33, 2020.

\bibitem{LiWWHWLLH-TKDE23}
Q.~Li, Z.~Wen, Z.~Wu, S.~Hu, N.~Wang, Y.~Li, X.~Liu, and B.~He, ``{A Survey on
  Federated Learning Systems: Vision, Hype and Reality for Data Privacy and
  Protection},'' \emph{IEEE Trans. on Knowledge and Data Engineering}, vol.~35,
  no.~4, pp. 3347--3366, 2023.

\bibitem{BeltranPSBBPPC-COMST23}
E.~Beltr{\'a}n, M.~P{\'e}rez, P.~S{\'a}nchez, S.~Bernal, G.~Bovet,
  M.~P{\'e}rez, G.~P{\'e}rez, and A.~Celdr{\'a}n, ``{Decentralized Federated
  Learning: Fundamentals, State of the Art, Frameworks, Trends, and
  Challenges},'' \emph{IEEE Communications Surveys \& Tutorials}, vol.~25,
  no.~4, pp. 2983--3013, 2023.

\bibitem{ChuL-SIGMETRICS24}
T.~Chu and N.~Laoutaris, ``{FedQV: Leveraging Quadratic Voting in Federated
  Learning},'' \emph{Proceedings of the ACM on Measurement and Analysis of
  Computing Systems}, vol.~8, no.~2, pp. 22:1--22:36, 2024.

\bibitem{McmahanMRHA-AISTATS17}
B.~McMahan, E.~Moore, D.Ramage, S.~Hampson, and B.~Arcas,
  ``{Communication-Efficient Learning of Deep Networks from Decentralized
  Data},'' in \emph{Proc. of the 20th AISTATS}, vol.~54, 2017, pp. 1273--1282.

\bibitem{LianZZHZL-NIPS17}
X.~Lian, C.~Zhang, H.~Zhang, C.~Hsieh, W.~Zhang, and J.~Liu, ``{Can
  Decentralized Algorithms Outperform Centralized Algorithms? A Case Study for
  Decentralized Parallel Stochastic Gradient Descent},'' in \emph{Proc. of the
  31st NIPS}, 2017, p. 5336–5346.

\bibitem{NedicOR-IEEE18}
A.~Nedi, A.~Olshevsky, and M.~Rabbat, ``{Network Topology and
  Communication-computation Tradeoffs in Decentralized Optimization},''
  \emph{Proceedings of the IEEE}, vol. 106, no.~5, pp. 953--976, 2018.

\bibitem{LianZZL-ICML18}
X.~Lian, W.~Zhang, C.~Zhang, and J.~Liu, ``{Asynchronous Decentralized Parallel
  Stochastic Gradient Descent},'' in \emph{Proc. of the 35th ICML}, 2018, pp.
  3043--3052.

\bibitem{LinKSJ-ICML21}
T.~Lin, S.~Karimireddy, S.~Stich, and M.~Jaggi, ``{Quasi-global Momentum:
  Accelerating Decentralized Deep Learning on Heterogeneous Data},'' in
  \emph{Proc. of the 38th ICML}, 2021, pp. 6654--6665.

\bibitem{LiDCH-ICDE22}
Q.~Li, Y.~Diao, Q.~Chen, and B.~He, ``{Federated Learning on Non-iid Data
  Silos: An Experimental Study},'' in \emph{Proc. of the 38th IEEE ICDE}, 2022,
  pp. 965--978.

\bibitem{LiLV-IOTJ21}
C.~Li, G.~Li, and P.~Varshney, ``{Decentralized Federated Learning via Mutual
  Knowledge Transfer},'' \emph{IEEE Internet of Things Journal}, vol.~9, no.~2,
  pp. 1136--1147, 2021.

\bibitem{ShiSWSYWT-ICML23}
Y.~Shi, L.~Shen, K.~Wei, Y.~Sun, B.~Yuan, X.~Wang, and D.~Tao, ``{Improving the
  Model Consistency of Decentralized Federated Learning},'' in \emph{Proc. of
  the 40th ICML}, vol. 202, 2023, pp. 31\,269--31\,291.

\bibitem{AketiHR-NeurIPS24}
S.~Aketi, A.~Hashemi, and K.~Roy, ``{Global Update Tracking: A Decentralized
  Learning Algorithm for Heterogeneous Data},'' in \emph{Proc. of the 37th
  NeurIPS}, 2024.

\bibitem{ZhangFLYLZ-MobiHoc22}
X.~Zhang, M.~Fang, Z.~Liu, H.~Yang, J.~Liu, and Z.~Zhu, ``{NET-FLEET: Achieving
  Linear Convergence Speedup for Fully Decentralized Federated Learning with
  Heterogeneous Data},'' in \emph{Proc. of the 23rd ACM MobiHoc}, 2022, p.
  71–80.

\bibitem{EsfandiariTJBHHS-ICML21}
Y.~Esfandiari, S.~Tan, Z.~Jiang, A.~Balu, E.~Herron, C.~Hegde, and S.~Sarkar,
  ``{Cross-Gradient Aggregation for Decentralized Learning from Non-IID
  Data},'' in \emph{Proc. of the 38th ICML}, 2021, pp. 3036--3046.

\bibitem{YuJY-ICML19}
H.~Yu, R.~Jin, and S.~Yang, ``{On the Linear Speedup Analysis of Communication
  Efficient Momentum {SGD} for Distributed Non-Convex Optimization},'' in
  \emph{Proc. of the 36th ICML}, 2019, pp. 7184--7193.

\bibitem{ZhuLH-NIPS19}
L.~Zhu, Z.~Liu, and S.~Han, ``{Deep leakage from gradients},'' in \emph{Proc.
  of the 33rd NIPS}, 2019, pp. 14\,747--14\,756.

\bibitem{ZhaoMB-ARXIV20}
B.~Zhao, K.~Mopuri, and H.~Bilen, ``{iDLG: Improved Deep Leakage from
  Gradients},'' \emph{arXiv preprint arXiv:2001.02610}, 2020.

\bibitem{WangHL-INFOCOM24}
F.~Wang, E.~Hugh, and B.~Li, ``{More than Enough is Too Much: Adaptive Defenses
  against Gradient Leakage in Production Federated Learning},'' in \emph{Proc.
  of the 42nd INFOCOM}, 2023, pp. 1--10.

\bibitem{Dwork-ICALP06}
C.~Dwork, ``{Differential Privacy},'' in \emph{Proc. of the 33rd ICALP}, 2006,
  pp. 1--12.

\bibitem{DworkR-FTTCS14}
C.~Dwork and A.~Roth, ``{The Algorithmic Foundations of Differential
  Privacy},'' \emph{Foundations and Trends in Theoretical Computer Science},
  vol.~9, pp. 211--407, 2014.

\bibitem{PonomarevaVXMKZ-KDD23}
N.~Ponomareva, S.~Vassilvitskii, Z.~Xu, B.~McMahan, A.~Kurakin, and C.~Zhang,
  ``How to dp-fy ml: A practical tutorial to machine learning with differential
  privacy,'' in \emph{Proc. of the 29th ACM SIGKDD}, 2023, p. 5823–5824.

\bibitem{ShokriS-CCS15}
R.~Shokri and V.~Shmatikov, ``{Privacy-Preserving Deep Learning},'' in
  \emph{Proc. of the 22nd ACM CCS}, 2015, pp. 1310--1321.

\bibitem{LinWLSHJ-TDSC23}
X.~Lin, J.~Wu, J.~Li, C.~Sang, S.~Hu, and M.~J. Deen, ``{Heterogeneous
  Differential-Private Federated Learning: Trading Privacy for Utility
  Truthfully},'' \emph{IEEE Trans. on Dependable and Secure Computing},
  vol.~20, no.~6, pp. 5113--5129, 2023.

\bibitem{ChenDBZJ-TKDE24}
L.~Chen, X.~Ding, Z.~Bao, P.~Zhou, and H.~Jin, ``{Differentially Private
  Federated Learning on Non-iid Data: Convergence Analysis and Adaptive
  Optimization},'' \emph{IEEE Trans. on Knowledge and Data Engineering},
  vol.~36, no.~9, pp. 4567--4581, 2024.

\bibitem{XuZW-TPAMI21}
J.~Xu, W.~Zhang, and F.~Wang, ``{A(DP)2SGD: Asynchronous Decentralized Parallel
  Stochastic Gradient Descent With Differential Privacy},'' \emph{IEEE Trans.
  on Pattern Analysis and Machine Intelligence}, vol.~44, no.~11, pp.
  8036--8047, 2021.

\bibitem{CyffersEBM-NIPS22}
E.~Cyffers, M.~Even, A.~Bellet, and L.~Massouli\'{e}, ``{Muffliato:
  Peer-to-Peer Privacy Amplification for Decentralized Optimization and
  Averaging},'' in \emph{Proc. of 36th NIPS}, vol.~35, 2022, pp.
  15\,889--15\,902.

\bibitem{SeemanS-JRC24}
J.~Seeman and D.~Susser, ``{Between Privacy and Utility: On Differential
  Privacy in Theory and Practice},'' \emph{ACM Journal on Responsible
  Computing}, vol.~1, no.~1, pp. 1--18, 2024.

\bibitem{BoydGPS-TIT06}
S.~Boyd, A.~Ghosh, B.~Prabhakar, and D.~Shah, ``{Randomized Gossip
  Algorithms},'' \emph{IEEE Trans. on Information Theory}, vol.~52, no.~6, pp.
  2508--2530, 2006.

\bibitem{ScamanBBLM-ICML17}
K.~Scaman, F.~Bach, S.~Bubeck, Y.~Lee, and L.~Massouli{\'{e}}, ``{Optimal
  Algorithms for Smooth and Strongly Convex Distributed Optimization in
  Networks},'' in \emph{Proc. of the 34th ICML}, 2017, pp. 3027--3036.

\bibitem{ScamanBBLM-NIPS18}
K.~Scaman, F.~Bach, S.~Bubeck, Y.~Lee, and L.~Massouli\'{e}, ``{Optimal
  Algorithms for Non-smooth Distributed Optimization in Networks},'' in
  \emph{Proc. of the 32nd NIPS}, 2018, p. 2745–2754.

\bibitem{PuN-MP21}
S.~Pu and A.~Nedi, ``{Distributed Stochastic Gradient Tracking Methods},''
  \emph{Mathematical Programming}, vol. 187, no.~1, pp. 409--457, 2021.

\bibitem{TakezawaBNSY-TMLR23}
Y.~Takezawa, H.~Bao, K.~Niwa, R.~Sato, and M.~Yamada, ``{Momentum Tracking:
  Momentum Acceleration for Decentralized Deep Learning on Heterogeneous
  Data},'' \emph{Trans. on Machine Learning Research}, vol. 2023, 2023.

\bibitem{AketiKR-TMLR23}
S.~Aketi, S.~Kodge, and K.~Roy, ``{Neighborhood Gradient Mean: An Efficient
  Decentralized Learning Method for Non-IID Data Distributions},'' \emph{Trans.
  on Machine Learning Research}, 2023.

\bibitem{HuangG-arXiv20}
Z.~Huang and Y.~Gong, ``{Differentially Private ADMM for Convex Distributed
  Learning: Improved Accuracy via Multi-step Approximation},'' \emph{arXiv
  preprint arXiv:2005.07890}, 2020.

\bibitem{WangYWL-ICDCS25}
L.~Wang, Y.~Yuan, C.~Wang, and F.~Li, ``Pdsl: Privacy-preserved decentralized
  stochastic learning with heterogeneous data distribution,'' in \emph{2025
  IEEE 45th ICDCS}.\hskip 1em plus 0.5em minus 0.4em\relax IEEE Computer
  Society, 2025, pp. 736--746.

\bibitem{ZhangCHWY-ICML22}
X.~Zhang, X.~Chen, M.~Hong, S.~Wu, and J.~Yi, ``{Understanding Clipping for
  Federated Learning: Convergence and Client-Level Differential Privacy},'' in
  \emph{Proc. of the 39th ICML}, vol. 162, 2022, pp. 26\,048--26\,067.

\bibitem{KasiviswanathanLNRS-FOCS08}
S.~Kasiviswanathan, H.~Lee, K.~Nissim, S.~Raskhodnikova, and A.~Smith, ``{What
  Can We Learn Privately?}'' in \emph{Proc. of the 49th IEEE FOCS}, 2008, p.
  531–540.

\bibitem{AbadiCGMMTZ-CCS16}
M.~Abadi, A.~Chu, I.~Goodfellow, H.~McMahan, I.~Mironov, K.~Talwar, and
  L.~Zhang, ``{Deep Learning with Differential Privacy},'' in \emph{Proc. of
  the 2016 ACM CCS}, 2016, p. 308–318.

\bibitem{GeipingBDM-NIPS20}
J.~Geiping, H.~Bauermeister, H.~Dr\"{o}ge, and M.~Moeller, ``{Inverting
  gradients - how easy is it to break privacy in federated learning?}'' in
  \emph{Proceedings of the 34th NIPS}, 2020.

\bibitem{FredriksonJR-CCS15}
M.~Fredrikson, S.~Jha, and T.~Ristenpart, ``{Model Inversion Attacks that
  Exploit Confidence Information and Basic Countermeasures},'' in \emph{Proc.
  of the 22nd ACM CCS}, 2015, p. 1322–1333.

\bibitem{Qian-NN19}
N.~Qian, ``On the momentum term in gradient descent learning algorithms,''
  \emph{Neural networks}, vol.~12, no.~1, pp. 145--151, 1999.

\bibitem{LecunBBH-IEEE98}
Y.~Lecun, L.~Bottou, Y.~Bengio, and P.~Haffner, ``Gradient-based learning
  applied to document recognition,'' \emph{Proc. of the IEEE}, vol.~86, no.~11,
  pp. 2278--2324, 1998.

\bibitem{Krizhevsky-09}
A.~Krizhevsky, ``Learning multiple layers of features from tiny images,'' 2009.

\bibitem{SimonyanZ-ARXIV14}
K.~Simonyan and A.~Zisserman, ``{Very Deep Convolutional Networks for
  Large-Scale Image Recognition},'' \emph{arXiv preprint arXiv:1409.1556},
  2014.

\bibitem{ZhuHZ-ICML21}
Z.~Zhu, J.~Hong, and J.~Zhou, ``{Data-Free Knowledge Distillation for
  Heterogeneous Federated Learning},'' in \emph{Proc. of the 38th ICML}, vol.
  139, 2021, pp. 12\,878--12\,889.

\bibitem{WeiLDMYFJQP-TIFS20}
K.~Wei, J.~Li, M.~Ding, C.~Ma, H.~Yang, F.~Farokhi, S.~Jin, T.~Quek, and H.~V.
  Poor, ``{Federated Learning With Differential Privacy: Algorithms and
  Performance Analysis},'' \emph{IEEE Trans. on Information Forensics and
  Security}, vol.~15, pp. 3454--3469, 2020.

\bibitem{ChenYZYC-TII22}
S.~Chen, D.~Yu, Y.~Zou, J.~Yu, and X.~Cheng, ``{Decentralized Wireless
  Federated Learning With Differential Privacy},'' \emph{IEEE Trans. on
  Industrial Informatics}, vol.~18, no.~9, pp. 6273--6282, 2022.

\bibitem{YuZCTTLC-TVV21}
D.~Yu, Z.~Zou, S.~Chen, Y.~Tao, B.~Tian, W.~Lv, and X.~Cheng, ``Decentralized
  parallel sgd with privacy preservation in vehicular networks,'' \emph{IEEE
  Transactions on Vehicular Technology}, vol.~70, no.~6, pp. 5211--5220, 2021.

\bibitem{WangSB-SSC03}
Z.~Wang, E.~Simoncelli, and A.~Bovik, ``{Multiscale Structural Similarity for
  Image Quality Assessment},'' in \emph{Proc. of the 37th ACSSC}, vol.~2, 2003,
  pp. 1398--1402.

\end{thebibliography}

\onecolumn

\appendices

\section{Proof of Theorem~\ref{thm:dp}} \label{sec:proofdp}
  In this section, we prove the efficacy of our DPDL algorithm in preserving differential privacy. In the decentralized learning system, each agent collaborates with its peers by exchanging private information (e.g.,  self-gradient and cross-gradient information related to its local data in our case); nevertheless, since the trustworthiness of the agents cannot be ensure, each agent has to preserve its privacy in the collaboration. For example, as we have shown in Sec.~\ref{ssec:expattacks}, an adversary can use gradient inversion technique to reconstruct the local data of any agent through its gradient information (e.g., $\ddot{g}^{ji}_{t}$ or $\tilde{g}^i_t$) by eavesdropping on the communications to its peers. Therefore, in this proof, we will show that the differential privacy of each agent can be ensured in fixed communication rounds by the random perturbation based on Gaussian distribution.

  The proof of the main theorem relies on the mechanism we use for perturbing and the variance of Gaussian noises. As illustrated in \textbf{Algorithm}~\ref{alg:dpdl}, each agent first clips the gradient of every individual data sample in the batch and then perturbs the averaged gradient over the batch. Thus, we first present \textbf{Lemma}~\ref{lem:sampled_gaussian_mgf}, an extension of \textbf{Lemma}~3 in [39], to establish a moment bound for aforementioned batch-level perturbation mechanism. We then compute moment bound of each $\tilde{g}^i_t$ based on its $l_2$-sensitivity. Next, we present \textbf{Lemma}~\ref{lem:momentmapping} to show inherit relationship of moments between $\tilde{g}^i_t$, $\tilde{v}^i_t$ and $\tilde{x}^i_t$. Finally, we use \textbf{Lemma}~\ref{lem:moment_dp} to bridge the gap between moment and DP, carrying out \textbf{Theorem}~\ref{thm:dp}.

  \begin{lem}\label{lem:sampled_gaussian_mgf}
    Consider a randomized mechanism $M(\mathcal{B}) = \sum_{\zeta_b \in \mathcal{B}} f(\zeta_b) + \mathcal{N}(0, \sigma^2 \mathbf{I})$, where $\mathcal{B}$ is the data set input with $B$ data samples where each sample $\zeta_b$ is selected with probability $q < \frac{1}{32\sigma}$, and $f \colon \mathbb{D} \rightarrow \mathbb{R}^d$ is a randomized function applied to data sample $\zeta_b$. Let $\sigma \geq c > 1$ and $l_2$-sensitivity $\Delta_2 \left( \sum_{\zeta_b \in \mathcal{B}} f(\zeta_b)\right) \leq c$. For any positive integer $\lambda<\frac{\sigma^2}{c^2} \ln \frac{1}{4q\sigma} - 1$, $M(\mathcal{B})$ satisfies
    \begin{align*}
      \alpha_{M}(\lambda) &\leq \frac{c^2 q^2\lambda(\lambda+1)} {(1-q)\sigma^2} + O \left( c^3q^3\lambda^3/\sigma^3 \right).
    \end{align*}
  \end{lem}
  \begin{proof}
    Supposing that two adjacent data sets $\mathcal{B}$ and ${\mathcal{B}}^{\prime}$ only differs from the $m$-th sample, i.e., $\zeta_b = \zeta_b^{\prime}, \forall b \neq m$. Without loss of generality, assume $f(\zeta_m) = \textbf{c} $, $f(\zeta_m^{\prime})=\textbf{0}$ and that the contributions from all other samples are zero, i.e., $\sum_{b \neq m}f(\zeta_b) = \sum_{b \neq m}f(\zeta_b^{\prime}) = \textbf{0}$. Under this assumption, the outputs of the mechanism $M$ on the adjacent data sets $\mathcal{B}$ and ${\mathcal{B}}^{\prime}$ can be explicitly expressed using their respective probability density functions: 
  \begin{align*}
    M(\mathcal{B}) &\sim \mu_0,\\
    M({\mathcal{B}}^{\prime}) &\sim \mu \stackrel{\Delta}{=} (1-q)\mu_0 + q\mu_1
  \end{align*}
  where $\mu_0 = \mathcal{N}(0, \sigma^2)$ and $\mu_1 = \mathcal{N}(c, \sigma^2)$. According to \textbf{Definition}~\ref{def:lambdamoments}, upper bounding $\alpha_{\mathcal{M}}(\lambda)$ reduces to bounding both $\mathbb{E}_{z \sim \mu} \left[ \left( \frac{\mu(z)}{\mu_0(z)} \right)^\lambda \right]$ and $\mathbb{E}_{z \sim \mu_0} \left[ \left( \frac{\mu_0(z)}{ \mu(z)}\right)^\lambda \right]$. Without loss of generality, we focus on bounding the second expectation, i.e., $\mathbb{E}_{z \sim \mu_0} \left[ \left( \frac{\mu_0(z)}{ \mu(z)}\right)^\lambda \right]$, as the procedure for the first one is analogous. By applying a binomial expansion to the integrand, we obtain:
  \begin{align}\label{eq:moment}
    \mathbb{E}_{z \sim \mu_0} \left[ \left( \frac{\mu_0(z)}{ \mu(z)}\right)^\lambda \right]
    = & \mathbb{E}_{z \sim \mu} \left[ \left( \frac{\mu_0(z)}{ \mu(z)}\right)^{\lambda+1} \right] = \mathbb{E}_{z \sim \mu} \left[ \left( 1 + \frac{\mu_0(z)-\mu(z)}{\mu(z)} \right)^{\lambda+1} \right] \nonumber\\
    = &\sum_{t=0}^{\lambda+1} {\lambda+1\choose t} \mathbb{E}_{z \sim \mu} \left[ \left(\frac{\mu_0(z)-\mu(z)}{\mu(z)} \right)^t \right]  \nonumber\\
    = & 1 + {{1+\lambda} \choose 1} \mathbb{E}_{z \sim \mu}\left[\frac{\mu_0(z) - \mu(z)}{\mu(z)}\right] + {{1+\lambda} \choose 2} \mathbb{E}_{z \in \mu} \left[\left(\frac{\mu_0(z) - \mu(z)}{\mu(z)}\right)^2\right] \nonumber\\
    & + \sum_{t=3}^{\lambda+1} {\lambda+1\choose t} \mathbb{E}_{z \sim \mu} \left[ \left(\frac{\mu_0(z)-\mu(z)}{\mu(z)} \right)^t \right] \,.
  \end{align}
  Moreover, under our assumption on the sampling probability, we have $\frac{\mu_0(z)-\mu(z)}{\mu(z)} \leq \frac{q}{1-q} < 1$ which ensures the convergence of the binomial expansion and enables high-order term control in the series. Next, we bound the second term in (\ref{eq:moment}):
  \begin{align}\label{eq:moment_second}
    {{1+\lambda} \choose 1} \mathbb{E}_{z \sim \mu}\left[\frac{\mu_0(z) - \mu(z)}{\mu(z)}\right] 
    &= {{1+\lambda} \choose 1} \int_{-\infty}^\infty \mu(z)\frac{\mu_0(z) - \mu(z)}{\mu(z)}  dz\\ \nonumber
    &= {{1+\lambda} \choose 1} \left[ \int_{-\infty}^\infty \mu_0(z) dz - \int_{-\infty}^\infty \mu(z) dz \right] = 0
  \end{align}
  And the third term in (\ref{eq:moment}) can be wrote as:
  \begin{align} \label{eq:moment_third}
    & {{1+\lambda} \choose 2} \mathbb{E}_{z \sim \mu}\left[\left(\frac{\mu_0(z) - \mu(z)}{\mu(z)} \right)^2\right] \nonumber\\ 
    =& \frac{\lambda(\lambda+1)q^{2}}{2} \mathbb{E}_{z \sim \mu} \left[\left(\frac{\mu_0(z)-\mu_1(z)}{\mu(z)} \right)^2\right] = \frac{\lambda(\lambda+1)q^{2}}{2} \int_{-\infty}^\infty \frac{(\mu_0(z) - \mu_1(z))^{2}}{\mu(z)} dz \nonumber\\ 
    \overset{\tcircle{1}}{\leq} & \frac{\lambda(\lambda+1)q^{2}}{2(1-q)} \int_{-\infty}^\infty \frac{(\mu_0(z) - \mu_1(z))^{2}}{\mu_{0}(z)} dz  = \frac{\lambda(\lambda+1)q^{2}}{2(1-q)} \mathbb{E}_{z \sim \mu_0}\left[\left(\frac{\mu_0(z) - \mu_1(z)}{\mu_0(z)}\right)^2\right] \nonumber\\ 
    \overset{\tcircle{2}}{=} & \frac{\lambda(\lambda+1)q^{2}}{2(1-q)} \mathbb{E}_{z \sim \mu_0} \left[\left(1- \exp\left( \frac{2cz-c^2}{2\sigma^2} \right) \right)^2 \right] \nonumber\\ 
    = & \frac{\lambda(\lambda+1)q^{2}}{2(1-q)} \left( 1 - 2\mathbb{E}_{z \sim \mu_0}\left[\exp\left( \frac{2cz-c^2}{2\sigma^2} \right) \right] + \mathbb{E}_{z \sim \mu_0} \left[\exp \left(\frac{4cz-2c^2}{2\sigma^2} \right)\right] \right) \nonumber\\ 
    \overset{\tcircle{3}}{=} & \frac{\lambda(\lambda+1)q^{2}}{2(1-q)} \left( 1 - 2\exp\left(\frac{c^2}{2\sigma^{2}}\right) \cdot \exp\left( \frac{-c^2}{2\sigma^{2}}\right) + \exp\left( \frac{4c^2}{2\sigma^{2}}\right) \cdot \exp\left(\frac{-c^2}{\sigma^{2}} \right) \right) \nonumber\\ 
    = & \frac{\lambda(\lambda+1) q^{2}}{2(1-q)} \cdot \left(\exp\left(\frac{c^2}{\sigma^2}\right) - 1\right) \leq \frac{\lambda(\lambda+1) c^2q^{2}}{(1-q)\sigma^{2}}.
  \end{align}
  where $\tcircle{1}$ comes from the fact of $\mu(z) \geq (1-q)\mu_0(z)$. $\tcircle{2}$ is based on $\mathbb{E}_{z \sim \mu_0} \exp\left(\frac{2az}{2\sigma^2}\right) = \exp\left(\frac{a^2}{2\sigma^2}\right)$ for any $a \in \mathbb{R}^d$. $\tcircle{3}$ is derived from the properties of normal distribution.
  %
  %
  To bound the remaining terms, we change $\left| \mu_0(z) - \mu_1(z) \right|$ into the form of $\mu_0(z)$ and $\mu_1(z)$:
  \begin{align*}
    \left| \mu_0(z) - \mu_1(z) \right| &= \left| \frac{1}{\sqrt{2\pi}\sigma} e^{-\frac{z^2}{2\sigma^2}}- \frac{1}{\sqrt{2\pi}\sigma} e^{-\frac{(z-c)^2}{2\sigma^2}} \right| = \frac{1}{\sqrt{2\pi}\sigma} \left|e^{\frac{-z^2}{2\sigma^2}}- e^{\frac{-z^2 + 2cz - c^2}{2\sigma^2}} \right| \\
    &=\frac{1}{\sqrt{2\pi}\sigma} e^{\frac{-z^2}{2\sigma^2}} \left|1- e^{\frac{2cz - c^2}{2\sigma^2}} \right| = \frac{1}{\sqrt{2\pi}\sigma} e^{\frac{-z^2 + 2cz - c^2}{2\sigma^2}} \left|e^{\frac{c^2 - 2cz}{2\sigma^2}} - 1 \right|\\
    &=\mu_0(z)\left|1- e^{\frac{2cz - c^2}{2\sigma^2}} \right| = \mu_1(z) \left|e^{\frac{c^2 - 2cz}{2\sigma^2}} - 1 \right|
  \end{align*}
  We then divide $\left| \mu_0(z) - \mu_1(z) \right|$ into multiple segments. $\forall z \leq 0$, we have:
  \begin{align*}
    \left|\mu_0(z) - \mu_1(z) \right| 
    = \mu_0(z) \left(1- e^{\frac{2cz - c^2}{2\sigma^2}} \right) 
    \overset{\tcircle{1}}{\leq} \mu_0(z) \left(\frac{c^2 - 2cz}{2\sigma^2} \right) 
    \leq \frac{c^2 - cz}{\sigma^2}\mu_0(z)
  \end{align*}
  where we have $\tcircle{1}$ since $e^z \geq 1+z \Rightarrow 1-e^z\leq -z$. Using the same method, $\forall z \geq c$ we can get:
  \begin{align*}
    \left|\mu_0(z) - \mu_1(z) \right| 
    = \mu_0(z) \left(e^{\frac{2cz-c^2}{2\sigma^2}}-1 \right)
    \overset{\tcircle{1}}{\leq} \mu_0(z) \left( \frac{2cz-c^2}{2\sigma^2} \right) 
    \leq \frac{cz}{\sigma^2}\mu_0(z)
  \end{align*}
  where $\tcircle{1}$ comes from $e^z -1 \geq z$. And $\forall 0 \leq z \leq c$:
  \begin{align*}
    |\mu_0(z) - \mu_1(z)| \overset{\tcircle{1}}{\leq} \mu_0(z) \left( e^{\frac{c^2}{2\sigma^2}} - 1 \right) \leq \frac{c^2}{\sigma^2}\mu_0(z)
  \end{align*}
  where $\tcircle{1}$ is derived that the maximum value of $\left| 1- e^{\frac{2cz - c^2}{2\sigma^2}} \right|$ in the interval $[0,c]$ is obtained at $z=c$. Based on the above inequality, we can then rewrite
  \begin{align}\label{eq:moment_remain}
    \mathbb{E}_{z \sim \mu}\left[\left(\frac{\mu_0(z) - \mu(z)}{\mu(z)} \right)^t\right] \leq &\int_{-\infty}^0 \mu(z) \left|\left(\frac{\mu_0(z) - \mu(z)}{\mu(z)} \right)^t\right| dz \nonumber\\ 
    &+ \int_{0}^c \mu(z) \left|\left(\frac{\mu_0(z) - \mu(z)}{\mu(z)} \right)^t\right| dz + \int_c^{\infty} \mu(z) \left|\left(\frac{\mu_0(z) - \mu(z)}{\mu(z)}\right)^t\right|dz.
  \end{align}
  The first term in (\ref{eq:moment_remain}) can be bounded by
  \begin{align*}
    &\int_{-\infty}^0 \mu(z) \left|\left(\frac{\mu_0(z) - \mu(z)}{\mu(z)}\right)^t\right| dz \\
    =&q^t\int_{-\infty}^0 \mu(z) \left|\left(\frac{\mu_0(z) - \mu_1(z)}{\mu(z)}\right)^t\right| dz
    \leq \frac{q^t}{(1-q)^{t-1}} \int_{-\infty}^{0} \mu_0(z) \frac{|\mu_0(z) - \mu_1(z)|^t}{{\mu_0(z)}^t} dz\\
    \leq& \frac{c^tq^t}{(1-q)^{t-1}\sigma^{2t}} \int_{-\infty}^{0} \mu_0(z) |z-c|^t dz 
    =\frac{c^tq^t}{(1-q)^{t-1}\sigma^{2t}} \sum_{k=0}^t {t \choose k} c^{t-k} \int_{-\infty}^{0} \mu_0(z) (-z)^k dz\\
    \overset{\tcircle{1}}{=}& \frac{c^tq^t}{2(1-q)^{t-1}\sigma^{2t}} \sum_{k=0}^t {t \choose k} c^{t-k} \int_{-\infty}^{+\infty} \mu_0(z) |z|^k dz \\
    \overset{\tcircle{2}}{=}&\frac{c^tq^t}{2(1-q)^{t-1}\sigma^{2t}} \sum_{k=0}^t {t \choose k} c^{t-k} \sigma^{k} (k-1)!! 
    \leq \frac{c^tq^t}{2(1-q)^{t-1}\sigma^{2t}} \sigma^{t} (t-1)!! \sum_{k=0}^t {t \choose k} c^{t-k} \\
    \leq & \frac{q^t(c+1)^{2t} (t-1)!! }{2(1-q)^{t-1}\sigma^t}.
  \end{align*}
  where $\tcircle{1}$ comes from the symmetry of integral that $\int_{-\infty}^{0} \mu_0(z) (-z)^k dz = \frac{1}{2} \int_{-\infty}^{+\infty} \mu_0(z) |z|^k dz$. $\tcircle{2}$ is derived from $\int_{-\infty}^{+\infty} \mu_0(z) |z|^k dz \leq \sigma^{t} (t-1)!!$. The second term in (\ref{eq:moment_remain}) is at most
  \begin{align*}
    \int_{0}^c \mu(z) \left|\left(\frac{\mu_0(z)- \mu(z)}{\mu(z)} \right)^t\right| dz 
    = &\frac{q^t}{(1-q)^t} \int_{0}^c \mu(z) \left|\left(\frac{\mu_0(z)- \mu_1(z)}{\mu_0(z)}\right)^t\right| dz \\
    \leq &\frac{q^t}{(1-q)^t}\int_0^{c} \mu(z) \frac{c^{2t}}{\sigma^{2t}} dz
    \leq \frac{c^{2t}q^t}{(1-q)^t \sigma^{2t}}.
  \end{align*}
  Similarly, the third term in (\ref{eq:moment_remain}) has the upper bound:
  \begin{align*}
    &\int_c^{\infty} \mu(z) \left|\left(\frac{\mu_0(z) - \mu(z)}{\mu(z)}\right)^t\right|dz \\
    \leq &\frac{q^t}{(1-q)^{t-1}} \int_c^{\infty} \mu_{0}(z) \left(\frac{|\mu_0(z) - \mu_1(z)|}{\mu_{0}(z)}\right)^t dz \\
    \leq &\frac{c^tq^t}{(1-q)^{t-1}\sigma^{2t}} \int_c^{\infty} \mu_{0}(z) \left(\frac{z\mu_1(z)}{\mu_{0}(z)}\right)^t dz    
    \leq \frac{c^tq^t}{(1-q)^{t-1}\sigma^{2t}}\int_c^{\infty} \mu_{0}(z) e^{\frac{2tcz - tc^2}{2\sigma^2}} z^t dz \\
    \leq &\frac{c^tq^t e^{(c^2t^2-c^2t)/2\sigma^2}}{(1-q)^{t-1}\sigma^{2t}}\int_{c}^\infty \mu_0(z-ct) z^t dz
    = \frac{c^tq^t e^{(c^2t^2-c^2t)/2\sigma^2}}{(1-q)^{t-1}\sigma^{2t}}\int_{c-ct}^\infty \mu_0(z) (z+ct)^t dz\\
    \leq &\frac{c^tq^t e^{(c^2t^2-c^2t)/2\sigma^2}}{(1-q)^{t-1}\sigma^{2t}} \int_{0}^\infty \mu_0(z) (z+ct)^t dz\\
    \overset{\tcircle{1}}{\leq} &\frac{c^tq^t e^{(c^2t^2-c^2t)/2\sigma^2}}{(1-q)^{t-1}\sigma^{2t}} \int_{0}^\infty 2^{t-1} \mu_0(z) (z^t+(ct)^t) dz\\
    \leq &\frac{(2cq)^t e^{(c^2t^2-c^2t)/2\sigma^2}}{2(1-q)^{t-1}\sigma^{2t}} \left[ \int_{-\infty}^{+\infty} \mu_0(z) |z|^k dz + (ct)^t \int_{-\infty}^{+\infty} \mu_0(z) dz \right]\\
    \leq &\frac{(2cq)^t e^{(c^2t^2-c^2t)/2\sigma^2}) (\sigma^t(t-1)!! + (ct)^t)}{2(1-q)^{t-1}\sigma^{2t}}.
  \end{align*}
  where $\tcircle{1}$ comes from $(a+b)^t \leq 2^{t-1}(a^t + b^t),\; \forall t \geq 1$. Combining the sum of three terms, we can get the $t$ remaining term of (\ref{eq:moment}):
  \begin{align*}
    &{\lambda+1\choose t} \mathbb{E}_{z \sim \mu} \left[ \left( \frac{\mu_0(z)-\mu(z)}{\mu(z)} \right)^t \right]\\
    \leq &{\lambda+1\choose t} \left[ \frac{q^t(c+1)^{2t} (t-1)!! }{2(1-q)^{t-1}\sigma^t} + \frac{c^{2t}q^t}{(1-q)^t \sigma^{2t}} + \frac{(2cq)^t e^{\frac{c^2t^2-c^2t}{2\sigma^2}} (\sigma^t(t-1)!! + (ct)^t)}{2(1-q)^{t-1}\sigma^{2t}} \right] \\
    \leq& {\lambda+1\choose t} \cdot \frac{((c+1)^2q\sigma)^t(t-1)!! + (c^2q)^t + (2cq\sigma)^t e^{\frac{c^2t^2-c^2t}{2\sigma^2}}(t-1)!! + (2c^2qt)^t e^{\frac{c^2t^2-c^2t}{2\sigma^2}}}{2(1-q)^{t}\sigma^{2t}}\\
    \leq & {\lambda+1\choose t} \cdot \frac{(2c^2q\sigma t)^t e^{(c^2t^2-c^2t)/2\sigma^2}}{2(1-q)^{t}\sigma^{2t}} = \Gamma(t)
  \end{align*}
  Hence, we have
  \begin{align*}
    \frac{\Gamma(t+1)}{\Gamma(t)} 
    =& \frac{{\lambda+1\choose t+1}}{{\lambda+1\choose t}} \frac{(2c^2q\sigma (t+1))^{t+1} e^{(c^2(t+1)^2-c^2(t+1))/2\sigma^2}}{2(1-q)^{t+1}\sigma^{2(t+1)}} \frac{2(1-q)^{t}\sigma^{2t}}{(2c^2q\sigma t)^t e^{(c^2t^2-c^2t)/2\sigma^2}} \\
    =& \frac{\lambda-t+1}{t+1} \cdot \frac{2c^2q}{\sigma} \cdot \frac{(t+1)^{t+1}}{t^t} \cdot e^{\frac{c^2t}{\sigma^2}} 
    = (\lambda-t+1) \cdot \frac{2c^2q}{\sigma} \cdot (1+\frac{1}{t})^{t} \cdot e^{\frac{c^2t}{\sigma^2}} \\
    \overset{\tcircle{1}}{\leq}& (\lambda-t+1) \cdot \frac{2c^2q}{\sigma} \cdot e^{\frac{c^2t}{\sigma^2}+1} 
    \overset{\tcircle{2}}{\leq} 2q\sigma e^{\frac{c^2(\lambda+1)}{\sigma^2}}
    \overset{\tcircle{3}}{\leq} \frac{1}{2}
  \end{align*}
  where $\tcircle{1}$ holds for $(1+\frac{1}{t})^{t} \leq e$. $\tcircle{2}$ is derived that the derivation of $f(t) = (\lambda-t+1) \cdot \frac{2c^2q}{\sigma} \cdot e^{\frac{c^2t}{\sigma^2}+1}$ equals to $f'(t)= \frac{2c^2q}{\sigma}(\frac{c^2(\lambda-t+1)}{\sigma^2}-1)e^{\frac{c^2t}{\sigma^2}+1}$, thus $f(t)$ has maximum value when $t=\lambda+1-\frac{\sigma^2}{c^2}$. $\tcircle{3}$ is satisfied since assumption $q<\frac{1}{32\sigma}$ and $\lambda<\frac{\sigma^2}{c^2}ln\frac{1}{4q\sigma} - 1$. In short, the $t$ remaining term of (\ref{eq:moment}) decreases exponentially for $t>3$, thus the remaining terms is dominated by the $t=3$ term:
  \begin{align} \label{eq:moment_o}
    & {\lambda+1\choose 3} \mathbb{E}_{z \sim \mu} \left[ \left(\frac{\mu_0(z)-\mu(z)}{\mu(z)} \right)^3 \right] \nonumber\\ 
    \leq &\frac{\lambda^3 - \lambda}{6} \cdot \frac{8q^3(c+1)^3\sigma^3 + 2c^6q^3 + 8c^3q^3e^{3c^2/\sigma^2}(2\sigma^3 + 27c^3 )}{2(1-q)^3\sigma^6} \nonumber\\
    =& O(q^3c^3\lambda^3/\sigma^3)
  \end{align}
  By combining (\ref{eq:moment}), (\ref{eq:moment_second}), (\ref{eq:moment_third}) and (\ref{eq:moment_o}), we complete the proof.
  \end{proof}

  Next we calculate the upper bound of $\alpha_{\tilde{g}^i_t}(\lambda)$, i.e., $\lambda$-th moment of $\tilde{g}^i_t$ in our DPDL. As shown in Lines~\ref{ln:rxpertgrad}$\sim$\ref{ln:agggrad} in \textbf{Algorithm}~\ref{alg:dpdl}, we rewrite $\tilde{g}^i_t$ as
  \begin{align} \label{eq:divide_agggrad}
    \tilde{g}^i_t &= \sum_{j \in \mathcal{N}_i} \left( \frac{1}{\sqrt{w_{ij}}N} \ddot{g}^{ij}_{t} + \alpha w_{ij} {\mathcal{C}}^{ij}_t \hat{g}^{ii}_{t} \right) \nonumber \\
    &= \frac{1}{B} \sum_b \sum_{j \in \mathcal{N}_i} \left( \frac{1}{\sqrt{w_{ij}}N} \hat{g}^{ij}_{t,b} + \alpha w_{ij} {\mathcal{C}}^{ij}_t \hat{g}^{ii}_{t,b} \right) + \frac{1}{B} \sum_{j \in \mathcal{N}_i} \frac{1}{\sqrt{w_{ij}}N}{z^{ij}_t}
  \end{align}
  Compared with the content in \textbf{Lemma}~\ref{lem:sampled_gaussian_mgf}, we turn the problem of calculating $\alpha_{\tilde{g}^i_t}(\lambda)$ into bounding the noise level of $\frac{1}{B} \sum_{j \in \mathcal{N}_i} \frac{1}{\sqrt{w_{ij}}N}{z^{ij}_t}$ and $l_2$ sensitivity of $\frac{1}{B} \sum_b \sum_{j \in \mathcal{N}_i} \left( \frac{1}{\sqrt{w_{ij}}N} \hat{g}^{ij}_{t,b} + \alpha w_{ij} {\mathcal{C}}^{ij}_t \hat{g}^{ii}_{t,b} \right)$, i.e. $\Delta_2 \phi_i$. Since the Gaussian noise $z^{ij}_t$ added at each agent $j \in \mathcal{N}_i$ is independently and identically distributed, the linear combination of $\{z^{ij}_t\}_{j \in \mathcal{N}_i}$, i.e., $z$, is also a Gaussian random variable such that
  \begin{align}\label{eq:noise_level}
    \mathbb{E}[z] &= \mathbb{E} \left[\frac{1}{B} \sum_{j \in \mathcal{N}_i} \frac{1}{\sqrt{w_{ij}}N} z^{ij}_t \right] = 0  \nonumber \\
    \mathrm{Var}[z] &= \mathrm{Var} \left[\frac{1}{B} \sum_{j \in \mathcal{N}_i} \frac{1}{\sqrt{w_{ij}}N} z^{ij}_t \right] = \frac{1}{B^2N^2} \sum_{j \in \mathcal{N}_i} \frac{1}{w_{ij}} \mathrm{Var}(z^{ij}_t) = \frac{C^2}{B^2N^2} \sum_{j \in \mathcal{N}_i} \frac{1}{w_{ij}} \sigma^2
  \end{align}
  thus $z \sim \mathcal{N}\left(0,\frac{C^2}{B^2N^2} \sum_{j \in \mathcal{N}_i} \frac{1}{w_{ij}} \sigma^2 \mathbf{I}\right)$. Then we turn to $\Delta_2 \phi$, specifically:
  \begin{align} \label{eq:seni_pre}
    \Delta_2 \phi =& \frac{1}{B} \max_{\mathcal{B}^i_t,{\mathcal{B}^i_t}^{\prime}} \left\| \sum_b  \sum_{j \in \mathcal{N}_i} \left( \frac{1}{\sqrt{w_{ij}}N} \hat{g}^{ij}_{t,b} + \alpha w_{ij} {\mathcal{C}}^{ij}_t \hat{g}^{ii}_{t,b} - \frac{1}{\sqrt{w_{ij}}N} \left( \hat{g}^{ij}_{t,b} \right)^{\prime} - \alpha w_{ij} \left({\mathcal{C}}^{ij}_t\right)^{\prime} \left( \hat{g}^{ii}_{t,b} \right)^{\prime} \right) \right\| \nonumber \\
    =& \frac{1}{B} \max_{\mathcal{B}^i_t,{\mathcal{B}^i_t}^{\prime}} \left\| \sum_b  \sum_{j \in \mathcal{N}_i} \frac{1}{\sqrt{w_{ij}}N} \left( \hat{g}^{ij}_{t,b} - \left( \hat{g}^{ij}_{t,b} \right)^{\prime} \right) \right\| \nonumber \\
    & + \max_{\mathcal{B}^i_t,{\mathcal{B}^i_t}^{\prime}} \left\| \sum_{j \in \mathcal{N}_i} \alpha w_{ij} \left( \frac{1}{B} \sum_b {\mathcal{C}}^{ij}_t \hat{g}^{ii}_{t,b} - \frac{1}{B} \sum_b \left({\mathcal{C}}^{ij}_t \right)^{\prime} \left( \hat{g}^{ii}_{t,b} \right)^{\prime} \right) \right\|
  \end{align}
  The first term of (\ref{eq:seni_pre}) can be proved as below:
  \begin{align}\label{eq:sensiw}
      \frac{1}{B} \max_{\mathcal{B}^i_t,{\mathcal{B}^i_t}^{\prime}} \left\| \sum_b  \sum_{j \in \mathcal{N}_i} \frac{1}{\sqrt{w_{ij}}N} \left( \hat{g}^{ij}_{t,b} - \left( \hat{g}^{ij}_{t,b} \right)^{\prime} \right) \right\| &\overset{\tcircle{1}}{\leq} \frac{1}{B} \max_{\mathcal{B}^i_t,{\mathcal{B}^i_t}^{\prime}} \left\| \sum_b  \frac{1}{\sqrt{w_{ii}}N} \left( \hat{g}^{ij}_{t,b} - \left( \hat{g}^{ij}_{t,b} \right)^{\prime} \right) \right\| \nonumber \\
      &\leq \frac{2C}{B\sqrt{w_{ii}}N}
  \end{align}
  where $\tcircle{1}$ follows the fact that $\hat{g}^{ij}_{t,b} = \left( \hat{g}^{ij}_{t,b} \right)^{\prime}, \forall i \neq j $. The second term of (\ref{eq:seni_pre}) can be proved as below:
  \begin{align} \label{eq:sensi_alpha}
    & \max_{\mathcal{B}^i_t,{\mathcal{B}^i_t}^{\prime}} \left\| \sum_{j \in \mathcal{N}_i} \alpha w_{ij} \left( \frac{1}{B} \sum_b {\mathcal{C}}^{ij}_t \hat{g}^{ii}_{t,b} - \frac{1}{B} \sum_b \left({\mathcal{C}}^{ij}_t \right)^{\prime} \left( \hat{g}^{ii}_{t,b} \right)^{\prime} \right) \right\| \nonumber \\
    \leq &  \max_{\mathcal{B}^i_t, {\mathcal{B}^i_t}^{\prime}} \left\| \alpha \sum_{j \in \mathcal{N}_i, j\neq i} w_{ij} \left( \frac{1}{B} \sum_b {\mathcal{C}}^{ij}_t \hat{g}^{ii}_{t,b} - \frac{1}{B} \sum_b \left({\mathcal{C}}^{ij}_t \right)^{\prime} \left( \hat{g}^{ii}_{t,b} \right)^{\prime} \right) \right\| \nonumber \\
    & + \max_{\mathcal{B}^i_t, {\mathcal{B}^i_t}^{\prime}} \left\|  \alpha w_{ii} \left( \frac{1}{B} \sum_b {\mathcal{C}}^{ii}_t \hat{g}^{ii}_{t,b} - \frac{1}{B} \sum_b \left({\mathcal{C}}^{ii}_t \right)^{\prime} \left( \hat{g}^{ii}_{t,b} \right)^{\prime} \right) \right\|
  \end{align}
  Next we focus on processing the terms of (\ref{eq:sensi_alpha}). Before we present the formal proof of (\ref{eq:sensi_alpha}), we list out some lemmas which will be used in the following proof.
  \begin{lem}\label{lem:sigmoidderivation}
    Let $g(x) = \frac{1}{1+e^{S(x,r)}}$ be a real-valued function of vector $x$, where $r \in \mathbb{R}^d$ is a fixed non-zero vector and $S(x,r)$ is the cosine similarity between vector $x$ and $r$. Define the vector-valued function $G(x) = g(x)x$, then for any non-zero random vector $x_1, x_2 \in \mathbb{R}^d$, the following inequality holds:
    \begin{equation} \label{eq:sigmoidderivation}
      \left\| G(x_1) - G(x_2) \right\| \leq \frac{15}{16} \left\| x_1 - x_2 \right\|
    \end{equation}
  \end{lem}
  \begin{proof}
  In the domain where $x \neq 0$, it is straightforward to verify that the function $G(x)$ is continuously differentiable. According to \textit{Differential Mean Value Theorem}, we have the following inequality:
  \begin{equation} \label{eq:dmvt}
      \left\| G(x_1) - G(x_2) \right\| \leq \sup_{\xi \in \text{CH}(x_1, x_2)} \left\| \nabla G(\xi) \right\| \left\| x_1 - x_2 \right\|
  \end{equation}
  where $\nabla G(x)$ represents the Jacobian matrix of $G(x)$ and $\text{CH}(x_1, x_2)$ denotes the convex hull of $x_1$ and $x_2$. To proceed, we focus on upper bounding the norm of $\nabla G(x)$. Applying the product rule of differentiation, we derive $\nabla G(\xi)$ as:
  \begin{align*}
    \nabla G(x) &= \nabla g(x)\cdot x^T + g(x)\mathbf{I}\\
    &= g(x)\mathbf{I} - g(x)\left(1- g(x)\right)\nabla S(x,r) \cdot x^T\\
    &= g(x)\mathbf{I} - g(x)\left(1-g(x)\right)\left[\frac{r}{\|x\| \|r\|} - \frac{(x\cdot r)x}{\|x\|^3 \|r\|} \right] \cdot x^T 
  \end{align*}
  The norm of $\nabla G(x)$ can thus be bounded by:
  \begin{align}\label{eq:nablaG_norm}
    \left\| \nabla G(x) \right\| &\leq \left\| g(x)\mathbf{I} \right\| + \left\| g(x)\left(1-g(x)\right) \right\| \left\| \frac{r}{\|x\| \|r\|} - \frac{(x\cdot r)x}{\|x\|^3 \|r\|} \right\| \cdot \| x^T \|  \nonumber \\ 
    &\overset{\tcircle{1}}{\leq} g(x) + \left( g(x)\left(1-g(x)\right) \right) \left\| \frac{r}{\|x\|} - \frac{(x\cdot r)x}{\|x\|^3 \|r\|}\right\| \cdot \| x \| 
  \end{align}
  where $\tcircle{1}$ is derived from that $g(x)$ is real-valued. Besides, we can prove that
  \begin{align}\label{eq:partialcosine}
    \left\| \frac{r}{\|x\| \|r\|} - \frac{(x\cdot r)x}{\|x\|^3 \|r\|} \right\|^2 &= \frac{\|r\|^2}{\|x\|^2 \|r\|^2} - 2 \frac{(x\cdot r)(x\cdot r)}{\|x\|^4 \|r\|^2} + \frac{(x\cdot r)^2\|x\|^2}{\|x\|^6 \|r\|^2} \nonumber \\
    &= \frac{1}{\|x\|^2} - \frac{(x\cdot r)^2}{\|x\|^4 \|r\|^2} 
    \overset{\tcircle{1}}{=} \frac{1}{\|x\|^2} - \frac{\cos^2\theta}{\|x\|^2}
    =\frac{\sin^2\theta}{\|x\|^2}
  \end{align}
  where $\tcircle{1}$ uses $x\cdot r = \|x\| \|r\| \cos \theta$ with $\theta$ being the angle between $x$ and $r$. Thus, substituting (\ref{eq:partialcosine}) into (\ref{eq:nablaG_norm}), we obtain:
  \begin{align*}
    \left\| \nabla G(x) \right\| &\leq g(x) + \left( g(x)\left(1-g(x)\right) \right) \left\| \frac{r}{\|x\|} - \frac{(x\cdot r)x}{\|x\|^3 \|r\|}\right\| \cdot \| x \| \\
    &= g(x) + \left( g(x)\left(1-g(x)\right) \right) |\sin \theta|\\
    &\overset{\tcircle{1}}{\leq} g(x)\left(2-g(x)\right) \leq \frac{15}{16}
  \end{align*}
  where $\tcircle{1}$ by considering the facts that $\frac{1}{1+e^{x}} \leq \frac{1}{1+e^{-1}} \leq \frac{3}{4}$.
  \end{proof}

  \begin{lem}\label{lem:sigmoidderivation2}
    Let $g(x) = \frac{1}{1+e^{S(x,x+r)}}$ be a real-valued function of vector $x$, where $r \in \mathbb{R}^d$ is a fixed non-zero vector and $S(x,x+r)$ is the cosine similarity between $x$ and $x+r$. Define the vector-valued function $G(x) = g(x)x$. Suppose that $h = \max \frac{\left\| x \right\|}{\left\| x+r \right\|} + 1 $, then for any non-zero random vector $x_1, x_2\in \mathbb{R}^d$, we have
    \begin{equation} \label{eq:sigmoidderivation2}
      \left\| G(x_1) - G(x_2) \right\| \leq \frac{(h+1)^2}{4h} \left\| x_1 - x_2 \right\|
    \end{equation}
  \end{lem}
  
  \begin{proof}
  In the domain of $x$, we can easily prove that $G(x)$ is continuously differentiable. Using the same trick as (\ref{eq:dmvt}) we have the following relationship:
  \begin{align*}
    \nabla G(x) &= \nabla g(x)\cdot x^T + g(x)\mathbf{I}\\
    &= g(x)\mathbf{I} - g(x)\left(1- g(x)\right)\nabla S(x,r) \cdot x^T
  \end{align*}
  Thus, the norm of $\nabla G(x)$:
  \begin{align}\label{eq:nablaG_norm_2}
    \left\| \nabla G(x) \right\| &\leq \left\| g(x)\mathbf{I} \right\| + \left\| g(x)\left(1-g(x)\right) \right\| \left\| \nabla S(x,x+r) \right\| \cdot \| x^T \|  \nonumber \\
    &\overset{\tcircle{1}}{\leq} g(x) + \left( g(x)\left(1-g(x)\right) \right) \left\| \nabla S(x,x+r) \right\| \cdot \| x \| 
  \end{align}
  where $\tcircle{1}$ is derived from that $g(x)$ is real-valued. By recording $p = \frac{x}{\|x\|}$, $q = \frac{x+r}{\|x+r\|}$, we can prove that
  \begin{align}\label{eq:nablacos_square}
    \left\| \nabla S(x,x+r) \right\|^2 & = \left\| \frac{2x + r}{\|x\| \|x + r\|} - \frac{\left(\|x\|^2 +(x \cdot r) \right) \left( \frac{x}{\|x\|^2} + \frac{x+r}{\|x + r\|^2}\right)}{\|x\| \|x + r\|} \right\|^2 \nonumber \\
    &= \left\| \frac{p}{\|x+r\|} + \frac{q}{\|x\|} - (p \cdot q)\left( \frac{p}{\|x\|} + \frac{q}{\|x+r\|} \right)\right\|^2 \nonumber \\
    & = \frac{1}{\|x+r\|} \left\| p - (p \cdot q) q \right\|^2 + \frac{1}{\|x\|} \left\| q - (p \cdot q) p \right\|^2 \nonumber \\
    &\overset{\tcircle{1}}{\leq} \frac{1}{\|x+r\|} + \frac{1}{\|x\|}
    %
  \end{align}
  where $\tcircle{1}$ comes from $\| p - (p \cdot q) q \| = \| q - (p \cdot q) p \| = \sqrt{1 - (p \cdot q)^2} \leq 1$. Thus, substituting (\ref{eq:nablacos_square}) into (\ref{eq:nablaG_norm_2}), we have
  \begin{align*}
    \left\| \nabla G(x) \right\| &\leq g(x) + \left( g(x)\left(1-g(x)\right) \right) \left( \frac{\|x\|}{\|x+r\|} + 1 \right) \\
    &\leq g(x) + h\left( g(x)\left(1-g(x)\right) \right) = -hg^2(x) + (h+1)g(x) \\
    &\leq \frac{(h+1)^2}{4h}
  \end{align*}
  \end{proof}
  
  Based on \textbf{Lemma}~\ref{lem:sigmoidderivation}, we can bound the first term of (\ref{eq:sensi_alpha}) as below:
  \begin{align}\label{eq:sensisigmoid1}
    & \max_{\mathcal{B}^i_t, {\mathcal{B}^i_t}^{\prime}} \left\| \alpha \sum_{j \in \mathcal{N}_i, j\neq i} w_{ij} \left( \frac{1}{B} \sum_b {\mathcal{C}}^{ij}_t \hat{g}^{ii}_{t,b} - \frac{1}{B} \sum_b \left({\mathcal{C}}^{ij}_t \right)^{\prime} \left( \hat{g}^{ii}_{t,b} \right)^{\prime} \right) \right\| \nonumber \\
    \leq & \alpha \sum_{j \in \mathcal{N}_i, j\neq i} \frac{15 w_{ij}}{16} \left\|  \frac{1}{B} \sum_b  \hat{g}^{ii}_{t,b} - \frac{1}{B} \sum_b \left( \hat{g}^{ii}_{t,b} \right)^{\prime} \right\| \leq \frac{15\alpha (1- w_{ii})C}{8B}
  \end{align}
  Based on \textbf{Lemma}~\ref{lem:sigmoidderivation2}, assuming that $h = \max_{i \in \mathcal{N}, t<T, \zeta^{i}_{t,b} \in \mathcal{B}^{i}_t} \frac{\left\| \frac{1}{B} \sum_b \hat{g}^{ii}_{t,b} \right\|}{\left\| \frac{1}{B} \sum_b \hat{g}^{ii}_{t,b}  + z^{ii}_t \right\|} +1 $, we can bound the second term of (\ref{eq:sensi_alpha}) as below:
  \begin{align}\label{eq:sensisigmoid2}
    & \max_{\mathcal{B}^i_t, {\mathcal{B}^i_t}^{\prime}} \left\| \alpha w_{ii} \left( \frac{1}{B} \sum_b {\mathcal{C}}^{ii}_t \hat{g}^{ii}_{t,b} - \frac{1}{B} \sum_b \left({\mathcal{C}}^{ii}_t \right)^{\prime} \left( \hat{g}^{ii}_{t,b} \right)^{\prime} \right) \right\| \nonumber \\
    \leq& \frac{\alpha w_{ii} (h+1)^2}{4h} \max_{\mathcal{B}^i_t,{\mathcal{B}^i_t}^{\prime}} \left\|  \frac{1}{B} \sum_b  \hat{g}^{ii}_{t,b} - \frac{1}{B} \sum_b \left( \hat{g}^{ii}_{t,b} \right)^{\prime} \right\| \nonumber \\
    \leq& \frac{\alpha w_{ii} (h+1)^2C}{2Bh}
  \end{align}
  Substituting (\ref{eq:sensisigmoid1}) and (\ref{eq:sensisigmoid2}) into (\ref{eq:sensi_alpha}), (\ref{eq:sensi_alpha}) and (\ref{eq:sensiw}) into (\ref{eq:seni_pre}), the upper bound of $L_2$ sensitivity $\Delta_2 \phi_i$ is:
  \begin{equation}\label{eq:seni}
    \Delta_2 \phi_i \leq \frac{2C}{\sqrt{w_{ii}}BN} + \frac{15\alpha (1- w_{ii})C}{8B}+ \frac{\alpha w_{ii} (h+1)^2C}{2Bh}
  \end{equation} 
  Combining (\ref{eq:seni}), (\ref{eq:noise_level}) and \textbf{Lemma}~\ref{lem:sampled_gaussian_mgf} can we get the upper bound of $\alpha_{\tilde{g}^i_t}(\lambda)$ for some explicit constant $c_2$ as follows: 
  \begin{align*}
    \alpha_{\tilde{g}^i_t}(\lambda) &\leq \frac{\left( \Delta_2 \phi_i \right)^2 q^2\lambda(\lambda+1)}{(1-q)\left( \frac{C\sigma}{NB}  \sqrt{\sum_{j \in \mathcal{N}_i} \frac{1}{w_{ij}}} \right)^2} + O\left( \frac{\left( \Delta_2 \phi_i \right)^3q^3\lambda^3}{\left( \frac{C\sigma}{NB}  \sqrt{\sum_{j \in \mathcal{N}_i} \frac{1}{w_{ij}}}\right)^3} \right). \\
    &\leq \frac{\left( \frac{2C}{\sqrt{w_{ii}}BN} + \frac{15\alpha (1- w_{ii})C}{8B}+ \frac{\alpha w_{ii} (h+1)^2C}{2Bh} \right)^2 c_2^2 q^2\lambda^2}{\left( \frac{C\sigma}{NB}  \sqrt{\sum_{j \in \mathcal{N}_i} \frac{1}{w_{ij}}} \right)^2}
  \end{align*}

  Since what is actually transmitted is the updated parameters and momentum in our DPDL algorithm, it is necessary to connect the privacy moment bounds of the data-dependent vector $\tilde{g}^i_t$ to those of $\tilde{x}^i_t$ and $\tilde{v}^i_t$. We formalize this connection by leveraging the post-processing property of moments, summarized in the following lemma:
  \begin{lem}~\label{lem:momentmapping}
  Let $\mathcal{M} \colon \mathbb{D} \rightarrow \mathbb{R}^d$ be a mechanism with outcome $\psi$, $f\colon \mathbb{R}^d \rightarrow \mathbb{R}^{d^{\prime}}$ be an data-independent randomized mapping with outcome ${\psi}^{\prime} = f(\psi)$. Then for any pair of neighboring datasets $\mathcal{D}, \mathcal{D}' \in \mathbb{D}$ and any $\lambda > 0$, the $\lambda$-th moment of $f \circ \mathcal{M} \colon \mathbb{D} \rightarrow \mathbb{R}^{d^{\prime}}$ satisfies:
  \begin{align*}
    \alpha_{f \circ M}(\lambda; \mathcal{D}, \mathcal{D}^{\prime}) = \alpha_{\mathcal{M}}(\lambda; \mathcal{D}, \mathcal{D}^{\prime}) 
    \end{align*}
  \end{lem}
  
  \begin{proof}
  \begin{align}
    \alpha_{f \circ \mathcal{M}}(\lambda; \mathcal{D}, \mathcal{D}^{\prime}) &\triangleq \max_{\mathcal{D}, \mathcal{D}^{\prime}} \log \mathbb{E}_{{\psi}^{\prime} \sim f \circ \mathcal{M}(\mathcal{D})} \left[ e^{\lambda c(\psi; f \circ \mathcal{M}, \mathcal{D}, \mathcal{D}^{\prime})} \right] \nonumber\\ 
    & = \max_{\mathcal{D}, \mathcal{D}^{\prime}} \log \mathbb{E}_{{\psi}^{\prime} \sim f \circ \mathcal{M}(\mathcal{D})} \left[ \frac{Pr[f \circ \mathcal{M}(\mathcal{D})={\psi}^{\prime}]}{Pr[f \circ \mathcal{M}(\mathcal{D}^{\prime})={\psi}^{\prime}]} \right]^{\lambda} \nonumber\\ 
    & = \max_{\mathcal{D}, \mathcal{D}^{\prime}} \log \mathbb{E}_{\psi \sim \mathcal{M}(\mathcal{D})} \left[ \frac{Pr[\mathcal{M}(\mathcal{D})=\psi]}{Pr[\mathcal{M}( \mathcal{D}^{\prime})=\psi]} \right]^{\lambda} \nonumber\\ 
    & = \alpha_{\mathcal{M}}(\lambda; \mathcal{D}, \mathcal{D}^{\prime})
  \end{align}
  \end{proof}
  That is, data-independent mappings do not affect the value of moment, since they do not introduce additional randomness conditioned on the data. Rethinking our DPDL algorithm, for any single iteration, the momentum update and model update steps are data-independent transformations derived solely from the data-dependent vector $\tilde{g}^i_t$. Thus, for $\tilde{v}^i_{t}, \forall t < T$ and $\tilde{x}^i_{t}, \forall t < T$ we have:
  \begin{align*}
    \alpha_{\tilde{v}^i_{t}}(\lambda; \mathcal{D}, \mathcal{D}^{\prime}) = \alpha_{\tilde{x}^i_{t}}(\lambda; \mathcal{D}, \mathcal{D}^{\prime}) = 
    \alpha_{\tilde{g}^i_t}(\lambda; \mathcal{D}, \mathcal{D}^{\prime})
  \end{align*}
  Let $w_{min} = \min_{i \in \mathcal{N}, j\in\mathcal{N}_i} w_{ij}$, $w_{max} = \max_{i \in \mathcal{N}, j\in\mathcal{N}_i} w_{ij}$, we can get the $\lambda$-th moment of our DPDL algorithm in fixed round $t$:
  \begin{align}\label{eq:dpdl_moment}
    \alpha_{{DPDL}_t}(\lambda; \mathcal{D}, \mathcal{D}^{\prime}) \overset{\Delta}{=} \max_{i} \alpha_{\tilde{x}^i_{t}}(\lambda; \mathcal{D}, \mathcal{D}^{\prime})  \leq \frac{\left( \frac{2C}{\sqrt{w_{min}}BN} + \frac{15\alpha (1- w_{min})C}{8B}+ \frac{\alpha w_{max} (h+1)^2C}{2Bh} \right)^2 c_2^2 q^2\lambda^2}{\left( \frac{C\sigma}{NB}  \sqrt{\sum_{j \in \mathcal{N}_i} \frac{1}{w_{min}}} \right)^2}
  \end{align}
  Next, we use the following lemma to bridge the gap between moment and DP: 
  \begin{lem}~\label{lem:moment_dp}
    Let $\mathcal{M} \colon \mathbb{D} \rightarrow \mathbb{R}^d$ be a randomized mechanism with $\lambda$-th moment as $\alpha_{\mathcal{M}}(\lambda)$. For any subset of outputs $\mathcal{S} \subseteq \mathbb{R}^d$ and $\epsilon>0$, $\mathcal{M}$ satisfies:
    \begin{equation} \label{eq:moment_dp}
      \Pr\left(\mathcal{M}(\mathcal{D}) \in \mathcal{S}\right) \leq e^{\epsilon} \Pr\left(\mathcal{M}(\mathcal{D}^{\prime}) \in \mathcal{S}\right) + e^{\alpha_{\mathcal{M}}(\lambda) - \lambda\epsilon},
    \end{equation}
    where $\mathcal{D}$ and $\mathcal{D}'$ are adjacent datasets differing in at most one sample.
  \end{lem}
  \begin{proof}
    Let $\Psi = \{ \psi, \mathcal{L}(\mathcal{M}, \mathcal{D}, \mathcal{D}^{\prime}) \geq \epsilon \}$ and its complement $\overline{\Psi} = \{ \psi, \mathcal{L}(\mathcal{M}, \mathcal{D}, \mathcal{D}^{\prime}) < \epsilon \}$
    \begin{align} \label{eq:moment_dp_step}
      \Pr\left(\mathcal{M}(\mathcal{D}) \in \mathcal{S}\right) & = \Pr\left(\mathcal{M}(\mathcal{D}) \in \mathcal{S} \cap \Psi \right) + \Pr\left(\mathcal{M}(\mathcal{D}) \in \mathcal{S} \cap \overline{\Psi}\right) \nonumber \\
      & \overset{\tcircle{1}}{\leq} \Pr\left(\mathcal{M}(\mathcal{D}) \in \Psi \right) + e^{\epsilon} \Pr\left(\mathcal{M}(\mathcal{D}^{\prime}) \in \mathcal{S} \cap \overline{\Psi}\right) \nonumber \\
      & \leq \Pr\left(\mathcal{M}(\mathcal{D}) \in \Psi \right) + e^{\epsilon} \Pr\left(\mathcal{M}(\mathcal{D}^{\prime}) \in \mathcal{S} \right) 
    \end{align}
    where $\tcircle{1}$ holds for applying $\mathcal{L}(\mathcal{M}, \mathcal{D}, \mathcal{D}^{\prime}) < \epsilon \Rightarrow \Pr(M(\mathcal{D})= \psi) < e^{\epsilon} \Pr(M(\mathcal{D}^{\prime})= \psi)$ in \textbf{Definition}~\ref{def:lambdamoments}. Then the first term can be bounded as:
    \begin{align}\label{eq:moment_markov}
      \Pr\left(\mathcal{M}(\mathcal{D}) \in \Psi  \right) 
      &= \Pr_{\psi \sim \mathcal{M}(\mathcal{D}) } \left(\mathcal{L}(\mathcal{M}, \mathcal{D}, \mathcal{D}^{\prime}) \geq \epsilon \right) 
      = \Pr_{\psi \sim \mathcal{M}(\mathcal{D})} \left(e^{\lambda \mathcal{L}(\mathcal{M}, \mathcal{D}, \mathcal{D}^{\prime})} \geq e^{\lambda \epsilon} \right) \nonumber \\
      & \leq \frac{E_{\psi \sim \mathcal{M}(\mathcal{D})} \left[ e^{\lambda \mathcal{L}(\mathcal{M}, \mathcal{D}, \mathcal{D}^{\prime})}\right]}{e^{\lambda \epsilon}} 
      \leq e^{\alpha_{\mathcal{M}}(\lambda) - \lambda\epsilon}
    \end{align}
    By substituting (\ref{eq:moment_markov}) into (\ref{eq:moment_dp_step}), we complete the proof.
  \end{proof}

  Based on \textbf{Lemma}~\ref{lem:moment_dp} and the assumption that $q$, $\sigma$ and $\lambda$ satisfy the conditions of \textbf{Lemma}~\ref{lem:sampled_gaussian_mgf}, to guarantee that our DPDL algorithm is $(\epsilon, \delta)$-differentially private, it suffices for a $\delta$ which satisfies $\delta \geq \min e^{\alpha_{\mathcal{M}}(\lambda) - \lambda\epsilon}$. So we turn to solve:
  \begin{align*}
    \alpha_{\mathcal{M}}(\lambda) &\leq \frac{\lambda \epsilon}{2} \\
    e^{ \frac{\lambda \epsilon}{2} - \lambda\epsilon} &\leq  \delta
  \end{align*}
  Combining $T$ rounds composition of moments in (\ref{eq:dpdl_moment}), we have:
  \begin{align*}
    \ln \left(\frac{1}{\delta} \right) \leq \frac{\lambda \epsilon}{2} \leq \frac{\left( \frac{2C}{\sqrt{w_{min}}BN} + \frac{15\alpha (1- w_{min})C}{8B}+ \frac{\alpha w_{max} (h+1)^2C}{2Bh} \right)^2 c_2^2 q^2\lambda^2}{\left( \frac{C\sigma}{NB}  \sqrt{\sum_{j \in \mathcal{N}_i} \frac{1}{w_{min}}} \right)^2} T 
  \end{align*}
  When $\epsilon \leq c_1q^2T$, we can satisfy all these conditions by setting
  \begin{align*}
    \frac{C\sigma}{NB}  \sqrt{\sum_{j \in \mathcal{N}_i} \frac{1}{w_{min}}} \geq \left( \frac{2C}{\sqrt{w_{min}}BN} + \frac{15\alpha (1- w_{min})C}{8B}+ \frac{\alpha w_{max} (h+1)^2C}{2Bh} \right) \frac{c_2 q\sqrt{T \ln(1/\delta)}}{\epsilon}
  \end{align*}
  for some explicit constants $c_1$ and $c_2$. Thus, we have the following noise bound to make each agent satisfies $(\epsilon, \delta)$-DP throughout $T$ rounds
  \begin{align*}
    \sigma \geq \frac{\left( \frac{2}{\sqrt{w_{min}}} + \frac{15\alpha (1- w_{min})}{8}+ \frac{\alpha w_{max} (h+1)^2N}{2h} \right) c_2q\sqrt{T \ln(1/\delta)}}{\sqrt{\sum_{j \in \mathcal{N}_i} \frac{1}{w_{min}}} \cdot \epsilon}
  \end{align*}

\section{Proof of Theorem~\ref{thm:convergence}}
\label{sec:convergenceproof}
  Our roadmap to prove \textbf{Theorem}~\ref{thm:convergence} is given in Fig~\ref{fig:thm2-roadmap}. In particular, in \textbf{Lemma}~\ref{lem:avg_clbr_grad}, we show the upper bound on the expected gradient aggregation averaged over the agents in each round, based on which, we propose \textbf{Lemma}~\ref{lem:avg_grad_diff} to illustrate the upper bound of the expected average difference between the aggregated gradients and the self-gradients in each round. We show in \textbf{Lemma}~\ref{lem:G_diff} that, in each round, the agents have their self-gradients ``similar'' to its aggregated gradients on expectation. Based on \textbf{Lemma}~\ref{lem:G_diff}, it is revealed in \textbf{Lemma}~\ref{lem:consensus_error} that, the cumulative expected difference between the agents' local model and the average model over the agents is bounded. We reveal the evolution of the average models across the time span in \textbf{Lemmas}~\ref{lem:u_seq}$\sim$\ref{lem:u_x_diff}. Finally, we complete the proof by combining all these lemmas. 
  \begin{figure}[H]
  \begin{center}
    \center\includegraphics[width=.7\textwidth]{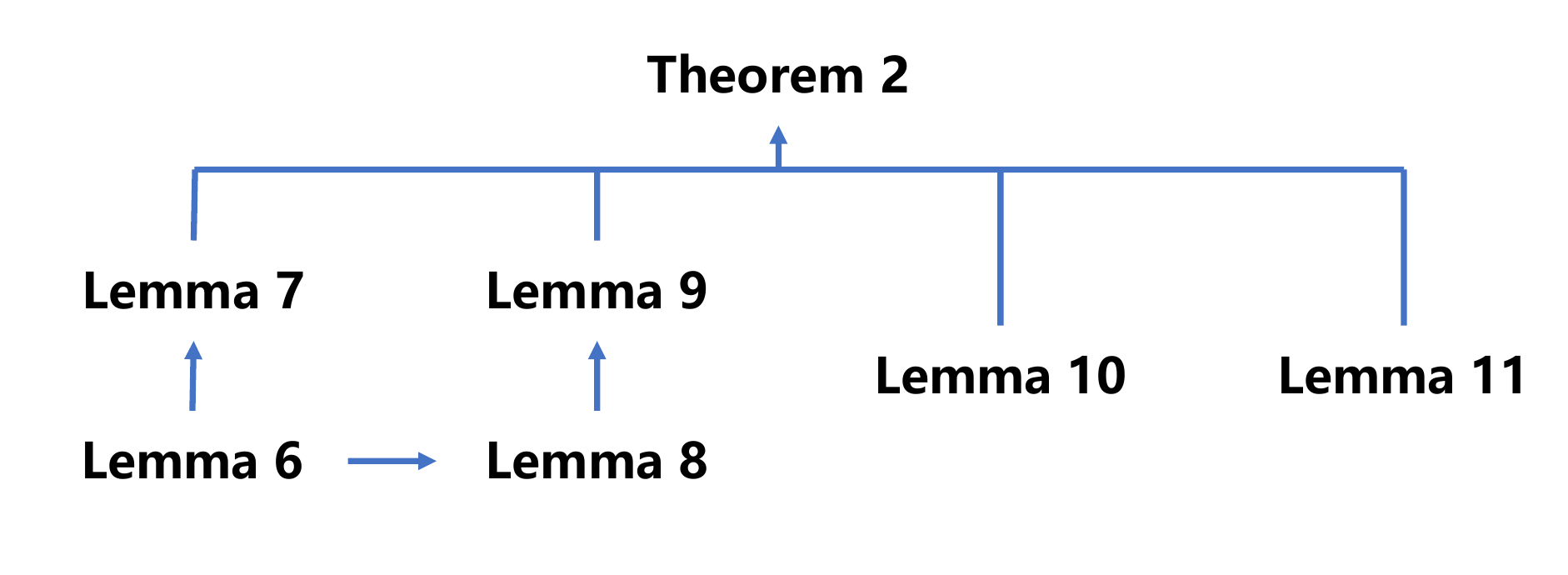}
  \caption{Our roadmap to prove \textbf{Theorem}~\ref{thm:convergence}.}
  \label{fig:thm2-roadmap}
  \end{center}
  \end{figure}

  Recall that $\tilde{g}^{i}_{t}$ denotes the aggregated gradient of agent $i$ in round $t$, and $\frac{1}{B} \sum_b g^{ii}_{t,b}$ represents its batch-averaged self-gradient. As mentioned above, we first give \textbf{Lemma}~\ref{lem:avg_clbr_grad} and \textbf{Lemma}~\ref{lem:avg_grad_diff} as follows. They reveal the upper bounds of $\mathbb{E}\left[\left\Vert\frac{1}{N}\sum_{i=1}^N{\tilde{g}}^{i}_{t} \right\Vert^2\right]$ and $\mathbb{E}\left\Vert \frac{1}{N} \sum_{i=1}^{N} \left({\tilde{g}}^i_t - \frac{1}{B} \sum_b g^{ii}_{t,b} \right) \right\Vert^2$, respectively.
  \begin{lem}~\label{lem:avg_clbr_grad} 
    Assume \textbf{Assumptions}~\ref{assum:smooth}$\sim$\ref{assum:dsm} hold and let $S_0 = \sum_{i=1}^{N} \sum_{j \in \mathcal{N}_i} \frac{1}{w_{ij}^2}$. In each round $t$, we have
    \begin{equation} \label{eq:avg_clbr_grad} 
      \mathbb{E}\left[\left\Vert\frac{1}{N}\sum_{i=1}^N{\tilde{g}}^{i}_{t} \right\Vert^2\right] 
      \leq \frac{3 S_0 C^2}{N^3} + \frac{3S_0\sigma^2 C^2d}{B^2N^4} + \frac{27\alpha^2 C^2}{16}
    \end{equation}
  \end{lem}
  \begin{proof}
    Supposing $\Lambda^{ij}_{t,b} \overset{\Delta}{=} \min\left\{1,\frac{C}{\left\Vert g^{ij}_{t,b} \right\Vert}\right\}$, we have
    \begin{align*}
      \quad \mathbb{E}\left[\left\Vert\frac{1}{N}\sum_{i=1}^N{\tilde{g}}^{i}_{t} \right\Vert^2\right] 
      & = \mathbb{E}\left\Vert \frac{1}{N} \sum_{i=1}^{N} \sum_{j \in \mathcal{N}_i} \left( \frac{1}{\sqrt{w_{ij}}BN} \left( \sum_b \hat{g}^{ij}_{t,b} + z^{ij}_t \right) + \frac{\alpha w_{ij} {\mathcal{C}}^{ij}_t }{B} \sum_b \hat{g}^{ii}_{t,b} \right)\right\Vert^2 \\
      & = \mathbb{E}\left\Vert \frac{1}{N} \sum_{i=1}^{N} \sum_{j \in \mathcal{N}_i} \left( \frac{1}{\sqrt{w_{ij}}BN} \left( \sum_b \Lambda^{ij}_{t,b} g^{ij}_{t,b} + z^{ij}_t \right) + \frac{\alpha w_{ij} {\mathcal{C}}^{ij}_t }{B} \sum_b \Lambda^{ii}_{t,b} g^{ii}_{t,b} \right)\right\Vert^2 \\
      & \overset{\tcircle{1}}{\leq} 3\mathbb{E}\left\Vert \frac{1}{BN^2} \sum_{i=1}^{N} \sum_{j \in \mathcal{N}_i} \frac{1}{\sqrt{w_{ij}}} \sum_b  \Lambda^{ij}_{t,b} g^{ij}_{t,b} \right\Vert^2 
      + 3\mathbb{E}\left\Vert \frac{1}{BN^2} \sum_{i=1}^{N} \sum_{j \in \mathcal{N}_i} \frac{z^{ij}_t}{\sqrt{w_{ij}}} \right\Vert^2 \\
      &{\quad} + 3\mathbb{E}\left\Vert \frac{\alpha}{BN} \sum_{i=1}^{N} \left(\sum_{j \in \mathcal{N}_i} w_{ij} {\mathcal{C}}^{ij}_t \right) \left(\sum_b \Lambda^{ii}_{t,b} g^{ii}_{t,b} \right) \right\Vert^2 \\
      & \overset{\tcircle{2}}{\leq} \frac{3}{B^2N^3} \sum_{i=1}^{N} \mathbb{E} \left\Vert \sum_{j \in \mathcal{N}_i} \frac{1}{\sqrt{w_{ij}}} \sum_b  \Lambda^{ij}_{t,b} g^{ij}_{t,b} \right\Vert^2 
      + \frac{3}{B^2N^4} \sum_{i=1}^{N} \mathbb{E} \left\Vert \sum_{j \in N(i)} \frac{z^{ij}_t}{\sqrt{w_{ij}}} \right\Vert^2  \\
      &{\quad} + \frac{3\alpha^2}{B^2N} \sum_{i=1}^{N} \mathbb{E} \left\Vert \sum_{j \in \mathcal{N}_i} w_{ij} {\mathcal{C}}^{ij}_t \sum_b \Lambda^{ii}_{t,b} g^{ii}_{t,b} \right\Vert^2 \\
      & \overset{\tcircle{3}}{\leq} \frac{3}{B^2N^3} \sum_{i=1}^{N} \sum_{j \in N(i)} \frac{1}{w_{ij}^2} \mathbb{E} \left\Vert \sum_b  \Lambda^{ij}_{t,b} g^{ij}_{t,b} \right\Vert^2 
      + \frac{3}{B^2N^4} \sum_{i=1}^{N} \sum_{j \in N(i)} \frac{1}{w_{ij}^2} \mathbb{E} \left\Vert z^{ij}_t \right\Vert^2 \\
      &{\quad} + \frac{27\alpha^2}{16 B^2N} \sum_{i=1}^{N} \sum_{j \in N(i)} w_{ij} \mathbb{E} \left\Vert \sum_b \Lambda^{ii}_{t,b} g^{ii}_{t,b} \right\Vert^2 \\
      & \overset{\tcircle{4}}{\leq} \frac{3 S_0 C^2}{N^3} + \frac{3S_0\sigma^2 C^2d}{B^2N^4} + \frac{27\alpha^2 C^2}{16}
    \end{align*}
    where we have $\tcircle{1}$ and $\tcircle{2}$ since $ \mathbb{E}\left[ \left\| \sum_{i=1}^{N} a_i \right\|^2 \right]= N \sum_{i=1}^{N} \mathbb{E}\left[ \| a_i \|^2 \right]$ holds for any vector $a_i \in \mathbb{R}^d$, $\tcircle{3}$ by considering the Jensen's Inequality of $\left\Vert \cdot \right\Vert^2$ that $ \left\Vert \sum_{i} w_i x_i \right\Vert^2  \leq \sum_{i} w_i \left\Vert x_i \right\Vert^2 $ when $\sum_i w_i =1$, and $\tcircle{4}$ by considering $\left\Vert \Lambda^{ij}_{t,b} g^{ij}_{t,b} \right\Vert^2 \leq C^2$ and $\mathbb{E} \left\Vert z^{ij}_t \right\Vert^2 \leq \sigma^2C^2d$.
  \end{proof}

  \begin{lem} \label{lem:avg_grad_diff}
    Assume \textbf{Assumptions}~\ref{assum:smooth}$\sim$\ref{assum:dsm} hold and let $S_0 =  \sum_{i=1}^{N} \sum_{j \in \mathcal{N}_i} \frac{1}{w_{ij}^2} $. In each round $t$, we have
    \begin{align} \label{eq:avg_grad_diff}
      &\quad \mathbb{E}\left\Vert \frac{1}{N} \sum_{i=1}^{N} \left({\tilde{g}}^i_t - \frac{1}{B} \sum_b g^{ii}_{t,b} \right) \right\Vert^2 
      \leq \frac{6S_0 C^2}{N^3} + \frac{6S_0\sigma^2 C^2d}{B^2N^4} + \frac{27\alpha^2 C^2}{8} + \frac{4\xi^2}{N} + 4\mathbb{E}\left\Vert \frac{1}{N} \sum_{i=1}^{N} \nabla f_i(x_{t-1}^i)\right\Vert^2
    \end{align}
  \end{lem}
  \begin{proof}
    We first derive the bound of $\mathbb{E}\left\Vert \frac{1}{N} \sum_{i=1}^{N} \frac{1}{B} \sum_b  g^{ii}_{t,b} \right\Vert^2$ as follows
    \begin{align}\label{eq:avg_selfgrad}
      & \mathbb{E}\left\Vert \frac{1}{N} \sum_{i=1}^{N} \frac{1}{B} \sum_b  g^{ii}_{t,b} \right\Vert^2 \\
      \leq & \frac{1}{B^2} \mathbb{E}\left\Vert \frac{1}{N} \sum_{i=1}^{N} \sum_b \nabla F_i \left( x^i_{t-1}; \zeta^i_{t,b} \right) \right\Vert^2 \nonumber\\
      \leq & \frac{2}{B^2}  \mathbb{E}\left\Vert \frac{1}{N} \sum_{i=1}^{N} \sum_b \left( \nabla F_i \left( x^i_{t-1}; \zeta^i_{t,b} \right) - \nabla f_i \left( x_{t-1}^i \right) \right) \right\Vert^2  
      + \frac{2}{B^2} \mathbb{E}\left\Vert \frac{1}{N} \sum_{i=1}^{N} \sum_b \nabla f_i \left( x_{t-1}^i \right) \right\Vert^2 \nonumber\\
      \overset{\tcircle{1}}{\leq} & \frac{2}{BN^2} \sum_{i=1}^{N} \sum_b \mathbb{E} \left\Vert  \nabla F_i \left( x^i_{t-1}; \zeta^i_{t,b} \right) - \nabla f_i \left( x_{t-1}^i \right) \right\Vert^2 
      + 2\mathbb{E}\left\Vert \frac{1}{N} \sum_{i=1}^{N} \nabla f_i \left( x_{t-1}^i \right) \right\Vert^2 \nonumber\\
      \overset{\tcircle{2}}{\leq} & \frac{2\xi^2}{N} + 2\mathbb{E}\left\Vert \frac{1}{N} \sum_{i=1}^{N} \nabla f_i(x_{t-1}^i)\right\Vert^2 
    \end{align}
    where we have $\tcircle{1}$ according to $\nabla f_i({x}) = \mathbb{E}_{\zeta_i \sim \mathcal{D}_i} \left[\nabla F_i({x}; \zeta_i)\right]$ and $\mathbb{E}\left[ \left\| \sum_{i=1}^{N} a_i \right\|^2 \right]=\sum_{i=1}^{N} \mathbb{E}\left[ \|a_i\|^2 \right]$ for any in dependent zero-mean vectors $a_1, a_i, \cdots, a_N \in \mathbb{R}^d$, and $\tcircle{2}$ based on \textbf{Assumption}~\ref{assum:variance}. Combining \textbf{Lemma}~\ref{lem:avg_clbr_grad} and (\ref{eq:avg_selfgrad}), we obtain
    \begin{align*}
      \mathbb{E}\left\Vert \frac{1}{N} \sum_{i=1}^{N} \left({\tilde{g}}^i_t - \frac{1}{B} \sum_b g^{ii}_{t,b} \right) \right\Vert^2 
      & \leq 2\mathbb{E}\left\Vert \frac{1}{N} \sum_{i=1}^{N} {\tilde{g}}^i_t \right\Vert^2 + 2\mathbb{E}\left\Vert \frac{1}{N} \sum_{i=1}^{N} \frac{1}{B} \sum_b g^{ii}_{t,b} \right\Vert^2 \\
      &\leq \frac{6S_0 C^2}{N^3} + \frac{6S_0\sigma^2 C^2d}{B^2N^4} + \frac{27\alpha^2 C^2}{8} + \frac{4\xi^2}{N} + 4\mathbb{E}\left\Vert \frac{1}{N} \sum_{i=1}^{N} \nabla f_i(x_{t-1}^i)\right\Vert^2
    \end{align*}
  \end{proof}

  Next, to capture the iterative progress from a global perspective, we present a matrix representation of each parameter in \textbf{Remark}~\ref{rmk:matrixs}. We then outline several important properties of matrices based on two commonly used norms: the \textit{spectral norm} and the \textit{Frobenius norm}. Specifically, \textbf{Fact}~\ref{fact:dsm} presents key properties of doubly stochastic matrices; \textbf{Fact}~\ref{fact:fnorm} highlights a characteristic of the Frobenius norm; and \textbf{Fact}~\ref{fact:fnorm_to_l2norm} demonstrates the relationship between the Frobenius norm and the $L_2$ norm. Finally, \textbf{Lemma}~\ref{lem:G_diff} provides a matrix-form expression of the distance between the calibrated gradient and the self-gradient.
  %
  %
  \begin{remark} \label{rmk:matrixs}
    Here are some matrix representation of different parameters which will be used in the following proof:
    \begin{align}
    \begin{cases}
      &\mathbf{\tilde{G}}_t \triangleq \left[ {\tilde{g}}^1_t,{\tilde{g}}^2_t,\ldots,{\tilde{g}}^N_t \right] \vspace{1ex}\\
      & \mathbf{G}_t \triangleq \left[\frac{1}{B} \sum_b {g}^{11}_{t,b},  \frac{1}{B} \sum_b {g}^{22}_{t,b}, \ldots, \frac{1}{B} \sum_b {g}^{NN}_{t,b} \right] \vspace{1ex}\\
      &\mathbf{\tilde{V}}_t \triangleq  \left[{\tilde{v}}^1_t,{\tilde{v}}^2_t,\ldots,{\tilde{v}}^{N}_t \right]\vspace{1ex}\\
      &\mathbf{V}_t \triangleq  \left[{v}^1_t, {v}^2_t,\ldots, {v}^N_t \right]\vspace{1ex}\\
      &\mathbf{\tilde{X}}_t  \triangleq \left[\tilde{{x}}^1_t,\tilde{{x}}^2_t,\ldots,\tilde{{x}}^N_t \right]\vspace{1ex}\\
      &\mathbf{X}_t \triangleq \left[{x}^1_t,{x}^2_t,\ldots,{x}^N_t \right]\vspace{1ex}\\
      &\mathbf{H}_t  \triangleq \left[\nabla f_1({x}^1_{t-1}),\nabla f_2({x}^2_{t-1}),...,\nabla f_N({x}^N_{t-1}) \right]\vspace{1ex}\\
      & \mathbf{J}_t \triangleq \left[\nabla f_1(\bar{{x}}^1_{t-1}),\nabla f_2(\bar{{x}}^2_{t-1}),\ldots,\nabla f_N(\bar{{x}}^N_{t-1})\right]
    \end{cases}
    \end{align}
  \end{remark}

  \begin{fact} \label{fact:dsm}
    Let $\mathbf{I}$ be the identical matrix and $\mathbf{Q}=\frac{1}{N}\mathbf{1}\mathbf{1}^\top$ where $\mathbf{1}$ be a vector with each element being $1$. Let $\left\Vert \cdot \right\Vert_\mathfrak{S}$ represent spectral norm. For any doubly stochastic matrix $\mathbf{W}$, we have
    \begin{itemize}
      \item $\mathbf{Q}\mathbf{W}=\mathbf{W}\mathbf{Q}$;\vspace{0.3em}
      \item $(\mathbf{I}-\mathbf{Q})\mathbf{W}=\mathbf{W}(\mathbf{I}-\mathbf{Q})$;\vspace{0.3em}
      \item For any integer $k\geq 1$, $\left\Vert(\mathbf{I}-\mathbf{Q})\mathbf{W}^k\right\Vert_\mathfrak{S}\leq{\sqrt{\rho}}^k$.\vspace{0.3em}
    \end{itemize}
  \end{fact}
  \begin{fact} \label{fact:fnorm}
    Let $\mathbf{A}_1, \mathbf{A}_2, \cdots, \mathbf{A}_N$ be $N$ arbitrary real-valued matrices. Let $\left\Vert \cdot \right\Vert_\mathfrak{F}$ represent Frobenius norm. It follows that
    \begin{equation}
      \left\Vert \sum_{i=1}^N\mathbf{A}_i \right\Vert^2_\mathfrak{F} \leq \sum_{i=1}^N \sum_{j=1}^N \big\Vert \mathbf{A}_i \big\Vert_\mathfrak{F} \big\Vert \mathbf{A}_j \big\Vert_\mathfrak{F}
    \end{equation}
  \end{fact}
  \begin{fact}\label{fact:fnorm_to_l2norm}
    For any matrix $\mathbf{A} \in \mathbb{R}^{d \times N}$ with dimension $d\times N$, we have $\Vert\mathbf{A}\Vert_\mathfrak{F}^2=\sum_{i=1}^N\Vert\mathbf{a}_i\Vert^2$, where $\mathbf{a}_i \in \mathbb{R}^d$ denotes the $i$-th column vector of $\mathbf{A}$. Letting $\bar{x}_t = \frac{1}{N} \sum_{i=1}^N x^i_t$, we have
    \begin{equation}
        \big\Vert  \mathbf{X}_t(\mathbf{I}-\mathbf{Q})\big\Vert_\mathfrak{F}^2 = \sum_{i=1}^N\left\Vert  {x}^i_t-\bar{{x}}_t \right\Vert^2.
    \end{equation}
    where $\mathbf{I}$ and $\mathbf{Q}$ are defined in \textbf{Fact}~\ref{fact:dsm}
  \end{fact}

  \begin{lem}\label{lem:G_diff}
    Assume \textbf{Assumptions}~\ref{assum:smooth}$\sim$\ref{assum:dsm} hold. Let $S_0 =  \sum_{i=1}^{N} \sum_{j \in N(i)} \frac{1}{w_{ij}^2}$. In each round $t$, we have
    \begin{align} \label{eq:G_diff}
      \mathbb{E}\left[ \left\Vert \mathbf{\tilde{G}}_{t}-\mathbf{G}_{t} \right\Vert_\mathfrak{F}^2 \right] 
      \leq & \frac{6S_0C^2}{N^2} + \frac{6S_0\sigma^2C^2d}{B^2N^2} + \frac{27\alpha C^2N}{8} + 20L^2\sum_{i=1}^N \mathbb{E} \left\Vert  x_{t-1}^i - \bar{x}_{t-1} \right\Vert^2 \nonumber\\ 
      & + 10N\xi^2 + 10N\kappa^2  + 10N\mathbb{E} \left\Vert \frac{1}{N} \sum_{k=1}^{N}\nabla f_k(x_{t-1}^k)\right\Vert^2
    \end{align}
  \end{lem}
  \begin{proof}
    Supposing $\Lambda^{ij}_{t,b} \overset{\Delta}{=} \min\left\{1,\frac{C}{\left\Vert g^{ij}_{t,b} \right\Vert}\right\}$, we have
    \begin{align}\label{eq:sum_selfgrad}
      & \sum_{i=1}^N \mathbb{E}\left\Vert \frac{1}{B} \sum_b {g}^{ii}_{t,b} \right\Vert^2 \nonumber \\
      \leq & \sum_{i=1}^N \mathbb{E} \left\Vert
        \begin{aligned} 
          & \frac{1}{B} \sum_b \nabla F \left(x^i_{t-1,b};\zeta^i_{t-1,b} \right) - \nabla f_i(x_{t-1}^i) + \nabla f_i(x_{t-1}^i) - \nabla f_i(\bar{x}_{t-1}) + \nabla f_i(\bar{x}_{t-1}) \\
          - & \frac{1}{N} \sum_{k=1}^{N}\nabla f_k(\bar{x}_{t-1}) + \frac{1}{N} \sum_{k=1}^{N}\nabla f_k(\bar{x}_{t-1}) - \frac{1}{N}\sum_{k=1}^{N}\nabla f_k(x_{t-1}^k) + \frac{1}{N}\sum_{k=1}^{N}\nabla f_k(x_{t-1}^k) \\
        \end{aligned} 
      \right\Vert^2 \nonumber\\
      \overset{\tcircle{1}}{\leq} & 5 \sum_{i=1}^N \mathbb{E} \left\Vert  \frac{1}{B} \sum_b \nabla F \left(x^i_{t-1,b};\zeta^i_{t-1,b} \right) - \frac{1}{B} \sum_b \nabla f_i(x_{t-1}^i)\right\Vert^2  + 5\sum_{i=1}^N \mathbb{E} \bigg\|  \nabla f_i(x_{t-1}^i) - \nabla f_i(\bar{x}_{t-1}) \bigg\|^2  \nonumber\\
      & + 5\sum_{i=1}^N \mathbb{E} \left\Vert \frac{1}{N} \sum_{k=1}^{N}\nabla f_k(x_{t-1}^k)\right\Vert^2 + 5\sum_{i=1}^N \mathbb{E} \left\Vert \nabla f_i(\bar{x}_{t-1}) - \frac{1}{N} \sum_{k=1}^{N}\nabla f_k(\bar{x}_{t-1}) \right\Vert^2 \nonumber\\
      & + 5\sum_{i=1}^N \mathbb{E} \left\Vert \frac{1}{N} \sum_{k=1}^{N}\nabla f_k(\bar{x}_{t-1}) -  \frac{1}{N}\sum_{k=1}^{N}\nabla f_k(x_{t-1}^k) \right\Vert^2 \nonumber\\
      \overset{\tcircle{2}}{\leq} & 5N\xi^2 + 5N\kappa^2 + 10L^2 \sum_{i=1}^N \mathbb{E} \left\Vert  x_{t-1}^i - \bar{x}_{t-1} \right\Vert^2 + 5N\mathbb{E} \left\Vert \frac{1}{N} \sum_{k=1}^{N}\nabla f_k(x_{t-1}^k)\right\Vert^2
    \end{align}
    where we have $\tcircle{1}$ since $\left\Vert \sum_{i=1}^{N} \mathbf{A}_i \right\Vert^2 \leq N \sum_{i=1}^{N} \left\Vert\mathbf{A}_i \right\Vert^2$, and $\tcircle{2}$ by considering \textbf{Assumptions}~\ref{assum:smooth}-2 as well as the fact that $\left\Vert \frac{1}{N} \sum_{k=1}^{N}\nabla f_k(\bar{x}_{t-1}) -  \frac{1}{N}\sum_{k=1}^{N}\nabla f_k(x_{t-1}^k) \right\Vert^2 \leq \frac{1}{N} \left\Vert  \nabla f_i(x_{t-1}^i) - \nabla f_i(\bar{x}_{t-1}) \right\Vert^2$.
    %
    %
    Using the technique in the proof of \textbf{Lemma}~\ref{lem:avg_clbr_grad}, we get
    \begin{align} \label{eq:sum_clbrgrad}
      \sum_{i=1}^N \mathbb{E} \left[ \left\Vert  {\tilde{g}}^{i}_t \right\Vert^2 \right]
      = & \sum_{i=1}^N \mathbb{E}\left\Vert \sum_{j \in \mathcal{N}_i} \left(\frac{1}{\sqrt{w_{ij}}BN} \left(\sum_b \Lambda^{ij}_{t,b} g^{ij}_{t,b} + z^{ij}_t \right) + \frac{\alpha w_{ij}{\mathcal{C}}^{ij}_t}{B} \sum_b \Lambda^{ii}_{t,b} {g}^{ii}_{t,b} \right) \right\Vert^2 \nonumber\\
      \leq & \frac{3}{B^2N^2} \sum_{i=1}^{N} \mathbb{E} \left\Vert \sum_{j \in \mathcal{N}_i} \sum_b \frac{\Lambda^{ij}_{t,b} g^{ij}_{t,b} }{\sqrt{w_{ij}}} \right\Vert^2 
      + \frac{3}{B^2N^2} \sum_{i=1}^{N} \mathbb{E} \left\Vert \sum_{j \in \mathcal{N}_i} \frac{z^{ij}_t }{\sqrt{w_{ij}}} \right\Vert^2 \nonumber \\
      & + \frac{3}{B^2} \sum_{i=1}^{N} \mathbb{E} \left\Vert \sum_{j \in \mathcal{N}_i} \alpha w_{ij} {\mathcal{C}}^{ij}_t \sum_b \Lambda^{ii}_{t,b} {g}^{ii}_{t,b} \right\Vert^2 \nonumber\\
      \leq & \frac{3S_0C^2}{N^2} + \frac{3S_0\sigma^2C^2d}{B^2N^2} + \frac{27\alpha C^2N}{16}.
    \end{align}
    We finally complete the proof by combining (\ref{eq:sum_clbrgrad}) and (\ref{eq:sum_selfgrad}).
  \end{proof}

  Next, we give \textbf{Lemma}~\ref{lem:consensus_error} of DPDL, which describes that the average expected difference of all agent parameters, which is also known as \textit{overall consensus error} describing the inconsistency of model parameters between multiple agents, has a composite upper bound. From \textbf{Lemma~\ref{lem:consensus_error}} we can see that the consensus error of parameters is mainly affected by four parts: intra-gradient variance, inter-gradient variance, degree of Gaussian noises, and squared norms of average gradients. 
  \begin{lem}~\label{lem:consensus_error}
    Assume \textbf{Assumptions}~\ref{assum:smooth}-3 hold. Let $S_0 =  \sum_{i=1}^{N} \sum_{j \in \mathcal{N}_i} \frac{1}{w_{ij}^2}$ and $\bar{x}_t = \frac{1}{N} \sum^N_{i=1} x^i_t$. When momentum coefficient $\beta$ and learning rate $\eta$ satisfy
    \begin{equation}
      \eta \leq \frac{(1-\beta)(1-\rho)}{2\sqrt{30}L}
    \end{equation}
    we have
    \begin{align} \label{eq:consensus_error}
      \frac{1}{N}\sum_{t=1}^{T}\sum_{i=1}^N\mathbb{E}\left[\left\Vert\bar{{x}}_{t-1}-{x}^i_{t-1}\right\Vert^2\right]
      \leq& \frac{\left(24S_0C^2 + \frac{24}{B^2}S_0 \sigma^2 C^2d \right)\eta^2} {N^3(1-\beta)^2(1-\sqrt{\rho})^2} T + \frac{\left(\frac{27\alpha C^2}{2} + 60B^2(\xi^2 +\kappa^2) \right) \eta^2} {(1-\beta)^2(1-\sqrt{\rho})^2B^2}T \nonumber\\
      & + \frac{60\eta^2}{(1-\beta)^2(1-\sqrt{\rho})^2}\sum_{t=1}^{T} \mathbb{E} \left\Vert \frac{1}{N} \sum_{k=1}^{N}\nabla f_k(x_{t-1}^k)\right\Vert^2        
    \end{align}
  \end{lem}
  \begin{proof}
    Given the update rule $\mathbf{V}_t = \mathbf{\tilde{V}}_t \mathbf{W}$ and $\mathbf{\tilde{V}}_t = \beta \mathbf{V}_{t-1} + \mathbf{\tilde{G}}_{t}$ in matrix form (see \textbf{Algorithm}~\ref{alg:dpdl}), we have
    \begin{align}\label{eq:v_matrix}
      \mathbf{V}_t &= \mathbf{\tilde{V}}_t \mathbf{W} = \left(\beta \mathbf{V}_{t-1} + \mathbf{\tilde{G}}_{t} \right) \mathbf{W} = \beta^{t} \mathbf{V}_0 \mathbf{W}^t + \sum_{\tau = 1}^{t} \beta^{t-\tau} \mathbf{\tilde{G}}_\tau \mathbf{W}^{t-\tau+1} = \sum_{\tau = 1}^{t} \beta^{t-\tau} \mathbf{\tilde{G}}_\tau \mathbf{W}^{t-\tau+1}
    \end{align}
    Moreover, recalling the definition of $\mathbf{X}_t$ in (\ref{rmk:matrixs}), we have
    \begin{equation} \label{eq:x_matrix}
      \mathbf{X}_t = \mathbf{\tilde{X}}_t \mathbf{W} =  \left(\mathbf{X}_{t-1} - \eta \mathbf{\tilde{V}}_t \right) \mathbf{W} = \mathbf{X}_0 \mathbf{W}^t - \eta \sum_{\tau = 1}^{t} \mathbf{V}_{t-\tau} \mathbf{W}^{t}        
    \end{equation}
    Multiplying both sides of (\ref{eq:x_matrix}) by $\mathbf{I}-\mathbf{Q}$ (where $\mathbf{Q}$ is defined in \textbf{Fact}~\ref{fact:dsm}), we have
    \begin{align*}
      \mathbf{X}_t(\mathbf{I}-\mathbf{Q}) &=\mathbf{X}_0 \mathbf{W}^t(\mathbf{I}-\mathbf{Q}) - \eta \sum_{\tau = 1}^{t} \mathbf{V}_{t-\tau} \mathbf{W}^{t}(\mathbf{I}-\mathbf{Q}) \\
      & \overset{\tcircle{1}}{=} - \eta \sum_{\tau = 1}^{t} \mathbf{V}_{t-\tau} \mathbf{W}^{t}(\mathbf{I}-\mathbf{Q}) \\
      & \overset{\tcircle{2}}{=} -\eta\sum_{\tau=1}^t \sum_{j=1}^{\tau}\mathbf{\tilde{G}}_j\beta^{\tau-j}\mathbf{W}^{\tau - j +1} (\mathbf{I}-\mathbf{Q}) \\
      %
      %
      & = -\eta\sum_{\tau=1}^{t}\frac{1-\beta^{t-\tau}}{1-\beta}\mathbf{\tilde{G}}_{\tau} \mathbf{W}^{t-\tau +1}(\mathbf{I}-\mathbf{Q}).        
    \end{align*}
    where we have $\tcircle{1}$ since $\mathbf{X}_0(\mathbf{I}-\mathbf{Q}) = \mathbf{0}$, and $\tcircle{2}$ according to (\ref{eq:v_matrix}). Therefore, for any $t \geq 1$, we have
    \begin{align} \label{eq:x_minus}
      \mathbb{E}\left[\left\Vert \mathbf{X}_t(\mathbf{I}-\mathbf{Q})\right\Vert^2_\mathfrak{F}\right] =& \eta^2 \mathbb{E} \left\Vert\sum_{\tau=1}^{t}\frac{1-\beta^{t-\tau}}{1-\beta}\mathbf{\tilde{G}}_{\tau}(\mathbf{I}-\mathbf{Q})\mathbf{W}^{t-\tau+1}\right\Vert^2_\mathfrak{F} \nonumber\\
      \overset{\tcircle{1}}{\leq}& \underbrace{2\eta^2\mathbb{E} \left\Vert\sum_{\tau=1}^{t}\frac{1-\beta^{t-\tau}}{1-\beta}(\mathbf{\tilde{G}}_{\tau}-\mathbf{G}_{\tau})(\mathbf{I}-\mathbf{Q})\mathbf{W}^{t-\tau+1}\right\Vert^2_\mathfrak{F} }_{\mathbf{\Rmnum{1}}} \nonumber\\
      & + \underbrace{2\eta^2\mathbb{E} \left\Vert\sum_{\tau=1}^{t}\frac{1-\beta^{t-\tau}}{1-\beta}\mathbf{G}_{\tau}(\mathbf{I}-\mathbf{Q})\mathbf{W}^{t-\tau+1}\right\Vert^2_\mathfrak{F} }_{\mathbf{\Rmnum{2}}}
    \end{align}
    where $\tcircle{1}$ follows from the inequality $\Vert\mathbf{A}+\mathbf{B}\Vert_\mathfrak{F}^2\leq 2\Vert\mathbf{A} \Vert_\mathfrak{F}^2+2\Vert\mathbf{B}\Vert_\mathfrak{F}^2$. 
    We then derive the upper bound of term \textbf{\Rmnum{1}} as follows
    %
    %
    \begin{align}\label{eq:x_minus_G_minus}
      & \mathbb{E} \left\Vert\sum_{\tau=1}^{t}\frac{1-\beta^{t-\tau}}{1-\beta}\left(\mathbf{\tilde{G}}_{\tau}-\mathbf{G}_{\tau}\right)(\mathbf{I}-\mathbf{Q})\mathbf{W}^{t-\tau+1}\right\Vert^2_\mathfrak{F}  \nonumber\\ 
      \overset{\tcircle{1}}{\leq} & \sum_{\tau=1}^{t}\sum_{\tau^\prime=1}^{t}\mathbb{E}  \left\Vert\frac{1-\beta^{t-\tau}}{1-\beta}\left(\mathbf{\tilde{G}}_{\tau}-\mathbf{G}_{\tau}\right)(\mathbf{I}-\mathbf{Q})\mathbf{W}^{t-\tau+1}\right\Vert_\mathfrak{F} \cdot \left\Vert\frac{1-\beta^{t-\tau^\prime}}{1-\beta}\left(\mathbf{\tilde{G}}_{\tau^\prime}-\mathbf{G}_{\tau^\prime}\right)(\mathbf{I}-\mathbf{Q})\mathbf{W}^{t-\tau^\prime+1}\right\Vert_\mathfrak{F} \nonumber\\
      %
      %
      \overset{\tcircle{2}}{\leq} & \sum_{\tau=1}^{t}\sum_{\tau^\prime=1}^{t} \frac{\rho^{t+1-\frac{\tau+\tau^\prime}{2}}}{(1-\beta)^2} \mathbb{E}\left[\left\Vert\mathbf{\tilde{G}}_{\tau}-\mathbf{G}_{\tau}\right\Vert_\mathfrak{F}  \left\Vert\mathbf{\tilde{G}}_{\tau}-\mathbf{G}_{\tau^\prime}\right\Vert_\mathfrak{F}\right] \nonumber\\
      %
      %
      \overset{\tcircle{3}}{\leq} & \frac{1}{(1-\beta)^2}\sum_{\tau=1}^{t}\sum_{\tau^\prime=1}^{t}\rho^{t+1-\frac{\tau+\tau^\prime}{2}}\mathbb{E}\left[\left\Vert\mathbf{\tilde{G}}_{\tau}-\mathbf{G}_{\tau}\right\Vert_\mathfrak{F}^2 \right] \overset{\tcircle{4}}{\leq} \frac{1}{(1-\beta)^2(1-\sqrt{\rho})}\sum_{\tau=1}^{t}\rho^{\frac{t-\tau+1}{2}}\mathbb{E}\left[\left\Vert\mathbf{\tilde{G}}_{\tau}-\mathbf{G}_{\tau}\right\Vert_\mathfrak{F}^2 \right] \nonumber\\
      \overset{\tcircle{5}}{\leq} & \frac{6S_0C^2}{N^2(1-\beta)^2(1-\sqrt{\rho})^2} 
      + \frac{6S_0\sigma^2C^2d}{N^2B^2(1-\beta)^2(1-\sqrt{\rho})^2} 
      + \frac{27N\alpha C^2}{8(1-\beta)^2(1-\sqrt{\rho})^2} 
      + \frac{10N\xi^2 + 10N\kappa^2} {(1-\beta)^2(1-\sqrt{\rho})^2} \nonumber\\
      & + \frac{20L^2}{(1-\beta)^2(1-\sqrt{\rho})}\sum_{\tau=1}^{t}\rho^{\frac{t-\tau+1}{2}}\sum_{i=1}^N \mathbb{E} \left\Vert  x_{t-1}^i - \bar{x}_{t-1} \right\Vert^2 
      + \frac{10N}{(1-\beta)^2(1-\sqrt{\rho})}\sum_{\tau=1}^{t}\rho^{\frac{t-\tau+1}{2}}\mathbb{E} \left\Vert \frac{1}{N} \sum_{k=1}^{N}\nabla f_k(x_{t-1}^k)\right\Vert^2
    \end{align}
    where $\tcircle{1}$ follows from \textbf{Fact}~\ref{fact:fnorm}, $\tcircle{2}$ follows from the inequality $ab \leq \frac{1}{2}(a^2+b^2)$ for any two real numbers $a$ and $b$ and Fact~\ref{fact:dsm}, $\tcircle{3}$ follows from the fact that $\mathbf{G}_{\tau}$ and $\mathbf{G}_{\tau^\prime}$ are independent of each other, $\tcircle{4}$ holds as $\sum_{\tau=1}^{t} \rho^{\frac{t-\tau+1}{2}} \leq \frac{\sqrt{\rho}}{1-\sqrt{\rho}}$, and $\tcircle{5}$ follows from \textbf{Lemma}~\ref{lem:G_diff} and  $\sum_{\tau_1=1}^{t}\rho^{t+1-\frac{\tau_1+\tau}{2}} \leq (1-\sqrt{\rho})^{-1}\rho^{\frac{t-\tau+1}{2}}$. We next show the upper bound of term \textbf{\Rmnum{2}} by the similar technique as follows:
    \begin{align} \label{eq:x_minus_G}
      & \mathbb{E} \left\Vert \sum_{\tau=1}^{t} \frac{1-\beta^{t-\tau}}{1-\beta}\mathbf{G}_{\tau}(\mathbf{I}-\mathbf{Q})\mathbf{W}^{t-\tau+1}\right\Vert^2_\mathfrak{F}  \nonumber\\
      \leq & \sum_{\tau=1}^{t}\sum_{\tau^\prime=1}^{t} \mathbb{E} \left[ \left\Vert\frac{1-\beta^{t-\tau}}{1-\beta} \mathbf{G}_{\tau}(\mathbf{I}-\mathbf{Q})\mathbf{W}^{t-\tau+1}\right\Vert_\mathfrak{F} \cdot \left\Vert \frac{1-\beta^{t-\tau^\prime}}{1-\beta}\mathbf{G}_{\tau^\prime}(\mathbf{I}-\mathbf{Q})\mathbf{W}^{t-\tau^\prime+1}\right\Vert_\mathfrak{F} \right] \nonumber\\
      \leq & \frac{1}{(1-\beta)^2}\sum_{\tau=1}^{t}\sum_{\tau^\prime=1}^{t}\rho^{t+1-\frac{\tau+\tau^\prime}{2}}\mathbb{E}\left[ \left\Vert\mathbf{G}_{\tau}\right\Vert_\mathfrak{F} \left\Vert \mathbf{G}_{\tau^\prime}\right\Vert_\mathfrak{F}\right] \nonumber\\
      = & \frac{1}{(1-\beta)^2}\sum_{\tau=1}^{t}\sum_{\tau^\prime=1}^{t}\rho^{t+1-\frac{\tau+\tau^\prime}{2}}\mathbb{E}\left[ \left\Vert\mathbf{G}_{\tau}\right\Vert_\mathfrak{F}^2\right] 
      \leq \frac{1}{(1-\beta)^2(1-\sqrt{\rho})}\sum_{\tau=1}^{t}\rho^{\frac{t-\tau + 1}{2}}\mathbb{E}\left[ \left\Vert\mathbf{G}_{\tau}\right\Vert_\mathfrak{F}^2\right].
    \end{align}
    According to the definition of $\mathbf{J}_{\tau}$ in \textbf{Remark}~\ref{rmk:matrixs}, for any $\tau \geq 1$, we have:
    \begin{align}\label{eq:x_minus_EG}
      \mathbb{E}\left[\left\Vert\mathbf{G}_{\tau}\right\Vert_\mathfrak{F}^2\right] \overset{\tcircle{1}}{\leq} &  5\mathbb{E}\left[\left\Vert\mathbf{G}_{\tau} - \mathbf{H}_{\tau}\right\Vert_\mathfrak{F}^2\right] + 5\mathbb{E}\left[\left\Vert\mathbf{H}_{\tau} - \mathbf{J}_{\tau}\right\Vert_\mathfrak{F}^2\right] + 5\mathbb{E}\left[\left\Vert\mathbf{J}_{\tau} - \mathbf{J}_{\tau}\mathbf{Q} \right\Vert_\mathfrak{F}^2\right] \nonumber\\
      & + 5\mathbb{E}\left[\left\Vert(\mathbf{J}_{\tau} - \mathbf{H}_{\tau})\mathbf{Q}\right\Vert_\mathfrak{F}^2\right] +5\mathbb{E}\left[\left\Vert\mathbf{H}_{\tau} \mathbf{Q}\right\Vert_\mathfrak{F}^2\right] \nonumber\\
      \overset{\tcircle{2}}{\leq} & 5N\xi^2 + 5L^2\mathbb{E}\left[\left\Vert{\mathbf{X}}_{\tau-1}(\mathbf{I} - \mathbf{Q})\right\Vert_\mathfrak{F}^2\right] + 5N\kappa^2  \nonumber\\
      & + 5L^2\mathbb{E}\left[\left\Vert {\mathbf{X}}_{\tau-1}(\mathbf{I} - \mathbf{Q})\right\Vert_\mathfrak{F}^2\right] + 5N\mathbb{E}\left\Vert\frac{1}{N}\sum_{k=1}^{N} \nabla f_k({x}^k_{\tau-1})\right\Vert^2 \nonumber\\ 
      & = 10L^2\mathbb{E}\left[\left\Vert{\mathbf{X}}_{\tau-1}(\mathbf{I} - \mathbf{Q})\right\Vert_\mathfrak{F}^2\right] + 5N\xi^2 + 5N\kappa^2  + 5N\mathbb{E}\left[\left\Vert\frac{1}{N}\sum_{k=1}^{N} \nabla f_k({x}^k_{\tau-1})\right\Vert^2\right] 
    \end{align}
    Therein, we have $\tcircle{1}$ by applying the inequality $\left\Vert \sum_{i=1}^{n} \mathbf{A}_i \right\Vert_\mathfrak{F}^2 \leq n \sum_{i=1}^{n} \left\Vert\mathbf{A}_i \right\Vert_\mathfrak{F}^2$, $\tcircle{2}$ by considering $\left\Vert\mathbf{G}_\tau-\mathbf{H}_\tau\right\Vert_\mathfrak{F}^2 = \\ \sum_{i=1}^N \left\Vert \nabla F({x}_{\tau-1}^i;\zeta_i) -\nabla f_i({x}^i_{\tau-1})\right\Vert^2\leq N\xi^2$ (where the inequality follows from \textbf{Assumption}~\ref{assum:variance}), $\left\Vert\mathbf{H}_{\tau} - \mathbf{J}_{\tau}\right\Vert_\mathfrak{F}^2 = \\ \sum_{i=1}^{N} \left\Vert \nabla f_i({x}^i_{\tau-1}) - \nabla f_i(\bar{{x}}_{\tau-1})\right\Vert^2 \leq L^2 \left\Vert{\mathbf{X}}_{\tau-1} (\mathbf{I} - \mathbf{Q})\right\Vert_\mathfrak{F}^2$ (where the inequality holds due to the smoothness of each \(f_i(\cdot)\) as shown in \textbf{Assumption}~\ref{assum:smooth}), $\left\Vert \mathbf{J}_{\tau} - \mathbf{J}_{\tau}\mathbf{Q}\right\Vert_\mathfrak{F}^2$ $= \sum_{i=1}^{N} \left\Vert\nabla f_i(\bar{{x}}_{\tau-1}) - \mathcal{F}(\bar{{x}}_{\tau-1})\right\Vert^2 \leq N\kappa^2$ (where the inequality follows from \textbf{Assumption}~\ref{assum:variance}), and the facts that $\left\Vert(\mathbf{J}_{\tau} - \mathbf{H}_{\tau})\mathbf{Q}\right\Vert_\mathfrak{F}^2 \leq \left\Vert\mathbf{H}_{\tau} - \mathbf{J}_{\tau}\right\Vert_\mathfrak{F}^2$ and $\mathbb{E}\left[\left\Vert\mathbf{H}_\tau \mathbf{Q}\right\Vert^2_\mathfrak{F}\right]\leq N\mathbb{E}\left[\left\Vert\frac{1}{N}\sum_{k=1}^N\nabla f_k({x}^k_{\tau-1})\right\Vert^2\right]$. By substituting (\ref{eq:x_minus_EG}) into (\ref{eq:x_minus_G}), we get
    \begin{align} \label{eq:x_minus_G_fin}
      &\mathbb{E}\left[\left\Vert\sum_{\tau=1}^{t}\frac{1-\beta^{t-\tau}}{1-\beta}\mathbf{G}_{\tau}(\mathbf{I}-\mathbf{Q})\mathbf{W}^{t-\tau+1}\right\Vert^2_\mathfrak{F}\right] \nonumber\\
      \leq& \frac{5N}{(1-\beta)^2(1-\sqrt{\rho})}\sum_{\tau=1}^{t}\rho^{\frac{t-\tau+1}{2}}  \mathbb{E}\left[\left\Vert\frac{1}{N}\sum_{k=1}^N \nabla f_k({x}^k_{\tau-1})\right\Vert^2 \right] \nonumber\\
      & + \frac{10L^2}{(1-\beta)^2(1-\sqrt{\rho})}\sum_{\tau=1}^{t}\rho^{\frac{t-\tau+1}{2}} \mathbb{E}\left[\big\Vert{{\mathbf{X}}}_{\tau-1}(\mathbf{I} - \mathbf{Q})\big\Vert_\mathfrak{F}^2\right] 
      + \frac{5\sqrt{\rho}N(\xi^2+\kappa^2)}{(1-\beta)^2(1-\sqrt{\rho})^2}
    \end{align}
    We then substitute (\ref{eq:x_minus_G_fin}) and (\ref{eq:x_minus_G_minus}) into (\ref{eq:x_minus}) and obtain
    \begin{align} \label{eq:x_minus_fin}
      \mathbb{E}\left[\left\Vert{\mathbf{X}}_t(\mathbf{I}-\mathbf{Q})\right\Vert^2_\mathfrak{F}\right] 
      \leq & \underbrace{ \frac{12S_0C^2 + \frac{12}{B^2}S_0\sigma^2C^2d}{N^2(1-\beta)^2(1-\sqrt{\rho})^2}\eta^2 
      + \frac{27\eta^2N\alpha C^2}{4(1-\beta)^2(1-\sqrt{\rho})^2} 
      + \frac{30N\xi^2 + 30N\kappa^2} {(1-\beta)^2(1-\sqrt{\rho})^2} \eta^2}_{p} \nonumber\\
      & + \frac{60\eta^2L^2}{(1-\beta)^2(1-\sqrt{\rho})}\sum_{\tau=1}^{t}\rho^{\frac{t-\tau+1}{2}}\sum_{i=1}^N \mathbb{E} \left\Vert  x_{\tau-1}^i - \bar{x}_{\tau-1} \right\Vert^2 \nonumber\\
      & +  \frac{30\eta^2N}{(1-\beta)^2(1-\sqrt{\rho})}\sum_{\tau=1}^{t}\rho^{\frac{t-\tau+1}{2}}\mathbb{E} \left\Vert \frac{1}{N} \sum_{k=1}^{N}\nabla f_k(x_{\tau-1}^k)\right\Vert^2
    \end{align}
    Summing over $t\in\{1,2,\dots, T-1\}$ and considering that $\mathbb{E}\left[\left\Vert{\mathbf{X}}_0(\mathbf{I}-\mathbf{Q})\right\Vert^2_\mathfrak{F}\right] = 0$, we have
    \begin{align}\label{eq:x_minus_t}
      & (1-\beta)^2(1-\sqrt{\rho}) \left[ \sum_{t=0}^{T-1}\mathbb{E}\left[\left\Vert{\mathbf{X}}_t(\mathbf{I}-\mathbf{Q})\right\Vert^2_\mathfrak{F}\right] - pT  \right] \nonumber \\
      =& (1-\beta)^2(1-\sqrt{\rho}) \left[ \sum_{t=1}^{T}\sum_{i=1}^N \mathbb{E} \left[ \left\Vert x_{t-1}^i - \bar{x}_{t-1} \right\Vert^2 \right] - pT \right] \nonumber\\
      \overset{\tcircle{1}}{\leq} & 60\eta^2L^2 \sum_{t=1}^{T-1} \sum_{\tau=1}^{t} \rho^{\frac{t-\tau+1}{2}}\sum_{i=1}^N \mathbb{E} \left\Vert  x_{\tau-1}^i - \bar{x}_{\tau-1} \right\Vert^2 
      + 30\eta^2N \sum_{t=1}^{T-1}\sum_{\tau=1}^{t}\rho^{\frac{t-\tau+1}{2}}\mathbb{E} \left\Vert \frac{1}{N} \sum_{k=1}^{N}\nabla f_k(x_{\tau-1}^k)\right\Vert^2 \nonumber\\
      \leq & 60\eta^2L^2 \sum_{t=1}^{T} \sum_{\tau=1}^{t} \rho^{\frac{t-\tau+1}{2}}\sum_{i=1}^N \mathbb{E} \left\Vert  x_{\tau-1}^i - \bar{x}_{\tau-1} \right\Vert^2 
      + 30\eta^2N \sum_{t=1}^{T}\sum_{\tau=1}^{t}\rho^{\frac{t-\tau+1}{2}}\mathbb{E} \left\Vert \frac{1}{N} \sum_{k=1}^{N}\nabla f_k(x_{\tau-1}^k)\right\Vert^2 \nonumber \\
      =& 60\eta^2L^2 \sum_{t=1}^{T}\sum_{\tau=t}^{T}\rho^{\frac{T-\tau+1}{2}} \sum_{i=1}^N \mathbb{E} \left\Vert  x_{t-1}^i-\bar{x}_{t-1} \right\Vert^2 
      + 30\eta^2N \sum_{t=1}^{T}\sum_{\tau=t}^{T}\rho^{\frac{T-\tau+1}{2}} \mathbb{E} \left\Vert \frac{1}{N} \sum_{k=1}^{N}\nabla f_k(x_{t-1}^k)\right\Vert^2 \nonumber\\
      \overset{\tcircle{2}}{\leq} & 60\eta^2L^2 \sum_{t=1}^{T} \frac{1-\rho^{\frac{T-t}{2}}}{1-\sqrt{\rho}} \sum_{i=1}^N \mathbb{E} \left\Vert  x_{t-1}^i - \bar{x}_{t-1} \right\Vert^2
      + 30\eta^2N \sum_{t=1}^{T} \frac{1-\rho^{\frac{T-t}{2}}} {1-\sqrt{\rho}} \mathbb{E} \left\Vert \frac{1}{N} \sum_{k=1}^{N}\nabla f_k(x_{t-1}^k)\right\Vert^2 \nonumber \\
      \leq & \frac{60\eta^2L^2}{1-\sqrt{\rho}}\sum_{t=1}^{T} \sum_{i=1}^N \mathbb{E} \left\Vert  x_{t-1}^i - \bar{x}_{t-1} \right\Vert^2 
      + \frac{30\eta^2N}{1-\sqrt{\rho}}\sum_{t=1}^{T} \mathbb{E} \left\Vert \frac{1}{N} \sum_{k=1}^{N}\nabla f_k(x_{t-1}^k)\right\Vert^2        
    \end{align}
    %
    Therein, $\tcircle{1}$ follows (\ref{eq:x_minus_fin}), $\tcircle{2}$ is derived by applying $\sum_{\tau=t}^{T}\rho^{\frac{T-\tau+1}{2}} \leq \sqrt{\rho} \cdot \frac{1-\rho^{\frac{T-t}{2}}}{1-\sqrt{\rho}} \leq \frac{1-\rho^{\frac{T-t}{2}}}{1-\sqrt{\rho}}$. According to (\ref{eq:x_minus_t}), we can transform (\ref{eq:x_minus_fin}) into:
    \begin{align}
       & \left( 1-\frac{60L^2\eta^2}{(1-\beta)^2(1-\sqrt{\rho})^2} \right) \sum_{t=1}^{T}\frac{1}{N}\sum_{i=1}^N\mathbb{E}\left[\left\Vert\bar{{x}}_{t-1} - {x}^i_{t-1} \right\Vert^2\right]  \nonumber \\
       \leq & \frac{12S_0C^2 + \frac{12}{B^2}S_0\sigma^2C^2d}{N^3(1-\beta)^2(1-\sqrt{\rho})^2}\eta^2T + \frac{27\alpha C^2}{4(1-\beta)^2(1-\sqrt{\rho})^2}\eta^2T+ \frac{30\xi^2 + 30\kappa^2} {(1-\beta)^2(1-\sqrt{\rho})^2}\eta^2T \nonumber\\
       & + \frac{30\eta^2}{(1-\beta)^2(1-\sqrt{\rho})^2}\sum_{t=1}^{T} \mathbb{E} \left\Vert \frac{1}{N} \sum_{k=1}^{N}\nabla f_k(x_{t-1}^k)\right\Vert^2
    \end{align}
    Based on $\eta \leq \frac{(1-\beta)(1-\rho)}{2\sqrt{30}L}$, we divide both sides by $N$ and a positive number $ 1-\frac{60L^2\eta^2}{(1-\beta)^2(1-\sqrt{\rho})^2} \geq \frac{1}{2}$, and we can get the final version of \textbf{Lemma}~\ref{lem:consensus_error}:
    \begin{align*}\label{lemma1_fin}
        \frac{1}{N}\sum_{t=1}^{T}\sum_{i=1}^N\mathbb{E}\left[\left\Vert\bar{{x}}_{t-1} - {x}^i_{t-1} \right\Vert^2\right] 
        & \leq \frac{\left(24S_0C^2 + \frac{24}{B^2}S_0 \sigma^2 C^2d \right)\eta^2} {N^3(1-\beta)^2(1-\sqrt{\rho})^2} T + \frac{\left(\frac{27\alpha C^2}{2} + 60B^2(\xi^2 +\kappa^2) \right) \eta^2} {(1-\beta)^2(1-\sqrt{\rho})^2B^2}T \\
        & \quad +\frac{60\eta^2}{(1-\beta)^2(1-\sqrt{\rho})^2}\sum_{t=1}^{T} \mathbb{E} \left\Vert \frac{1}{N} \sum_{k=1}^{N}\nabla f_k(x_{t-1}^k)\right\Vert^2
    \end{align*}
    
    \end{proof}


    Next, we summarize the updating process of momentum and parameters, and analyze the rules therein. We first define an auxiliary sequence in \textbf{Remark}~\ref{rmk:u_seq}. Based on \textbf{Remark}~\ref{rmk:u_seq}, we propose \textbf{Lemma}~\ref{lem:u_seq} to depict the iteration relation of the auxiliary sequence and \textbf{Lemma}~\ref{lem:u_x_diff} to describe the overall difference between auxiliary sequence and model parameters in relation to the calibrated gradient.
  \begin{remark}\label{rmk:u_seq}
    Recall $\bar{x}_t = \frac{1}{N} \sum_{i=1}^N x^i_t$. We define $\bar{u}_t$ as
    \begin{equation}
      \bar{u}_t =\begin{cases} 
        \frac{1}{1-\beta}\bar{x}_t-\frac{\beta}{1-\beta}\bar{x}_{t-1} & \text{if} \quad t\geq 1 \\
        \bar{x}_t & \text{if} \quad  t=0
      \end{cases}
    \end{equation}
  \end{remark}

  \begin{lem} \label{lem:u_seq}
    For any $t \geq 0$, we have
    \begin{equation} \label{eq:u_seq}
      \bar{u}_{t+1}-\bar{u}_t = -\frac{\eta}{(1-\beta)N} \sum_{i=1}^N{\tilde{g}}^i_{t+1}
    \end{equation}
  \end{lem}
  \begin{proof}
    Let $\bar{v}_t = \frac{1}{N}\sum_{i=1}^N {v}^i_t$, then we have
    \begin{align} \label{eq:v_seq}
      \bar{v}_t = & \frac{1}{N}\sum_{i=1}^N \sum_{j \in N(i)} w_{ij} \left(\beta {v}^j_{t-1} + {\tilde{g}}^{j}_t \right) \nonumber\\
      = & \frac{1}{N}\sum_{i=1}^N \left( \left(\beta {v}^i_{t-1} + {\tilde{g}}_t^i \right)\cdot \sum_{j \in N(i)} w_{ij} \right) \nonumber\\
      = & \frac{1}{N}\sum_{i=1}^N \left(\beta {v}^i_{t-1} + {\tilde{g}}_t^i \right) \nonumber\\
      = & \beta \bar{v}_{t-1} + \frac{1}{N}\sum_{i=1}^N{\tilde{g}}^i_{t}
    \end{align}
    According to the definition of $\bar{x}_t$, we also have
    \begin{align}\label{eq:x_seq}
      \bar{x}_t & = \frac{1}{N}\sum_{i=1}^N \sum_{j \in N(i)} w_{ij} \left( {x}^j_{t-1} - \eta{\tilde{v}}^j_t \right) \nonumber\\
      & = \frac{1}{N}\sum_{i=1}^N {x}^i_{t-1} \left( \sum_{j \in N(i)} w_{ij} \right) - \frac{\eta}{N}\sum_{i=1}^N \sum_{j \in \mathcal{N}_i} w_{ij} {\tilde{v}}^i_t \nonumber\\
      & = \frac{1}{N}\sum_{i=1}^N {x}^i_{t-1} - \frac{\eta}{N}\sum_{i=1}^N {v}^i_t \nonumber\\
      & = \bar{{x}}_{t-1} - \eta\bar{v}_{t}        
    \end{align}
    We then prove this lemma by induction. For $t = 0$,
    \begin{align*}
      \bar{u}_{t+1}-\bar{u}_t = \bar{u}_{1}-\bar{u}_0 = \frac{1}{1-\beta}\bar{{x}}_{1}-\frac{\beta}{1-\beta}\bar{{x}}_{0} - \bar{{x}}_{0} = \frac{1}{1-\beta}(\bar{{x}}_{1} - \bar{{x}}_{0}) = \frac{-\eta}{N(1-\beta)}\sum_{i=1}^N{\tilde{g}}^i_{1}
    \end{align*}
    while for $t\geq1$, by considering (\ref{eq:x_seq}), we have
    \begin{align*}
      \bar{u}_{t+1}-\bar{u}_t &= \frac{1}{1-\beta}\bar{{x}}_{t+1}- \frac{\beta}{1-\beta}\bar{{x}}_{t} - \frac{1}{1-\beta}\bar{{x}}_t+ \frac{\beta}{1-\beta}\bar{{x}}_{t-1} \\ 
      &= \frac{1}{1-\beta}\big( \left(\bar{{x}}_{t+1}- \bar{{x}}_{t}\right) - \beta \left(\bar{{x}}_{t}-\bar{{x}}_{t-1}\right) \big) \\
      &= \frac{1}{1-\beta} \left( \bar{v}_{t+1} - \beta\bar{v}_{t} \right) \\
      &= \frac{-\eta}{N(1-\beta)}\sum_{i=1}^N{\tilde{g}}^i_{t+1}    
    \end{align*}
  \end{proof}

  \begin{lem} \label{lem:u_x_diff}
    For any $T\geq 1$, we have
    \begin{equation} \label{eq:u_x_diff}
     \sum_{t=1}^{T}\Vert\bar{u}_t-\bar{{x}}_t\Vert^2 \leq \frac{\eta^2\beta^2}{(1-\beta)^4}\sum_{t=1}^{T}\left\Vert\frac{1}{N}\sum_{i=1}^N{\tilde{g}}^{i}_{t}\right\Vert^2.
    \end{equation}
  \end{lem}
  \begin{proof}
    According to the definition of $\bar{u}_t$ (see \textbf{Remark}~\ref{rmk:u_seq}), for any $t \geq 1$, we have
    \begin{equation}\label{eq:u_x_distance}
      \bar{u}_t -\bar{{x}}_t 
      = \frac{\beta}{1-\beta}\left(\bar{{x}}_t-\bar{{x}}_{t-1}\right) 
      = \frac{\eta\beta}{1-\beta}\bar{v}_t
      \overset{\tcircle{1}}{=}\frac{-\eta\beta}{1-\beta}\sum_{\tau=1}^{t}\beta^{t-\tau}\left(\frac{1}{N}\sum_{i=1}^N{\tilde{g}}^i_\tau\right)
    \end{equation}
    where we have $\tcircle{1}$ by recursively applying (\ref{eq:v_seq}) for $\bar{v}_t$. Based on (\ref{eq:u_x_distance}), we have
    \begin{align} \label{eq:u_x_diff_per}
      \left\Vert \bar{u}_t -\bar{{x}}_t \right\Vert^2 & = \frac{\eta^2\beta^2}{(1-\beta)^2} \left(\sum_{\tau^{\prime}=1}^{t}\beta^{t-\tau^{\prime}} \right)^2 \left\Vert\sum_{\tau=1}^{t}\frac{\beta^{t-\tau}}{\sum_{\tau^{\prime}=1}^{t}\beta^{t-\tau^{\prime}}} \left(\frac{1}{N}\sum_{i=1}^N{\tilde{g}}^i_\tau \right)\right\Vert^2  \nonumber\\
      & \overset{\tcircle{1}}{\leq} \frac{\eta^2\beta^2}{(1-\beta)^2} \left(\sum_{\tau^{\prime}=1}^{t}\beta^{t-\tau^{\prime}} \right)^2 \sum_{\tau=1}^{t}\frac{\beta^{t-\tau}}{\sum_{\tau^{\prime}=1}^{t}\beta^{t-\tau^{\prime}}}\left\Vert \frac{1}{N}\sum_{i=1}^N{\tilde{g}}^i_\tau \right\Vert^2 \nonumber\\ 
      & \overset{\tcircle{2}}{=} \frac{\eta^2\beta^2(1-\beta^t)}{(1-\beta)^3} \sum_{\tau=1}^{t}\beta^{t-\tau}\left\Vert \frac{1}{N}\sum_{i=1}^N{\tilde{g}}^i_\tau \right\Vert^2 \nonumber\\
      & \overset{\tcircle{3}}{\leq} \frac{\eta^2\beta^2}{(1-\beta)^3}\sum_{\tau=1}^{t}\beta^{t-\tau}\left\Vert \frac{1}{N}\sum_{i=1}^N{\tilde{g}}^i_\tau \right\Vert^2        
    \end{align}
    where we have $\tcircle{1}$ according to Jensen's Inequality, $\tcircle{2}$ by considering $\sum_{\tau^{\prime}=1}^{t}\beta^{t-\tau^{\prime}} =\frac{1-{\beta}^t}{1-{\beta}}$, $\tcircle{3}$ according to the fact that $1-\beta^t \leq 1$. Given any $T \geq 1$, we sum (\ref{eq:u_x_diff_per}) over $t=1,2,\cdots,T$ such that
    \begin{align}\label{eq:u_x_diff-fin}
      \sum_{t=1}^{T}\left\Vert \bar{u}_t -\bar{{x}}_t \right\Vert^2 & \leq  \frac{\eta^2\beta^2}{(1-\beta)^3}\sum_{t=1}^{T}\sum_{\tau=1}^{t}\beta^{t-\tau}\left\Vert \frac{1}{N}\sum_{i=1}^N{\tilde{g}}^i_\tau \right\Vert^2  \nonumber\\ 
      & = \frac{\eta^2\beta^2}{(1-\beta)^3}\sum_{t=1}^{T}\left\Vert \frac{1}{N}\sum_{i=1}^N {\tilde{g}}^{i}_{t} \right\Vert^2 \sum_{\tau=t}^{T} \beta^{\tau - t} \nonumber\\ 
      &\overset{\tcircle{1}}{\leq} \frac{\eta^2\beta^2}{(1-\beta)^4}\sum_{t=1}^{T}\left\Vert \frac{1}{N}\sum_{i=1}^N{\tilde{g}}^{i}_{t} \right\Vert^2
      %
    \end{align}
    where we have $\tcircle{1}$ since $\sum_{\tau=t}^{T-1} \beta^{\tau - t} = \frac{1-\beta^{\tau - t}}{1-\beta} \leq \frac{1}{1-\beta}$.
  \end{proof}

  Now, we are ready to prove \textbf{Theorem}~\ref{thm:convergence}. Since each $f_i(\cdot)$ is \textit{L-smoothness}, $\mathcal{F}$ also satisfies \textit{L-smoothness} which has an expectation form as:
  \begin{equation} \label{eq:main_pre}
    \mathbb{E}\left[  \mathcal{F}\left( \bar{u}_{t}\right)\right] 
    \leq \mathbb{E}\left[  \mathcal{F}\left(\bar{u}_{t-1}\right)\right] 
    + \mathbb{E}\left[\big\langle  \nabla\mathcal{F}\left(\bar{u}_{t-1}\right), \bar{u}_{t} - \bar{u}_{t-1}\big\rangle \right] 
    + \frac{L}{2} \mathbb{E}\left[ \big\| \bar{u}_{t} - \bar{u}_{t-1} \big\|^2 \right]
  \end{equation}
  According to \textbf{Lemma}~\ref{lem:u_seq}, $\mathbb{E}\left[\big\langle  \nabla\mathcal{F}\left(\bar{u}_{t-1}\right), \bar{u}_{t} - \bar{u}_{t-1}\big\rangle \right]$ can be re-written as
  \begin{align}\label{eq:linear_term}
    & \mathbb{E}\left[\big\langle  \nabla\mathcal{F}\left(\bar{u}_{t-1}\right), \bar{u}_{t}- \bar{u}_{t-1} \big\rangle \right] 
    = \frac{-\eta}{1-\beta}\mathbb{E}\left[\left\langle\nabla\mathcal{F}\left(\bar{u}_{t-1}\right),\frac{1}{N}\sum_{i=1}^{N}{\tilde{g}}^i_{t}\right\rangle\right] \nonumber\\
    = & \frac{-\eta}{1-\beta}\mathbb{E}\left[\left\langle\nabla\mathcal{F}\left(\bar{u}_{t-1}\right) - \nabla\mathcal{F}\left(\bar{{x}}_{t-1}\right),\frac{1}{N}\sum_{i=1}^{N}{\tilde{g}}^i_{t}\right\rangle\right] - \frac{\eta}{1-\beta}\mathbb{E}\left[\left\langle\nabla\mathcal{F}\left(\bar{{x}}_{t-1}\right),\frac{1}{N}\sum_{i=1}^{N}{\tilde{g}}^i_{t}\right\rangle\right] \nonumber\\
    \leq & \underbrace{\frac{-\eta}{1-\beta}\mathbb{E}\left[\left\langle\nabla\mathcal{F}\left(\bar{u}_{t-1}\right)- \nabla\mathcal{F}\left(\bar{{x}}_{t-1}\right), \frac{1}{N}\sum_{i=1}^{N}{\tilde{g}}^i_{t}\right\rangle\right] }_{\mathbf{\Rmnum{3}}}  \nonumber\\
    & - \underbrace{\frac{\eta}{1-\beta} \mathbb{E}\left[ \left\langle\nabla \mathcal{F}\left(\bar{{x}}_{t-1}\right), \frac{1}{N} \sum_{i=1}^{N}\left({\tilde{g}}_{t}^{i} - \frac{1}{B} \sum_b \mathbf{g}_{t,b}^{ii}\right)\right\rangle \right]}_{\mathbf{\Rmnum{4}}} 
    - \underbrace{ \frac{\eta}{1-\beta} \mathbb{E}\left[ \left\langle\nabla \mathcal{F}\left(\bar{{x}}_{t-1}\right) , \frac{1}{N} \sum_{i=1}^{N} \frac{1}{B} \sum_b \mathbf{g}_{t,b}^{ii}\right\rangle \right]}_{\mathbf{\Rmnum{5}}}        
  \end{align}
  It is shown by the above inequality that $\mathbb{E}\left[\big\langle  \nabla\mathcal{F}\left(\bar{u}_{t-1}\right), \bar{u}_{t} - \bar{u}_{t-1}\big\rangle \right]$ can be bounded by $\mathbf{\Rmnum{3}}$, $\mathbf{\Rmnum{4}}$ and $\mathbf{\Rmnum{5}}$. We derive the bound of $\mathbf{\Rmnum{3}}$ by
  \begin{align}\label{eq:linear_term_1}
    & \frac{-\eta}{1-\beta}\mathbb{E}\left[\left\langle\nabla\mathcal{F}\left(\bar{u}_{t-1}\right)- \nabla\mathcal{F}\left(\bar{{x}}_{t-1}\right),\frac{1}{N}\sum_{i=1}^{N}{\tilde{g}}^i_{t} \right\rangle\right] \nonumber\\
    \overset{\tcircle{1}}{\leq} & \frac{(1-\beta)}{2\beta L}\mathbb{E}\left[  \big\Vert \nabla\mathcal{F}(\bar{u}_{t-1})- \nabla\mathcal{F}(\bar{{x}}_{t-1})\big\Vert^2 \right] + \frac{\beta L \eta^{2}}{2(1-\beta)^{3}}\mathbb{E}\left[\left\Vert \frac{1}{N} \sum_{i=1}^{N} {\tilde{g}}_{t}^{i}\right\Vert^{2} \right] \nonumber\\
    \leq & \frac{(1-\beta)L}{2 \beta}\mathbb{E}\ \big\Vert  \bar{u}_{t-1}-\bar{{x}}_{t-1}\big\Vert^{2} + \frac{\beta L \eta^2}{2(1-\beta)^{3}}\mathbb{E} \left\Vert\frac{1}{N} \sum_{i=1}^{N} {\tilde{g}}_{t}^{i}\right\Vert^{2}
  \end{align}
  where we have $\tcircle{1}$ by applying the inequality $\left\langle \mathbf{a},\mathbf{b} \right\rangle \leq \frac{1}{2} \| \mathbf{a} \|^2 + \frac{1}{2}\| \mathbf{b} \|^2$ where $\mathbf{a}=\sqrt{\frac{1-\beta}{\beta L}} \big( \nabla\mathcal{F}(\bar{u}_{t-1})- \nabla\mathcal{F}(\bar{{x}}_{t-1})\big)$ and $\mathbf{b} = -\frac{\eta \sqrt{\beta L}}{(1-\beta)^{\frac{3}{2}}}\frac{1}{N} \sum_{i=1}^{N}{\tilde{g}}^i_{t}$. $\mathbf{\Rmnum{4}}$ is bounded by
  \begin{align}\label{eq:linear_term_2}
    & \frac{-\eta}{(1-\beta)}\mathbb{E}\left[\left\langle\nabla \mathcal{F}\left(\bar{{x}}_{t-1}\right), \frac{1}{N} \sum_{i=1}^{N} \left({\tilde{g}}_{t}^{i} - \frac{1}{B} \sum_b \mathbf{g}_{t,b}^{ii} \right)\right\rangle\right] \nonumber\\
    \overset{\tcircle{1}}{\leq} & \frac{(1-\beta)}{32}\mathbb{E}\ \big\Vert \nabla \mathcal{F}(\bar{{x}}_{t-1})\big\Vert^{2} + \frac{8\eta^2 }{(1-\beta)^{3}}\mathbb{E}\left\Vert\frac{1}{N} \sum_{i=1}^{N} ({\tilde{g}}_{t}^{i} - \frac{1}{B} \sum_b \mathbf{g}_{t,b}^{ii})\right\Vert^{2}
  \end{align}
  where we have $\tcircle{3}$ by applying the inequality $\left\langle \mathbf{a},\mathbf{b} \right\rangle \leq \frac{1}{2} \| \mathbf{a} \|^2 + \frac{1}{2} \| \mathbf{b} \|^2$ again, where $\mathbf{a}= \frac{\sqrt{1-\beta}}{4} \nabla \mathcal{F}(\bar{{x}}_{t-1})$ and $\mathbf{b} = -\frac{4\eta}{(1-\beta)^{\frac{3}{2}}}\frac{1}{N} \sum_{i=1}^{N}\big({\tilde{g}}_{t}^{i}-\frac{1}{B} \sum_b \mathbf{g}_{t,b}^{ii}\big)$. The bound of $\mathbf{\Rmnum{5}}$ is as follows
  \begin{align}\label{eq:linear_term_3}
    & \mathbb{E}\left[\left\langle\nabla \mathcal{F}\left(\bar{{x}}_{t-1}\right) , \frac{1}{N} \sum_{i=1}^{N}\frac{1}{B} \sum_b \mathbf{g}_{t,b}^{ii} \right\rangle\right] \nonumber\\
    = & \mathbb{E}\left[\left\langle \nabla\mathcal{F}(\bar{{x}}_{t-1}),\frac{1}{N}\sum_{i=1}^{N}\nabla f_i({x}^i_{t-1})\right\rangle \right] \nonumber\\
    \overset{\tcircle{1}}{=}& \frac{1}{2}\mathbb{E}\ \big\Vert \nabla\mathcal{F} \left( \bar{{x}}_{t-1}\right) \big\Vert^{2} + \frac{1}{2} \mathbb{E} \left\Vert \frac{1}{N}\sum_{i=1}^{N} \nabla f_{i}({x}_{t-1}^{i})\right\Vert^{2} - \frac{1}{2} \mathbb{E}\ \bigg\Vert \nabla \mathcal{F}(\bar{{x}}_{t-1})-\frac{1}{N} \sum_{i=1}^{N} \nabla f_{i}({x}_{t-1}^{i}) \bigg\Vert^{2} \nonumber\\
    \overset{\tcircle{2}}{\geq} & \frac{1}{2} \mathbb{E}\ \big\Vert  \nabla \mathcal{F}(\bar{{x}}_{t-1})\big\Vert^{2} + \frac{1}{2}\mathbb{E}\ \bigg\Vert\frac{1}{N}\sum_{i=1}^{N} \nabla f_{i}({x}_{t-1}^{i})\bigg\Vert^{2} - \frac{L^{2}}{2N} \mathbb{E} \sum_{i=1}^{N}\left\Vert \bar{{x}}_{t-1}-{x}_{t-1}^{i}\right\Vert^{2}
  \end{align}
  where we have $\tcircle{1}$ comes according to the inequality $\left\langle \mathbf{a}, \mathbf{b} \right\rangle = \frac{1}{2}\left(\left\Vert \mathbf{a} \right\Vert^2 + \left\Vert \mathbf{b} \right\Vert^2 - \left\Vert \mathbf{a} - \mathbf{b} \right\Vert^2\right)$ where $\mathbf{a} = \nabla\mathcal{F} \left( \bar{{x}}_{t-1}\right)$ and $\mathbf{b}= \frac{1}{N}\sum_{i=1}^{N}\nabla f_i({x}^i_{t-1})$, and $\tcircle{2}$ according to 
  \begin{align*}
    \left\Vert \nabla\mathcal{F}(\bar{{x}}_{t-1})- \frac{1}{N}\sum_{i=1}^{N}\nabla f_i({x}^i_{t-1})\right\Vert^2 
    =& \left\Vert \frac{1}{N}\sum_{i=1}^{N}\nabla f_i(\bar{{x}}_{t-1}) - \frac{1}{N}\sum_{i=1}^{N}\nabla f_i({x}^i_{t-1})\right\Vert^2   \nonumber\\
    \leq& \ \frac{1}{N}\sum_{i=1}^{N} \left\Vert\nabla f_i(\bar{{x}}_{t-1}) - \nabla f_i({x}^i_{t-1})\right\Vert^2   \nonumber\\
    \leq& \frac{1}{N} \sum_{i=1}^{N} L^2 \left\Vert \bar{{x}}_{t-1}- {x}^i_{t-1}\right\Vert^2
  \end{align*}
  By substituting (\ref{eq:linear_term_1}), (\ref{eq:linear_term_2}) and (\ref{eq:linear_term_3}) into (\ref{eq:linear_term}) and then substituting (\ref{eq:u_seq}), (\ref{eq:u_x_diff}) and (\ref{eq:linear_term}) into (\ref{eq:main_pre}), we have
  \begin{align}\label{eq:main_mix}
    \mathbb{E}\left[ \mathcal{F}(\bar{u}_{t})\right] 
    \leq & \mathbb{E}\left[ \mathcal{F}(\bar{u}_{t-1})\right] + \frac{(1-\beta)L}{2\beta}\mathbb{E}\left[\left\Vert\bar{u}_{t-1}-\bar{{x}}_{t-1} \right\Vert^2\right] \nonumber\\
    & +\left(\frac{\beta L \eta^2}{2(1-\beta)^3} + \frac{L\eta^2}{2(1-\beta)^2}\right) \mathbb{E}\left[\left\Vert\frac{1}{N}\sum_{i=1}^{N}{\tilde{g}}^i_{t} \right\Vert^2\right] \nonumber\\
    & + \left(\frac{(1-\beta)}{32} - \frac{\eta}{(1-\beta)}\right) \mathbb{E}\left[\left\Vert\nabla \mathcal{F}(\bar{{x}}_{t-1})\right\Vert^2 \right] - \frac{\eta}{2(1-\beta)}\mathbb{E}\left[\left\Vert\frac{1}{N}\sum_{i=1}^{N}\nabla f_i({x}^i_{t-1})\right\Vert^2\right] \nonumber\\
    & + \frac{8\eta^2}{2(1-\beta)^3} \mathbb{E}\left[\left\Vert\frac{1}{N}\sum_{i=1}^{N}\left({\tilde{g}}^i_{t}-\frac{1}{B} \sum_b \mathbf{g}_{t,b}^{ii}\right)\right\Vert^2 \right] 
    + \frac{\eta L^2}{2(1-\beta)} \frac{1}{N}\sum_{i=1}^{N} \mathbb{E}\left[\left\Vert\bar{{x}}_{t-1}-{x}^i_{t-1} \right\Vert^2\right]
  \end{align}
  According to (\ref{eq:conconstants}), the above inequality is re-written as
  \begin{align} \label{eq:main_rearange}
    \mathbb{E}\left[\left\Vert\nabla\mathcal{F}(\bar{{x}}_{t-1})\right\Vert^2\right] 
    & \leq \frac{1}{C_1}\bigg(\mathbb{E}\left[ \mathcal{F}(\bar{u}_{t-1}) \right] - \mathbb{E}\left[ \mathcal{F}(\bar{u}_{t})\right]\bigg) 
    + C_2\:\mathbb{E}\left[\left\Vert\bar{u}_{t-1}-\bar{{x}}_{t-1}\right\Vert^2\right] \nonumber\\
    &\quad + C_3 \:\mathbb{E}\left[\left\Vert\frac{1}{N} \sum_{i=1}^{N} {\tilde{g}}^i_{t} \right\Vert^2\right] 
    - C_4\: \mathbb{E}\left[\left\Vert\frac{1}{N}\sum_{i=1}^{N}\nabla f_i({x}^i_{t-1})\right\Vert^2\right] \nonumber\\
    &\quad + C_5\:\mathbb{E}\left[\left\Vert\frac{1}{N}\sum_{i=1}^{N} \left({\tilde{g}}^i_{t}-\frac{1}{B} \sum_b \mathbf{g}_{t,b}^{ii}\right) \right\Vert^2\right] 
    + C_6\: \frac{1}{N}\sum_{i=1}^{N}\mathbb{E}\left[\left\Vert\bar{{x}}_{t-1}-{x}^i_{t-1}\right\Vert^2\right]        
  \end{align}
  %
  %
  when
  \begin{equation} \label{eq:c1condition0}
    C_1 = \frac{\eta}{2(1-\beta)} - \frac{1-\beta}{32} > 0
  \end{equation}
  and we thus have
  \begin{align*}
    \sum_{t=1}^{T} \mathbb{E}\left[\left\Vert\nabla \mathcal{F}\left(\bar{{x}}_{t-1}\right)\right\Vert^{2}\right] 
    &\leq \frac{1}{C_{1}}\bigg(\mathbb{E}\left[ \mathcal{F}\left(\bar{u}_{0} \right)\right]-\mathbb{E}\left[\mathcal{F}\left(\bar{u}_{t}\right)\right] \bigg) 
    + \underbrace{C_{2} \sum_{t=1}^{T} \mathbb{E}\left[\left\Vert\bar{u}_{t-1} -\bar{{x}}_{t-1}\right\Vert^{2}\right]}_{\mathbf{Lemma}~\ref{lem:u_x_diff}} \\
    &\quad + \underbrace{C_{3} \sum_{t=1}^{T} \mathbb{E}\left[ \left\Vert\frac{1}{N} \sum_{i=1}^{N} {\tilde{g}}_{t}^{i} \right\Vert^{2} \right]}_{\mathbf{Lemma}~\ref{lem:avg_clbr_grad}} 
    - C_{4} \sum_{t=1}^{T} \mathbb{E}\left[\left\Vert\frac{1}{N} \sum_{i=1}^{N} \nabla f_{i}\left({x}_{t-1}^{i}\right) \right\Vert^{2}\right] \\
    &\quad + \underbrace{C_{5} \sum_{t=1}^{T} \mathbb{E}\left[\left\Vert\frac{1}{N} \sum_{i=1}^{N} \left({\tilde{g}}^i_{t}-\frac{1}{B} \sum_b \mathbf{g}_{t,b}^{ii}\right) \right\Vert^{2}\right]}_{\mathbf{Lemma}~\ref{lem:avg_grad_diff}} 
    + \underbrace{C_{6} \sum_{t=1}^{T} \frac{1}{N} \sum_{i=1}^{N} \mathbb{E}\left[\left\Vert\bar{{x}}_{t-1}-{x}_{t-1}^{i}\right\Vert^{2}\right]}_{\mathbf{Lemma}~\ref{lem:consensus_error}}  
  \end{align*}
  By substituting \textbf{Lemma}~\ref{lem:avg_clbr_grad}, \textbf{Lemma}~\ref{lem:avg_grad_diff}, \textbf{Lemma}~\ref{lem:consensus_error} and \textbf{Lemma}~\ref{lem:u_x_diff} into the above equation, we have:
  \begin{align*}
    & \sum_{t=1}^{T} \mathbb{E}\left[\left\Vert\nabla \mathcal{F}\left(\bar{{x}}_{t-1}\right)\right\Vert^{2}\right] \\
    \leq & \frac{1}{C_{1}}\bigg(\mathbb{E}\left[ \mathcal{F}\left(\bar{u}_{0}\right)\right]-\mathbb{E}\left[ \mathcal{F}\left(\bar{u}_{t}\right)\right]\bigg) 
    + C_{5} \left(\frac{6S_0 C^2}{N^3} + \frac{6S_0\sigma^2 C^2d}{B^2N^4} + \frac{27\alpha^2 C^2}{8} + \frac{4\xi^2}{N}\right)T  \\
    & + C_{6}\left(\frac{\left(24S_0C^2 + \frac{24}{B^2}S_0 \sigma^2 C^2d \right)\eta^2} {N^3(1-\beta)^2(1-\sqrt{\rho})^2} + \frac{\left(\frac{27\alpha C^2}{2} + 60B^2(\xi^2 +\kappa^2) \right) \eta^2} {(1-\beta)^2(1-\sqrt{\rho})^2B^2} \right) \cdot T \\
    & +\left(C_{2} \frac{\eta^2\beta^2}{(1-\beta)^4} + C_{3}\right) \left( \frac{3 S_0 C^2}{N^3} + \frac{3S_0\sigma^2 C^2d}{B^2N^4} + \frac{27\alpha^2 C^2}{16} \right) T \\
    & + \left(\frac{60\eta^2C_6} {(1-\beta)^2(1-\sqrt{\rho})^2}+ 4C_{5} - C_{4}\right) \cdot \sum_{t=1}^{T} \mathbb{E}\left[\left\Vert\frac{1}{N} \sum_{i=1}^{N} \nabla f_{i}\left({x}_{t-1}^{i}\right) \right\Vert^{2}\right]        
  \end{align*}
  Dividing both sides of the above inequality by $T$ while recording $\mathcal{F}^{\star}$ as the minimum value of $\mathcal{F}$ in (\ref{eq:decentobj}), we have
  \begin{align*}
    \frac{1}{T}\sum_{t=1}^{T}\mathbb{E}\left[\left\Vert\nabla \mathcal{F}\left(\bar{{x}}_{t-1}\right)\right\Vert^{2}\right]
    & \leq \frac{1}{C_{1}T}\left( \mathcal{F}\left(\bar{{x}}_{0}\right) -\mathcal{F}^{\star}\right) + C_{5} \left(\frac{6S_0 C^2}{N^3} + \frac{6S_0\sigma^2 C^2d}{B^2N^4} + \frac{27\alpha^2 C^2}{8} + \frac{4\xi^2}{N}\right) \\
    & \quad + \left(C_{2} \frac{\eta^2\beta^2}{(1-\beta)^4} + C_{3}\right) \left( \frac{3 S_0 C^2}{N^3} + \frac{3S_0\sigma^2 C^2d}{B^2N^4} + \frac{27\alpha^2 C^2}{16} \right) \\
    &\quad + C_{6}\left(\frac{\left(24S_0C^2 + \frac{24}{B^2}S_0 \sigma^2 C^2d \right)\eta^2} {N^3(1-\beta)^2(1-\sqrt{\rho})^2} + \frac{\left(\frac{27\alpha C^2}{2} + 60B^2(\xi^2 +\kappa^2) \right) \eta^2} {(1-\beta)^2(1-\sqrt{\rho})^2B^2} \right)        
  \end{align*}
  by considering the facts that $\bar{u}_0 = \bar{{x}}_0$ and assuming
  \begin{equation} \label{eq:keyassumption}
    \frac{4C_{5}}{N} + \frac{60\eta^2 C_6}{(1-\beta)^2(1-\sqrt{\rho})^2} - C_{4} \leq 0
  \end{equation}
  We finally have (\ref{eq:convergence}) by considering $S_0 = \sum_{i=1}^N \sum_{j \in \mathcal{N}_i} \frac{1}{w_{ij}^2} \leq \frac{N^2}{w_{min}^2}$.
  
  As mentioned above, we have (\ref{eq:convergence}) based on $\eta$ respecting inequality (\ref{eq:keyassumption}). By considering another conditions on $\eta$, i.e., (\ref{eq:c1condition0}) and (\ref{eq:keyassumption}), the learning rate $\eta$ in our algorithm is supposed to satisfy (\ref{eq:learningrate}).

  \begin{figure}[htb!]
  \begin{center}
    \parbox{.32\textwidth}{\center\includegraphics[width=.32\textwidth]{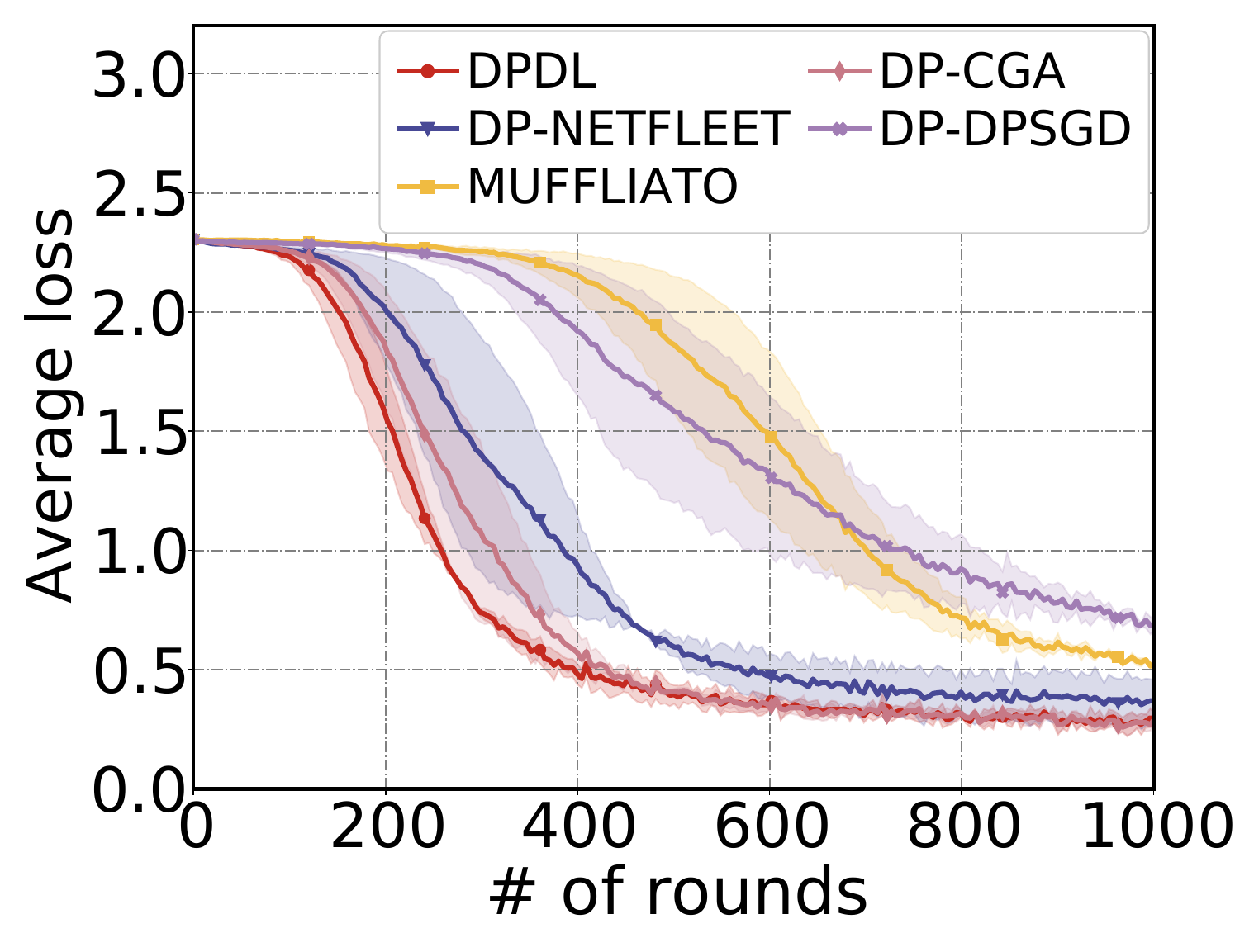}}
    \parbox{.32\textwidth}{\center\includegraphics[width=.32\textwidth]{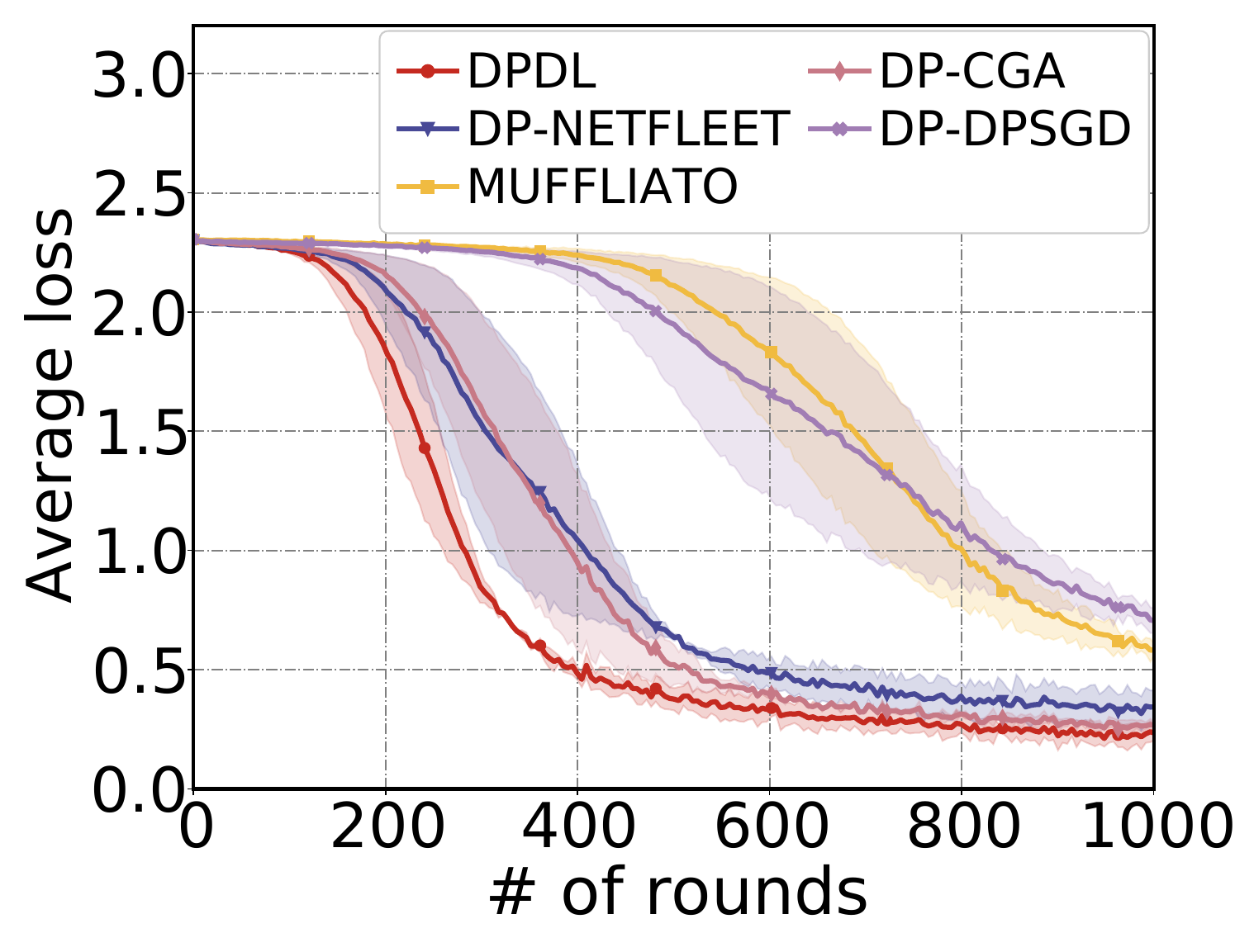}}
    \parbox{.32\textwidth}{\center\includegraphics[width=.32\textwidth]{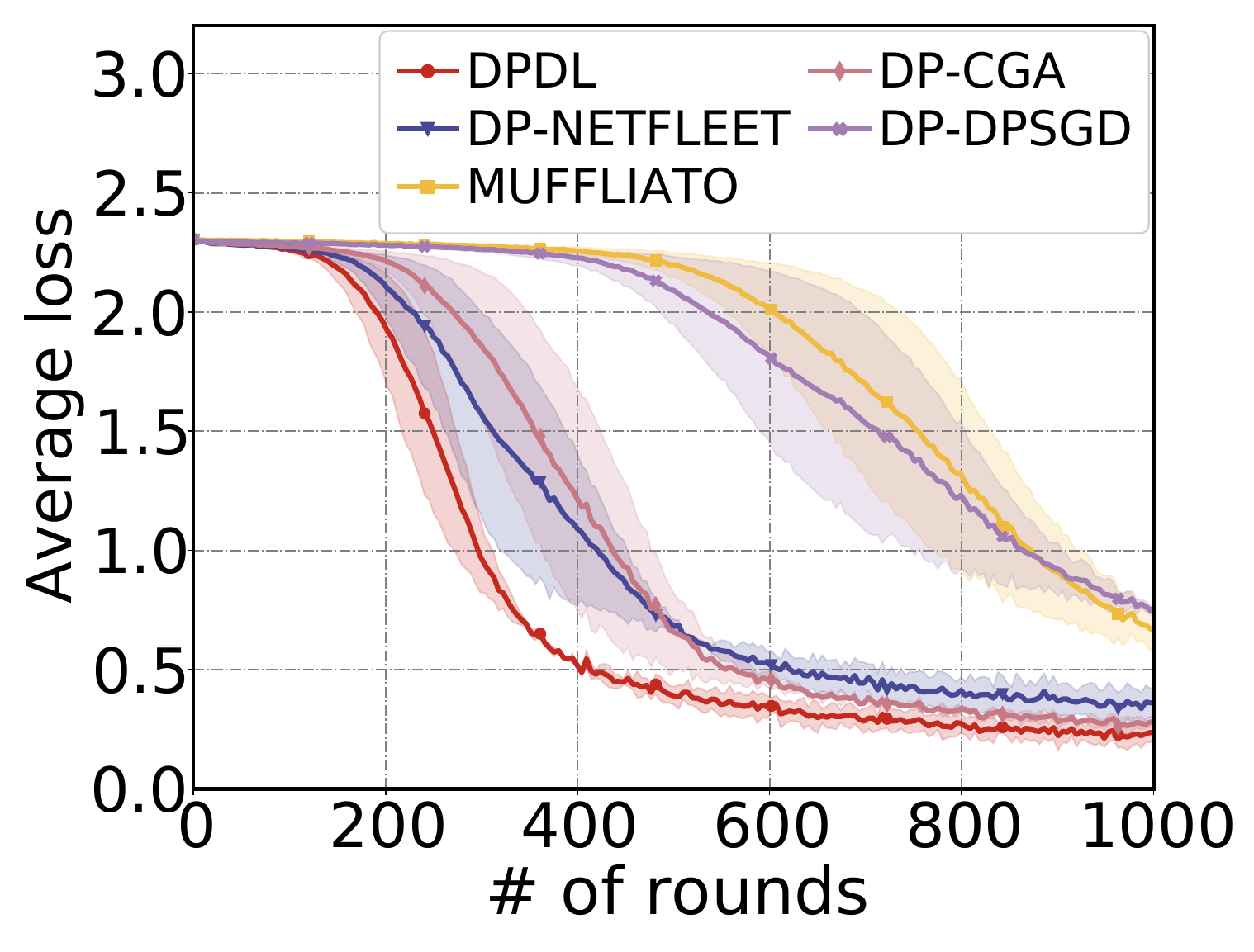}}
    \parbox{.32\textwidth}{\center\scriptsize(a1) $\epsilon=0.25, N=10$}
    \parbox{.32\textwidth}{\center\scriptsize(b1) $\epsilon=0.5, N=10$}
    \parbox{.32\textwidth}{\center\scriptsize(c1) $\epsilon=1.0, N=10$}
    \parbox{.32\textwidth}{\center\includegraphics[width=.32\textwidth]{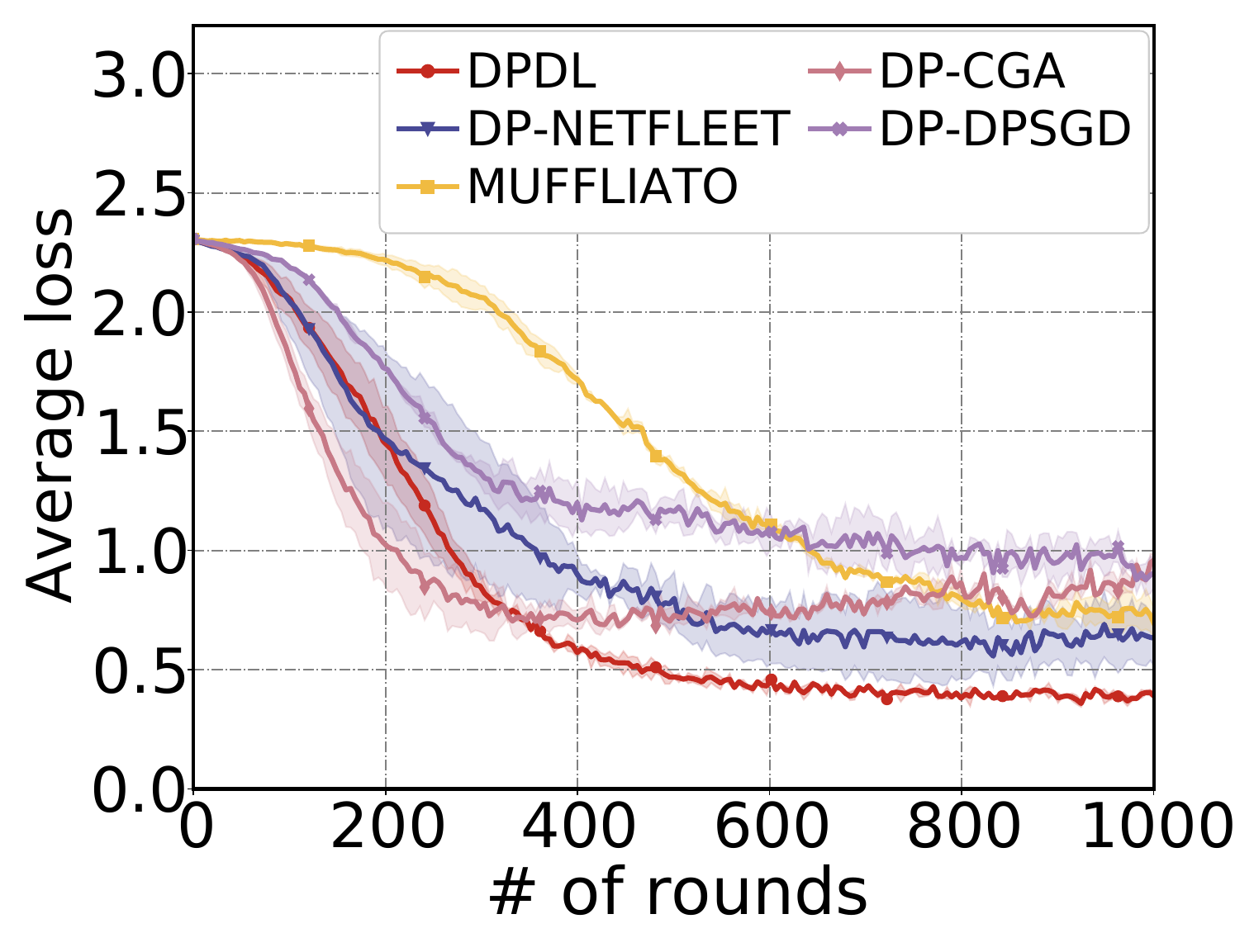}}
    \parbox{.32\textwidth}{\center\includegraphics[width=.32\textwidth]{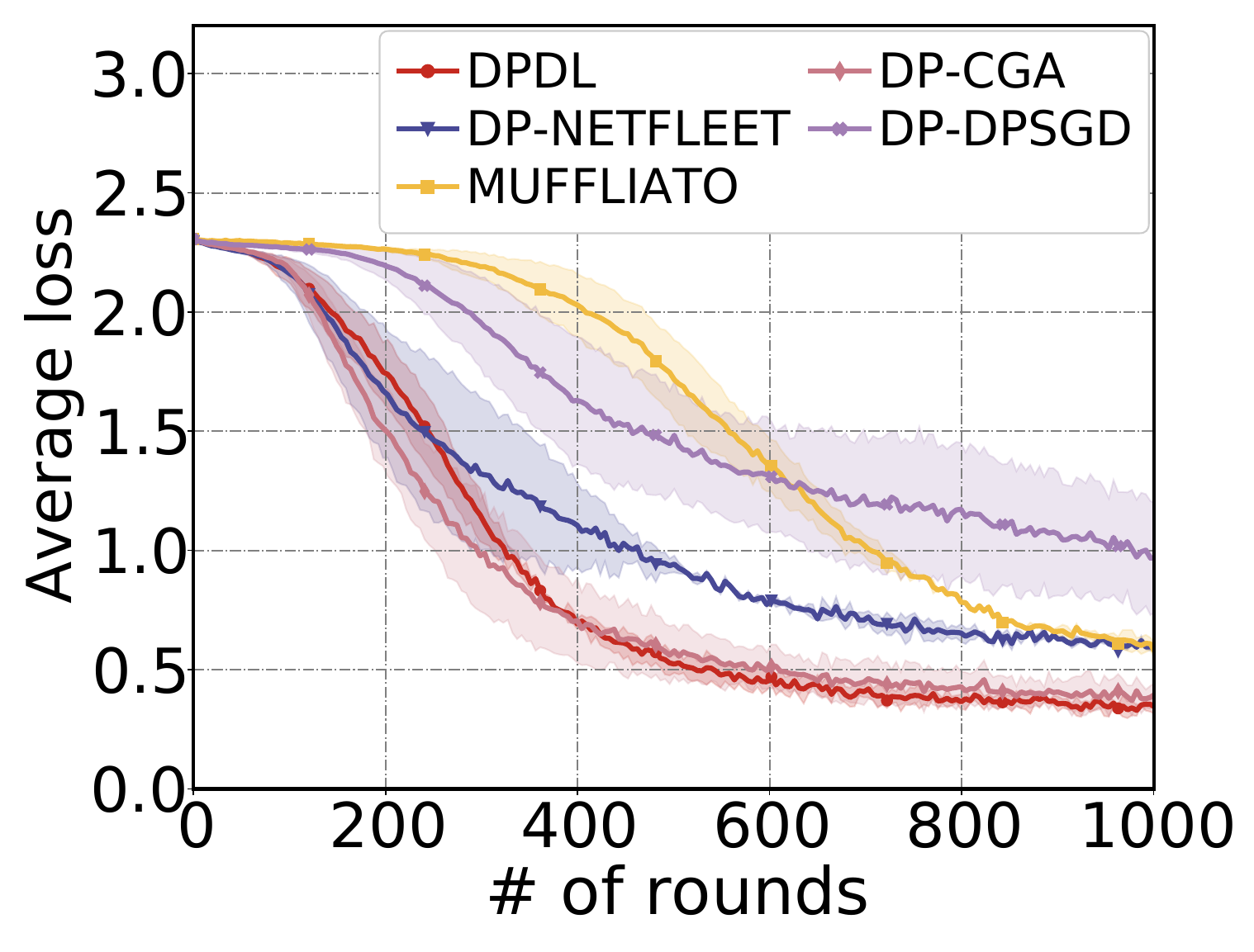}}
    \parbox{.32\textwidth}{\center\includegraphics[width=.32\textwidth]{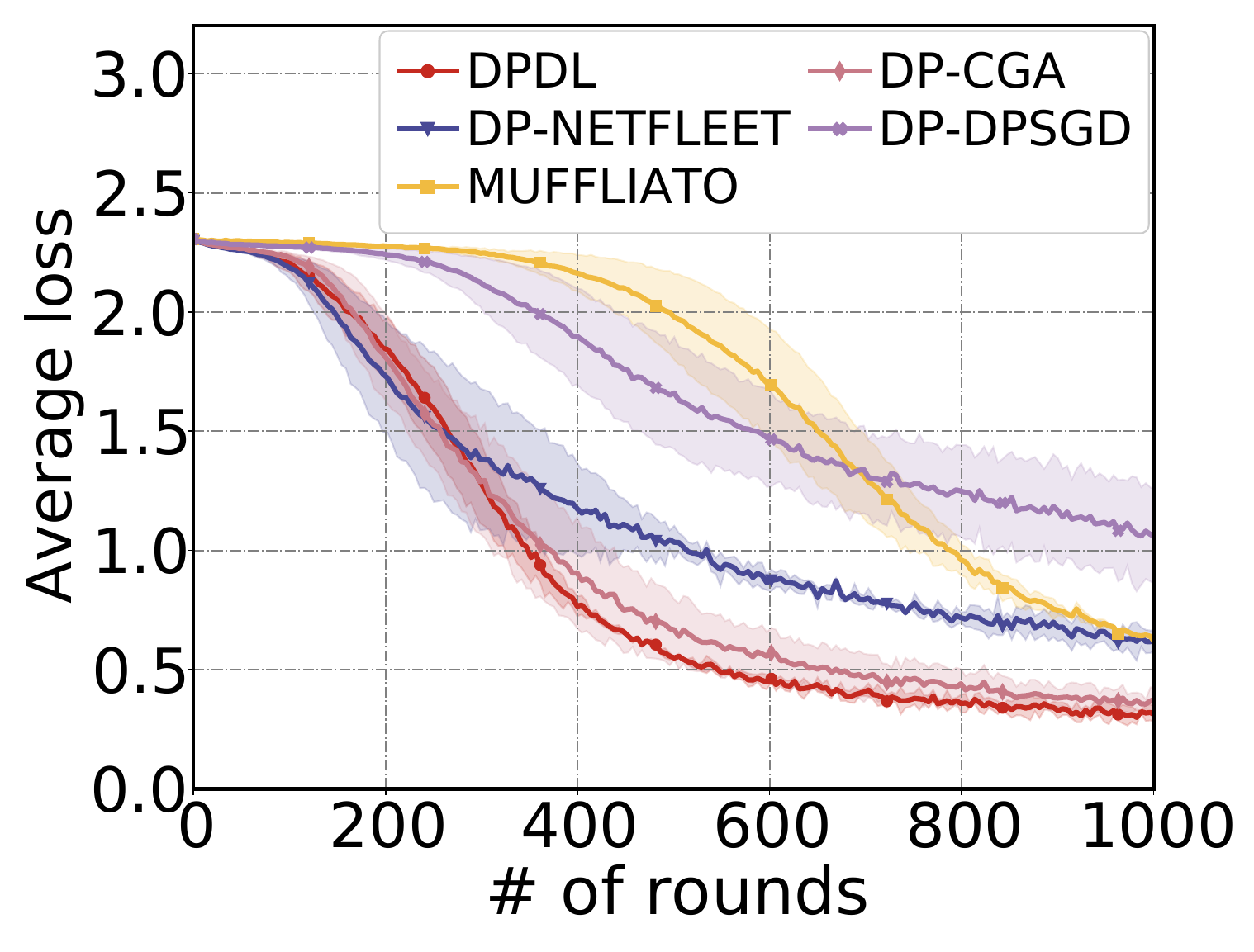}}
    \parbox{.32\textwidth}{\center\scriptsize(a2) $\epsilon=0.25, N=20$}
    \parbox{.32\textwidth}{\center\scriptsize(b2) $\epsilon=0.5, N=20$}
    \parbox{.32\textwidth}{\center\scriptsize(c2) $\epsilon=1.0, N=20$}
    \parbox{.32\textwidth}{\center\includegraphics[width=.32\textwidth]{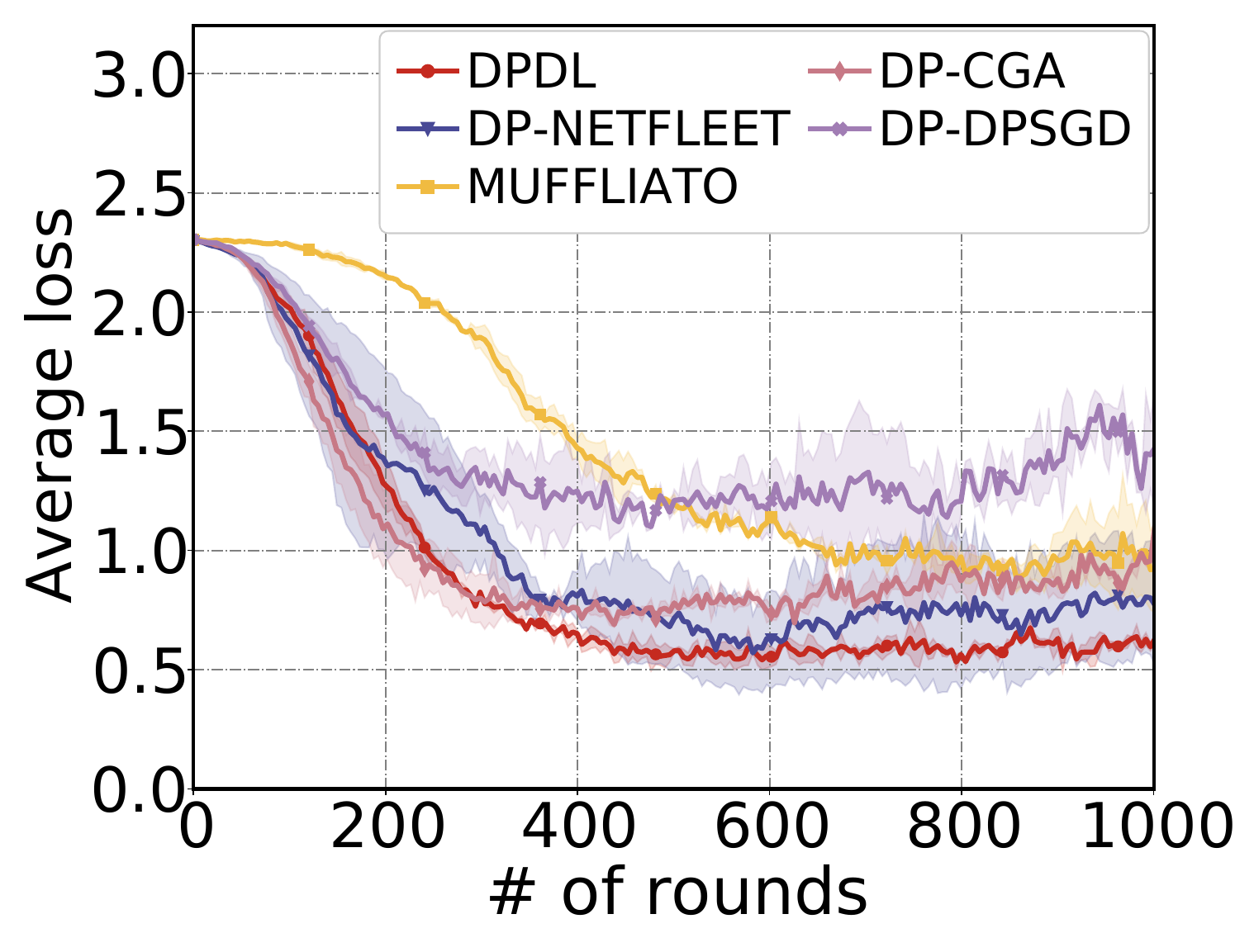}}
    \parbox{.32\textwidth}{\center\includegraphics[width=.32\textwidth]{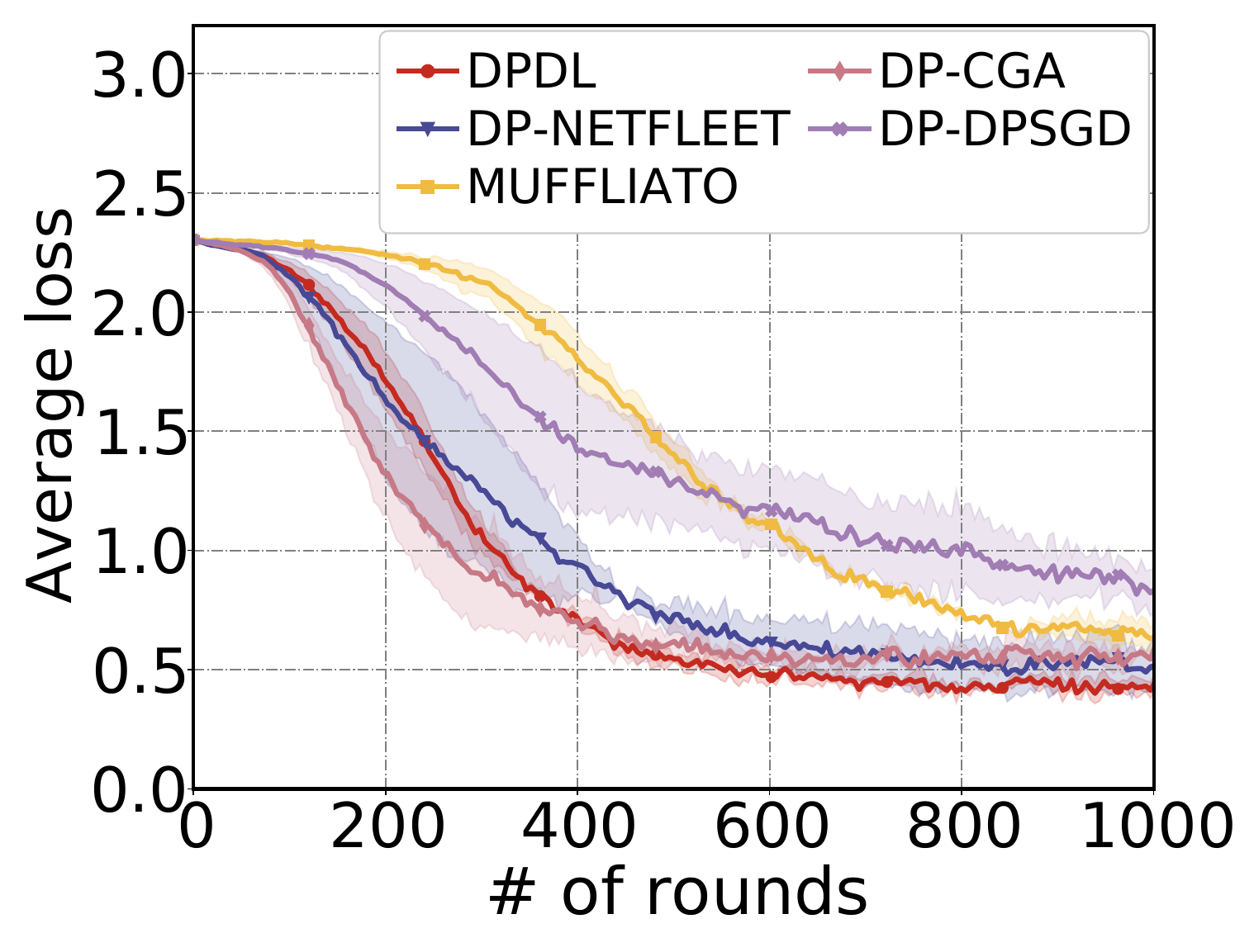}}
    \parbox{.32\textwidth}{\center\includegraphics[width=.32\textwidth]{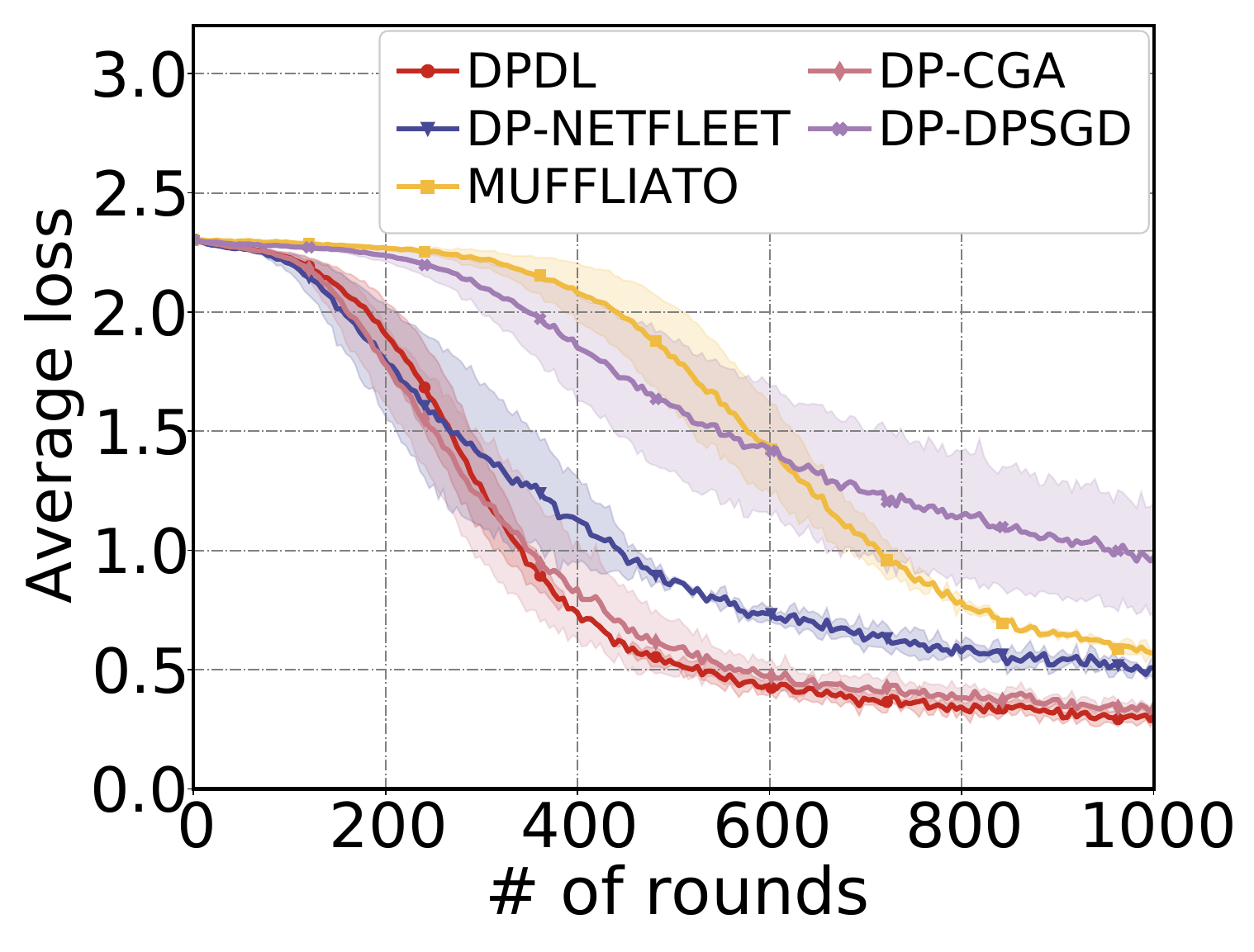}}
    \parbox{.32\textwidth}{\center\scriptsize(a3) $\epsilon=0.25, N=30$}
    \parbox{.32\textwidth}{\center\scriptsize(b3) $\epsilon=0.5, N=30$}
    \parbox{.32\textwidth}{\center\scriptsize(c3) $\epsilon=1.0, N=30$}
    \caption{Experiment results about convergence on MNIST dataset over ring graphs.}
  \label{fig:ring-mnist}
  \end{center}
  \end{figure}

  \begin{figure}[htb!]
  \begin{center}
    \parbox{.32\textwidth}{\center\includegraphics[width=.32\textwidth]{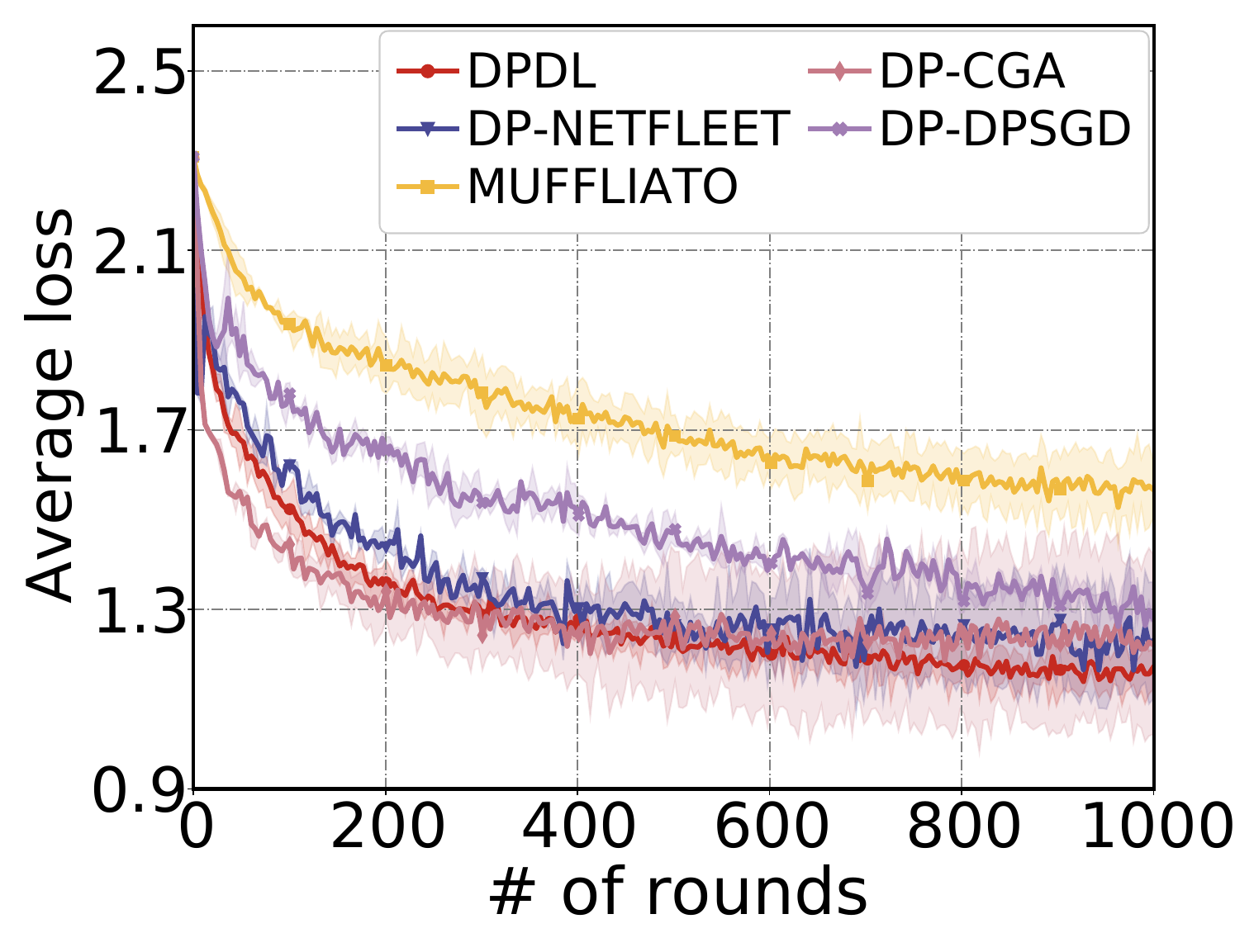}}
    \parbox{.32\textwidth}{\center\includegraphics[width=.32\textwidth]{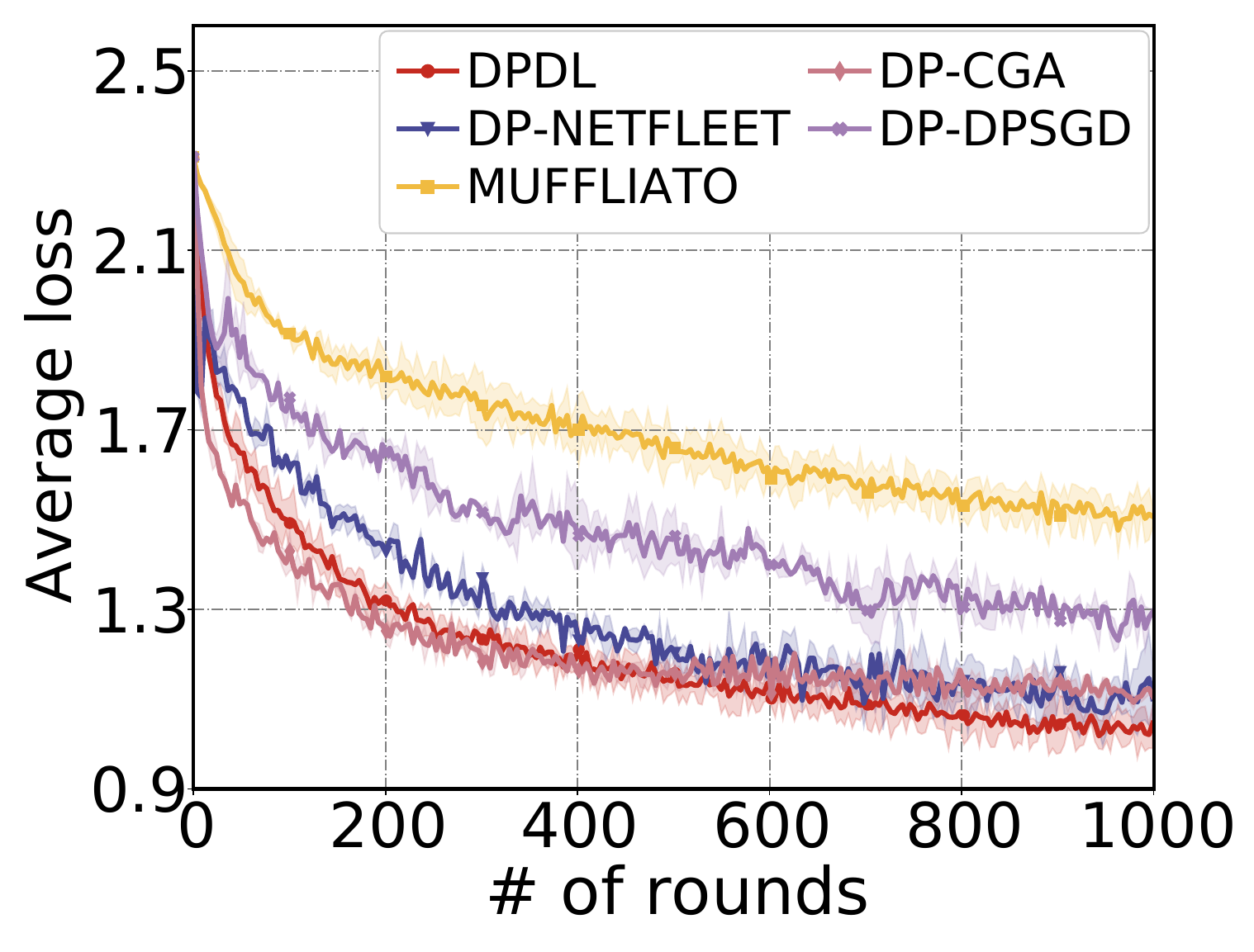}}
    \parbox{.32\textwidth}{\center\includegraphics[width=.32\textwidth]{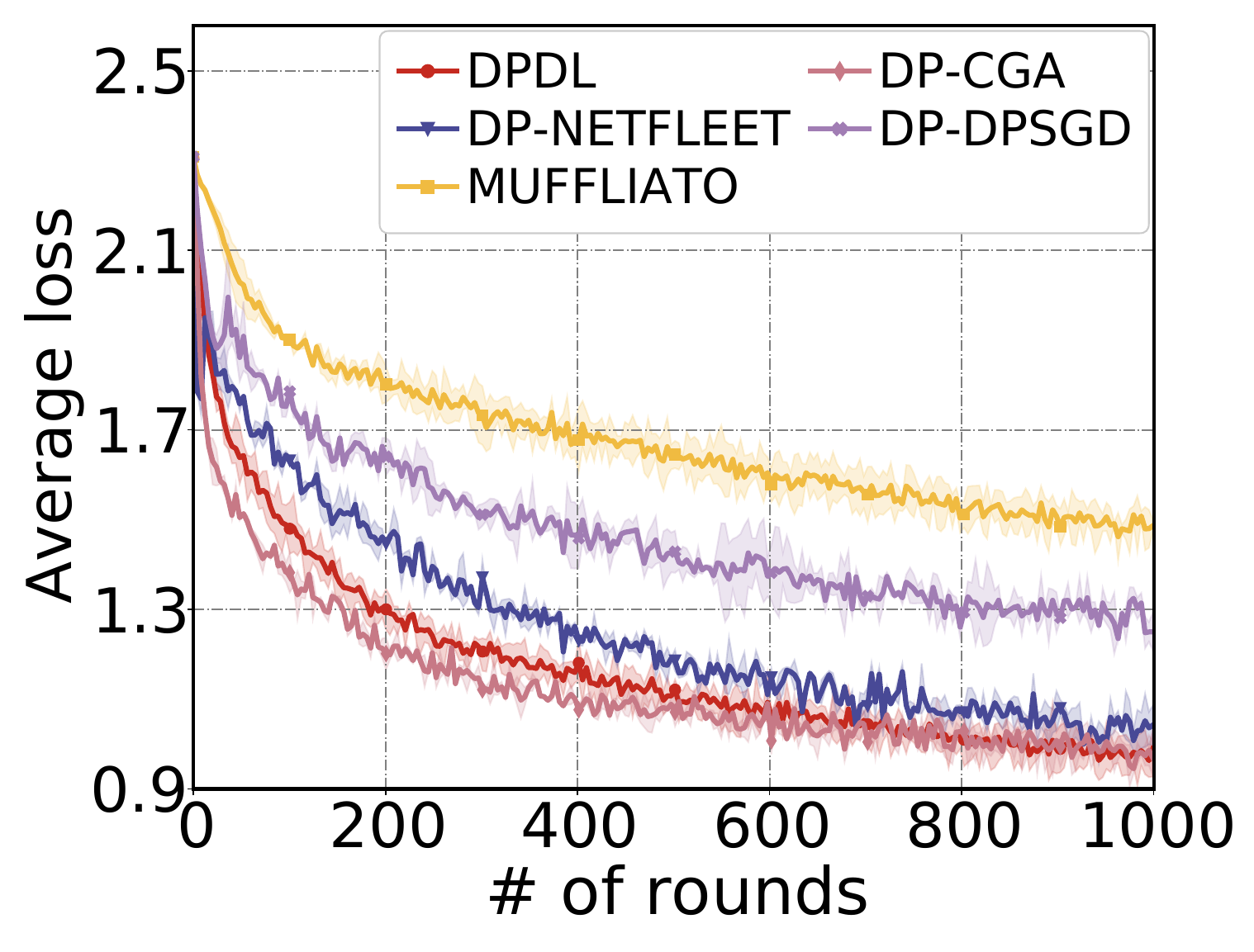}}
    \parbox{.32\textwidth}{\center\scriptsize(a1) $\epsilon=2.0, N=10$}
    \parbox{.32\textwidth}{\center\scriptsize(b1) $\epsilon=4.0, N=10$}
    \parbox{.32\textwidth}{\center\scriptsize(c1) $\epsilon=8.0, N=10$}
    \parbox{.32\textwidth}{\center\includegraphics[width=.32\textwidth]{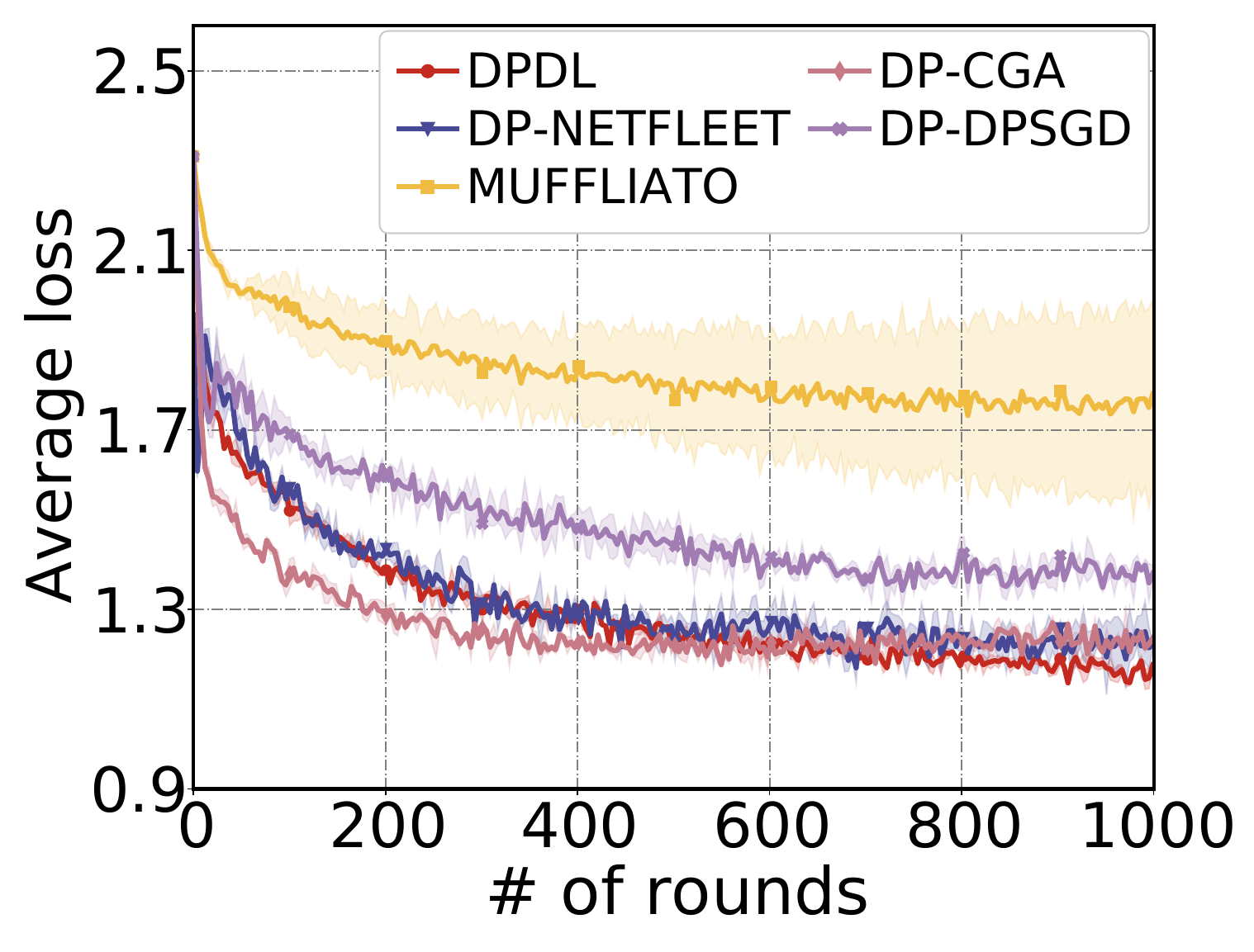}}
    \parbox{.32\textwidth}{\center\includegraphics[width=.32\textwidth]{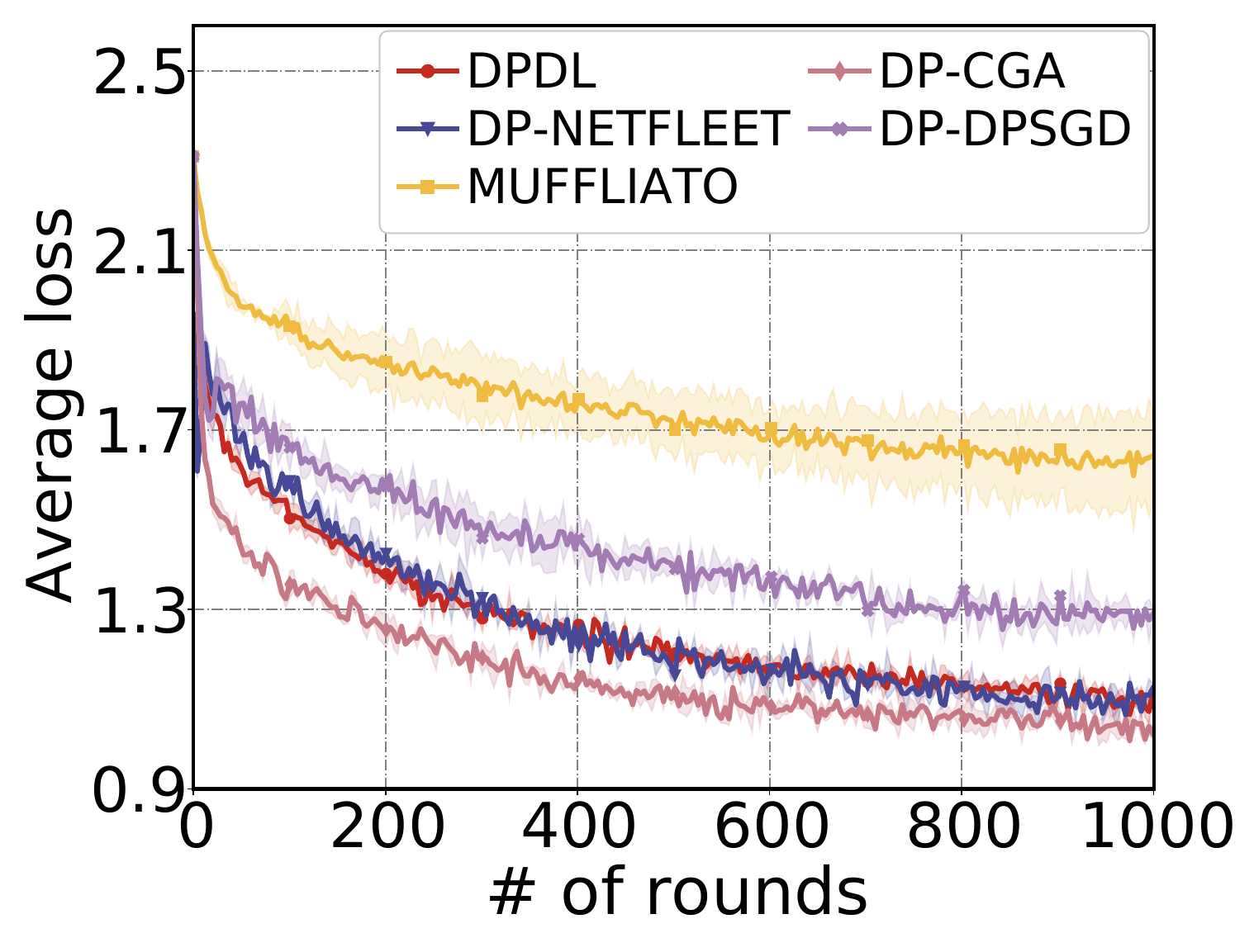}}
    \parbox{.32\textwidth}{\center\includegraphics[width=.32\textwidth]{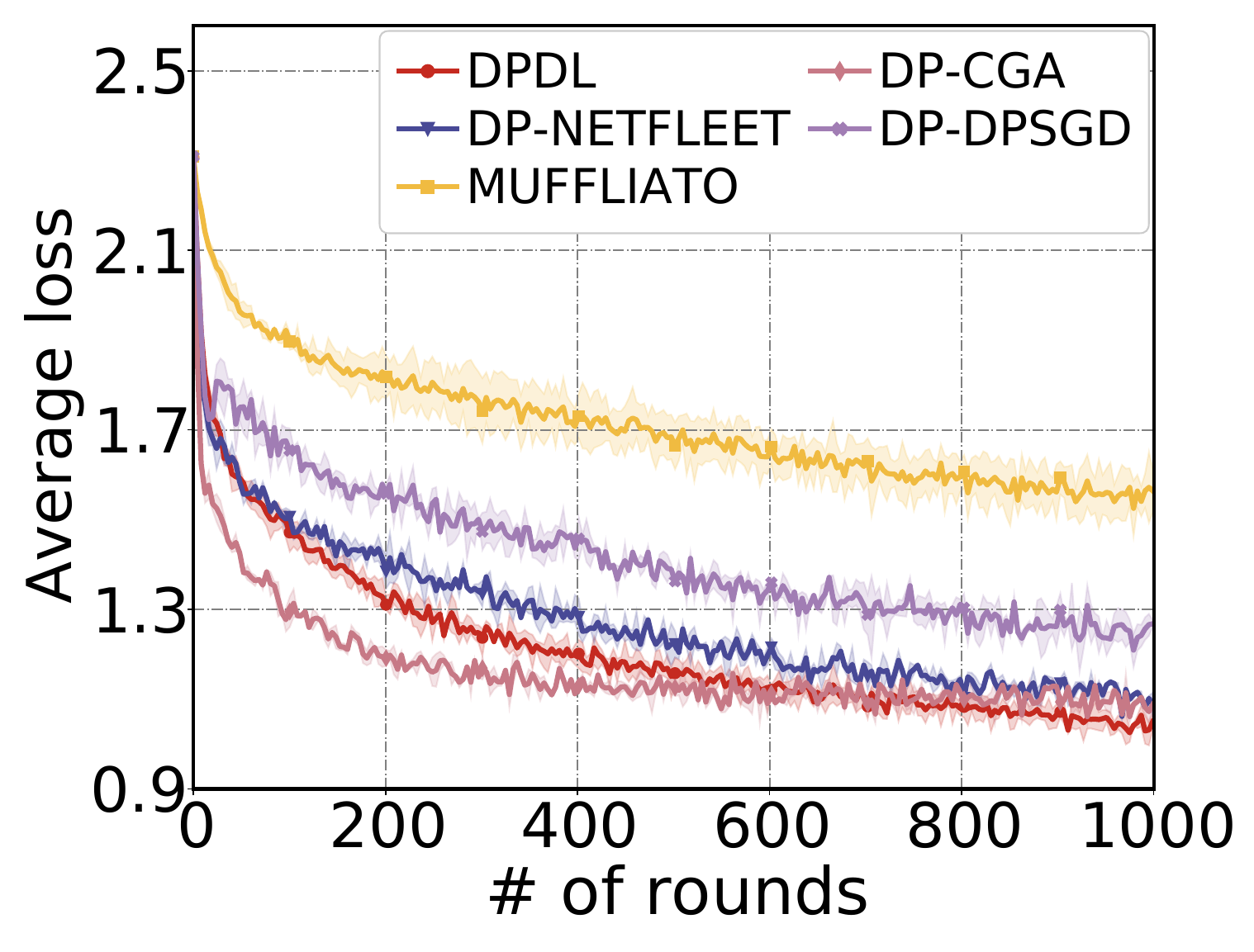}}
    \parbox{.32\textwidth}{\center\scriptsize(a2) $\epsilon=2.0, N=20$}
    \parbox{.32\textwidth}{\center\scriptsize(b2) $\epsilon=4.0, N=20$}
    \parbox{.32\textwidth}{\center\scriptsize(c2) $\epsilon=8.0, N=20$}
    \parbox{.32\textwidth}{\center\includegraphics[width=.32\textwidth]{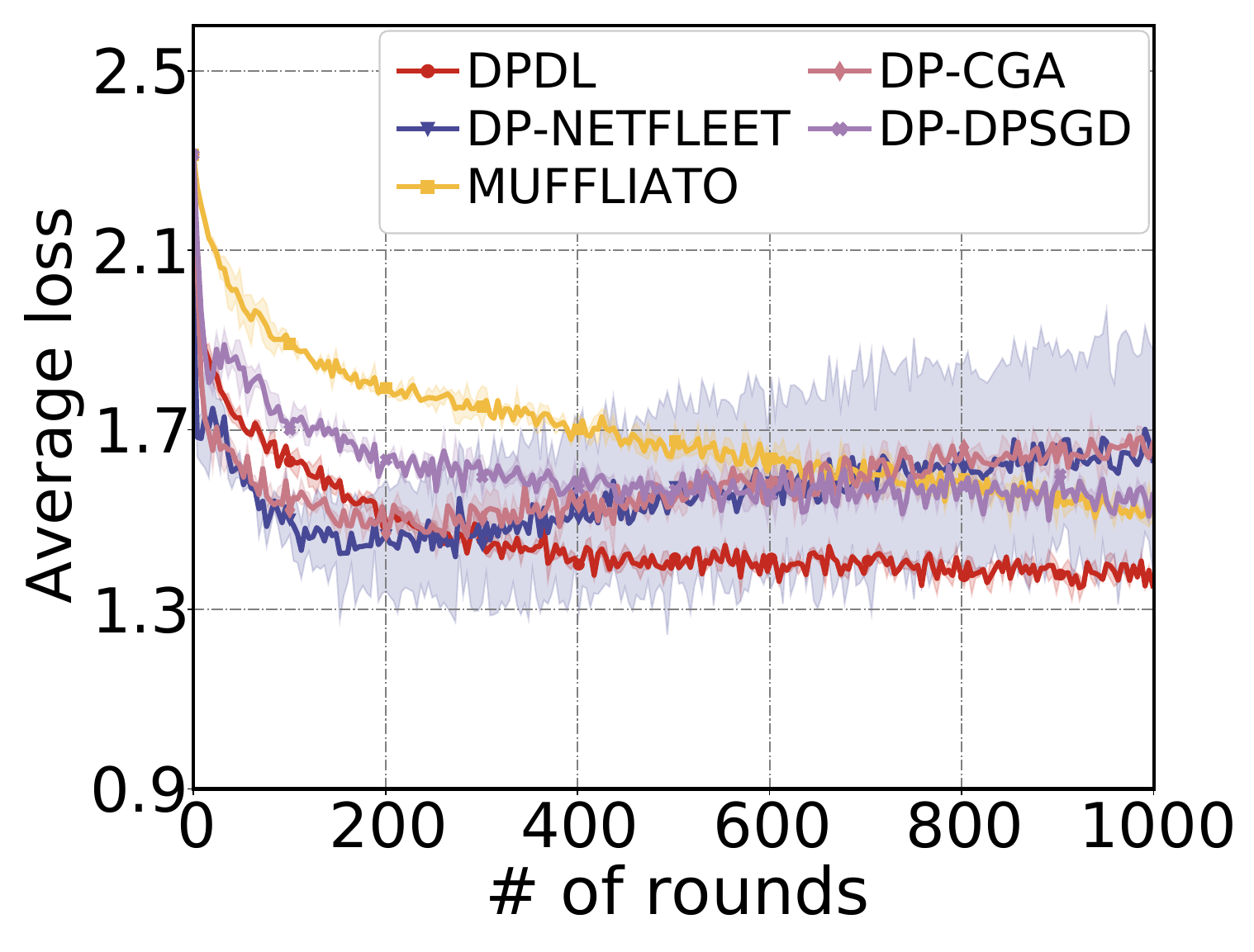}}
    \parbox{.32\textwidth}{\center\includegraphics[width=.32\textwidth]{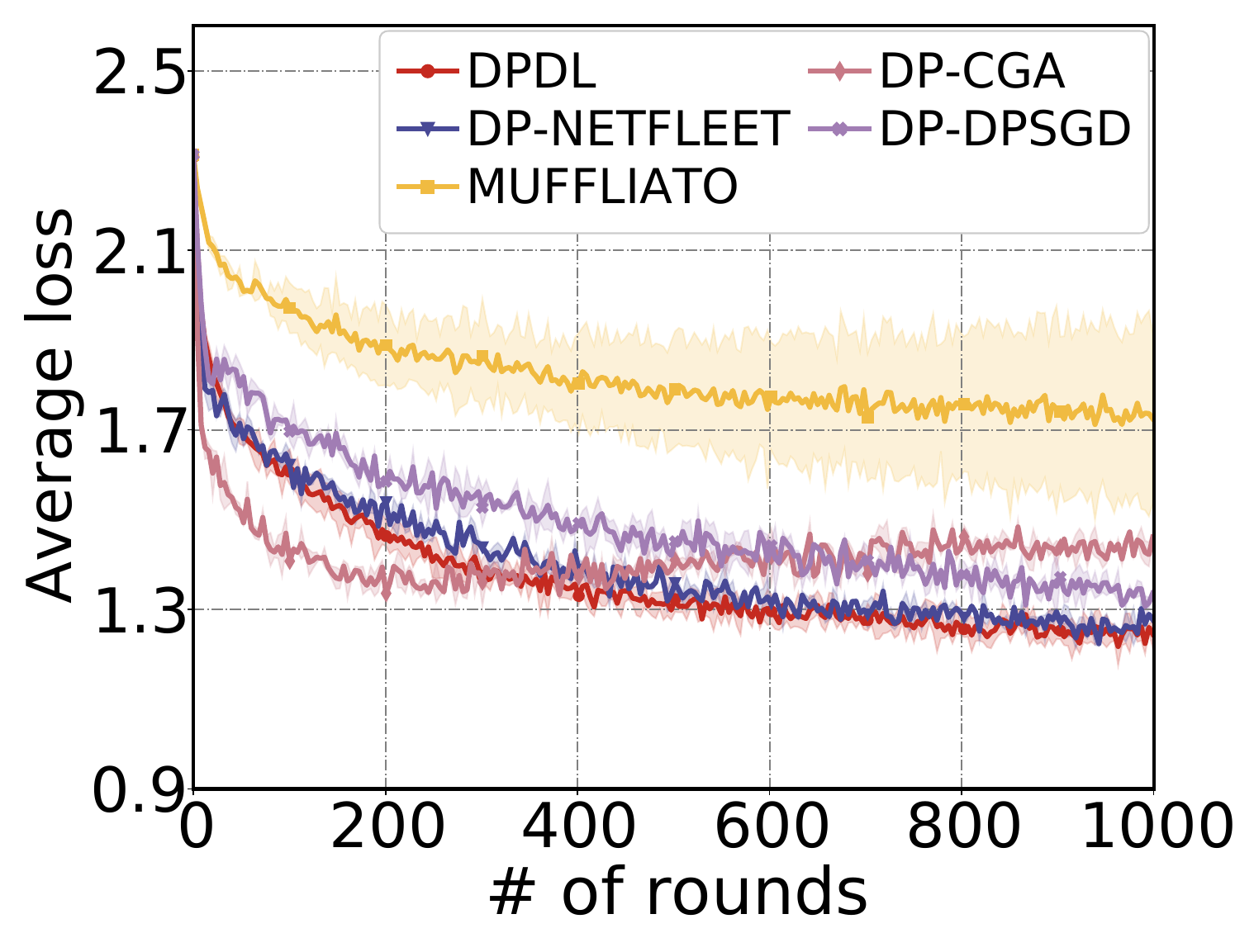}}
    \parbox{.32\textwidth}{\center\includegraphics[width=.32\textwidth]{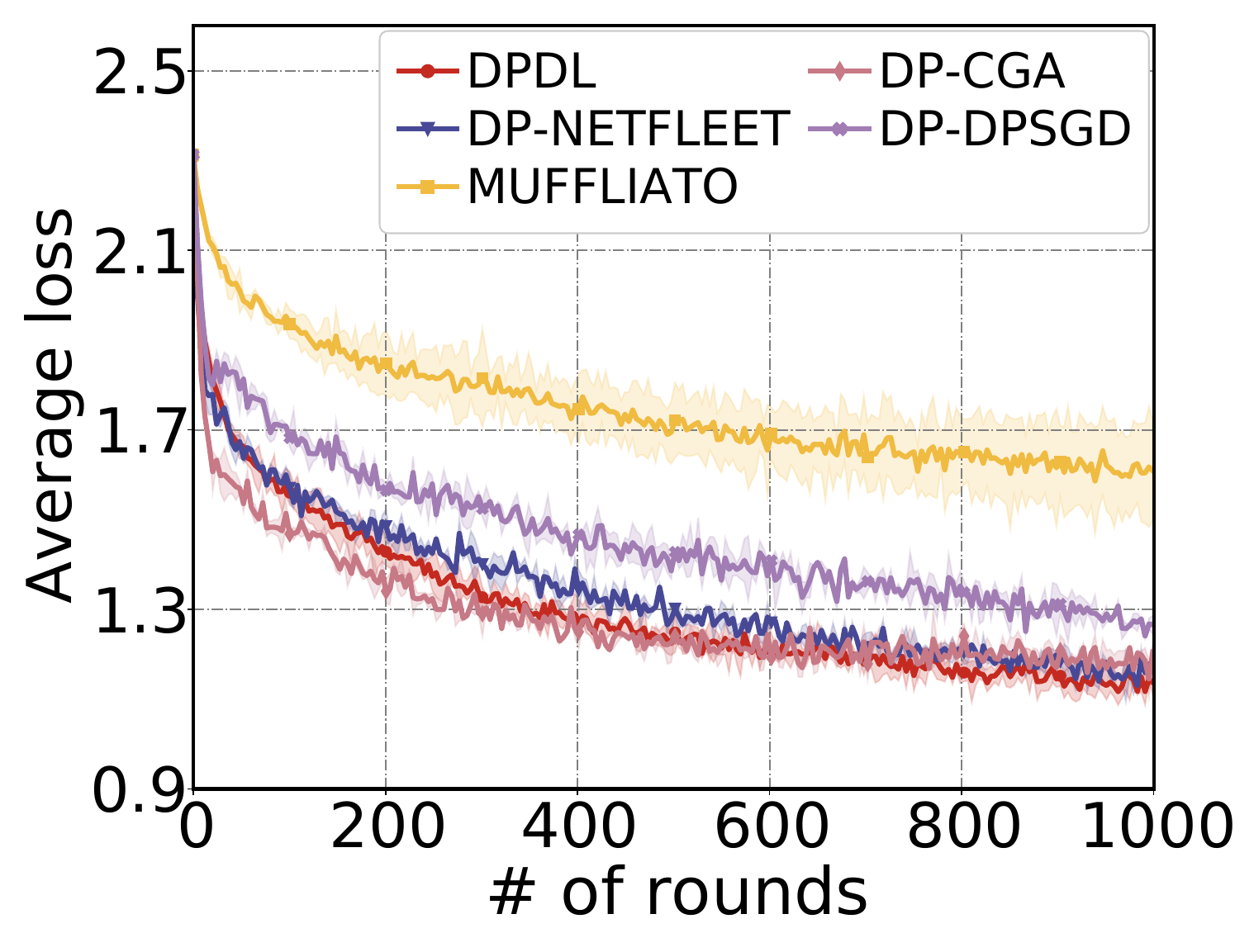}}
    \parbox{.32\textwidth}{\center\scriptsize(a3) $\epsilon=2.0, N=30$}
    \parbox{.32\textwidth}{\center\scriptsize(b3) $\epsilon=4.0, N=30$}
    \parbox{.32\textwidth}{\center\scriptsize(c3) $\epsilon=8.0, N=30$}
  \caption{Experiment results about convergence on CIFAR-10 dataset over ring graphs.}
  \label{fig:ring-CIFAR10}
  \end{center}
  \end{figure}

\section{Proof for Corollary~\ref{cor:convergence}} \label{sec:cor_proof}
  The right-hand side of (\ref{eq:convergence}) consists of four terms closely related to the variable $\eta$. If we let $\eta = \mathcal{O}\left(\frac{N\epsilon}{T}\right)$, and 
  \begin{equation}
    T \geq \max \left\{ 
    \frac{2\sqrt{30}NL\epsilon}{(1-\beta)(1-\rho)}, 
    \frac{\frac{30}{1-\sqrt{\rho}}N^2L^2\epsilon}{\sqrt{256+15(1-\beta)L^2N^2}-16(1-\sqrt{\rho})}\right\}
  \end{equation}
  we have 
  \begin{equation*}
    C_1 =  \frac{\eta}{2(1-\beta)}-\frac{(1-\beta)}{2\beta} =  \mathcal{O}\left( \sqrt{\frac{1}{T}} \right)
  \end{equation*} 
  and thus
  \begin{align*}
    \frac{\mathcal{F}(\bar{{x}}_0)-\mathcal{F}^*}{C_1T} = \mathcal{O} \left( \sqrt{\frac{1}{T}} \right)
  \end{align*}
  Similarly, the coefficients $C_2 \sim C_6$ (see (\ref{eq:conconstants})) can be represented as
  \begin{align*}
    C_2 = \mathcal{O}\left( \frac{1}{\eta} \right) = \mathcal{O}\left( \sqrt{T} \right),
    ~C_3= \mathcal{O}(\eta)=\mathcal{O}\left( \sqrt{\frac{1}{T}} \right),
    ~C_4 = \mathcal{O}(1),
    ~C_5= \mathcal{O}(\eta) = \mathcal{O}\left( \sqrt{\frac{1}{T}} \right),
    ~C_6=\mathcal{O}(1)
  \end{align*}
  by substituting which into the right-hand side of (\ref{eq:convergence}), we complete the proof.

\section{Supplement Experiment Results about Convergence and Accuracy}~\label{sec:supp_ex_con}

  We report the remaining convergence and accuracy results in Sec.~\ref{ssec:expconvergence}$\sim$~\ref{ssec:expaccuracy}, focusing on the performance of our proposed DPDL algorithm on ring graph compared with other reference algorithms. The experimental results for the MNIST dataset under ring graph are presented in Fig.~\ref{fig:ring-mnist} and Table~\ref{tab:acc-mnist-2}, while those for the CIFAR-10 dataset are provided in Fig.~\ref{fig:ring-CIFAR10} and Table~\ref{tab:acc-cifar-2}.

  As illustrated in Fig.~\ref{fig:ring-mnist}, when the number of agents $N=10$, our proposed algorithm demonstrates the fastest convergence among reference algorithms. For example, when $\epsilon=0.5$ and $N=10$, the average training loss of our DPDL algorithm achieves its minimum value around $0.39$ in about $500$ rounds, while the ones of the other four reference algorithms, i.e., DP-CGA, DP-NETFLEET, DP-DPSGD, and MUFFLIATO, are $0.56$, $0.68$, $1.99$, and $2.14$, respectively, within the same time horizon. As the number of agents increases, although our DPDL algorithm converges slightly slower than DP-CGA under certain privacy budget settings, it consistently achieves the lowest final loss, indicating superior optimization capability. Specifically, when $\epsilon=0.5$ and $N=30$, DPDL approaches $0.41$ average training loss after $1000$ rounds, while the other four reference algorithms, i.e., DP-CGA, DP-NETFLEET, DP-DPSGD, and MUFFLIATO, attain $0.58$, $0.52$, $0.83$, and $0.64$ average training loss, respectively.
  Besides, when $\epsilon$ is increasing from $0.25$ to $1.0$, our proposed algorithm maintains a relatively stable loss trajectory with minimal fluctuations, particularly when agent number $N$ is large. For instance, when $\epsilon = 0.25$, our algorithm consistently reaches its minimum training loss (approximately $0.36$ to $0.57$) within $600$ rounds for $N = 10, 20, 30$. In contrast, the performance of reference algorithms deteriorates when $N=20$ and $N=30$, exhibiting less stable behavior and higher final loss values.

  Next, we report the convergence results on the CIFAR-10 dataset in Fig.~\ref{fig:ring-CIFAR10}. It can be observed that our DPDL algorithm, exhibit comparable convergence performance to DP-CGA, and DP-NETFLEET,  with training average loss stabilizing within the range of $1.1$–$1.4$ in $600$–$800$ rounds, particularly when a higher privacy budget is used (e.g., $\epsilon = 4.0, 8.0$). In contrast, DP-DPSGD and MUFFLIATO demonstrate substantially slower convergence and lack a clear descending trend in training loss within the same number of rounds, for they do not account for data heterogeneity. It is worth noting that although DP-CGA and DP-NETFLEET may converge slightly faster and reach marginally lower training losses than our DPDL algorithm, their test accuracies are significantly lower. Moreover, when the privacy budget is small (e.g., $\epsilon=2.0$), DP-CGA shows instability and deviates from the correct optimization path which leads to increased training loss, while DP-NETFLEET suffers from considerable randomness in its convergence trajectory. In contrast, DPDL maintains robust and consistent convergence across all tested configurations.
  
  \begin{table*}[h]
  \caption{Experiment results about test accuracy on MNIST dataset over ring graphs.}
  \centering
    \setlength{\tabcolsep}{8pt}
    \setlength{\extrarowheight}{1.5pt} 
    \begin{tabular}{|c|c|c|c|c|c|c|c|c|c|c|}
    \hline
    \multirow{2}{*}{$\epsilon$} &
    \multirow{2}{*}{\textbf{Algorithms}} &
    \multicolumn{3}{c|}{\textbf{$\alpha_d=0.25$}} &
    \multicolumn{3}{c|}{\textbf{$\alpha_d=1.0$}} \\ 
    \cline{3-8} & & $N$=10 & $N$=20 & $N$=30 & $N$=10 & $N$=20 & $N$=30 \\ \hline
    \multirow{5}{*}{$0.25$} 
    & DP-DPSGD &84.2±2.0 &71.1±1.7 &60.2±1.4 &86.8±1.4 &85.1±0.6 &70.9±2.5 \\
    & DP-CGA &89.3±0.5 &69.0±4.2 &68.8±1.6 &91.4±1.1 &83.5±1.0 &62.8±1.2 \\ 
    & MUFFLITAO &86.1±1.5 &74.6±2.0 &68.0±3.0 &86.3±0.7 &84.4±1.4 &81.3±1.1 \\
    & DP-NETFLEET &88.7±2.4 &78.2±2.9 &72.1±4.3 &90.3±2.3 &85.0±3.1 &75.5±3.5  \\
    & \textbf{DPDL} &\textbf{91.1±0.9} &\textbf{86.3±0.4} &\textbf{81.2±1.7} &\textbf{93.4±1.0} &\textbf{90.8±1.4} &\textbf{85.4±1.4} \\ \hline
    \multirow{5}{*}{$0.5$}
    & DP-DPSGD &84.3±1.2 &76.0±1.3 &76.3±0.5 &87.6±1.4 &88.8±1.4 &87.8±1.9  \\
    & DP-CGA &91.3±0.2 &83.1±2.7 &75.2±0.9 &92.6±1.0 &90.8±1.1 &86.2±0.6  \\ 
    & MUFFLITAO &86.2±0.8 &82.1±2.4 &77.6±1.4 &87.4±0.3 &86.8±0.8 &85.9±0.6 \\
    & DP-NETFLEET &89.9±1.7 &82.1±0.8 &82.2±3.7 &91.1±1.5 &91.0±2.3 &87.9±2.5 \\
    & \textbf{DPDL} &\textbf{91.8±0.6} &\textbf{88.9±1.6} &\textbf{86.9±0.9} &\textbf{93.7±1.1} &\textbf{92.7±1.0} &\textbf{91.4±1.2} \\ \hline
    \multirow{5}{*}{$1.0$}
    & DP-DPSGD &84.3±1.5 &77.0±1.5 &79.0±0.8 &89.2±1.3 &87.3±1.3 &88.0±1.5 \\
    & DP-CGA &91.1±0.8 &86.6±1.3 &87.5±0.9 &92.7±0.6 &92.0±1.0 &91.5±1.2 \\ 
    & MUFFLITAO &85.4±0.9 &83.8±1.6 &83.8±2.2 &88.9±0.1 &88.1±0.6 &88.0±0.6 \\
    & DP-NETFLEET &89.7±1.5 &82.6±1.2 &83.4±1.3&91.5±1.3 &91.3±1.6 &90.7±1.8 \\
    & \textbf{DPDL} &\textbf{92.1±0.6} &\textbf{89.4±1.1} &\textbf{89.5±1.4} &\textbf{93.9±1.0} &\textbf{92.8±1.0} &\textbf{92.3±1.2} \\ \hline
  \end{tabular}
  \label{tab:acc-mnist-2}
  \end{table*}
  \begin{table*}[h]
  \caption{Experiment results about test accuracy on CIFAR-10 dataset over ring graphs.}
  \centering
    \setlength{\tabcolsep}{8pt}
    \setlength{\extrarowheight}{1.5pt} 
    \begin{tabular}{|c|c|c|c|c|c|c|c|c|c|c|}
    \hline
    \multirow{2}{*}{$\epsilon$} &
    \multirow{2}{*}{\textbf{Algorithms}} &
    \multicolumn{3}{c|}{\textbf{$\alpha_d=0.25$}} &
    \multicolumn{3}{c|}{\textbf{$\alpha_d=1.0$}} \\ 
    \cline{3-8} & & $N$=10 & $N$=20 & $N$=30 & $N$=10 & $N$=20 & $N$=30 \\ \hline
    \multirow{5}{*}{$2.0$} 
    & DP-DPSGD &48.9±0.9 &46.3±0.0 &42.5±0.8 &50.8±0.3 &50.2±0.4 &46.3±0.2 \\
    & DP-CGA &53.9±1.3 &41.6±1.1 &24.4±1.6 &56.9±0.3 &45.6±0.3 &28.2±0.7\\ 
    & MUFFLITAO &44.8±2.8 &37.6±7.0 &45.2±1.1 &51.4±0.8 &50.1±0.8 &47.7±0.5 \\
    & DP-NETFLEET &52.6±3.9 &49.6±1.3 &34.4±8.1 &55.3±0.2 &50.0±0.1 &39.7±0.2 \\
    & \textbf{DPDL} &\textbf{57.4±0.9} &\textbf{51.9±0.6} &\textbf{46.4±0.7} &\textbf{58.8±0.8} &\textbf{53.7±0.7} &\textbf{48.7±0.6} \\ \hline
    \multirow{5}{*}{$4.0$}
    & DP-DPSGD &49.6±0.8 &48.2±0.9 &48.0±0.4 &51.1±0.3 &49.9±0.6 &48.5±0.7  \\
    & DP-CGA &55.9±1.3 &55.1±0.9 &46.7±0.3 &60.7±0.5 &56.9±0.9 &50.5±0.7  \\ 
    & MUFFLITAO &46.3±1.2 &41.7±2.7 &39.8±5.8 &51.6±0.7 &50.6±0.6 &49.4±0.7  \\
    & DP-NETFLEET &56.2±1.4 &54.6±0.4 &49.3±1.1 &56.8±0.2 &54.7±0.3 &51.4±0.1 \\
    & \textbf{DPDL} &\textbf{60.2±1.0} &\textbf{55.4±0.7} &\textbf{53.4±0.4} &\textbf{62.4±0.8} &\textbf{58.1±1.1} &\textbf{54.7±0.2} \\ \hline
    \multirow{5}{*}{$8.0$}
    & DP-DPSGD &49.9±0.2 &48.5±0.5 &48.7±0.8 &51.1±0.3 &50.1±0.8 &48.8±0.8 \\
    & DP-CGA &56.7±2.4 &56.3±1.2 &54.5±1.2 &61.7±0.8 &58.6±0.8 &56.7±0.7 \\
    & MUFFLITAO &46.1±1.3 &43.8±0.6 &43.7±2.7 &51.7±0.5 &50.7±0.5 &49.6±0.6 \\
    & DP-NETFLEET &57.7±1.1 &56.1±0.6 &52.2±1.4 &57.6±0.1 &55.9±0.2 &53.6±0.2 \\
    & \textbf{DPDL} &\textbf{61.4±0.8} &\textbf{58.2±0.7} &\textbf{55.4±0.7} &\textbf{63.0±0.9} &\textbf{59.9±0.8} &\textbf{57.7±0.7} \\ \hline
  \end{tabular}
  \label{tab:acc-cifar-2}
  \end{table*}

  The accuracy performance on ring graph using the MNIST dataset is reported in Table~\ref{tab:acc-mnist-2}. The results indicate that our DPDL algorithm still has the highest prediction accuracy on extreme sparse topology with high data heterogeneity. For example, when $N=20$, $\epsilon=0.25$ and $\alpha_d=0.25$ over the ring graph, the test accuracy of our DPDL algorithm is $7.16\%$,  $21.4\%$, $12.3\%$ and $17.8\%$ higher than the ones of DP-NETFLEET, DP-CGA, MUFFLIATO and DP-DPSGD, respectively. Another important observation is that with the decrease of data heterogeneity, DPDL can obtain $1.95\%\sim5.21\%$ accuracy improvement and still maintain the highest accuracy among baselines, which also shows its robustness in various data scenarios.

  According to the results obtained on CIFAR-10 dataset (see Table~\ref{tab:acc-cifar-2}), the advantage of our proposed DPDL algorithm over the baseline methods remains evident. In particular, when $N=10$, $\epsilon=4.0$ and $\alpha_d=0.25$, the test accuracy of our algorithm over the ring graph is $60.2\%$, whereas the ones of the other reference algorithms are around only $46.3\% \sim 56.2\%$. When we increase the privacy budget $\epsilon$ to $8.0$, the test accuracy of our algorithm further improves to $61.4\%$, benefiting from reduced noise injection during gradient perturbation, which helps preserve the utility of gradient information. Moreover, even as data heterogeneity decreases, and some methods (e.g., DP-CGA) exhibit faster accuracy improvement under lower privacy budgets, our method consistently outperforms them and achieves the highest overall accuracy.

\end{document}